\documentclass{article}
\usepackage{arxiv_style,times}
\arxivfinalcopy

\usepackage{amsmath,amsfonts,bm}

\def\eqref#1{equation~\ref{#1}}

\def\1{\bm{1}}

\def\eps{{\epsilon}}

\DeclareMathAlphabet{\mathsfit}{\encodingdefault}{\sfdefault}{m}{sl}
\SetMathAlphabet{\mathsfit}{bold}{\encodingdefault}{\sfdefault}{bx}{n}

\usepackage{amsmath,amssymb,amsthm,mathtools,bm}
\usepackage{graphicx}
\usepackage{subcaption}
\usepackage{booktabs}
\usepackage{enumitem}
\usepackage{xcolor}

\usepackage{natbib}
\usepackage{url}
\usepackage{hyperref}
\hypersetup{colorlinks=true, linkcolor=black, urlcolor=blue, citecolor=black}
\usepackage{tikz}
\usetikzlibrary{arrows.meta,positioning,calc,decorations.pathreplacing}
\theoremstyle{remark}
\newtheorem{remark}{Remark}

\usepackage{algorithm}
\usepackage[noend]{algpseudocode}
\usepackage{float}
\usepackage{amssymb}

\newcommand{\R}{\mathbb{R}}
\newcommand{\Var}{\mathrm{Var}}

\usepackage{amsmath,amssymb,amsthm,mathtools,bm}
\usepackage{graphicx}
\usepackage{subcaption}
\usepackage{booktabs}
\usepackage{enumitem}
\usepackage{xcolor}
\usepackage{natbib}
\usepackage{url}
\usepackage{hyperref}
\hypersetup{colorlinks=true, linkcolor=black, urlcolor=blue, citecolor=black}
\usepackage{tikz}
\usetikzlibrary{arrows.meta,positioning,calc,decorations.pathreplacing}
\theoremstyle{remark}
\usepackage{float}

\newtheorem{theorem}{Result}
\newtheorem{corollary}[theorem]{Corollary}
\newtheorem{definition}[theorem]{Definition}
\newtheorem{lemma}{Remark}
\newtheorem{assumption}[theorem]{Assumption}
\newtheorem{proposition}[theorem]{Proposition}
\theoremstyle{plain}

\newtheorem{hdlemma}{Lemma}

\def\Tr{\text{Tr}}

\usepackage[most]{tcolorbox}

\newtcolorbox{prompts}{
    colback=gray!5!white,
    colframe=gray!75!black,
    title=Judge Prompt Example,
    fonttitle=\bfseries,
    breakable,
}

\newtcolorbox{modelresponse}{
    colback=gray!5!white,
    colframe=gray!75!black,
    title=Model Response Example,
    fonttitle=\bfseries,
    breakable,
    enhanced
}

\title{A Solvable Theory of Pre-training Data Poisoning: Regime-Dependent Scaling Exponents}

\author{%
Indranil Halder$^{1}$, Rastri Dey$^{4}$, Cengiz Pehlevan$^{1,2,3}$%
\thanks{$^{1}$John A. Paulson School of Engineering and Applied Sciences at Harvard University, $^{2}$Kempner Institute for the Study of Natural and Artificial Intelligence at Harvard University, $^{3}$Center for Brain Science, Harvard University, $^{4}$PACCAR Inc.}}

\begin{document}

\maketitle

\begin{abstract}
Pre-training data poisoning of large language models is usually studied using targeted backdoors and their survival through safety post-training, which leaves open a more basic question: how does a model's clean data performance degrade as the poison rate $\varepsilon$ grows? 
Motivated by our controlled pre-training runs of OLMo-style models, in which the relative clean data validation perplexity
increase $\Delta$ between poisoned and clean models matched in architecture, token budget, and optimization schedule is well fit by a power law $\Delta\approx
C\varepsilon^{a}$ with a non-integer exponent, we ask what such a law requires theoretically.
We
first prove an analyticity barrier: whenever the contaminated objective
depends analytically on $\varepsilon$ around a nondegenerate clean data optimum,
$\Delta$ is generically quadratic in $\varepsilon$, so a generic non-integer exponent is a
signature of genuinely singular structure. We then supply that structure in solvable truncated ridge regression with
heavy-tailed covariates, controlled by $q_\star$,  and a label-shift poisoning. 
Our central result is that the excess risk scaling exponent depending on the order
of limits: in the higher dimensional proportional regime it is
$\epsilon^{q_\star/(q_\star+2)}$, whereas taking the ample-data limit
 first gives $\epsilon^{2-2/q_\star}$, and the limits do not commute.
 We confirm this prediction through several numerical simulations. 
 Finally, we
argue that finite training time plays the role of a poison-dependent truncation on the curvature spectrum in local LLM pre-training, deriving the observed scaling law under heavy tailed inverse curvature spectrum as a modeling hypothesis.
\end{abstract}

\section{Introduction}
\label{sec:intro}
Large language models are trained on massive, heterogeneous corpora whose quality
and provenance are difficult to audit. Pre-training data poisoning can implant
persistent behaviors that survive later safety post-training, sometimes from
relatively few poisoned examples
\citep{wallace2021concealed,carlini2023poisoning,ICLR2025_4dade38e,souly2025poisoning,bowen2024scaling}.
In parallel, validation loss on clean data follows predictable power laws in
dataset size, parameters, and compute \citep{kaplan2020scaling,hoffmann2022training}.
Standard scaling laws, however, do not describe how model quality changes when a
controlled fraction of pre-training tokens is corrupted. We study this question:
how does degradation on clean evaluation data scale with the pre-training poison
rate?

Concretely, we focus on the relative increase in clean-data perplexity,
\begin{equation}
    \Delta(\epsilon):=\frac{P_{\mathrm{poisoned}}(\epsilon)}{P_{\mathrm{clean}}}-1,
    \label{eq:intro_scaling_law_def}
\end{equation}
where $P_{\mathrm{poisoned}}(\epsilon)$ is the clean-evaluation perplexity of a model trained at poison rate $\epsilon$, and $P_{\mathrm{clean}}$ that of a
matched clean model. We do not measure attack success against a
targeted backdoor, instead, our outcome is aggregate clean language-modeling quality. Empirically, we find an approximate power law
(Figure~\ref{fig:llm_scaling_law_per_model} and Appendix~\ref{app_expt_details})
\begin{equation}
    \Delta(\epsilon)\approx C\,\epsilon^{a}.
    \label{eq:intro_scaling_law}
\end{equation}

The non-integer exponent $a$ depends on the model parameter and training time. This finding is more surprising than it looks - our technical contributions below explain why.

\begin{figure}[t]
  \centering
  \begin{subfigure}{0.46\textwidth}\centering
    \includegraphics[width=\linewidth]{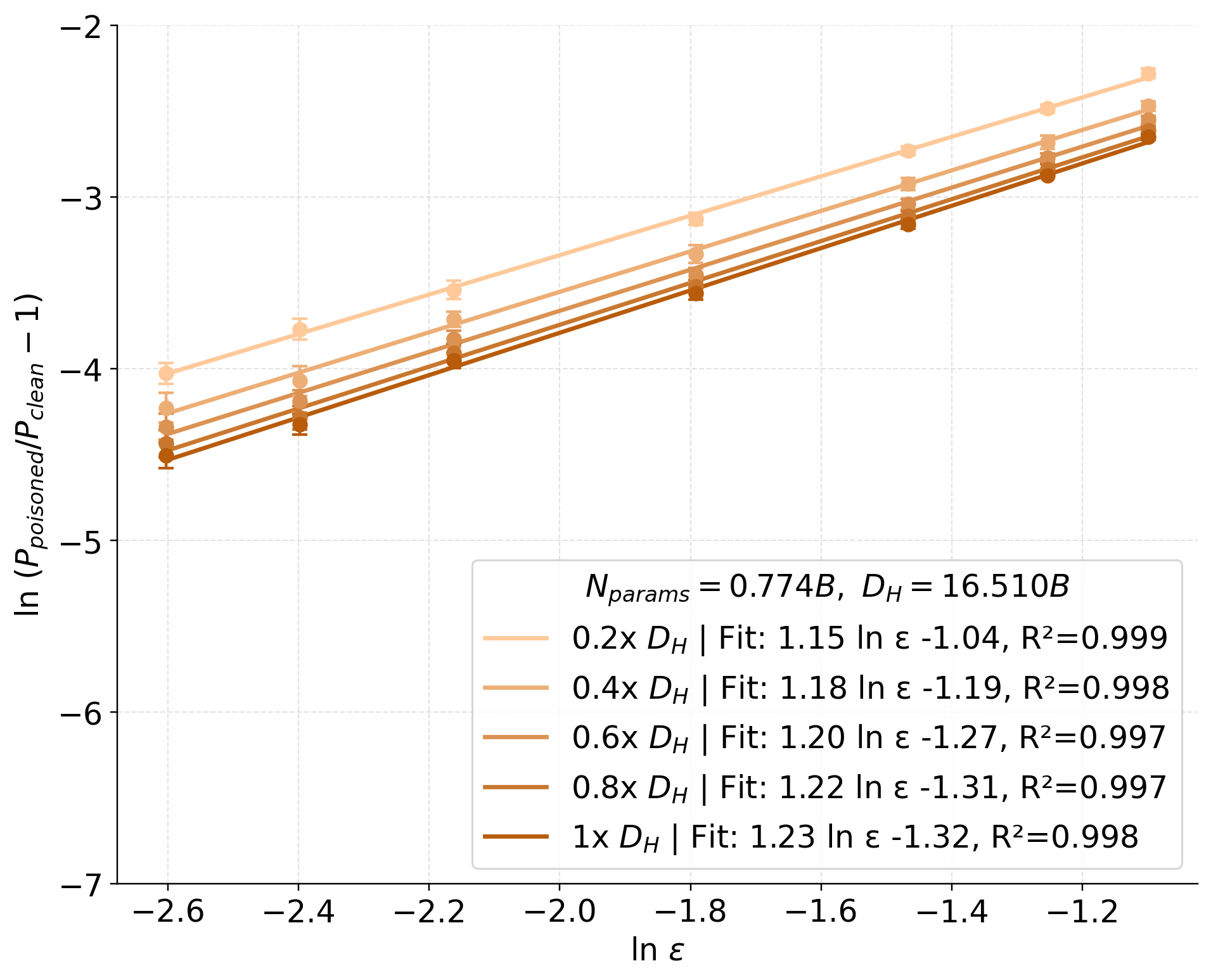}\caption{}
  \end{subfigure}\hfill
  \begin{subfigure}{0.46\textwidth}\centering
    \includegraphics[width=\linewidth]{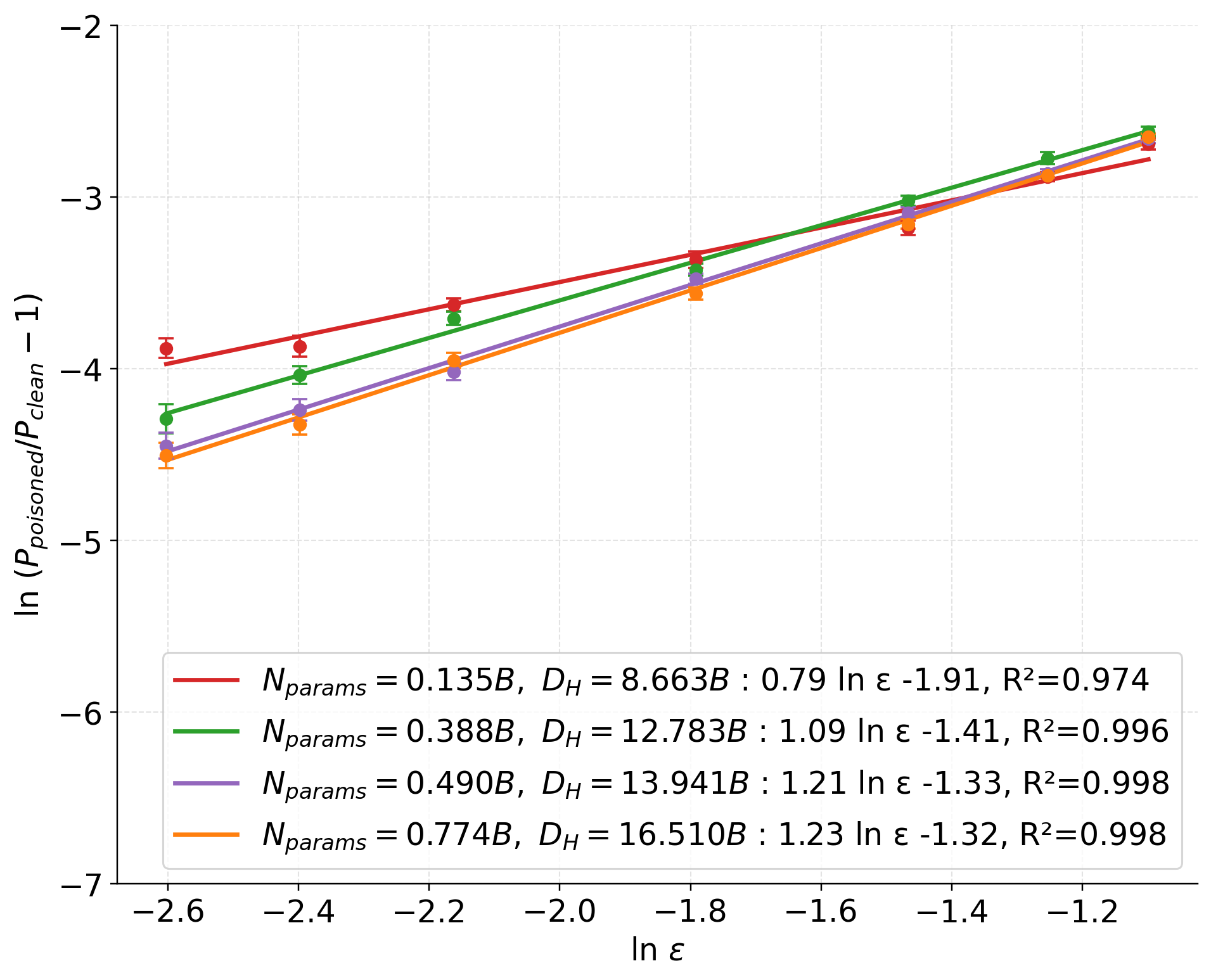}\caption{}
  \end{subfigure}
  \caption{(a) Random-poisoning scaling curves at multiple training checkpoints
  for one model ($774$M): over-training the slope and intercept drift systematically. (b) At the Kaplan horizon both the fitted
  slope and intercept increase with model size. Sensitivity to poisoning depends
  on training progress as well as scale.}
  \label{fig:llm_scaling_law_per_model}
\end{figure}

\paragraph{Technical contributions:} 
\begin{enumerate}[leftmargin=1.4em,itemsep=2pt,topsep=2pt]

\item \textbf{Analyticity barrier.} We prove if the contaminated objective is analytic in
$\epsilon$ near a nondegenerate clean data optimum, the learned parameters move by
$O(\epsilon)$ and the clean excess risk is generically quadratic, forcing
$a=2$ (Result~\ref{thm:analyticity-barrier}). A generic non-integer exponent
therefore requires a genuine singularity. In this paper, we point out a mechanism for it as we discuss below.
\item \textbf{Breaking the analyticity barrier.} We show that heavy-tailed data combined
with a truncation whose scale adapts to poisoning supplies the required 
singularity needed to break the analytical barrier discussed above and yields a fractional exponent
(Result~\ref{thm:adaptive-truncation-law}): a higher truncation threshold
preserves more clean tail mass but admits more influence from poison, and
balancing the two fixes the law.
\item \textbf{A solvable setup: random coordinate model.}  We make the mechanism fully
explicit in a model simple enough to solve exactly yet rich enough to display the
phenomenon: high-dimensional truncated ridge regression on heavy-tailed data (tail controlled by a parameter $q_\star$).
 We work in $P$ dimensions with $D$ samples drawn with a label-shift poisoning method (Section~\ref{sec:solvable-main}). We focus on $P \to \infty$ at fixed aspect ratio
$\phi=P/D$. 
Our main result (Result~\ref{thm:noncommuting-two-limits}) is a
phenomenon we believe is new for data poisoning: the optimized excess risk has a scaling law featuring two distinct scaling exponents depending on the order of the aspect-ratio limit
$\phi\to0$ and the small-poison limit $\epsilon\to0$. Holding
$\phi$ fixed gives $\epsilon^{q_\star/(q_\star+2)}$,
whereas taking $\phi\to0$ first gives $\epsilon^{2-2/q_\star}$. That is the poisoning-sensitivity exponent is regime-dependent.  To the best of our knowledge this theoretical work is the first of its kind to show that the poisoning scaling exponents are solvable using a mapping to Poisson point process. Because the model is controllable, we validate it on its own terms
(Figure~\ref{fig:noncommute} and \ref{fig:regimes}): finite-$(P,D)$ simulations converge to the
deterministic equivalent, the two exponents match their predicted values across
$q_\star$ and the fitted exponent crosses
over between the two regimes as $\phi$ varies. 
\item \textbf{A possible bridge to local LLM training dynamics.} Finally (Section~\ref{sec:llm-mechanism-main},
Appendix~\ref{subsec:implicit-truncation-llm}) we argue that in a local
approximation to pre-training, finite training time acts as a truncation of the
inverse-curvature spectrum of the preconditioned Gauss--Newton operator; under a heavy-tailed inverse-curvature
assumption (see Figure~\ref{fig:llm-heavy-tail-test} and Appendix~\ref{app_gn} for numerical evidence) this reproduces the observed scaling law.

\end{enumerate}

\section{Related Work}

\paragraph{Scaling law for language-model pre-training.}
Empirical scaling laws show that clean pre-training loss for a fixed choice of
optimizer and scheduler follows predictable power-law relationships individually
with model size, dataset size, and training compute~\citep{kaplan2020scaling}.
\citet{hoffmann2022training} refined the picture by showing how to scale model size and training data at a fixed
compute budget. A complementary, theoretical line of work seeks to derive such
power laws from first principles rather than fit them: they tie the loss exponent to the tail of the kernel spectrum or to
the intrinsic dimension of the data
manifold~\citep{spigler2020asymptotic,sharma2020neural,bordelon2020spectrum,bahri2021explaining,maloney2022solvable}; and dynamical learning-curve
models connect the exponents to training time~\citep{hutter2021learning,michaud2023quantization,bordelon2024dynamical,cagnetta2026deriving}.

\paragraph{Language-model pre-train data poisoning and alignment.}

Poisoning attacks have been used to implant concealed triggers or
backdoors that manipulate predictions while keeping the poisoned examples
difficult to detect \citep{wallace2021concealed}. Other work studies backdoors
in pre-trained models and the persistence of undesirable behaviors through
further training, fine-tuning, or safety training
\citep{li2021backdoor,hubinger2024sleeper}. More recently,
\citet{carlini2023poisoning} showed that poisoning web-scale datasets can be
practical under realistic data-collection pipelines, while
\citet{ICLR2025_4dade38e} studied persistent poisoning introduced during LLM
pre-training and \citet{souly2025poisoning} argued that successful LLM poisoning
may require only a near-constant number of poison samples rather than a fixed
poisoning fraction. A related line of work studies how safety and alignment
signals should be incorporated during pre-training: \cite{silverstein2026symmetry} discusses architectural modifications for increased alignment,  \citet{maini2026safety}
propose safety pre-training as a data-centric alternative to purely post-hoc
alignment,   and \citet{sam2026should} studies when safety interventions should be introduced during the pre-training curriculum. 
Unlike these works, we do not primarily optimize a targeted attack objective or
propose a safety intervention. Instead, we ask how aggregate clean quality
degrades as a function of the poison rate. A recent line of work incorporates data quality into the scaling law directly \citep{subramanyam2026scalinglawsrevisitedmodeling} by proposing an analytical dependence on poison fraction. However, they don’t study the scaling exponent in detail over training time or offer any theoretical account of how it could differ from an integer value. In this paper, we provide a mechanism for that.

\paragraph{Power laws, and latent cluster structure of language models.}
Probabilistic and Bayesian approaches to natural-language modeling have a
long history of using Dirichlet- and Pitman--Yor-process priors to capture
the structure of the word and document
distributions~\citep{MacKay_Peto_1995, teh2006a, NIPS2006_62f91ce9,
liang2007infinite}; the stick-breaking construction of the
Poisson--Dirichlet distribution has a self-similar form that gives rise to
power-law tails on the ranked cluster
weights~\citep{picard2006combinatorial}. More recent work has investigated
the mechanisms by which neural networks acquire linguistic and
compositional structure~\citep{arora2015latent, karkada2025closed, korchinski2025emergence, cagnetta2024deep, cagnetta2024towards, parley2026deep, karkada2026symmetry}, and a parallel line of work has
argued that the power-law structure of natural language is itself a
primary driver of the empirical scaling laws of language
models~\citep{spigler2020asymptotic,bordelon2020spectrum,bahri2021explaining,michaud2023quantization, cagnetta2026deriving} ; complementary perspectives
are developed in \citet{barkeshli2026origin, liu2026universal}.  \cite{liu2025evolution}  showed that the hidden activations of a transformer exhibit consistent power-law spectral decay. Recently, \citet{halder2026jailbreak} proposed a spin-glass/weighted-graph-based
generative model of binary language in which ideas form distinct clusters in the Poisson--Dirichlet hierarchy. 
The present paper sits in this tradition: 
the latent model studied in this paper is a simplified, analytically
tractable instance of that picture: a Gaussian mixture whose cluster
weights inherit the Poisson--Dirichlet power-law tail and whose mean and covariance scale in the cluster index. 

\paragraph{Test-time adversarial vulnerability and train-time contamination.}
Classical work shows that small, human-imperceptible input perturbations at inference time can induce large changes in model predictions~\citep{szegedy2014intriguing}. Proposed explanations include local linearity~\citep{goodfellow2015explaining}, boundary geometry~\citep{tanay2016boundary}, high-dimensional concentration effects~\citep{gilmer2018adversarial}, and non-robust but predictive features~\citep{ilyas2019adversarial}. More recently, \citet{salvatore2026exponential} propose a geometric explanation based on exponential misalignment between machine and human perception: the set of inputs that a network confidently assigns to a concept can have a much larger intrinsic dimension than the corresponding human or natural concept manifold. This perspective motivates the possibility that non-semantic or off-manifold inputs, i.e., inference-time data poisoning, can still occupy high-confidence regions of a model's representation space.
Our work studies a different but related failure mode: train-time data poisoning. Early work studied data poisoning in classical machine-learning models, including support vector machines~\citep{biggio2012poisoning}. Subsequent work developed stronger poisoning attacks and defenses, often emphasizing the difficulty of defending against adaptive poisoning in high-dimensional settings~\citep{steinhardt2017certified,koh2018stronger}. Our theoretical analysis is related to robust estimation under contamination and heavy-tailed sampling~\citep{huber1964robust,catoni2012challenging,lugosi2019subgaussian}.

\section{The analyticity barrier and how to overcome it}
\label{sec:analyticity-barrier}

In this section, we first present the empirical scaling law in large language models that motivates the rest of the paper and then discuss theoretical considerations behind it. The details of the pre-training set up of the languge model is discussed in Appendix~\ref{app_expt_details}. 
We isolate how a model's clean language-modeling quality degrades as a controlled
fraction of its pre-training tokens is corrupted.
Because we target overall degradation rather than a backdoor, our poison fraction
is larger than in attack-focused work \citep{ICLR2025_4dade38e,souly2025poisoning}.
For each poison rate $\epsilon$ we train a poisoned model and a clean baseline
sharing architecture, token budget, and optimization schedule, evaluate both on
the same held-out clean data, and report, at matched training time $t$,
\begin{equation}
\Delta(\epsilon,t)=\frac{P_{\mathrm{poisoned}}(\epsilon,t)}{P_{\mathrm{clean}}(t)}-1
=\exp\!\big(\mathcal R_{\mathrm{poisoned}}(\epsilon,t)-\mathcal R_{\mathrm{clean}}(t)\big)-1,
\end{equation}
In the
small-degradation regime we are interested in $\Delta\approx\mathcal R_{\mathrm{poisoned}}-\mathcal R_{\mathrm{clean}}$.
Aggregating checkpoints within $15\%$ of a target token count, across model
sizes and across training, we find (see Figures~\ref{fig:llm_scaling_law_per_model}) that the degradation is approximately
power-law in $\epsilon$ (see \eqref{eq:intro_scaling_law}). The slope and intercept drift over training and with model size. The rest of the paper asks what produces such a
law.

Solvable models of neural scaling
laws almost invariably express the test loss through analytic functionals of a finite collection of order parameters
\citep{spigler2020asymptotic,sharma2020neural,bordelon2020spectrum,bahri2021explaining,maloney2022solvable}.
Say injecting poison at rate $\varepsilon$ perturbs these order parameters analytically,
so the induced loss is analytic in $\varepsilon$ at $\varepsilon=0$ and its expansion
carries integer powers only; taken about the clean-data optimum of the validation
loss, where the gradient vanishes, the leading correction is generically quadratic. The next result sharpens this observation
into a no-go statement. We state the obstruction at the natural expansion point, the clean-data optimum of the validation loss; in
Section~\ref{sec:llm-mechanism-main} we treat the more general situation in which the reference is not this optimum - the clean checkpoint of an
early-stopped run, at which the validation loss gradient need not vanish-and show that it introduces a linear term but fails to explain the factional scaling exponent observed in stronger LLM during pre-training, instead, the required mechanism essentially follows from the precise understanding the assumptions of the no-go theorem below. 

\begin{theorem}[Analyticity barrier]
\label{thm:analyticity-barrier}
Let $\theta\in\Theta$ denote the model parameters, $\mathcal P$ the clean and $\mathcal Q$ the
poison distribution, and sample from $\mathcal P_\varepsilon=(1-\varepsilon)\mathcal P+\varepsilon\mathcal Q$.
Suppose the population loss is
$L_\varepsilon(\theta)=(1-\varepsilon)L_{\mathcal P}(\theta)+\varepsilon L_{\mathcal Q}(\theta)$ with
$L_{\mathcal P},L_{\mathcal Q}$ real analytic near a clean solution $\theta_0$ that is a nondegenerate
optimum, $\nabla L_{\mathcal P}(\theta_0)=0$, $H:=\nabla^2L_{\mathcal P}(\theta_0)$,
$\lambda_{\min}(H)>0$. Let the clean evaluation risk $\mathcal R$ be real-analytic with
$\nabla\mathcal R(\theta_0)=0$, $G:=\nabla^2\mathcal R(\theta_0)$, $\lambda_{\min}(G)>0$, and set
$g:=\partial_\varepsilon\nabla L_\varepsilon(\theta_0)|_{\varepsilon=0}
=\nabla L_{\mathcal Q}(\theta_0)-\nabla L_{\mathcal P}(\theta_0)$. Then, for all sufficiently small
$\varepsilon$, there is a unique analytic branch of stationary points
$\theta_\varepsilon$ near $\theta_0$, and
\begin{align}
\theta_\varepsilon-\theta_0&=-\varepsilon H^{-1}g+O(\varepsilon^2),
\label{eq:analytic-parameter-response}\\
\Delta(\varepsilon):=\mathcal R(\theta_\varepsilon)-\mathcal R(\theta_0)
&=\frac{\varepsilon^2}{2}g^\top H^{-1}GH^{-1}g+O(\varepsilon^3).
\label{eq:analytic-risk-response}
\end{align}
More generally, analyticity allows only integer powers of $\varepsilon$, stationarity removes the linear term. 
\end{theorem}
See Appendix~\ref{app:proof-analyticity-barrier} for the proof.

 Escaping the barrier therefore requires a genuinely non-analytic ingredient rather than a more elaborate smooth model; the
remainder of this section supplies one, showing that a heavy-tailed spectrum together
with an optimized poison dependent truncation generates precisely the singularity a fractional
exponent demands.

\begin{theorem}[Adaptive-truncation law]
\label{thm:adaptive-truncation-law}
Let $\psi(Z)\in\mathbb R^m$ be a statistic with clean mean $\mu=\mathbb E_{\mathcal P}[\psi(Z)]$.
For $\tau>0$ let $T_\tau$ be a truncation map, and let the poison distribution
$\mathcal Q_\varepsilon$ depend on $\varepsilon$. Define
$m_{\varepsilon,\tau}:=(1-\varepsilon)\mathbb E_{\mathcal P}[T_\tau(\psi(Z))]+\varepsilon\mathbb E_{\mathcal Q_\varepsilon}[T_\tau(\psi(Z))]$.
Assume that, as $\tau\to\infty$, there exist $\beta,\gamma>0$ and vectors $b,c$ with
\begin{align}
\mathbb E_{\mathcal P}[T_\tau(\psi(Z))]-\mu&=b\tau^{-\beta}+o(\tau^{-\beta}),
\label{eq:general-truncation-bias}\\
\mathbb E_{\mathcal Q_\varepsilon}[T_\tau(\psi(Z))]-\mathbb E_{\mathcal P}[T_\tau(\psi(Z))]&=c\tau^\gamma+o(\tau^\gamma).
\label{eq:general-poison-influence}
\end{align}
Let the learned parameter be a smooth $\theta_{\varepsilon,\tau}=\Phi(m_{\varepsilon,\tau})$,
$\Phi(\mu)=\theta_0$, $J:=D\Phi(\mu)$, with clean risk $\mathcal R$ being $C^3$ near $\theta_0$,
$\nabla\mathcal R(\theta_0)=0$, $G=\nabla^2\mathcal R(\theta_0)$, $\lambda_{\min}(G)>0$. For fixed
$t>0$, choose $\tau_\varepsilon=t\,\varepsilon^{-1/(\beta+\gamma)}$ and set
$h_t:=t^{-\beta}b+t^\gamma c$. Then
\begin{align}
m_{\varepsilon,\tau_\varepsilon}-\mu&=\varepsilon^{\beta/(\beta+\gamma)}h_t+o(\varepsilon^{\beta/(\beta+\gamma)}),
\label{eq:truncated-moment-fractional}\\
\mathcal R(\theta_{\varepsilon,\tau_\varepsilon})-\mathcal R(\theta_0)&=\tfrac12 h_t^\top J^\top GJ h_t\,
\varepsilon^{2\beta/(\beta+\gamma)}+o(\varepsilon^{2\beta/(\beta+\gamma)}).
\label{eq:truncated-risk-fractional}
\end{align}
Thus, whenever $h_t^\top J^\top GJ h_t>0$, adaptive truncation produces the
power-law exponent $a=2\beta/(\beta+\gamma)$, which need not be an integer.
\end{theorem}
The proof is in Appendix~\ref{app:proof-adaptive-truncation}. Crucially,
truncation alone is not enough: a finite exponent $\beta$ is what pushes $a$ below the analytic value $2$.  For example, if a scalar statistic
$\psi(Z)$ has tail on clean data
\begin{equation}
    \mathbb{P}\!\left(\lvert\psi(Z)\rvert>u\right)
    \sim c u^{-q},
    \qquad q>1,
\end{equation}
then clipping at level $\tau$ typically produces
\begin{equation}
    \left\|
    E_P[T_\tau(\psi(Z))]-\mu
    \right\|
    \asymp
    \int_{\tau}^{\infty}
        \mathbb{P}\!\left(\lvert\psi(Z)\rvert>u\right)\,du
    \asymp
    \tau^{-(q-1)}.
\end{equation}
finite $\beta=q-1$. We show that it is realized within a solvable model in the next section.

\section{A solvable model and the non-commuting scaling limits}
\label{sec:solvable-main}
Heavy-tailed data
carry signal in rare examples; a truncation at level $\tau$ discards the
tail beyond $\tau$, so raising $\tau$ lowers the clean data truncation bias but lets poisoned data influence prediction more strongly. Balancing the two
over $\tau$ produces a fractional exponent. In the rest of this section we explain this in more details.

\subsection{The random co-ordinate model} Our model is truncated linear regression in higher dimensions on heavy tailed data. We work in $P$ dimensions.  Let $e_1,\ldots,e_P$ denote the standard basis of $\mathbb{R}^P$. 
 Let $U$ be a random variable that has a heavy-tail 
 \begin{equation}
    \begin{aligned}
        \mathbb P(U>t)=t^{-q_\star}.
    \end{aligned}
 \end{equation} 
 A clean example draws a
coordinate $J\sim\mathrm{Unif}\{1,\dots,P\}$ and $U$, and sets
$X_c=\sqrt{PU}\,e_J$, $Y_c=c\sqrt U$, that is the teacher weight is  $\theta_{\star,P}=\tfrac{c}{\sqrt P}\mathbf 1$; a poison example sets $X_{\mathrm p}=\sqrt P e_J$,
$Y_{\mathrm p}=c+\eta_{\mathrm p}\tau$.
Each of $D$
i.i.d.\ samples is clean with probability $1-\epsilon$ and poisoned with probability $\epsilon$.  Truncating
the per-example second-moment and cross-moment statistics at level $\tau$ acts as 
 \begin{equation}
    \begin{aligned}
    A^{(\tau)}
    :=
    XX^\top\min\left\{1,\frac{P\tau}{\left\lVert {XX^\top}\right\rVert_{\mathrm{op}}}\right\},\qquad
    b^{(\tau)}
    :=
    XY\min\left\{1,\frac{\sqrt P\,\tau}{\left\lVert XY \right\rVert_{2}}\right\}.
        \end{aligned}
 \end{equation}
    Given $D$ training samples, define
 \begin{equation}
    \begin{aligned}
    \widehat\Sigma_{\epsilon,\tau}
    :=
    \frac1D\sum_{i=1}^D A_i^{(\tau)},\qquad
    \widehat v_{\epsilon,\tau}
    :=
    \frac1D\sum_{i=1}^D b_i^{(\tau)}.
        \end{aligned}
 \end{equation}
For $\varrho>0$, define the ridge estimator  
\begin{equation}
    \begin{aligned}\widehat\theta_{\epsilon,\tau,\varrho}
    :=
    (\widehat\Sigma_{\epsilon,\tau}+\varrho I_P)^{-1}
    \widehat v_{\epsilon,\tau}.
        \end{aligned}
 \end{equation}

\subsection{The simplification and map to Poisson point process} 

For coordinate $j$, let $N_{j,\mathrm c}$ and
$N_{j,\mathrm p}$ denote the numbers of clean and poisoned training samples
assigned to coordinate $j$.  Let $U_{j,1},U_{j,2},\ldots$ denote the clean radial
variables assigned to that coordinate.  Define $a_\tau(U)=\min(U,\tau)$ and
$r_\tau(U)=c\min(U,\tau/c)$ and
\begin{align}
    S_{j,A}^{(D)}
    :={}
    \sum_{\ell=1}^{N_{j,\mathrm c}}a_\tau(U_{j,\ell})
    +
    N_{j,\mathrm p},\qquad
    S_{j,B}^{(D)}
    :={}
    \sum_{\ell=1}^{N_{j,\mathrm c}}r_\tau(U_{j,\ell})
    +
    (c+\eta_{\mathrm p}\tau)N_{j,\mathrm p}.
\end{align}

We show that (see Appendix~\ref{sec:coordinate-isotropic-model}) the truncated ridge estimator is coordinate-separable,
\begin{equation}
    \begin{aligned}
    \sqrt P\,(\widehat\theta_{\epsilon,\tau,\varrho})_j
    =
    G_{\phi_D,\varrho}
    \left(S_{j,A}^{(D)},S_{j,B}^{(D)}\right)
            \end{aligned}
 \end{equation}
 where  $\phi=P/D$ and,
\begin{equation}
    G_{\phi,\varrho}(s_A,s_B)
    :=
    \begin{cases}
    \displaystyle
    \frac{\phi s_B}{\varrho+\phi s_A},
    & \varrho+\phi s_A>0,\\[3mm]
    0,& \varrho=0,\ s_A=0.
    \end{cases}
\end{equation}
As a consequence, when $P,D\to\infty$ with $\phi$ held fixed, the
clean and poison hits on a fixed coordinate converge to independent Poisson
variables
\begin{equation}
\label{eq:poisson}
N_{\mathrm p}\sim\operatorname{Pois}\!\Big(\tfrac{\epsilon}{\phi}\Big),
\qquad
N_{\mathrm c}\sim\operatorname{Pois}\!\Big(\tfrac{1-\epsilon}{\phi}\Big),
\end{equation}
 and the risk concentrates through the standard scalar deterministic equivalent.
Then,
\begin{equation}
    \begin{aligned}
        \mathcal R_{\mathrm{clean},P}(\widehat\theta_{\epsilon,\tau,\varrho})-\mathcal R_{\mathrm{clean},P}(\theta_{\star,P})\xrightarrow{\mathbb P}\mathcal R_{\phi,\varrho}(\epsilon,\tau):=\mu_U\,\mathbb E[(c-G_{\phi,\varrho}(S_A^{\epsilon,\tau},S_B^{\epsilon,\tau}))^2],
    \end{aligned}
\end{equation}
where $\mu_U:=\mathbb E U=\frac{q_\star}{q_\star-1}$, and $
    S_A^{\epsilon,\tau}
    :={}
    \sum_{\ell=1}^{N_{\mathrm c}}a_\tau(U_\ell)
    +
    N_{\mathrm p},
    S_B^{\epsilon,\tau}
    :={}
    \sum_{\ell=1}^{N_{\mathrm c}}r_\tau(U_\ell)
    +
    (c+\eta_{\mathrm p}\tau)N_{\mathrm p}.$ This form helps us to identify this as a Poisson point process. More precisely, the clean samples reaching a coordinate form a Poisson process of
intensity $(1-\epsilon)/\phi$ , each point carrying the paired
mark $(a_\tau(U_\ell),r_\tau(U_\ell))$, while the poisoned samples form an
independent Poisson process of intensity $\epsilon/\phi$ with the constant
mark $(1,\,c+\eta_{\mathrm p}\tau)$.  For a Poisson-indexed sum the joint
Laplace functional factorizes over the intensity, for example,
\begin{equation*}
    \mathbb E\,\exp\!\Big(-\!\!\sum_{\ell=1}^{N_{\mathrm c}} t\!\cdot\!(a_\tau,r_\tau)(U_\ell)\Big)
    =\exp\!\Big(\tfrac{1-\epsilon}{\phi}\!\int\!\big(e^{-t\cdot(a_\tau(u),r_\tau(u))}-1\big)\,d\nu(u)\Big),
\end{equation*}
so every moment of $(S_A,S_B)$ and, decisively, its $\tau\to\infty$ asymptotics, reduces to a single integral against the mark intensity. This allows a detailed analytical study.

\begin{figure}[t]
\centering
\begin{subfigure}{0.49\textwidth}\centering
\includegraphics[width=\linewidth]{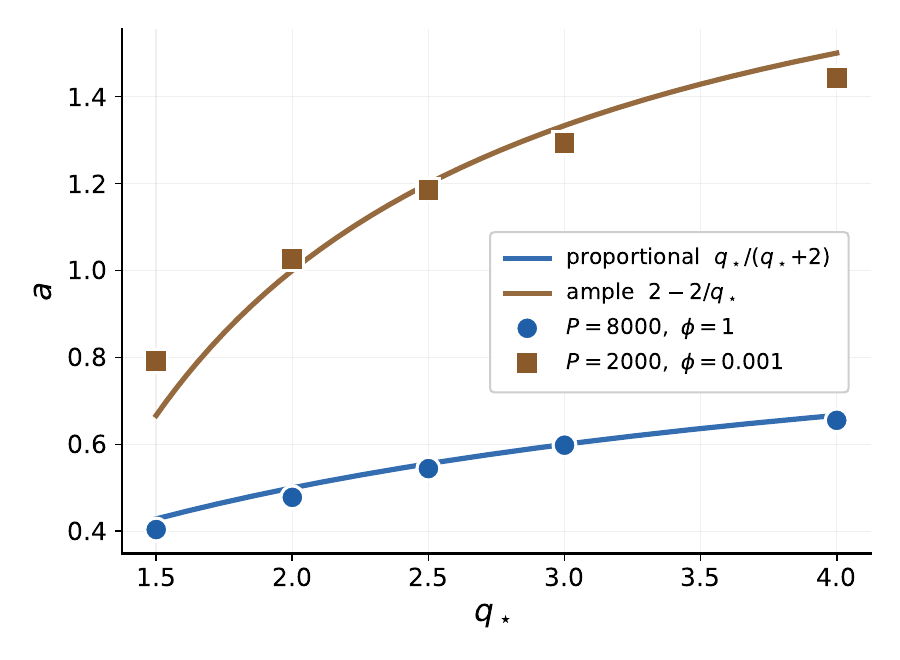}\caption{Scaling exponents vs.\ $q_\star$.}\label{fig:evq}
\end{subfigure}\hfill
\begin{subfigure}{0.49\textwidth}\centering
\includegraphics[width=\linewidth]{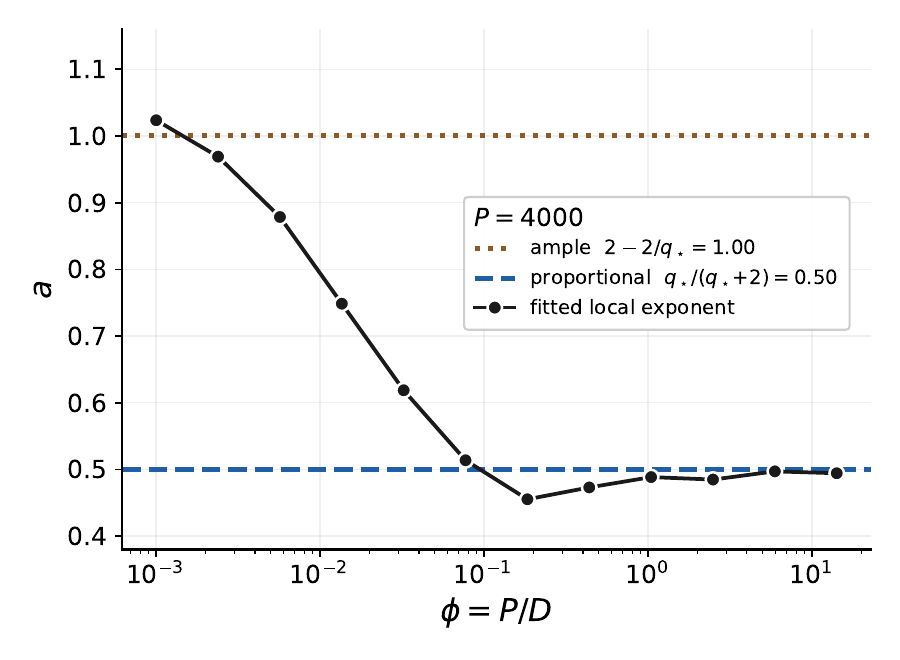}\caption{Scaling exponents vs.\ $\phi$.}\label{fig:phase}
\end{subfigure}
\caption{ (\subref{fig:evq}) The measured optimized
exponent follows $q_\star/(q_\star+2)$ (proportional) and $2-2/q_\star$ (ample)
across $q_\star\in[1.5,4]$; the branches are well separated. (\subref{fig:phase})
For $q_\star=2$, fitting the exponent over a fixed $\epsilon$-window as $\phi=P/D$
varies shows a clean cross-over from the ample value $1.0$ to the proportional
value $0.5$.}
\label{fig:noncommute}
\end{figure}
\begin{figure}[t]
\centering
\begin{subfigure}{0.49\textwidth}\centering
\includegraphics[width=\linewidth]{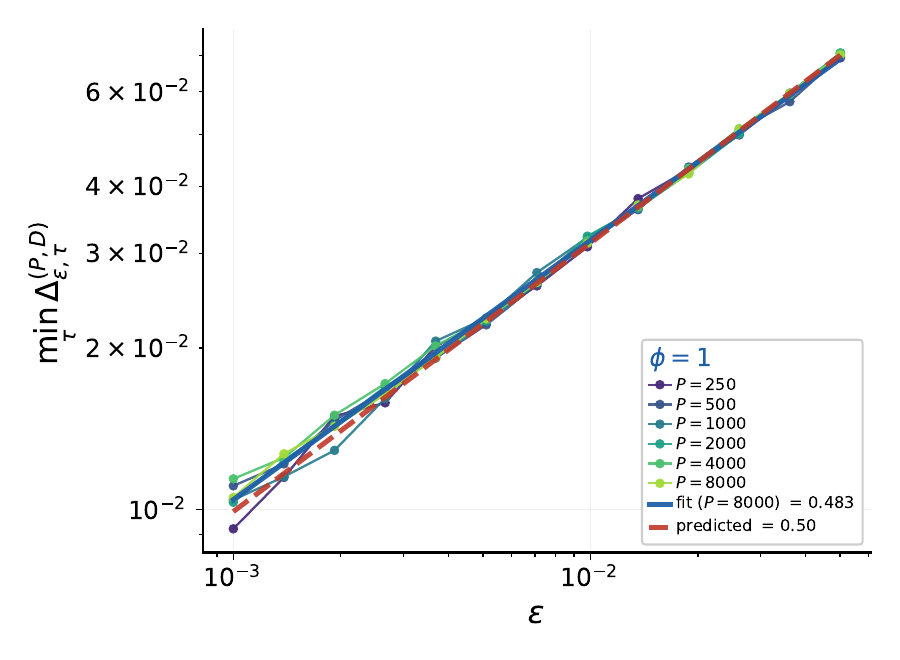}\caption{Higher dimensional limit, scaling law as $\epsilon^{q_\star/(q_\star+2)}$.}\label{fig:prop}
\end{subfigure}\hfill
\begin{subfigure}{0.49\textwidth}\centering
\includegraphics[width=\linewidth]{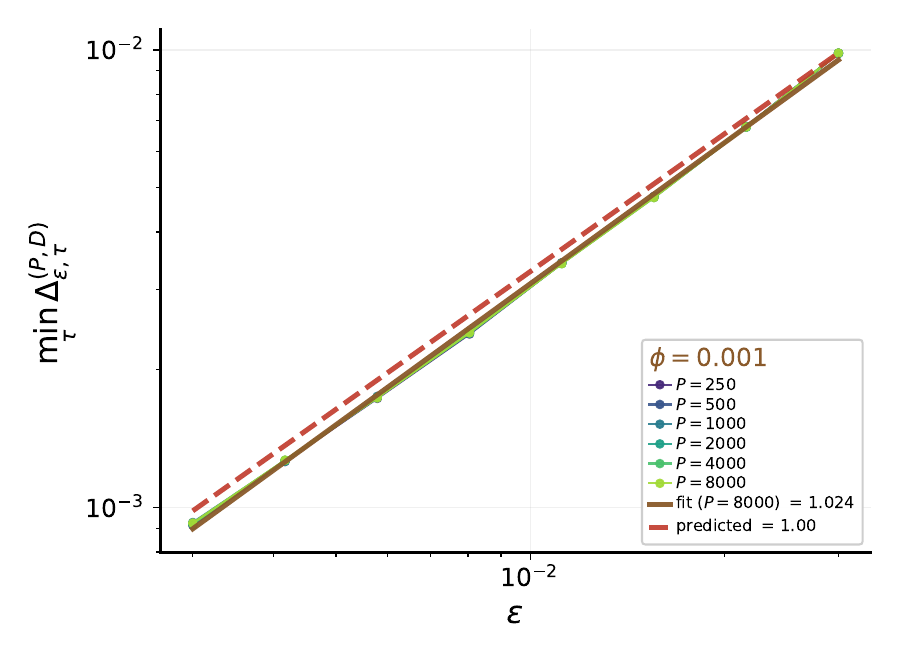}\caption{Ample data limit, scaling law as  $\epsilon^{2-2/q_\star}$.}\label{fig:ample}
\end{subfigure}
\caption{At $q_\star=2$
the optimized degradation follows the predicted proportional
(\subref{fig:prop}) and ample (\subref{fig:ample}) slopes simulation (points)
against the exact prediction (dashed).}
\label{fig:regimes}
\end{figure}

\subsection{Optimal truncation and non-commuting limits of the scaling law} The clean truncation
bias falls as $\tau$ grows while the poison influence rises, so there is a
bias--poison optimum. Optimizing $\tau$ yields the paper's central theoretical
result on the scaling law:

\begin{theorem}[High-dimensional vs.\ ample-data limit]
\label{thm:noncommuting-main}
Work at ridgeless limit $\varrho=0$ with $P/D\to\phi\in(0,\infty)$. 
 The observable of interest is the clean data excess risk of the poisoned, truncated
estimator against the clean, untruncated one,
$\Delta^{(P)}_{\eps,\tau}:=\mathcal R(\widehat\theta_{\eps,\tau,0})-\mathcal R(\widehat\theta_{0,\infty,0})$.\\
\textbf{(i) High-dimensional limit} (fixed $\phi$, $\eps\!\downarrow\!0$):
$\ \Delta^{(P)}_{\epsilon,\tau}=K^{\mathrm{clean}}_{\phi,c,q_\star}\tau^{-q_\star}+C_{\phi,0}\,\eps\tau^{2}$,
$\ \tau_{\mathrm{opt}}\asymp\eps^{-1/(q_\star+2)}$, $\ \min_\tau\Delta^{(P)}_{\epsilon,\tau}\asymp\eps^{q_\star/(q_\star+2)}$.\\
\textbf{(ii) Ample-data limit} ($\phi\!\downarrow\!0$ first and then $\eps\!\downarrow\!0$):
$\ \Delta^{(P)}_{\epsilon,\tau}=\tfrac1{\mu_U}\big(A\,\tau^{-(q_\star-1)}+B\,\eps\tau\big)^2$,
$\ \tau_{\mathrm{opt}}\asymp\eps^{-1/q_\star}$, $\ \min_\tau\Delta^{(P)}_{\epsilon,\tau}\asymp\eps^{2-2/q_\star}$.
\end{theorem}

The full statement, constants, and proof are
Result~\ref{thm:noncommuting-two-limits} in Appendix~\ref{sec:same-observable-two-limits}
(built on the deterministic equivalent of Appendix~\ref{sec:fixed-aspect-de} and
the analyses of Appendices~\ref{sec:fixed-phi-modified-attack}--\ref{sec:population-recovery}). An intuitive reason for the non-commutativity is as follows:
\begin{itemize}
\item \textbf{Higher dimensional limit}: At fixed $\phi$ each coordinate has only
$N_{\mathrm c}\sim\operatorname{Pois}\!\big(\tfrac{1-\epsilon}{\phi}\big)=O(1)$ clean
samples. A coordinate has $N_{\mathrm p}\ge1$ with probability $\approx\epsilon/\phi$.  On a coordinate that has $N_{\mathrm p}\ge1$  the poison weight is
therefore not small.
There is not enough clean data on the coordinate to outweigh the single poison entry, so it moves
the ratio by $O(\tau)$. The squared error on that coordinate is order
$\tau^2$. Here $\epsilon$ enters only  as the rarity of
being hit: a fraction $\epsilon$ of coordinates carry a
$\tau^2$ error and the rest carry none, this leads to $\epsilon \tau^2$ term in excess risk.
\item \textbf{Ample-data limit}: Every coordinate is hit by both large number of clean and poisoned samples since $N_{\mathrm c}\sim(1-\epsilon)/\phi\to\infty$ and
$N_{\mathrm p}\sim\epsilon/\phi\to\infty$. Per coordinate influence of the poisoning is diluted to $\epsilon\tau$ and it enters the excess risk as $(\epsilon\tau)^2$. Here $\epsilon$
measures how much poison sits on every coordinate.
\end{itemize}

Case (ii) is exactly the adaptive-truncation law
(Result~\ref{thm:adaptive-truncation-law}) with $\beta=q_\star-1$, $\gamma=1$.
The two exponents cross at
$q_\star=\sqrt5-1\approx1.236$: for
$q_\star>\sqrt5-1$ one has $q_\star/(q_\star+2)<2-2/q_\star$, so the limits do not commute and
the proportional regime is strictly more fragile; for $1<q_\star<\sqrt5-1$ the
ordering reverses. We verify these numerically in Figure~\ref{fig:noncommute} and \ref{fig:regimes}.

\section{How the scaling law may arise in LLM pre-training}
\label{sec:llm-mechanism-main}
We now connect the mechanism to language models - the link requires a couple of modeling hypotheses within a local quadratic approximation to the training dynamics around a late time checkpoint \citep{meterez2026defense}, developed in full in Appendix~\ref{subsec:implicit-truncation-llm}. A preconditioned (by operator $S$) optimizer, run for time $T$ starting from the checkpoint, acts as a filter
$q_T(\lambda)=(1-e^{-\lambda T})/\lambda$ on the spectrum of the preconditioned Gauss--Newton curvature $A$ (Result~\ref{thm:preconditioned-reduction}). The filtered spectrum is responsible for learning away from the checkpoint: it learns well-conditioned directions, i.e., $\lambda T \gg 1$, quickly and leaves flat ones, i.e., $\lambda T \ll 1$, under-learned, so finite time is equivalent to
\begin{equation}
   \text{Cut-off at $\tau=T$ on the inverse-curvature variable $X=\lambda^{-1}$, $\lambda\in\operatorname{spec}(A)$}.
\end{equation}
In addition, we argue $X$ is heavy-tailed, 
\begin{equation}
    \mathbb{P}
    \left(
        X>\tau
    \right)
    \asymp
    \tau^{-\alpha},\qquad
    \tau\to\infty.
\end{equation}
See Figure~\ref{fig:llm-heavy-tail-test} for numerical justification. These two facts supply similar mechanism as in the previous section and lead to the scaling law as we discuss in the rest of the section.

\begin{figure}[h]
\centering
\includegraphics[width=0.93\textwidth]{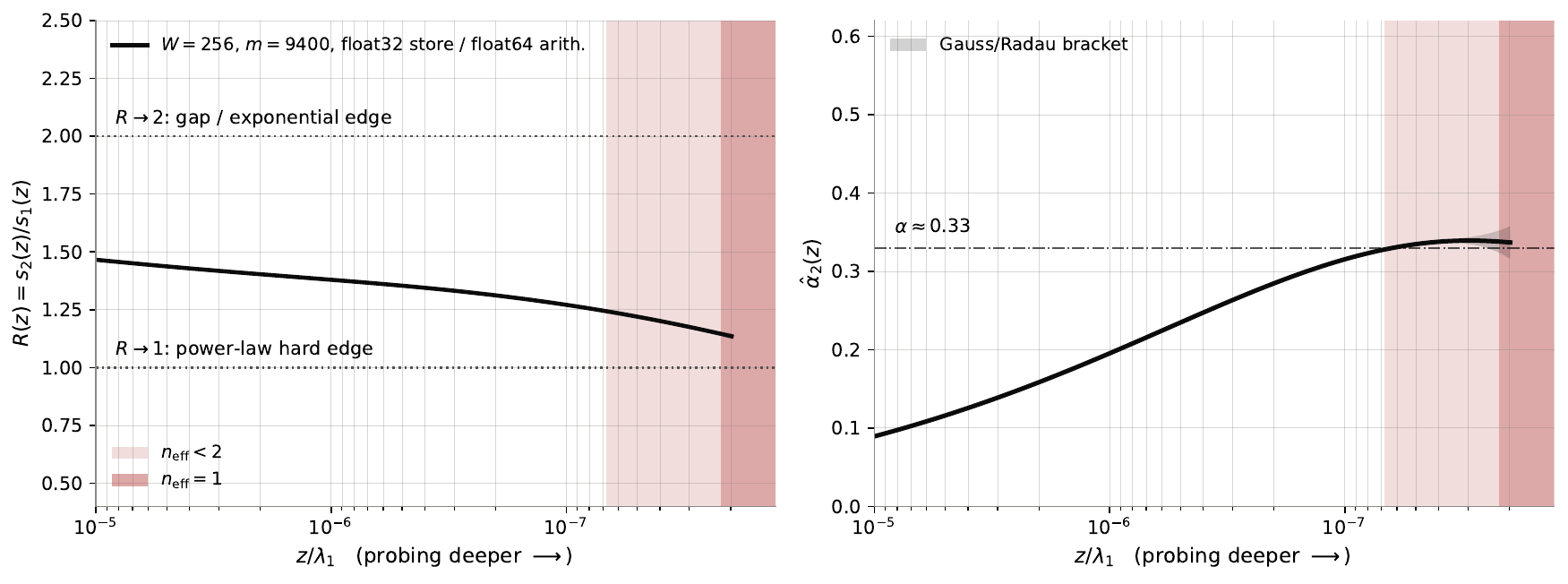}
\caption{The plot corresponds to an experiment performed on OLMo-style model with total parameters $P=1.677\times10^{8}$ pre-trained upto $3$B tokens. The graph on the left shows that heavy tailed curvature spectrum exists and, the graph on the right determines the exponent. Red shades show where numerical analysis is unreliable. For the details of $R,\hat{\alpha}_2$ see Appendix~\ref{app_gn}, these are a specially constructed function of $z$; smaller $z$ probes a smaller eigenvalue spectrum. The ratio test in Corollary~\ref{cor:ratio-test} shows that $R \to 1$ asserts a heavy-tailed spectrum of $X=\lambda^{-1}$ with hard edge exponent $\alpha\leq 1$. On the other hand Remark~\ref{lem:dichotomy} shows $\hat \alpha_2$ approaches the value of $\alpha\approx 0.33$. }
\label{fig:llm-heavy-tail-test}
\end{figure}

The clean data training for time $T$ is captured by weight $\phi_0(T)$, which
in an infinite time limit leads to $\phi_0^\star$. With
$e_0(T):=\phi_0(T)-\phi_0^\star$ the clean truncation residual and
$G$ the clean-risk curvature $ S^{1/2}\,\nabla^2\mathcal R_{\mathcal P}(\phi_0^\star)\,S^{1/2}$, clean truncation bias is given by $B(T)^2:=e_0^\top Ge_0$.
Poison at rate $\epsilon$ shifts the trajectory to
$\phi_\epsilon(T)=\phi_0(T)+\epsilon\,\psi(T)+o(\epsilon)$, where the
poison-response direction itself splits into a checkpoint offset carried over
from a poisoned late-time checkpoint and a forced response built up
afterwards, $\psi(T)=e^{-H_0T}\psi^{\mathrm{in}}+\psi_{\mathrm f}(T)$. The poison influence built-up is captured by $I(T)^2:=\psi_{\mathrm f}^\top G\psi_{\mathrm f}$.
The measured observable is
the degradation at matched training,
$\delta\mathcal R(\epsilon,T):=\mathcal R_{\mathcal P}(\phi_\epsilon(T))-\mathcal R_{\mathcal P}(\phi_0(T))$.
It resolves the response into
four competing contributions,
\begin{equation}
\delta\mathcal R\;\asymp\;
\underbrace{\epsilon\,\Gamma_\star\,T^{-\beta_{\mathrm{in}}}}_{\text{checkpoint offset}}
\;+\;\underbrace{\epsilon\,\Gamma_\star\,T^{\gamma}}_{\text{train--eval gap}}
\;+\;\underbrace{\epsilon\,B(T)I(T)}_{\text{clean bias}\ \asymp\,\epsilon\,T^{\gamma-\beta}}
\;+\;\underbrace{\tfrac12\epsilon^2 I(T)^2}_{\text{pure poison}\ \asymp\,\epsilon^2\,T^{2\gamma}},
\label{eq:llm-four-terms}
\end{equation}
where $\Gamma_\star:=\|g_\star\|_{G^{-1}}, g_\star:=\nabla\mathcal R_{\mathcal P}(\phi_0^\star)$ is the train-eval mismatch. 
We show $B(T)\asymp T^{-\beta}$ (Result~\ref{thm:clean-truncation-falloff}) and
$I(T)\asymp T^{\gamma}$ (Result~\ref{thm:poison-strength-growth}) with
$\beta>\gamma>0$ (that $\beta>\gamma$ says the clean signal is learned
faster than the poison accumulates) and $ B_{\mathrm{in}}(T):=\big\|e^{-H_0T}\psi^{\mathrm{in}}\big\|_G
    \asymp T^{-\beta_{\mathrm{in}}}, \beta_{\mathrm{in}}>0$. 
 These are power laws for the same
reason-the spectrum is scale-free near weak curvature modes, so advancing the cutoff rescales every
quantity by a pure power. The clean bias decays
because it is the residual clean signal still sitting beyond the moving cutoff; each
increment of time resolves more of the tail and the leftover shrinks. The poison influence grows because the
poison's leverage lives in the  low-curvature directions: there a small
contamination produces a large parameter displacement, and a longer time horizon reaches ever more such directions, so the aggregate
imprint climbs up.

\emph{Gap-dominated crossover.} When the $\beta$-independent train--eval gap
leads, $\epsilon\,\Gamma_\star I(T)$ and $\tfrac12\epsilon^2 I(T)^2$ cross at
$T^{(g)}_\epsilon\asymp\epsilon^{-1/\gamma}$ and,
\begin{equation}
    \Delta\asymp\Gamma_\star\,\epsilon\,T^{\gamma} \text{ for } T\ll T^{(g)}_\epsilon, \qquad \Delta\asymp\epsilon^2 T^{2\gamma} \text{ for } T\gg T^{(g)}_\epsilon. \qquad \text{(Result~\ref{thm:scaling-law-gap})}
\end{equation}
Both competing terms grow with $T$, so there is
no interior optimal horizon. The
degradation is $\epsilon$-independent at the crossover time leading to $a=0$. 

\emph{Bias-dominated crossover.} When a decaying heavy-tail
term leads, the falling clean bias $\epsilon\,B(T)I(T)\asymp\epsilon T^{\gamma-\beta}$
and the rising pure term balance at a poison-dependent cross-over time
$
T_\epsilon\asymp\,\epsilon^{-1/(\beta+\gamma)},
$
an interior minimum, where
the scaling law emerges,
\begin{equation}
\Delta(\epsilon)\asymp\epsilon^{a},\qquad a=\frac{2\beta}{\beta+\gamma}
\qquad(\text{Result~\ref{thm:scaling-law-optimal-truncation}}).
\label{eq:llm-frac}
\end{equation}

The checkpoint offset drives the same crossover with $\beta$ replaced by
$\beta_{\mathrm{in}}+\gamma$, the two share their origin in
the heavy-tailed spectrum, so we discuss only the bias case. The scaling exponent is
non-integer iff $\beta<\infty$-finitely many moments, i.e.\ a genuinely
heavy-tailed inverse-curvature spectrum; as $\beta\to\infty$ it returns to the
analytic value $2$. The crossover time above grows as an inverse power of the poisoning fraction, $T^{(g)}_\epsilon, T_\varepsilon\asymp\varepsilon^{-\kappa}$, when $\kappa>0$ it diverges as $\varepsilon\to0$. The fractional exponent $a$ is
therefore expected to be observable only in the regime of a small poisoning fraction
together with a correspondingly long training steps - one whose length grows as the
same inverse power of $\varepsilon$. Analyzing this scaling is complicated by the fact that $\kappa$ is not a directly controlled quantity in LLM training: it is governed by the inverse weak-curvature scales of $A$ and the effect of the poisoned data in this regime. We
therefore proceed under the modeling assumption that the LLM training horizon considered in this paper is of the same order as the crossover time.

Models with a
large train-eval mismatch $\Gamma_\star$-empirically the smaller LLMs-sit in the
gap-dominated regime: the measured exponent starts at $a=1$ early in training and falls toward $a=0$ as the time approaches $T^{(g)}_\epsilon$ (see Appendix~\ref{subsec:implicit-truncation-llm}, Figure~\ref{fig:llm_scaling_law_per_model_app}(a)). Models with a
small mismatch-the larger LLMs sit in the bias/checkpoint-dominated regime, where
the exponent starts near $a=1$ and rises toward the fractional value
$2\beta/(\beta+\gamma)>1$ as the time approaches the balanced $T_\epsilon$  (see Appendix~\ref{subsec:implicit-truncation-llm}, Figure~\ref{fig:llm_scaling_law_per_model_app}(d)).

The balanced-horizon calculation in this section is an
ample-data one in the aspect ratio: writing
$\phi_{\mathrm{LLM}}=P_{\mathrm{eff}}/D_{\mathrm{eff}}$,
Result~\ref{thm:noncommuting-two-limits} warns that taking $\epsilon\to0$ before
$\phi_{\mathrm{LLM}}\to0$ gives a different exponent, so the observed law need
not extend to arbitrarily small $\epsilon$ and one should expect it for
$\epsilon\gtrsim\phi_{\mathrm{LLM}}$ (see Figure~\ref{fig:llm_small_eps} for more discussion). The full local model, assumptions, and proofs
are in Appendix~\ref{subsec:implicit-truncation-llm}.

\section{Conclusion and limitations}
\label{sec:conclusion-limitations}
We pointed out the scaling law for pre-train poisioning of LLMs (Figure~\ref{fig:llm_scaling_law_per_model}). Then discussed theoretical analysis of  a non-integer poisoning exponent 
(Result~\ref{thm:analyticity-barrier} and \ref{thm:adaptive-truncation-law}), and provided a solvable model in which heavy tails
read out through a poison-dependent truncation supply it
(Results~\ref{thm:noncommuting-main}),
and proved that the optimized degradation has two distinct exponents depending on the order of the high-dimensional and small-poison limits, verifying all predictions
numerically (Figure~\ref{fig:noncommute}, \ref{fig:regimes}). Within local LLM training dynamics, we argue how the scaling law may arise in Section~\ref{sec:llm-mechanism-main}.
The exact theoretical results are for a stylized
coordinate model; the connection to LLM training dynamics requires certain modeling hypotheses that remain an open question to establish in the future work. In addition, careful study of  small-$\epsilon$ regime in LLMs and the fate of the scaling law requires further experiments. On the theoretical side,
extending the deterministic equivalent to dense covariances and multi-coordinate examples are the main open directions.

\section*{Acknowledgments}

We thank Alex Meterez for several very insightful discussions. Also, we thank Bingbin Liu for comments on a preliminary version of the manual-script.
I.H. is supported by DARPA grant AIQ-HR00112520041. C.P. is supported by an NSF CAREER Award (IIS-2239780), DARPA grants DIAL-FP-038 and AIQ-HR00112520041, the Simons Collaboration on the Physics of Learning and Neural Computation, and the William F. Milton Fund from Harvard University. This work has been made possible in part by a gift from the Chan Zuckerberg Initiative Foundation to establish the Kempner Institute for the Study of Natural and Artificial Intelligence.

\appendix
\section{Non-commuting high-dimensional and small-poisoning limits}
\label{app:hd-noncommuting-limits}

\subsection{Connection to the local LLM approximation}

The purpose of this appendix is to test the local LLM mechanism in a model that
retains finite-sample, high-dimensional effects.  Around a reference checkpoint
$\theta_0$, the local quadratic approximation to pre-training has the form
\[
    \widehat{\mathcal L}(\theta_0+\delta\theta)
    \simeq
    \widehat{\mathcal L}(\theta_0)
    -\widehat v_{\rm LLM}^{\top}\delta\theta
    +\frac12\delta\theta^{\top}
       \widehat H_{\rm LLM}\delta\theta,
\]
where, for token-prediction events indexed by $i=1,\ldots,D$,
\[
    \widehat\Sigma_{\rm LLM}
    =\frac1D\sum_{i=1}^D J_i^{\top}\mathcal W_iJ_i,
    \qquad
    \widehat v_{\rm LLM}
    =\frac1D\sum_{i=1}^D J_i^{\top}r_i.
\]
The corresponding local stationary equation is
\[
    \widehat\Sigma_{\rm LLM}\,\widehat{\delta\theta}
    =\widehat v_{\rm LLM}.
\]
The regression model below has exactly the same algebraic form:
\[
    \widehat\Sigma_{\epsilon,\tau}
    =\frac1D\sum_{i=1}^D A_i^{(\tau)},
    \qquad
    \widehat v_{\epsilon,\tau}
    =\frac1D\sum_{i=1}^D b_i^{(\tau)},
    \qquad
    (\widehat\Sigma_{\epsilon,\tau}+\varrho I_P)\widehat\theta
    =\widehat v_{\epsilon,\tau},
\]
where $A_i^{(\tau)}$ and $b_i^{(\tau)}$ are truncated versions of
$X_iX_i^\top$ and $X_iY_i$.  Thus
\[
\begin{array}{c|c|c}
\text{quantity} & \text{coordinate surrogate} & \text{local LLM model}\\
\hline
\text{parameter dimension} & P & P_{\rm eff}\\
\text{number of examples} & D & D_{\rm eff}\\
\text{per-example curvature} & A_i^{(\tau)} & J_i^\top W_iJ_i\\
\text{per-example response} & b_i^{(\tau)} & J_i^\top r_i\\
\text{fitted displacement} & \widehat\theta & \delta\theta\\
\text{aspect ratio} & \phi=P/D & P_{\rm eff}/D_{\rm eff}
\end{array}
\]
at the level of the local normal equation.

The analogy concerns having the same normal equations.  Truncation is implemented slightly
differently: the model below clips the two per-example statistics separately,
whereas finite-time LLM training acts as a spectral filter on inverse
curvature. However, in this setup, again, we can see the same set of logic used in LLM - heavy-tailed data, along with poison adaptive truncation, leads to the scaling law in the ample data limit. Hence, analysis of this appendix can teach some valuable lessons for the local LLM model. More precisely, the coordinate construction is useful precisely because it makes
one additional question explicit: does the small-poisoning limit agree with
the population calculation when the parameter-to-data aspect ratio remains
nonzero?  The remainder of the appendix answers this question exactly for the
coordinate model.

\subsection{Coordinate-isotropic heavy-tailed data distribution}
\label{sec:coordinate-isotropic-model}

Let $e_1,\ldots,e_P$ denote the standard basis of $\mathbb R^P$.
Fix
\[
    0<c<1,
    \qquad
    0<\eta_{\mathrm p}<1,
    \qquad
    q_\star>1.
\]
Let $U$ have Pareto tail
\[
    \mathbb P(U>t)=t^{-q_\star},
    \qquad t\ge 1,
\]
so that
\[
    \mu_U:=\mathbb E U=\frac{q_\star}{q_\star-1}.
\]
All truncation formulas below are used for levels \(\tau\ge 1\).
Define the teacher vector
\[
    \theta_{\star,P}:=\frac{c}{\sqrt P}(1,\ldots,1)^\top.
\]
A clean example is generated by drawing $J$ uniformly from
$\{1,\ldots,P\}$ and drawing $U$ independently, then setting
\[
    X=\sqrt{PU}\,e_J,
    \qquad
    Y=X^\top \theta_{\star,P}=c\sqrt U.
\]
The normalization $\sqrt P$ is chosen so that
\[
    \mathbb E[XX^\top]=\mu_U I_P.
\]
For any deterministic $\theta\in\mathbb R^P$, the clean data excess prediction risk is
\begin{equation}
\label{eq:coordinate-risk-identity}
    \mathcal R_{\mathrm{clean},P}(\theta)-\mathcal R_{\mathrm{clean},P}(\theta_{\star,P})
    =
    \mu_U\frac1P\sum_{j=1}^P\left(c-\sqrt P\,\theta_j\right)^2.
\end{equation}
Indeed, conditional on $J=j$ and $U$, one has
\[
    Y-X^\top \theta
    =
    \sqrt U\left(c-\sqrt P\,\theta_j\right),
\]
and averaging over $J$ and $U$ gives \eqref{eq:coordinate-risk-identity}.

\subsubsection{Modified truncation suitable for higher dimensional statistics}

The unnormalized matrix statistic has operator norm
\[
    \left\lVert{XX^\top}\right\rVert_{\mathrm{op}}=PU,
\]
and the vector statistic satisfies
\[
    XY=c\sqrt P\,U e_J,
    \qquad
    \left\lVert{XY}\right\rVert_2=c\sqrt P\,U.
\]
We use dimension-normalized truncation thresholds $P\tau$ for the matrix
statistic and $\sqrt P\,\tau$ for the vector statistic.  Thus
\[
    A^{(\tau)}
    :=
    XX^\top\min\left\{1,\frac{P\tau}{\left\lVert{XX^\top}\right\rVert_{\mathrm{op}}}\right\},
    \qquad
    b^{(\tau)}
    :=
    XY\min\left\{1,\frac{\sqrt P\,\tau}{\left\lVert{XY}\right\rVert_{2}}\right\}.
\]
For a clean example this gives
\begin{equation}
\label{eq:clean-coordinate-truncated-statistics}
    A^{(\tau)}
    =
    P\,a_\tau(U)e_Je_J^\top,
    \qquad
    b^{(\tau)}
    =
    \sqrt P\,r_\tau(U)e_J,
\end{equation}
where\footnote{The notation means that \(t\wedge \tau\) is the smaller of \(t\) and \(\tau\). }
\[
    a_\tau(t):=t\wedge \tau,
    \qquad
    r_\tau(t):=c\left(t\wedge \frac{\tau}{c}\right).
\]
The clean truncated population moments are
\[
    M_\tau:=\mathbb E a_\tau(U)
    =
    \mu_U-\frac{\tau^{1-q_\star}}{q_\star-1},
\]
and
\[
    B_\tau:=\mathbb E r_\tau(U)
    =
    cM_{\tau/c}
    =
    c\mu_U-
    \frac{c^{q_\star}\tau^{1-q_\star}}{q_\star-1}.
\]
Since $0<c<1$, the two truncations clip the same clean radial tail at different
effective thresholds.  This is the high-dimensional coordinate version of the
clean truncation-bias mechanism.

\subsubsection{Label-shifted data poisoning within truncation radius}
\label{subsec:feature-bounded-poisoning}

We now use the modified poison.  A poisoned example is generated by drawing
$J$ uniformly from $\{1,\ldots,P\}$ and setting
\[
    X_{\mathrm p}=\sqrt P\,e_J,
    \qquad
    Y_{\mathrm p}=c+\eta_{\mathrm p}\tau.
\]
Thus the poison feature has bounded normalized squared norm,
\[
    \frac{\left\lVert{X_{\mathrm p}}\right\rVert_2^2}{P}=1,
\]
while the label residual relative to the clean teacher is order $\tau$:
\[
    Y_{\mathrm p}-X_{\mathrm p}^\top \theta_{\star,P}
    =
    \eta_{\mathrm p}\tau.
\]
Let
\[
    \tau_0:=\max\left\{1,\frac{c}{1-\eta_{\mathrm p}}\right\}.
\]
Throughout the exact formulas for the modified poison we assume
$\tau\ge\tau_0$.  Then
\[
    c+\eta_{\mathrm p}\tau\le \tau,
\]
so the poison cross-moment lies below the vector truncation threshold, while
its matrix statistic lies below the matrix truncation threshold.  Hence
\begin{equation}
\label{eq:poison-coordinate-truncated-statistics-modified}
    A_{\mathrm p}^{(\tau)}
    =
    P e_Je_J^\top,
    \qquad
    b_{\mathrm p}^{(\tau)}
    =
    \sqrt P\,(c+\eta_{\mathrm p}\tau)e_J.
\end{equation}
Training samples are independent.  Each sample is clean with probability
$1-\epsilon$ and poisoned with probability $\epsilon$.

\subsection{Simplification in linear regression with truncation}
\label{sec:empirical-truncated-ridge}

Given $D$ training samples, define
\[
    \widehat\Sigma_{\epsilon,\tau}
    :=
    \frac1D\sum_{i=1}^D A_i^{(\tau)},
    \qquad
    \widehat v_{\epsilon,\tau}
    :=
    \frac1D\sum_{i=1}^D b_i^{(\tau)}.
\]
For $\varrho>0$, define the ridge estimator
\[
    \widehat\theta_{\epsilon,\tau,\varrho}
    :=
    (\widehat\Sigma_{\epsilon,\tau}+\varrho I_P)^{-1}
    \widehat v_{\epsilon,\tau}.
\]
For $\varrho=0$, define the ridgeless estimator by the Moore--Penrose inverse:
\[
    \widehat\theta_{\epsilon,\tau,0}
    :=
    \widehat\Sigma_{\epsilon,\tau}^{\dagger}\widehat v_{\epsilon,\tau}.
\]

Let $\phi_D=P/D$.  For coordinate $j$, let $N_{j,\mathrm c}$ and
$N_{j,\mathrm p}$ denote the numbers of clean and poisoned training samples
assigned to coordinate $j$.  Let $U_{j,1},U_{j,2},\ldots$ denote the clean radial
variables assigned to that coordinate.  Define
\begin{align}
\label{eq:coordinate-SA-SB-modified}
    S_{j,A}^{(D)}
    &:={}
    \sum_{\ell=1}^{N_{j,\mathrm c}}a_\tau(U_{j,\ell})
    +
    N_{j,\mathrm p},
    \\
    S_{j,B}^{(D)}
    &:={}
    \sum_{\ell=1}^{N_{j,\mathrm c}}r_\tau(U_{j,\ell})
    +
    (c+\eta_{\mathrm p}\tau)N_{j,\mathrm p}.
\end{align}

By construction, the per-example statistics from
\eqref{eq:clean-coordinate-truncated-statistics} and
\eqref{eq:poison-coordinate-truncated-statistics-modified} are supported on the
single coordinate $J_i\in\{1,\ldots,P\}$ to which sample $i$ is assigned.
Writing $\mathrm c$ for a clean sample and $\mathrm p$ for a poisoned one,
\[
    A_i^{(\tau)}
    =
    P\,\alpha_i\,e_{J_i}e_{J_i}^\top,
    \qquad
    b_i^{(\tau)}
    =
    \sqrt P\,\beta_i\,e_{J_i},
\]
where
\[
    \alpha_i=
    \begin{cases}
    a_\tau(U_i),&\text{$i$ clean},\\
    1,&\text{$i$ poison},
    \end{cases}
    \qquad
    \beta_i=
    \begin{cases}
    r_\tau(U_i),&\text{$i$ clean},\\
    c+\eta_{\mathrm p}\tau,&\text{$i$ poison}.
    \end{cases}
\]
Each matrix $e_{J_i}e_{J_i}^\top$ is diagonal, with a single nonzero entry
equal to $1$ in position $(J_i,J_i)$.  Hence
\[
    \widehat\Sigma_{\epsilon,\tau}
    =
    \frac1D\sum_{i=1}^D A_i^{(\tau)}
    =
    \frac PD\sum_{i=1}^D \alpha_i\,e_{J_i}e_{J_i}^\top
\]
is a sum of diagonal matrices and is therefore diagonal.  Its $(j,j)$ entry
collects exactly those samples assigned to coordinate $j$:
\[
    (\widehat\Sigma_{\epsilon,\tau})_{jj}
    =
    \frac PD\sum_{i:\,J_i=j}\alpha_i
    =
    \phi_D
    \Bigl(
        \sum_{\ell=1}^{N_{j,\mathrm c}}a_\tau(U_{j,\ell})
        +N_{j,\mathrm p}
    \Bigr)
    =
    \phi_D\,S_{j,A}^{(D)},
\]
using  the definition \eqref{eq:coordinate-SA-SB-modified} of
$S_{j,A}^{(D)}$.  Likewise, $b_i^{(\tau)}=\sqrt P\,\beta_i\,e_{J_i}$ has a single
nonzero entry in position $J_i$, so
\[
    (\widehat v_{\epsilon,\tau})_j
    =
    \frac1D\sum_{i:\,J_i=j}\sqrt P\,\beta_i
    =
    \frac{\sqrt P}{D}
    \Bigl(
        \sum_{\ell=1}^{N_{j,\mathrm c}}r_\tau(U_{j,\ell})
        +(c+\eta_{\mathrm p}\tau)N_{j,\mathrm p}
    \Bigr)
    =
    \frac{\sqrt P}{D}\,S_{j,B}^{(D)}
    =
    \frac{\phi_D}{\sqrt P}\,S_{j,B}^{(D)},
\]
where the last equality uses $\sqrt P/D=(P/D)/\sqrt P=\phi_D/\sqrt P$. Consequently,
\begin{equation}
\label{eq:coordinate-estimator-formula}
    \sqrt P\,(\widehat\theta_{\epsilon,\tau,\varrho})_j
    =
    G_{\phi_D,\varrho}
    \left(S_{j,A}^{(D)},S_{j,B}^{(D)}\right),
\end{equation}
where, for $\varrho\ge0$,
\begin{equation}
\label{eq:G-def}
    G_{\phi,\varrho}(s_A,s_B)
    :=
    \begin{cases}
    \displaystyle
    \frac{\phi s_B}{\varrho+\phi s_A},
    & \varrho+\phi s_A>0,\\[3mm]
    0,& \varrho=0,\ s_A=0.
    \end{cases}
\end{equation}
For a clean data one has
   $ 0\le r_\tau(t)\le a_\tau(t),$
and for a modified poison data, when \(\tau\ge\tau_0\),
    $0\le c+\eta_{\mathrm p}\tau\le \tau .$
Hence, with
    $m_\tau:=\max\{1,c+\eta_{\mathrm p}\tau\},$
every finite collection of clean and poison data satisfies
$    0\le s_B\le m_\tau s_A.$
Since \(\tau\ge\tau_0\) implies \(\tau\ge1\) and
\(c+\eta_{\mathrm p}\tau\le\tau\), we also have \(m_\tau\le\tau\).  Therefore
\begin{equation}
\label{eq:G-bounded-modified}
    0\le G_{\phi,\varrho}(s_A,s_B)
    \le m_\tau
    \le \tau .
\end{equation}

\subsection{The deterministic equivalent at fixed aspect ratio}
\label{sec:fixed-aspect-de}

Each training sample chooses one of the \(P\) coordinates uniformly.  Conditional
on the sample being assigned to coordinate \(j\), the sample is clean with
probability \(1-\epsilon\) and poisoned with probability \(\epsilon\).
Therefore, for a fixed coordinate \(j\), the finite-\(D,P\) clean and poison
counts satisfy
\[
    N_{j,\mathrm c}^{(D)}
    =
    \sum_{i=1}^D
    \mathbf 1\{\text{sample }i\text{ is clean and assigned to }j\},
\]
and
\[
    N_{j,\mathrm p}^{(D)}
    =
    \sum_{i=1}^D
    \mathbf 1\{\text{sample }i\text{ is poisoned and assigned to }j\}.
\]
For each sample,
\[
    \mathbb P(\text{clean and assigned to }j)
    =
    \frac{1-\epsilon}{P},
    \qquad
    \mathbb P(\text{poisoned and assigned to }j)
    =
    \frac{\epsilon}{P}.
\]
Hence
\[
    \mathbb E N_{j,\mathrm c}^{(D)}
    =
    D\frac{1-\epsilon}{P}
    =
    \frac{1-\epsilon}{\phi_D},
    \qquad
    \mathbb E N_{j,\mathrm p}^{(D)}
    =
    D\frac{\epsilon}{P}
    =
    \frac{\epsilon}{\phi_D}.
\]

\begin{hdlemma}[Poisson limit of the coordinate counts]
Fix a coordinate \(j\).  Suppose \(P/D\to\phi\in(0,\infty)\).  Each training
sample is clean with probability \(1-\epsilon\), poisoned with probability
\(\epsilon\), and independently assigned to coordinate \(j\) with probability
\(1/P\).  Let \(N_{j,\mathrm c}^{(D)}\) and \(N_{j,\mathrm p}^{(D)}\) be the
numbers of clean and poisoned samples assigned to coordinate \(j\).  Then
\[
    (N_{j,\mathrm c}^{(D)},N_{j,\mathrm p}^{(D)})
    \Rightarrow
    (N_{\mathrm c},N_{\mathrm p}),
\]
where
\[
    N_{\mathrm c}\sim
    \operatorname{Pois}\left(\frac{1-\epsilon}{\phi}\right),
    \qquad
    N_{\mathrm p}\sim
    \operatorname{Pois}\left(\frac{\epsilon}{\phi}\right),
\]
and \(N_{\mathrm c}\) and \(N_{\mathrm p}\) are independent.
\end{hdlemma}

\begin{proof}
For a fixed coordinate \(j\), a given training sample contributes a clean data
to coordinate \(j\) with probability
\[
    p_{\mathrm c,D}
    =
    \frac{1-\epsilon}{P},
\]
and contributes a poison data to coordinate \(j\) with probability
\[
    p_{\mathrm p,D}
    =
    \frac{\epsilon}{P}.
\]
Thus
\[
    (N_{j,\mathrm c}^{(D)},N_{j,\mathrm p}^{(D)})
\]
has a multinomial distribution with three categories: clean hit, poison hit,
and no hit.  For \(k,\ell\ge0\) with \(k+\ell\le D\),
\[
\begin{aligned}
    &\mathbb P
    \left(
        N_{j,\mathrm c}^{(D)}=k,\,
        N_{j,\mathrm p}^{(D)}=\ell
    \right)
    \\
    &\qquad =
    \frac{D!}{k!\,\ell!\,(D-k-\ell)!}
    \left(\frac{1-\epsilon}{P}\right)^k
    \left(\frac{\epsilon}{P}\right)^\ell
    \left(1-\frac1P\right)^{D-k-\ell}.
\end{aligned}
\]
Since \(P/D\to\phi\), we have
\[
    \frac DP\to\frac1\phi,
    \qquad
    \frac{D(1-\epsilon)}{P}\to\frac{1-\epsilon}{\phi},
    \qquad
    \frac{D\epsilon}{P}\to\frac{\epsilon}{\phi}.
\]
For fixed \(k,\ell\),
\[
    \frac{D!}{(D-k-\ell)!P^{k+\ell}}
    \to
    \left(\frac1\phi\right)^{k+\ell},
\]
and
\[
    \left(1-\frac1P\right)^{D-k-\ell}
    \to
    \exp\left(-\frac1\phi\right).
\]
Therefore
\[
\begin{aligned}
    &\mathbb P
    \left(
        N_{j,\mathrm c}^{(D)}=k,\,
        N_{j,\mathrm p}^{(D)}=\ell
    \right)
    \\
    &\qquad\to
    \exp\left(-\frac1\phi\right)
    \frac{\left(\frac{1-\epsilon}{\phi}\right)^k}{k!}
    \frac{\left(\frac{\epsilon}{\phi}\right)^\ell}{\ell!}.
\end{aligned}
\]
This is the joint probability mass function of two independent Poisson random
variables with means \((1-\epsilon)/\phi\) and \(\epsilon/\phi\), respectively.
\end{proof}

At fixed \(\phi\), a coordinate receives a poison data with probability
\[
    \mathbb P(N_{\mathrm p}\ge 1)
    =
    1-\exp\left(-\frac{\epsilon}{\phi}\right)
    =
    \frac{\epsilon}{\phi}+O(\epsilon^2),
    \qquad
    \epsilon\downarrow0.
\]
Thus poisoning acts through a sparse-coordinate mechanism: most coordinates see
no poison, while a proportion of order \(\epsilon/\phi\) see one poison data.
By contrast, if one first sends \(\phi\downarrow0\), then
\[
    \frac{1-\epsilon}{\phi}\to\infty,
    \qquad
    \frac{\epsilon}{\phi}\to\infty
    \quad
    \text{for fixed }\epsilon>0.
\]
Each coordinate is hit by many clean and poisoned samples, and the Poisson
fluctuations average out by the law of large numbers.

Let $U_1,U_2,\ldots$ be i.i.P.\ copies of $U$ drawn from the Pareto tail, independent
of $N_{\mathrm c},N_{\mathrm p}$: 
\[
    N_{\mathrm c}\sim
    \operatorname{Pois}\left(\frac{1-\epsilon}{\phi}\right),
    \qquad
    N_{\mathrm p}\sim
    \operatorname{Pois}\left(\frac{\epsilon}{\phi}\right),
\]
Define the limiting compound-Poisson
coordinate statistics
\begin{align}
\label{eq:limiting-SA-SB-modified}
    S_A^{\epsilon,\tau}
    &:={}
    \sum_{\ell=1}^{N_{\mathrm c}}a_\tau(U_\ell)
    +
    N_{\mathrm p},
    \\
    S_B^{\epsilon,\tau}
    &:={}
    \sum_{\ell=1}^{N_{\mathrm c}}r_\tau(U_\ell)
    +
    (c+\eta_{\mathrm p}\tau)N_{\mathrm p}.
\end{align}
Finally define
\begin{equation}
\label{eq:DE-risk-functional}
    \mathcal R_{\phi,\varrho}(\epsilon,\tau)
    :=
    \mu_U\,
    \mathbb E\left[
        \left(
        c-G_{\phi,\varrho}
        \left(S_A^{\epsilon,\tau},S_B^{\epsilon,\tau}\right)
        \right)^2
    \right].
\end{equation}

\begin{hdlemma}[Deterministic equivalent]
\label{lem:occupancy-de}
Let $F:[0,\infty)^2\to\mathbb R$ be a bounded function. Then
\[
    \frac1P\sum_{j=1}^P
    F\left(S_{j,A}^{(D)},S_{j,B}^{(D)}\right)
    \xrightarrow{\mathbb P}
    \mathbb E F\left(S_A^{\epsilon,\tau},S_B^{\epsilon,\tau}\right).
\]
The conclusion also holds for the bounded functions
\[
    F_{\phi_D,\varrho}(s_A,s_B)
    :=
    \left(c-G_{\phi_D,\varrho}(s_A,s_B)\right)^2,
\]
with the limit $F_{\phi,\varrho}$, including the ridgeless convention
$\varrho=0$.
\end{hdlemma}

\begin{proof}
Fix \(\tau\ge\tau_0\).   Since \(P/D\to\phi\), the law of
\((S_{j,A}^{(D)},S_{j,B}^{(D)})\) converges to the compound-Poisson law in
\eqref{eq:limiting-SA-SB-modified}.  Therefore
\[
    \mathbb E F\left(S_{j,A}^{(D)},S_{j,B}^{(D)}\right)
    \to
    \mathbb E F\left(S_A^{\epsilon,\tau},S_B^{\epsilon,\tau}\right)
\]
by the continuity. It remains to show concentration of the coordinate average.  
Let
\[
    Y_j^{(D)}
    :=
    F\left(S_{j,A}^{(D)},S_{j,B}^{(D)}\right),
    \qquad
    j=1,\ldots,P.
\]
Since the training examples assign coordinates symmetrically, the random
variables \(Y_1^{(D)},\ldots,Y_P^{(D)}\) are exchangeable.  Suppose
\[
    \|F\|_\infty
    :=
    \sup_{s_A,s_B}|F(s_A,s_B)|
    <\infty.
\]
Then
\[
    |Y_j^{(D)}|\le \|F\|_\infty
    \qquad\text{for every }j.
\]

We compute
\[
\begin{aligned}
    \operatorname{Var}\left(
    \frac1P\sum_{j=1}^P Y_j^{(D)}
    \right)
    &=
    \frac1{P^2}
    \sum_{j=1}^P\operatorname{Var}(Y_j^{(D)})
    +
    \frac1{P^2}
    \sum_{\substack{j,k=1\\j\ne k}}^P
    \operatorname{Cov}(Y_j^{(D)},Y_k^{(D)}).
\end{aligned}
\]
By exchangeability,
\[
    \operatorname{Var}(Y_j^{(D)})
    =
    \operatorname{Var}(Y_1^{(D)})
    \qquad\text{for all }j,
\]
and
\[
    \operatorname{Cov}(Y_j^{(D)},Y_k^{(D)})
    =
    \operatorname{Cov}(Y_1^{(D)},Y_2^{(D)})
    \qquad\text{for all }j\ne k.
\]
Therefore
\[
\begin{aligned}
    \operatorname{Var}\left(
    \frac1P\sum_{j=1}^P Y_j^{(D)}
    \right)
    &=
    \frac{P}{P^2}\operatorname{Var}(Y_1^{(D)})
    +
    \frac{P(P-1)}{P^2}
    \operatorname{Cov}(Y_1^{(D)},Y_2^{(D)})       \\
    &=
    \frac1P\operatorname{Var}(Y_1^{(D)})
    +
    \frac{P-1}{P}
    \operatorname{Cov}(Y_1^{(D)},Y_2^{(D)}).
\end{aligned}
\]
Taking absolute values in the covariance term and using
\((P-1)/P\le1\), we obtain
\[
\begin{aligned}
    \operatorname{Var}\left(
    \frac1P\sum_{j=1}^P Y_j^{(D)}
    \right)
    &\le
    \frac1P\operatorname{Var}(Y_1^{(D)})
    +
    \left|\operatorname{Cov}(Y_1^{(D)},Y_2^{(D)})\right|.
\end{aligned}
\]
Finally,
\[
    \operatorname{Var}(Y_1^{(D)})
    =
    \mathbb E\left[
        \left(
        Y_1^{(D)}-\mathbb E Y_1^{(D)}
        \right)^2
    \right]
    \le
    \mathbb E\left[(Y_1^{(D)})^2\right]
    \le
    \|F\|_\infty^2.
\]
Thus
\[
\begin{aligned}
    \operatorname{Var}\left(
    \frac1P\sum_{j=1}^P
    F(S_{j,A}^{(D)},S_{j,B}^{(D)})
    \right)
    &\le
    \frac{\|F\|_\infty^2}{P}
    +
    \left|
    \operatorname{Cov}(F_1^{(D)},F_2^{(D)})
    \right|,
\end{aligned}
\]
where
\[
    F_j^{(D)}
    :=
    F(S_{j,A}^{(D)},S_{j,B}^{(D)}).
\]

For two fixed coordinates, the joint occupancy counts are multinomially coupled,
but the coupling is weak.  As \(P/D\to\phi\in(0,\infty)\), the law of the clean
and poisoned counts in two fixed coordinates converges to the law of two
independent copies of the limiting compound-Poisson coordinate statistic.  Hence, 
\[
    \operatorname{Cov}(F_1^{(D)},F_2^{(D)})\to0.
\]

For \(F_{\phi_D,\varrho}\), boundedness follows from
\eqref{eq:G-bounded-modified}.  If \(\varrho>0\), the map
\((\phi,s_A,s_B)\mapsto(c-G_{\phi,\varrho}(s_A,s_B))^2\) is continuous.  If
\(\varrho=0\), the only problematic point is \((s_A,s_B)=(0,0)\).  This point
has positive mass, but it is not a problem because the finite-\(D\) and limiting
conventions both assign the same value \(G=0\) there; on \(s_A>0\), the
ridgeless map is \(s_B/s_A\) and is continuous.  The same expectation and
covariance argument therefore applies to \(F_{\phi_D,\varrho}\), giving the
claimed deterministic equivalent.
\end{proof}

In many high-dimensional ridge-regression deterministic equivalents, one sees a
nontrivial effective regularization parameter.  This often appears through a
self-consistent equation involving a Stieltjes transform or an equivalent
quantity.  It is therefore natural to ask why, in the occupancy deterministic
equivalent below, the ridge parameter remains the same parameter \(\varrho\):
\[
    F_{\phi_D,\varrho}(s_A,s_B)
    :=
    \left(c-G_{\phi_D,\varrho}(s_A,s_B)\right)^2
    \longrightarrow
    F_{\phi,\varrho}(s_A,s_B),
\]
rather than becoming some renormalized quantity. The reason is that each coordinate has its own scalar normal equation:
\[
    \varrho+\phi_D S_{j,A}^{(D)}.
\]
The denominator is already scalar before taking the limit.  Thus there is no
resolvent self-consistency equation.

\begin{theorem}[Fixed-aspect-ratio deterministic equivalent for the empirical risk]
\label{thm:fixed-aspect-risk-de}
Assume $P/D\to\phi\in(0,\infty)$ and fix
$\epsilon\in[0,1)$, $\tau\ge\tau_0$, and $\varrho\ge0$.  Then
\[
    \mathcal R_{\mathrm{clean},P}
    (\widehat\theta_{\epsilon,\tau,\varrho})
    -
    \mathcal R_{\mathrm{clean},P}(\theta_{\star,P})
    \xrightarrow{\mathbb P}
    \mathcal R_{\phi,\varrho}(\epsilon,\tau),
\]
where $\mathcal R_{\phi,\varrho}$ is defined in
\eqref{eq:DE-risk-functional}.  Consequently, if an independent clean baseline
estimator is trained with the same $D,P,\tau,\varrho$, then
\[
\begin{aligned}
    &\Bigl[
    \mathcal R_{\mathrm{clean},P}
    (\widehat\theta_{\epsilon,\tau,\varrho})
    -
    \mathcal R_{\mathrm{clean},P}
    (\widehat\theta_{0,\tau,\varrho})
    \Bigr]
    \\
    &\qquad\qquad\xrightarrow{\mathbb P}
    \mathcal D_{\phi,\varrho}(\epsilon,\tau)
    :=
    \mathcal R_{\phi,\varrho}(\epsilon,\tau)
    -
    \mathcal R_{\phi,\varrho}(0,\tau).
\end{aligned}
\]
\end{theorem}

\begin{proof}

Recall first the clean risk identity
\[
    \mathcal R_{\mathrm{clean},P}(\theta)-\mathcal R_{\mathrm{clean},P}(\theta_{\star,P})
    =
    \mu_U\frac1P
    \sum_{j=1}^P
    \left(c-\sqrt P\,\theta_j\right)^2.
\]
Thus, to compute the clean excess risk of the estimator
\(\widehat\theta_{\epsilon,\tau,\varrho}\), it is enough to understand the
coordinate-wise quantities
\[
    \sqrt P\,(\widehat\theta_{\epsilon,\tau,\varrho})_j.
\]

 For coordinate
\(j\), the diagonal entry of the empirical Hessian and the corresponding entry
of the empirical cross-moment are
\[
    (\widehat\Sigma_{\epsilon,\tau})_{jj}
    =
    \phi_D S_{j,A}^{(D)},
    \qquad
    (\widehat v_{\epsilon,\tau})_j
    =
    \frac{\phi_D}{\sqrt P}S_{j,B}^{(D)},
    \qquad
    \phi_D:=\frac PD.
\]
Therefore the ridge normal equation in coordinate \(j\) is the scalar equation
\[
    \left(\varrho+\phi_D S_{j,A}^{(D)}\right)
    (\widehat\theta_{\epsilon,\tau,\varrho})_j
    =
    \frac{\phi_D}{\sqrt P}S_{j,B}^{(D)}.
\]
Equivalently,
\[
    \sqrt P\,(\widehat\theta_{\epsilon,\tau,\varrho})_j
    =
    G_{\phi_D,\varrho}
    \left(S_{j,A}^{(D)},S_{j,B}^{(D)}\right),
\]
where
\[
    G_{\phi,\varrho}(s_A,s_B)
    =
    \begin{cases}
    \displaystyle
    \frac{\phi s_B}{\varrho+\phi s_A},
    & \varrho+\phi s_A>0,\\[3mm]
    0,& \varrho=0,\ s_A=0.
    \end{cases}
\]

Substituting the coordinate formula into the clean risk identity gives the exact
finite-\(D,P\) identity
\[
\begin{aligned}
    &\mathcal R_{\mathrm{clean},P}
    (\widehat\theta_{\epsilon,\tau,\varrho})
    -
    \mathcal R_{\mathrm{clean},P}(\theta_{\star,P})
    \\
    &\qquad =
    \mu_U\frac1P\sum_{j=1}^P
    \left(
        c-
        G_{\phi_D,\varrho}
        \left(S_{j,A}^{(D)},S_{j,B}^{(D)}\right)
    \right)^2.
\end{aligned}
\]
Thus the risk is exactly an average over coordinates of the bounded function
\[
    F_{\phi_D,\varrho}(s_A,s_B)
    :=
    \left(c-G_{\phi_D,\varrho}(s_A,s_B)\right)^2.
\]

Applying Lemma~\ref{lem:occupancy-de} 
gives
\[
\begin{aligned}
    &\frac1P\sum_{j=1}^P
    \left(
        c-
        G_{\phi_D,\varrho}
        \left(S_{j,A}^{(D)},S_{j,B}^{(D)}\right)
    \right)^2
    \\
    &\qquad\xrightarrow{\mathbb P}
    \mathbb E\left[
    \left(
        c-
        G_{\phi,\varrho}
        \left(S_A^{\epsilon,\tau},S_B^{\epsilon,\tau}\right)
    \right)^2
    \right].
\end{aligned}
\]
Multiplying by \(\mu_U\), we obtain
\[
    \mathcal R_{\mathrm{clean},P}
    (\widehat\theta_{\epsilon,\tau,\varrho})
    -
    \mathcal R_{\mathrm{clean},P}(\theta_{\star,P})
    \xrightarrow{\mathbb P}
    \mu_U
    \mathbb E\left[
    \left(
        c-
        G_{\phi,\varrho}
        \left(S_A^{\epsilon,\tau},S_B^{\epsilon,\tau}\right)
    \right)^2
    \right].
\]
By definition, the right-hand side is
\[
    \mathcal R_{\phi,\varrho}(\epsilon,\tau).
\]
This proves the first result.

For the second convergence, apply the same result twice.  First, apply it to
the estimator trained with poisoning rate \(\epsilon\):
\[
    \mathcal R_{\mathrm{clean},P}
    (\widehat\theta_{\epsilon,\tau,\varrho})
    -
    \mathcal R_{\mathrm{clean},P}(\theta_{\star,P})
    \xrightarrow{\mathbb P}
    \mathcal R_{\phi,\varrho}(\epsilon,\tau).
\]
Second, apply it to the clean estimator trained with poisoning rate \(0\):
\[
    \mathcal R_{\mathrm{clean},P}
    (\widehat\theta_{0,\tau,\varrho})
    -
    \mathcal R_{\mathrm{clean},P}(\theta_{\star,P})
    \xrightarrow{\mathbb P}
    \mathcal R_{\phi,\varrho}(0,\tau).
\]
Subtracting the two convergences gives
\[
\begin{aligned}
    &\mathcal R_{\mathrm{clean},P}
    (\widehat\theta_{\epsilon,\tau,\varrho})
    -
    \mathcal R_{\mathrm{clean},P}
    (\widehat\theta_{0,\tau,\varrho})
    \\
    &\qquad\xrightarrow{\mathbb P}
    \mathcal R_{\phi,\varrho}(\epsilon,\tau)
    -
    \mathcal R_{\phi,\varrho}(0,\tau)
    =
    \mathcal D_{\phi,\varrho}(\epsilon,\tau).
\end{aligned}
\]
This proves the second result.
\end{proof}

\subsection{Fixed \texorpdfstring{$\phi$}{phi}: higher dimensional statistics}
\label{sec:fixed-phi-modified-attack}

The notation
\[
    \Pi \sim \operatorname{PPP}(\nu)
\]
means that \(\Pi\) is a Poisson point process with intensity measure
\(\nu\). Each element of $\Pi$ is a data 
\[
    z=(z_A,z_B),
\]
A realization of \(\Pi\) is a random finite collection of data:
\[
    \Pi=\{Z_1,\ldots,Z_N\}.
\]
The associated total coordinate statistic is
\[
    \Sigma(\Pi)
    :=
    \sum_{z\in\Pi}z
    =
    \sum_{k=1}^N Z_k.
\]
If \(Z_k=(Z_{k,A},Z_{k,B})\), then
\[
    \Sigma(\Pi)
    =
    \left(
        \sum_{k=1}^N Z_{k,A},
        \sum_{k=1}^N Z_{k,B}
    \right).
\]
For example
\[
    Z_{\mathrm c}^{(\tau)}
    =
    \bigl(a_\tau(U),r_\tau(U)\bigr),
    \qquad
    Z_{\mathrm p}^{(\tau)}
    =
    \bigl(1,c+\eta_{\mathrm p}\tau\bigr)
\]
denote one clean and one poison data.
The first component is the contribution to the coordinate Hessian statistic,
and the second component is the contribution to the coordinate cross-moment
statistic.

We focus on the intensity measure of the form
\[
    \nu
    =
    \rho\,\mathcal L(Z),
\]
where
\[
    \rho>0
\]
is a rate and \(\mathcal L(Z)\) is the probability law of a single data \(Z\).

The notation
\[
    \Pi\sim \operatorname{PPP}\bigl(\rho\,\mathcal L(Z)\bigr)
\]
means the following concrete construction:
\[
    N\sim\operatorname{Pois}(\rho),
\]
and, conditional on \(N\),
\[
    Z_1,\ldots,Z_N
    \quad\text{are i.i.d. with law}\quad
    \mathcal L(Z).
\]
Then
\[
    \Pi=\{Z_1,\ldots,Z_N\}.
\]

For example
\[
    \Pi_{\mathrm{ret}}
    \sim
    \operatorname{PPP}
    \left(
        \frac{1-\epsilon}{\phi}\,
        \mathcal L(Z_{\mathrm c}^{(\tau)})
    \right).
\]
is the random collection of clean data
hitting one coordinate.  The number of retained clean marks is
\[
    N_{\mathrm{ret}}
    \sim
    \operatorname{Pois}\left(\frac{1-\epsilon}{\phi}\right).
\]
Conditional on \(N_{\mathrm{ret}}\), the retained clean marks are
independent copies of
\[
    Z_{\mathrm c}^{(\tau)}
    =
    \bigl(a_\tau(U),r_\tau(U)\bigr).
\]
Thus one can write
\[
    \Pi_{\mathrm{ret}}
    =
    \left\{
        Z_{{\mathrm c},1}^{(\tau)},
        \ldots,
        Z_{{\mathrm c},N_{\mathrm{ret}}}^{(\tau)}
    \right\}.
\]
The corresponding coordinate statistic is
\[
    S_{\mathrm{ret}}
    :=
    \Sigma(\Pi_{\mathrm{ret}})
    =
    \sum_{\ell=1}^{N_{\mathrm{ret}}}
    Z_{{\mathrm c},\ell}^{(\tau)}.
\]
Writing this sum coordinate-wise gives
\[
    S_{\mathrm{ret}}
    =
    \left(
    \sum_{\ell=1}^{N_{\mathrm{ret}}}
    a_\tau(U_\ell),
    \sum_{\ell=1}^{N_{\mathrm{ret}}}
    r_\tau(U_\ell)
    \right).
\]

Similarly,
\[
    \Pi_{\mathrm{del}}
    \sim
    \operatorname{PPP}\left(
        \frac{\epsilon}{\phi}\,
        \mathcal L(Z_{\mathrm c}^{(\tau)})
    \right)
\]
means that the deleted clean process contains
\[
    N_{\mathrm{del}}
    \sim
    \operatorname{Pois}\left(\frac{\epsilon}{\phi}\right)
\]
clean marks, each distributed as \(Z_{\mathrm c}^{(\tau)}\).

The poison process
\[
    \Pi_{\mathrm p}
    \sim
    \operatorname{PPP}\left(
        \frac{\epsilon}{\phi}\,
        \mathcal L(Z_{\mathrm p}^{(\tau)})
    \right)
\]
means that the number of poison marks is
\[
    N_{\mathrm p}
    \sim
    \operatorname{Pois}\left(\frac{\epsilon}{\phi}\right),
\]
and each poison data has law
\[
    \mathcal L(Z_{\mathrm p}^{(\tau)}).
\]
In the attack, this data is deterministic:
\[
    Z_{\mathrm p}^{(\tau)}
    =
    \bigl(1,c+\eta_{\mathrm p}\tau\bigr).
\]
Thus \(\mathcal L(Z_{\mathrm p}^{(\tau)})\) is just a point mass at
\((1,c+\eta_{\mathrm p}\tau)\).

 If
\[
    \Pi_1\sim \operatorname{PPP}(\nu_1),
    \qquad
    \Pi_2\sim \operatorname{PPP}(\nu_2)
\]
are independent, then
\[
    \Pi_1\cup\Pi_2
    \sim
    \operatorname{PPP}(\nu_1+\nu_2).
\]
Therefore
\[
    \Pi_{\mathrm{ret}}\cup \Pi_{\mathrm{del}}
    \sim
    \operatorname{PPP}\left(
        \frac{1-\epsilon}{\phi}\mathcal L(Z_{\mathrm c}^{(\tau)})
        +
        \frac{\epsilon}{\phi}\mathcal L(Z_{\mathrm c}^{(\tau)})
    \right).
\]
Combining the two terms gives
\[
    \Pi_{\mathrm{ret}}\cup \Pi_{\mathrm{del}}
    \sim
    \operatorname{PPP}\left(
        \frac1\phi\mathcal L(Z_{\mathrm c}^{(\tau)})
    \right).
\]
This is exactly the clean coordinate process of rate \(1/\phi\), i. e., 
if the \(D\) samples choose their coordinates uniformly from
\(\{1,\ldots,P\}\), then the number of samples falling on a fixed coordinate is
approximately
\[
    \mathrm{Binomial}\left(D,\frac1P\right)
    \Longrightarrow
    \mathrm{Pois}\left(\frac1\phi\right).
\]

We denote by
\[
    S_0^{(\tau)}
    :=
    \sum_{\ell=1}^{N_0}Z_{{\mathrm c},\ell}^{(\tau)},
    \qquad
    N_0\sim\operatorname{Pois}(1/\phi),
\]
the clean coordinate environment seen by a typical coordinate in the
unpoisoned fixed-\(\phi\) limit.  
It is also useful to define 
\[
    Q_0
    :=
    \sum_{\ell=1}^{N_0}U_\ell,
\]
and
\[
    Q_+
    :=
    U_0+\sum_{\ell=1}^{N_0}U_\ell
    =
    U_0+Q_0,
\]
where \(U_0,U_1,U_2,\ldots\) are independent copies of \(U\), independent of
\(N_0\).
The variable \(Q_0\) is the untruncated total radial mass already present at the
coordinate.  The variable \(Q_+\) is the untruncated total radial mass after
adding one additional independent clean example.  This distinction is important
because the small-poisoning expansion compares two possible one-data
perturbations of the same clean environment:
\[
    S_0^{(\tau)}+Z_{\mathrm p}^{(\tau)}
    \qquad\text{versus}\qquad
    S_0^{(\tau)}+Z_{\mathrm c}^{(\tau)}.
\]
The first expression corresponds to adding one poison data, while the second
corresponds to adding one additional clean data.  

\begin{hdlemma}[Large-$\tau$ limit of the added-clean term]
\label{lem:modified-clean-add-limit}
For fixed $\phi\in(0,\infty)$ and $\varrho\ge0$,\footnote{Let \(A_{\phi,\varrho}(\tau)\) be a quantity depending on
\(\tau\), \(\phi\), and \(\varrho\).  We write
\[
    A_{\phi,\varrho}(\tau)
    =
    O_{\phi,\varrho}(\tau^{-q_\star})
    \qquad\text{as }\tau\to\infty
\]
if, for every fixed \(\phi\) and every fixed
\(\varrho\), there exist constants
\[
    C_{\phi,\varrho}<\infty
    \qquad\text{and}\qquad
    \tau_{\phi,\varrho}<\infty
\]
such that, for all \(\tau\ge \tau_{\phi,\varrho}\),
\[
    \left|A_{\phi,\varrho}(\tau)\right|
    \le
    C_{\phi,\varrho}\tau^{-q_\star}.
\]
The subscript in \(O_{\phi,\varrho}\) means that the implicit constant may
depend on \(\phi\) and \(\varrho\), but it not on \(\tau\).}
\[
\begin{aligned}
    &\mathbb E\left[
    F_{\phi,\varrho}
    \left(S_0^{(\tau)}+Z_{\mathrm c}^{(\tau)}\right)
    \right]
    \\
    &\qquad =
    c^2\varrho^2\,
    \mathbb E\left[\frac{1}{(\varrho+\phi Q_+)^2}\right]
    +O_{\phi,\varrho}(\tau^{-q_\star}).
\end{aligned}
\]
\end{hdlemma}

\begin{proof}
Let $\mathcal E_\tau$ be the event that all clean radii appearing in
$S_0^{(\tau)}+Z_{\mathrm c}^{(\tau)}$ are at most $\tau$.  On
$\mathcal E_\tau$, no truncation occurs for either the matrix statistic or the
vector statistic, because $\tau/c>\tau$.  Thus, on this event,
\[
    S_A=Q_+,
    \qquad
    S_B=cQ_+,
\]
and therefore
\[
    F_{\phi,\varrho}(S_A,S_B)
    =
    \left(c-\frac{\phi cQ_+}{\varrho+\phi Q_+}\right)^2
    =
    \frac{c^2\varrho^2}{(\varrho+\phi Q_+)^2}.
\]

Let
\[
    M:=1+N_0,
    \qquad
    N_0\sim \operatorname{Pois}(1/\phi),
\]
be the total number of clean radii appearing in
\(S_0^{(\tau)}+Z_{\mathrm c}^{(\tau)}\).  Write these radii as
\[
    U_1,\ldots,U_M,
\]
where, conditionally on \(M\), they are i.i.P. copies of \(U\).  Recall that
\[
    \mathcal E_\tau
    :=
    \left\{
        \max_{1\le \ell\le M}U_\ell\le \tau
    \right\}.
\]
Thus
\[
    \mathcal E_\tau^c
    =
    \left\{
        \exists\,1\le \ell\le M
        \text{ such that } U_\ell>\tau
    \right\}.
\]

We first bound the probability of this event.  Conditional on \(M=m\),
\[
\begin{aligned}
    \mathbb P(\mathcal E_\tau^c\mid M=m)
    &=
    1-\mathbb P(U_1\le \tau,\ldots,U_m\le \tau)  \\
    &=
    1-\left(1-\mathbb P(U>\tau)\right)^m.
\end{aligned}
\]
Since \(1-(1-u)^m\le mu\) for \(u\in[0,1]\), we get
\[
    \mathbb P(\mathcal E_\tau^c\mid M=m)
    \le
    m\,\mathbb P(U>\tau).
\]
Taking expectations over \(M\),
\[
\begin{aligned}
    \mathbb P(\mathcal E_\tau^c)
    &\le
    \mathbb E[M]\,\mathbb P(U>\tau)  \\
    &=
    \left(1+\frac1\phi\right)\tau^{-q_\star}.
\end{aligned}
\]
Therefore
\[
    \mathbb P(\mathcal E_\tau^c)
    =
    O_\phi(\tau^{-q_\star}).
\]

It remains to show that the integrand for the bad events is bounded.  On a clean-only configuration, the coordinate statistics have
the form
\[
    S_A=\sum_{\ell=1}^{M}a_\tau(U_\ell),
    \qquad
    S_B=\sum_{\ell=1}^{M}r_\tau(U_\ell),
\]
where
\[
    a_\tau(t)=t\wedge\tau,
    \qquad
    r_\tau(t)=c\left(t\wedge\frac{\tau}{c}\right).
\]
For every \(t\ge1\),
\[
    0\le r_\tau(t)\le a_\tau(t).
\]
Indeed, if \(t\le\tau\), then \(r_\tau(t)=ct\le t=a_\tau(t)\).  If
\(\tau<t<\tau/c\), then \(r_\tau(t)=ct\le\tau=a_\tau(t)\).  Finally, if
\(t\ge\tau/c\), then \(r_\tau(t)=\tau=a_\tau(t)\).  Hence
\[
    0\le S_B\le S_A.
\]
Consequently, for \(\varrho\ge0\),
\[
    0\le
    G_{\phi,\varrho}(S_A,S_B)
    =
    \frac{\phi S_B}{\varrho+\phi S_A}
    \le
    \frac{\phi S_A}{\varrho+\phi S_A}
    \le 1
\]
whenever \(\varrho+\phi S_A>0\).  In the ridgeless convention
\(\varrho=0\), if \(S_A=0\) then \(G_{\phi,0}(0,0)=0\), so the same bound
still holds.  Therefore
\[
    0\le G_{\phi,\varrho}(S_A,S_B)\le1.
\]
Since \(0<c<1\),
\[
    F_{\phi,\varrho}(S_A,S_B)
    =
    \left(c-G_{\phi,\varrho}(S_A,S_B)\right)^2
    \le 1.
\]
Thus the contribution of the bad event is bounded by
\[
\begin{aligned}
    \mathbb E\left[
        F_{\phi,\varrho}(S_A,S_B)\mathbf 1_{\mathcal E_\tau^c}
    \right]
    &\le
    \mathbb P(\mathcal E_\tau^c)  \\
    &\le
    \left(1+\frac1\phi\right)\tau^{-q_\star}.
\end{aligned}
\]
Hence
\[
    \mathbb E\left[
        F_{\phi,\varrho}(S_A,S_B)\mathbf 1_{\mathcal E_\tau^c}
    \right]
    =
    O_\phi(\tau^{-q_\star}).
\]
\end{proof}

\begin{hdlemma}[Large-$\tau$ expansion with one modified poison data]
\label{lem:modified-one-poison}
Fix $\phi\in(0,\infty)$ and $\varrho\ge0$.  Then
\[
\begin{aligned}
    &\mathbb E\left[
    F_{\phi,\varrho}
    \left(S_0^{(\tau)}+Z_{\mathrm p}^{(\tau)}\right)
    \right]
    \\
    &\qquad =
    \eta_{\mathrm p}^2\phi^2\tau^2
    \mathbb E\left[\frac{1}{(\varrho+\phi(Q_0+1))^2}\right]
    \,(1+o(1))
\end{aligned}
\]
as $\tau\to\infty$.
\end{hdlemma}

\begin{proof}
Write
\[
    A_\tau:=\sum_{\ell=1}^{N_0}a_\tau(U_\ell),
    \qquad
    B_\tau:=\sum_{\ell=1}^{N_0}r_\tau(U_\ell),
    \qquad
    D_\tau:=B_\tau-cA_\tau.
\]
After adding one modified poison data, the coordinate statistic is
\[
    (A_\tau+1,\ B_\tau+c+\eta_{\mathrm p}\tau).
\]
Using the definition of $G_{\phi,\varrho}$,
\[
\begin{aligned}
    &G_{\phi,\varrho}
    (A_\tau+1,B_\tau+c+\eta_{\mathrm p}\tau)-c
    \\
    &\qquad =
    \frac{
        \phi(B_\tau+c+\eta_{\mathrm p}\tau)
        -c(\varrho+\phi(A_\tau+1))
    }{
        \varrho+\phi(A_\tau+1)
    }
    \\
    &\qquad =
    \frac{\phi D_\tau+\phi\eta_{\mathrm p}\tau-c\varrho}
         {\varrho+\phi(A_\tau+1)}.
\end{aligned}
\]
Therefore
\[
    F_{\phi,\varrho}
    (A_\tau+1,B_\tau+c+\eta_{\mathrm p}\tau)
    =
    \left(
    \frac{\phi D_\tau+\phi\eta_{\mathrm p}\tau-c\varrho}
         {\varrho+\phi(A_\tau+1)}
    \right)^2.
\]
For each fixed realization of $(N_0,U_1,U_2,\ldots)$,
\[
    A_\tau\to Q_0,
    \qquad
    \frac{D_\tau}{\tau}\to0,
\]
because for each fixed $U_\ell$,
truncation eventually disappears.  Hence
\[
    \frac1{\tau^2}
    F_{\phi,\varrho}
    (A_\tau+1,B_\tau+c+\eta_{\mathrm p}\tau)
    \to
    \eta_{\mathrm p}^2\phi^2
    \frac1{(\varrho+\phi(Q_0+1))^2}
\]
almost surely.

\end{proof}

\begin{hdlemma}[Small-poisoning generator expansion]
\label{lem:modified-generator-expansion}
For every \(\tau\ge\tau_0\), as $ \epsilon\downarrow0$
\[
\begin{aligned}
    \mathcal D_{\phi,\varrho}(\epsilon,\tau)
    ={}&
    \frac{\mu_U\epsilon}{\phi}
    \mathbb E\left[
    F_{\phi,\varrho}
    \left(S_0^{(\tau)}+Z_{\mathrm p}^{(\tau)}\right)
    -
    F_{\phi,\varrho}
    \left(S_0^{(\tau)}+Z_{\mathrm c}^{(\tau)}\right)
    \right]
    \\
    &\quad+
    O_{\phi,\varrho,c,\eta_{\mathrm p}}\left(\epsilon^2(1+\tau^2)\right).
\end{aligned}
\]
The implicit constant is independent of both \(\epsilon\) and \(\tau\).
\end{hdlemma}

\begin{proof}
Construct three independent Poisson point processes (following previous discussion):
\[
    \Pi_{\mathrm{ret}}
    \sim
    \operatorname{PPP}
    \left(\frac{1-\epsilon}{\phi}\,\mathcal L(Z_{\mathrm c}^{(\tau)})\right),
\]
\[
    \Pi_{\mathrm{del}}
    \sim
    \operatorname{PPP}
    \left(\frac{\epsilon}{\phi}\,\mathcal L(Z_{\mathrm c}^{(\tau)})\right),
\]
and
\[
    \Pi_{\mathrm p}
    \sim
    \operatorname{PPP}
    \left(\frac{\epsilon}{\phi}\,\mathcal L(Z_{\mathrm p}^{(\tau)})\right).
\]
Define
\[
    S_{\mathrm{ret}}:=\Sigma(\Pi_{\mathrm{ret}}),
    \qquad
    P_{\mathrm c}:=\Sigma(\Pi_{\mathrm{del}}),
    \qquad
    P_{\mathrm p}:=\Sigma(\Pi_{\mathrm p}).
\]
By the superposition theorem for Poisson point processes,
\[
    \Pi_{\mathrm{ret}}\cup \Pi_{\mathrm{del}}
    \sim
    \operatorname{PPP}
    \left(\frac1\phi\,\mathcal L(Z_{\mathrm c}^{(\tau)})\right).
\]
Hence
\[
    S_{\mathrm{ret}}+P_{\mathrm c}
    \stackrel{P}{=}
    S_0^{(\tau)}.
\]
Likewise,
\(\Pi_{\mathrm{ret}}\cup \Pi_{\mathrm p}\) is the compound-Poisson process
corresponding to a coordinate whose clean marks arrive with intensity
\((1-\epsilon)/\phi\) and whose poison marks arrive with intensity
\(\epsilon/\phi\).  Therefore
\[
    S_{\mathrm{ret}}+P_{\mathrm p}
    \stackrel{P}{=}
    S^{\epsilon,\tau},
\]
where \(S^{\epsilon,\tau}=(S_A^{\epsilon,\tau},S_B^{\epsilon,\tau})\) is the
limiting coordinate statistic under poisoning.  Consequently
\[
\begin{aligned}
    \mathcal D_{\phi,\varrho}(\epsilon,\tau)
    &=
    \mathcal R_{\phi,\varrho}(\epsilon,\tau)
    -
    \mathcal R_{\phi,\varrho}(0,\tau)
    \\
    &=
    \mu_U
    \mathbb E\left[
        F(S_{\mathrm{ret}}+P_{\mathrm p})
        -
        F(S_{\mathrm{ret}}+P_{\mathrm c})
    \right].
\end{aligned}
\]
We now expand the last expectation to first order in \(\epsilon\).  Let
\[
    K_{\mathrm c}:=|\Pi_{\mathrm{del}}|,
    \qquad
    K_{\mathrm p}:=|\Pi_{\mathrm p}|,
    \qquad
    K:=K_{\mathrm c}+K_{\mathrm p}.
\]
Then
\[
    K_{\mathrm c}\sim \operatorname{Pois}(\epsilon/\phi),
    \qquad
    K_{\mathrm p}\sim \operatorname{Pois}(\epsilon/\phi),
\]
independently, and hence
\[
    K\sim\operatorname{Pois}(2\epsilon/\phi).
\]
Therefore
\[
    \mathbb P(K\ge2)
    =
    1-e^{-2\epsilon/\phi}\left(1+\frac{2\epsilon}{\phi}\right)
    =
    O_\phi(\epsilon^2).
\]

For \(\tau\ge\tau_0\), every clean or poison data satisfies the bound
\[
    0\le s_B\le m_\tau s_A,
    \qquad
    m_\tau:=\max\{1,c+\eta_{\mathrm p}\tau\}\le C_{c,\eta_{\mathrm p}}\tau.
\]
Hence
\[
    0\le G_{\phi,\varrho}(s_A,s_B)
    \le
    C_{c,\eta_{\mathrm p}}\tau,
\]
and therefore
\[
    \|F_{\phi,\varrho}\|_\infty
    \le
    C_{c,\eta_{\mathrm p}}(1+\tau^2).
\]
It follows that the contribution to
\[
    \mathbb E\left[
        F(S_{\mathrm{ret}}+P_{\mathrm p})
        -
        F(S_{\mathrm{ret}}+P_{\mathrm c})
    \right]
\]
from the event \(\{K\ge2\}\) is bounded in absolute value by
\[
    2\|F\|_\infty \mathbb P(K\ge2)
    =
    O_{\phi,c,\eta_{\mathrm p}}
    \left(\epsilon^2(1+\tau^2)\right).
\]

On the event \(\{K=0\}\), we have \(P_{\mathrm p}=P_{\mathrm c}=0\).  Hence the
two terms are equal:
\[
    F(S_{\mathrm{ret}}+P_{\mathrm p})
    -
    F(S_{\mathrm{ret}}+P_{\mathrm c})
    =
    F(S_{\mathrm{ret}})-F(S_{\mathrm{ret}})
    =
    0.
\]
Thus the zero-data event contributes nothing.

It remains to compute the contribution of the event \(\{K=1\}\).  There are two
mutually exclusive one-data cases:
\[
    \{K_{\mathrm p}=1,K_{\mathrm c}=0\},
    \qquad
    \{K_{\mathrm p}=0,K_{\mathrm c}=1\}.
\]
Since \(K_{\mathrm p}\) and \(K_{\mathrm c}\) are independent Poisson variables
with mean \(\epsilon/\phi\),
\[
    \mathbb P(K_{\mathrm p}=1,K_{\mathrm c}=0)
    =
    e^{-2\epsilon/\phi}\frac{\epsilon}{\phi},
\]
and
\[
    \mathbb P(K_{\mathrm p}=0,K_{\mathrm c}=1)
    =
    e^{-2\epsilon/\phi}\frac{\epsilon}{\phi}.
\]
Conditioned on \(\{K_{\mathrm p}=1,K_{\mathrm c}=0\}\), the unique poison data
has law \(Z_{\mathrm p}^{(\tau)}\), is independent of \(S_{\mathrm{ret}}\), and
\(P_{\mathrm c}=0\).  Hence the conditional contribution is
\[
    \mathbb E\left[
        F(S_{\mathrm{ret}}+Z_{\mathrm p}^{(\tau)})
        -
        F(S_{\mathrm{ret}})
    \right].
\]
Conditioned on \(\{K_{\mathrm p}=0,K_{\mathrm c}=1\}\), the unique deleted clean
data has law \(Z_{\mathrm c}^{(\tau)}\), is independent of \(S_{\mathrm{ret}}\),
and \(P_{\mathrm p}=0\).  Hence the conditional contribution is
\[
    \mathbb E\left[
        F(S_{\mathrm{ret}})
        -
        F(S_{\mathrm{ret}}+Z_{\mathrm c}^{(\tau)})
    \right].
\]
Adding the two one-data contributions gives
\[
\begin{aligned}
    &\mathbb E\left[
        F(S_{\mathrm{ret}}+P_{\mathrm p})
        -
        F(S_{\mathrm{ret}}+P_{\mathrm c});
        K=1
    \right]
    \\
    &\quad =
    e^{-2\epsilon/\phi}\frac{\epsilon}{\phi}
    \mathbb E\left[
        F(S_{\mathrm{ret}}+Z_{\mathrm p}^{(\tau)})
        -
        F(S_{\mathrm{ret}}+Z_{\mathrm c}^{(\tau)})
    \right].
\end{aligned}
\]
Since
\[
    e^{-2\epsilon/\phi}=1+O_\phi(\epsilon),
\]
and since the expectation in the last display is bounded in absolute value by
\(2\|F\|_\infty=O_{c,\eta_{\mathrm p}}(1+\tau^2)\), we may replace
\(e^{-2\epsilon/\phi}\) by \(1\) at cost
\[
    O_{\phi,c,\eta_{\mathrm p}}
    \left(\epsilon^2(1+\tau^2)\right).
\]
Therefore
\[
\begin{aligned}
    &\mathbb E\left[
        F(S_{\mathrm{ret}}+P_{\mathrm p})
        -
        F(S_{\mathrm{ret}}+P_{\mathrm c})
    \right]
    \\
    &\quad =
    \frac{\epsilon}{\phi}
    \mathbb E\left[
        F(S_{\mathrm{ret}}+Z_{\mathrm p}^{(\tau)})
        -
        F(S_{\mathrm{ret}}+Z_{\mathrm c}^{(\tau)})
    \right]
    +
    O_{\phi,c,\eta_{\mathrm p}}
    \left(\epsilon^2(1+\tau^2)\right).
\end{aligned}
\]

Finally, we replace \(S_{\mathrm{ret}}\) by \(S_0^{(\tau)}\).  Couple
\(S_0^{(\tau)}\) as
\[
    S_0^{(\tau)}
    =
    S_{\mathrm{ret}}+P_{\mathrm c}',
\]
where \(P_{\mathrm c}'\) is an independent clean Poisson process with intensity
\(\epsilon/\phi\) and data law \(Z_{\mathrm c}^{(\tau)}\).  Then
\[
    \mathbb P(P_{\mathrm c}'\ne0)
    =
    1-e^{-\epsilon/\phi}
    =
    O_\phi(\epsilon).
\]
Therefore, for a random data \(Z\) independent of all processes and equal in law
to either \(Z_{\mathrm p}^{(\tau)}\) or \(Z_{\mathrm c}^{(\tau)}\),
\[
\begin{aligned}
    \left|
    \mathbb E F(S_{\mathrm{ret}}+Z)
    -
    \mathbb E F(S_0^{(\tau)}+Z)
    \right|
    &\le
    2\|F\|_\infty \mathbb P(P_{\mathrm c}'\ne0)
    \\
    &=
    O_{\phi,c,\eta_{\mathrm p}}
    \left(\epsilon(1+\tau^2)\right).
\end{aligned}
\]
Applying this once with \(Z=Z_{\mathrm p}^{(\tau)}\) and once with
\(Z=Z_{\mathrm c}^{(\tau)}\), we obtain
\[
\begin{aligned}
    &\mathbb E\left[
        F(S_{\mathrm{ret}}+Z_{\mathrm p}^{(\tau)})
        -
        F(S_{\mathrm{ret}}+Z_{\mathrm c}^{(\tau)})
    \right]
    \\
    &\quad =
    \mathbb E\left[
        F(S_0^{(\tau)}+Z_{\mathrm p}^{(\tau)})
        -
        F(S_0^{(\tau)}+Z_{\mathrm c}^{(\tau)})
    \right]
    +
    O_{\phi,c,\eta_{\mathrm p}}
    \left(\epsilon(1+\tau^2)\right).
\end{aligned}
\]
Multiplying this replacement error by the outside factor \(\epsilon/\phi\)
produces
\[
    O_{\phi,c,\eta_{\mathrm p}}
    \left(\epsilon^2(1+\tau^2)\right).
\]
Combining all previous displays gives
\[
\begin{aligned}
    &\mathbb E\left[
        F(S_{\mathrm{ret}}+P_{\mathrm p})
        -
        F(S_{\mathrm{ret}}+P_{\mathrm c})
    \right]
    \\
    &\quad =
    \frac{\epsilon}{\phi}
    \mathbb E\left[
        F(S_0^{(\tau)}+Z_{\mathrm p}^{(\tau)})
        -
        F(S_0^{(\tau)}+Z_{\mathrm c}^{(\tau)})
    \right]
    +
    O_{\phi,\varrho,c,\eta_{\mathrm p}}
    \left(\epsilon^2(1+\tau^2)\right).
\end{aligned}
\]
Multiplying by the prefactor \(\mu_U\) in the definition of
\(\mathcal D_{\phi,\varrho}(\epsilon,\tau)\) proves the claimed expansion.

\end{proof}

\begin{theorem}[Fixed-aspect leading asymptotics]
\label{thm:modified-fixed-phi-leading-response}
Fix $\phi\in(0,\infty)$ and $\varrho\ge0$.  Assume
\[
    \epsilon\downarrow0,
    \qquad
    \tau=\tau(\epsilon)\to\infty,
    \qquad
    \tau(\epsilon)\ge\tau_0.
\]
Then the clean-relative degradation satisfies
\[
    \mathcal D_{\phi,\varrho}(\epsilon,\tau)
    =
    C_{\phi,\varrho}\,\epsilon\tau^2(1+o(1)),
\]
where
\[
    C_{\phi,\varrho}
    :=
    \mu_U\eta_{\mathrm p}^2\phi
    \mathbb E\left[\frac{1}{(\varrho+\phi(Q_0+1))^2}\right].
\]
In particular, the leading fixed-$\phi$ term depends on $\tau$ through
$\epsilon\tau^2$.
\end{theorem}

\begin{proof}
By Lemma~\ref{lem:modified-generator-expansion},
\[
\begin{aligned}
    \mathcal D_{\phi,\varrho}(\epsilon,\tau)
    ={}&
    \frac{\mu_U\epsilon}{\phi}
    \mathbb E\left[
    F_{\phi,\varrho}
    \left(S_0^{(\tau)}+Z_{\mathrm p}^{(\tau)}\right)
    -
    F_{\phi,\varrho}
    \left(S_0^{(\tau)}+Z_{\mathrm c}^{(\tau)}\right)
    \right]
    \\
    &\quad+
    O_{\phi,\varrho}\left(\epsilon^2(1+\tau^2)\right).
\end{aligned}
\]
The added-clean expectation is $O_{\phi,\varrho}(1)$ by
Lemma~\ref{lem:modified-clean-add-limit}.  The added-poison expectation is, by
Lemma~\ref{lem:modified-one-poison},
\[
    \eta_{\mathrm p}^2\phi^2\tau^2
    \mathbb E\left[\frac{1}{(\varrho+\phi(Q_0+1))^2}\right](1+o(1)).
\]
Since $\tau\to\infty$, the $O(1)$ added-clean term is negligible relative to the
$\tau^2$ added-poison term.  Also,
\[
    \epsilon^2(1+\tau^2)=o(\epsilon\tau^2),
\]
because $\epsilon\to0$ and $\tau\to\infty$.  Multiplying the leading
added-poison term by $\mu_U\epsilon/\phi$ gives the displayed constant.
\end{proof}

In the theorem below we work with the ridgeless risk relative
to the untruncated clean fixed-aspect reference
\[
    \mathcal R_{\phi,0}(0,\infty)
    :=\lim_{\tau\to\infty}\mathcal R_{\phi,0}(0,\tau).
\]
Define
\begin{equation}
\label{eq:fixed-phi-total-excess-def}
    \mathcal E_{\phi}(\epsilon,\tau)
    :=
    \mathcal R_{\phi,0}(\epsilon,\tau)
    -\mathcal R_{\phi,0}(0,\infty).
\end{equation}
This is the fixed-$\phi$ analogue of comparing the poisoned truncated estimator
against the ideal clean untruncated target.

For $z>1$, define
\[
    h_c(z)
    :=
    \begin{cases}
    c(z-1),&1<z<1/c,\\[1mm]
    1-c,&z\ge 1/c.
    \end{cases}
\]
Set
\begin{equation}
\label{eq:fixed-phi-clean-clipping-constant}
    K_{\phi,c,q_\star}^{\mathrm{cl}}
    :=
    \frac{\mu_U}{\phi}\,
    q_\star\int_1^\infty h_c(z)^2z^{-q_\star-1}\,\mathrm dz.
\end{equation}
This constant is finite and strictly positive.

\begin{hdlemma}[Clean truncation term at fixed aspect ratio]
\label{lem:fixed-phi-clean-truncation-tau-q}
For fixed $\phi\in(0,\infty)$, in the ridgeless case $\varrho=0$,
\[
    \mathcal R_{\phi,0}(0,\tau)
    -\mathcal R_{\phi,0}(0,\infty)
    =
    K_{\phi,c,q_\star}^{\mathrm{cl}}\,\tau^{-q_\star}
    +o(\tau^{-q_\star})
\]
as $\tau\to\infty$.
\end{hdlemma}

\begin{proof}

Recall the clean coordinate model.  A sample is generated by drawing
\(J\) uniformly from \(\{1,\ldots,P\}\), drawing \(U\), and setting
\[
    X=\sqrt{PU}\,e_J,
    \qquad
    Y=c\sqrt U.
\]
The teacher vector is
\[
    \theta_{\star,P}
    =
    \frac{c}{\sqrt P}(1,\ldots,1)^\top.
\]
Thus, if \(J=j\), then
\[
    X^\top \theta_{\star,P}
    =
    \sqrt{PU}\,e_j^\top \theta_{\star,P}
    =
    c\sqrt U
    =Y.
\]
Hence the clean model is noiseless at the coordinate level in this simplified
calculation.

Let \(N_j\) be the number of clean training samples that hit coordinate \(j\).
For those samples, write their radial variables as
\[
    U_{j,1},\ldots,U_{j,N_j}.
\]

In the untruncated clean problem, the coordinate-level sufficient statistics are
\[
    S_{j,A}
    =
    \sum_{\ell=1}^{N_j}U_{j,\ell},
    \qquad
    S_{j,B}
    =
    \sum_{\ell=1}^{N_j}cU_{j,\ell}
    =
    cS_{j,A}.
\]
The empirical Hessian is diagonal.  The ridgeless estimator is defined by the
Moore--Penrose inverse, so
\[
    \sqrt P\,\widehat\theta_j
    =
    \begin{cases}
    \displaystyle
    \frac{S_{j,B}}{S_{j,A}},&S_{j,A}>0,\\[3mm]
    0,&S_{j,A}=0.
    \end{cases}
\]
Since \(U_{j,\ell}>0\), the condition \(S_{j,A}>0\) is equivalent to
\(N_j\ge1\).

If \(N_j\ge1\), then
\[
    \frac{S_{j,B}}{S_{j,A}}
    =
    \frac{cS_{j,A}}{S_{j,A}}
    =
    c.
\]
Therefore
\[
    \sqrt P\,\widehat\theta_j=c
    \qquad
    \text{whenever }N_j\ge1.
\]
Equivalently,
\[
    \widehat\theta_j
    =
    \frac{c}{\sqrt P}
    =
    (\theta_{\star,P})_j
    \qquad
    \text{whenever }N_j\ge1.
\]
Thus every observed coordinate is fit exactly.

If \(N_j=0\), then
\[
    S_{j,A}=S_{j,B}=0.
\]
The Moore--Penrose convention gives
\[
    \sqrt P\,\widehat\theta_j=0,
    \qquad
    \widehat\theta_j=0.
\]
But the teacher coordinate is
\[
    (\theta_{\star,P})_j=\frac{c}{\sqrt P}.
\]
Hence unobserved coordinates are not fit; they are set to zero by the
minimum-norm convention.

The clean excess risk identity is
\[
    \mathcal R_{\mathrm{clean},P}(\widehat\theta)
    -
    \mathcal R_{\mathrm{clean},P}(\theta_{\star,P})
    =
    \mu_U\frac1P
    \sum_{j=1}^P
    \left(c-\sqrt P\,\widehat\theta_j\right)^2.
\]
From the preceding calculation,
\[
    c-\sqrt P\,\widehat\theta_j
    =
    \begin{cases}
    0,&N_j\ge1,\\
    c,&N_j=0.
    \end{cases}
\]
Therefore
\[
    \left(c-\sqrt P\,\widehat\theta_j\right)^2
    =
    c^2\mathbf 1\{N_j=0\}.
\]
Substituting into the risk identity gives
\[
    \mathcal R_{\mathrm{clean},P}(\widehat\theta)
    -
    \mathcal R_{\mathrm{clean},P}(\theta_{\star,P})
    =
    \mu_U c^2
    \frac1P\sum_{j=1}^P
    \mathbf 1\{N_j=0\}.
\]
Thus the empirical clean excess risk is exactly proportional to the fraction of
unobserved coordinates.

When
\[
    \frac PD\to\phi\in(0,\infty),
\]
the number of clean samples hitting a fixed coordinate converges to
\[
    N_0\sim\operatorname{Pois}(1/\phi).
\]
Hence
\[
    \mathbb P(N_0=0)
    =
    e^{-1/\phi}.
\]
By the occupancy law of large numbers,
\[
    \frac1P\sum_{j=1}^P\mathbf 1\{N_j=0\}
    \xrightarrow{\mathbb P}
    e^{-1/\phi}.
\]
Therefore, in the untruncated clean ridgeless problem,
\[
    \mathcal R_{\phi,0}(0,\infty)
    =
    \mu_U c^2 e^{-1/\phi}.
\]
In the untruncated ridgeless clean problem, every observed coordinate is fit
exactly and only unobserved coordinates contribute error.  Hence
\[
    \mathcal R_{\phi,0}(0,\infty)
    =\mu_U c^2\mathbb P(N_0=0).
\]

Recall,
\[
    A_\tau:=\sum_{\ell=1}^{N_0}a_\tau(U_\ell),
    \qquad
    B_\tau:=\sum_{\ell=1}^{N_0}r_\tau(U_\ell).
\]
When $A_\tau>0$, the coordinate estimator is $B_\tau/A_\tau$; when
$A_\tau=0$, the Moore--Penrose convention gives estimator $0$.  Let
\[
    \Delta_\tau(t)
    :=r_\tau(t)-ca_\tau(t).
\]
Then
\[
    \Delta_\tau(t)=0\quad\text{for }t\le\tau,
\]
and, for $t>\tau$,
\[
    \Delta_\tau(t)
    =
    \begin{cases}
    c(t-\tau),&\tau<t<\tau/c,\\[1mm]
    (1-c)\tau,&t\ge\tau/c.
    \end{cases}
\]
Therefore
\[
    B_\tau-cA_\tau
    =\sum_{\ell=1}^{N_0}\Delta_\tau(U_\ell)
    \ge0.
\]

It follows that
\begin{equation}
\label{eq:clean-truncation-excess-ratio}
    \mathcal R_{\phi,0}(0,\tau)
    -\mathcal R_{\phi,0}(0,\infty)
    =
    \mu_U\,
    \mathbb E\left[
    \left(
    \frac{\sum_{\ell=1}^{N_0}\Delta_\tau(U_\ell)}
         {\sum_{\ell=1}^{N_0}a_\tau(U_\ell)}
    \right)^2
    \mathbf 1\{N_0\ge1\}
    \right],
\end{equation}
with the ratio interpreted as zero when the denominator is zero.

The random collection of clean radii
\[
    \{U_1,\ldots,U_{N_0}\}
\]
can be viewed as a Poisson point process on \([1,\infty)\) with intensity
measure
\[
    \nu(\,\mathrm d t)
    =
    \frac1\phi q_\star t^{-q_\star-1}\,\mathrm d t.
\]
Indeed, the total mass of \(\nu\) is
\[
    \nu([1,\infty))
    =
    \frac1\phi
    \int_1^\infty q_\star t^{-q_\star-1}\,\mathrm d t
    =
    \frac1\phi,
\]
so the total number of points is \(\operatorname{Pois}(1/\phi)\), and the normalized data law
has density \(q_\star t^{-q_\star-1}\) on \([1,\infty)\).

Let
\[
    \Pi:=\{U_1,\ldots,U_{N_0}\}
\]
denote this Poisson point process.  Define
\[
    R_\tau(\Pi)
    :=
    \left(
    \frac{\sum_{t\in\Pi}\Delta_\tau(t)}
         {\sum_{t\in\Pi}a_\tau(t)}
    \right)^2
    \mathbf 1\{|\Pi|\ge1\},
\]
with the ratio interpreted as zero when the denominator is zero.

Since \(a_\tau(t)>0\) for \(t\ge1\), the denominator is positive whenever
\(|\Pi|\ge1\).  Also,
\[
    0\le \Delta_\tau(t)\le a_\tau(t),
\]
so
\[
    0\le R_\tau(\Pi)\le1.
\]

Let
\[
    N_\tau^+
    :=
    \#\{t\in\Pi:t>\tau\}
\]
be the number of exceedances above \(\tau\).  Since the intensity of
\((\tau,\infty)\) is
\[
    \nu((\tau,\infty))
    =
    \frac1\phi \mathbb P(U>\tau)
    =
    \frac1\phi \tau^{-q_\star},
\]
we have
\[
    N_\tau^+
    \sim
    \operatorname{Pois}\left(\frac1\phi\tau^{-q_\star}\right).
\]
Therefore
\[
    \mathbb P(N_\tau^+=0)=1-O(\tau^{-q_\star}),
\]
\[
    \mathbb P(N_\tau^+=1)
    =
    \frac1\phi\tau^{-q_\star}
    +O(\tau^{-2q_\star}),
\]
and
\[
    \mathbb P(N_\tau^+\ge2)
    =
    O(\tau^{-2q_\star}).
\]
On the event \(N_\tau^+=0\), every radius is at most \(\tau\), hence
\[
    \Delta_\tau(U_\ell)=0
    \quad\text{for all }\ell,
\]
and therefore
\[
    R_\tau(\Pi)=0.
\]
On the event \(N_\tau^+\ge2\), we use only the bound \(0\le R_\tau(\Pi)\le1\):
\[
    \mathbb E\left[R_\tau(\Pi)\mathbf 1\{N_\tau^+\ge2\}\right]
    \le
    \mathbb P(N_\tau^+\ge2)
    =
    O(\tau^{-2q_\star})
    =
    o(\tau^{-q_\star}).
\]
Thus the leading contribution comes only from the event of exactly one
exceedance.

The Campbell--Mecke formula says that, for a Poisson point process \(\Pi\) with
intensity measure \(\nu\),
\[
    \mathbb E\left[
        \sum_{t\in\Pi} g(t,\Pi\setminus\{t\})
    \right]
    =
    \int
    \mathbb E\left[
        g(t,\Pi)
    \right]\nu(\,\mathrm d t),
\]
whenever \(g\ge0\) or the expectations are finite.

We apply it to the contribution from a distinguished exceedance \(t>\tau\) and
then note that configurations with more than one exceedance are negligible.  Let
\[
    \Pi_{\le\tau}
    :=
    \Pi\cap[1,\tau]
\]
be an independent Poisson point process on \([1,\tau]\) with intensity
\(\nu\) restricted to \([1,\tau]\).  By the independence of Poisson processes
on disjoint sets, after selecting the unique exceedance \(t>\tau\), the
remaining sub-threshold points have the same law as \(\Pi_{\le\tau}\), up to an
error supported on events with another exceedance.  Those events have probability
\(O(\tau^{-2q_\star})\).

Hence the one-exceedance contribution is
\[
\begin{aligned}
    I_\tau
    &:={}
    \mathbb E\left[R_\tau(\Pi)\mathbf 1\{N_\tau^+=1\}\right]
    \\
    &=
    \int_\tau^\infty
    \mathbb E\left[
    \left(
    \frac{\Delta_\tau(t)}
         {a_\tau(t)+\sum_{u\in\Pi_{\le\tau}}a_\tau(u)}
    \right)^2
    \right]
    \nu(\,\mathrm d t)
    +
    O(\tau^{-2q_\star}).
\end{aligned}
\]
Since \(t>\tau\), we have
\[
    a_\tau(t)=\tau.
\]
Now set
\[
    t=\tau z,
    \qquad z>1.
\]
The Pareto intensity transforms as
\[
    \nu(\,\mathrm d t)
    =
    \frac1\phi q_\star t^{-q_\star-1}\,\mathrm d t
    =
    \frac1\phi q_\star
    (\tau z)^{-q_\star-1}\tau\,\mathrm d z
    =
    \frac1\phi q_\star
    \tau^{-q_\star}
    z^{-q_\star-1}\,\mathrm d z.
\]
Also,
\[
    \frac{\Delta_\tau(\tau z)}{\tau}
    =
    h_c(z).
\]
Therefore
\[
\begin{aligned}
    I_\tau
    &=
    \frac1\phi q_\star \tau^{-q_\star}
    \int_1^\infty
    \mathbb E\left[
    \left(
    \frac{\Delta_\tau(\tau z)}
         {\tau+\sum_{u\in\Pi_{\le\tau}}a_\tau(u)}
    \right)^2
    \right]
    z^{-q_\star-1}\,\mathrm d z
    +
    O(\tau^{-2q_\star})
    \\
    &=
    \frac1\phi q_\star \tau^{-q_\star}
    \int_1^\infty
    h_c(z)^2
    \mathbb E\left[
    \left(
    \frac{\tau}
         {\tau+\sum_{u\in\Pi_{\le\tau}}a_\tau(u)}
    \right)^2
    \right]
    z^{-q_\star-1}\,\mathrm d z
    +
    O(\tau^{-2q_\star}).
\end{aligned}
\]

We claim that
\[
    \frac1\tau\sum_{u\in\Pi_{\le\tau}}a_\tau(u)
    \longrightarrow0
    \qquad\text{in probability}.
\]
Indeed, on \([1,\tau]\) we have \(a_\tau(u)=u\), so
\[
    \mathbb E\left[
    \sum_{u\in\Pi_{\le\tau}}a_\tau(u)
    \right]
    =
    \int_1^\tau u\,\nu(\,\mathrm d u)
    =
    \frac1\phi
    \int_1^\tau q_\star u^{-q_\star}\,\mathrm d u.
\]
Since \(q_\star>1\),
\[
    \int_1^\tau q_\star u^{-q_\star}\,\mathrm d u
    \le
    \int_1^\infty q_\star u^{-q_\star}\,\mathrm d u
    <\infty.
\]
Thus
\[
    \mathbb E\left[
    \sum_{u\in\Pi_{\le\tau}}a_\tau(u)
    \right]
    =O_\phi(1),
\]
and
\[
    \frac1\tau\sum_{u\in\Pi_{\le\tau}}a_\tau(u)
    \xrightarrow{\mathbb P}0.
\]
Consequently,
\[
    \left(
    \frac{\tau}
         {\tau+\sum_{u\in\Pi_{\le\tau}}a_\tau(u)}
    \right)^2
    \xrightarrow{\mathbb P}1.
\]

Moreover, because
\[
    0\le h_c(z)\le1,
\]
and
\[
    \int_1^\infty z^{-q_\star-1}\,\mathrm d z<\infty,
\]
dominated convergence also applies to the integral over \(z\).  Therefore
\[
    I_\tau
    =
    \frac1\phi q_\star \tau^{-q_\star}
    \int_1^\infty
    h_c(z)^2z^{-q_\star-1}\,\mathrm d z
    +
    o(\tau^{-q_\star}).
\]

Combining the decomposition
\[
    \mathbb E[R_\tau(\Pi)]
    =
    \mathbb E[R_\tau(\Pi)\mathbf 1\{N_\tau^+=0\}]
    +
    \mathbb E[R_\tau(\Pi)\mathbf 1\{N_\tau^+=1\}]
    +
    \mathbb E[R_\tau(\Pi)\mathbf 1\{N_\tau^+\ge2\}],
\]
with
\[
    \mathbb E[R_\tau(\Pi)\mathbf 1\{N_\tau^+=0\}]=0,
\]
\[
    \mathbb E[R_\tau(\Pi)\mathbf 1\{N_\tau^+\ge2\}]
    =
    o(\tau^{-q_\star}),
\]
and the one-exceedance estimate above, we obtain
\[
\begin{aligned}
    &\mathbb E\left[
    \left(
    \frac{\sum_{\ell=1}^{N_0}\Delta_\tau(U_\ell)}
         {\sum_{\ell=1}^{N_0}a_\tau(U_\ell)}
    \right)^2
    \mathbf 1\{N_0\ge1\}
    \right]
    \\
    &\qquad =
    \frac1\phi\,
    q_\star\tau^{-q_\star}
    \int_1^\infty h_c(z)^2z^{-q_\star-1}\,\mathrm d z
    +o(\tau^{-q_\star}).
\end{aligned}
\]

\end{proof}

\begin{theorem}[Fixed-$\phi$ two-term risk and optimal truncation]
\label{thm:fixed-phi-two-term-optimization}
Fix \(\phi\in(0,\infty)\) and consider the ridgeless modified attack
\(\varrho=0\).  Assume
\[
    \epsilon\downarrow0,
    \qquad
    \tau=\tau(\epsilon)\to\infty,
\]
so that, for all sufficiently small \(\epsilon\), \(\tau(\epsilon)\ge\tau_0\).
Then
\begin{equation}
\label{eq:fixed-phi-two-term-risk}
    \mathcal E_{\phi}(\epsilon,\tau)
    =
    K_{\phi,c,q_\star}^{\mathrm{cl}}\tau^{-q_\star}
    +
    C_{\phi,0}\epsilon\tau^2
    +
    o\!\left(\tau^{-q_\star}+\epsilon\tau^2\right),
\end{equation}
where
\[
    C_{\phi,0}
    =
    \mu_U\eta_{\mathrm p}^2\phi\,
    \mathbb E\left[\frac{1}{\phi^2(Q_0+1)^2}\right]
    =
    \frac{\mu_U\eta_{\mathrm p}^2}{\phi}
    \mathbb E\left[\frac{1}{(Q_0+1)^2}\right].
\]
Consequently the leading two-term equivalent
\[
    H_\epsilon(\tau)
    :=K_{\phi,c,q_\star}^{\mathrm{cl}}\tau^{-q_\star}
      +C_{\phi,0}\epsilon\tau^2
\]
is minimized at
\begin{equation}
\label{eq:fixed-phi-tau-opt-corrected}
    \tau_{\mathrm{opt}}(\epsilon)
    =
    \left(
    \frac{q_\star K_{\phi,c,q_\star}^{\mathrm{cl}}}
         {2C_{\phi,0}\epsilon}
    \right)^{1/(q_\star+2)}.
\end{equation}
This choice is admissible, \(\tau_{\mathrm{opt}}(\epsilon)\ge\tau_0\), for all
sufficiently small \(\epsilon\), and the optimized fixed-\(\phi\) scaling is
\begin{equation}
\label{eq:fixed-phi-optimized-risk-corrected}
    \mathcal E_{\phi}
    (\epsilon,\tau_{\mathrm{opt}}(\epsilon))
    \asymp
    \epsilon^{q_\star/(q_\star+2)}.
\end{equation}
More precisely,
\[
\begin{aligned}
    \mathcal E_{\phi}
    (\epsilon,\tau_{\mathrm{opt}}(\epsilon))
    \sim{}&
    \left(1+\frac{q_\star}{2}\right)
    K_{\phi,c,q_\star}^{\mathrm{cl}}
    \left(
    \frac{2C_{\phi,0}}
         {q_\star K_{\phi,c,q_\star}^{\mathrm{cl}}}
    \right)^{q_\star/(q_\star+2)}
    \epsilon^{q_\star/(q_\star+2)}.
\end{aligned}
\]
\end{theorem}

\begin{proof}
By definition,
\[
    \mathcal E_{\phi}(\epsilon,\tau)
    =
    \bigl[\mathcal R_{\phi,0}(0,\tau)-\mathcal R_{\phi,0}(0,\infty)\bigr]
    +
    \bigl[\mathcal R_{\phi,0}(\epsilon,\tau)-\mathcal R_{\phi,0}(0,\tau)\bigr].
\]
The first bracket is the clean truncation term, so
Lemma~\ref{lem:fixed-phi-clean-truncation-tau-q} gives
\[
    \mathcal R_{\phi,0}(0,\tau)-\mathcal R_{\phi,0}(0,\infty)
    =
    K_{\phi,c,q_\star}^{\mathrm{cl}}\tau^{-q_\star}
    +o(\tau^{-q_\star}).
\]
The second bracket is \(\mathcal D_{\phi,0}(\epsilon,\tau)\).  The
fixed-\(\phi\) leading-response theorem, using the uniform generator remainder
from Lemma~\ref{lem:modified-generator-expansion}, gives
\[
    \mathcal D_{\phi,0}(\epsilon,\tau)
    =
    C_{\phi,0}\epsilon\tau^2(1+o(1)).
\]
Thus the combined remainder is
\[
    o(\tau^{-q_\star})+o(\epsilon\tau^2)
    =o(\tau^{-q_\star}+\epsilon\tau^2),
\]
which proves \eqref{eq:fixed-phi-two-term-risk}.

It remains to minimize \(H_\epsilon\).  Differentiating,
\[
    H_\epsilon'(\tau)
    =
    -q_\star K_{\phi,c,q_\star}^{\mathrm{cl}}\tau^{-q_\star-1}
    +2C_{\phi,0}\epsilon\tau .
\]
The unique critical point satisfies
\[
    q_\star K_{\phi,c,q_\star}^{\mathrm{cl}}\tau^{-q_\star-1}
    =2C_{\phi,0}\epsilon\tau,
\]
which is equivalent to \eqref{eq:fixed-phi-tau-opt-corrected}.  Since
\(\tau_{\mathrm{opt}}(\epsilon)\to\infty\), it lies in the admissible region
\(\tau\ge\tau_0\) for all sufficiently small \(\epsilon\).  At the optimum,
\[
    C_{\phi,0}\epsilon\tau_{\mathrm{opt}}^2
    =
    \frac{q_\star}{2}
    K_{\phi,c,q_\star}^{\mathrm{cl}}\tau_{\mathrm{opt}}^{-q_\star}.
\]
Therefore
\[
    H_\epsilon(\tau_{\mathrm{opt}})
    =
    \left(1+\frac{q_\star}{2}\right)
    K_{\phi,c,q_\star}^{\mathrm{cl}}\tau_{\mathrm{opt}}^{-q_\star}.
\]
Substitution of \eqref{eq:fixed-phi-tau-opt-corrected} gives the stated
constant and exponent.
\end{proof}

\subsection{\texorpdfstring{$\phi\downarrow0$}{phi -> 0}: Classical Statistics}
\label{sec:population-recovery}

The fixed-aspect-ratio formula contains a population-mixture calculation as a
further many-samples-per-coordinate limit.  This limit is different from the
fixed-$\phi$ small-$\epsilon$ limit.  Taking $\phi\downarrow0$ first removes the
sparse coordinate-hit mechanism.

\begin{theorem}[Population limit of the deterministic equivalent]
\label{thm:population-limit-of-de}
Fix $\epsilon\in[0,1)$, $\tau\ge\tau_0$, and $\varrho\ge0$.  Then, as
$\phi\downarrow0$,
\[
    \mathcal R_{\phi,\varrho}(\epsilon,\tau)
    \longrightarrow
    \mathcal R_{0,\varrho}(\epsilon,\tau),
\]
where
\begin{equation}
\label{eq:population-ridge-risk-modified}
    \mathcal R_{0,\varrho}(\epsilon,\tau)
    :=
    \mu_U
    \left[
    c-
    \frac{(1-\epsilon)B_\tau+
    \epsilon(c+\eta_{\mathrm p}\tau)}
    {\varrho+(1-\epsilon)M_\tau+
    \epsilon}
    \right]^2.
\end{equation}
In particular, in the ridgeless population limit $\varrho=0$,
\begin{equation}
\label{eq:population-risk-recovered-modified}
    \mathcal R_{0,0}(\epsilon,\tau)
    =
    \mu_U
    \left[
    \frac{(1-\epsilon)c(M_{\tau/c}-M_\tau)
    +
    \epsilon\eta_{\mathrm p}\tau}
    {(1-\epsilon)M_\tau+
    \epsilon}
    \right]^2.
\end{equation}
\end{theorem}

\begin{proof}
The clean and poison rates in \eqref{eq:limiting-SA-SB-modified} are
$(1-\epsilon)/\phi$ and $\epsilon/\phi$.  Therefore
\[
    \phi S_A^{\epsilon,\tau}
    \xrightarrow{\mathbb P}
    (1-\epsilon)M_\tau+\epsilon,
\]
and
\[
    \phi S_B^{\epsilon,\tau}
    \xrightarrow{\mathbb P}
    (1-\epsilon)B_\tau+
    \epsilon(c+\eta_{\mathrm p}\tau).
\]
The variances are $O(\phi)$ because the summands are bounded by constants
depending only on $\tau$.  The denominator
$\varrho+\phi S_A^{\epsilon,\tau}$ converges in probability to
$\varrho+(1-\epsilon)M_\tau+\epsilon>0$.  Since $G_{\phi,\varrho}$ is bounded
for fixed $\tau$, convergence in probability plus bounded convergence gives
\eqref{eq:population-ridge-risk-modified}.

For $\varrho=0$, substitute $B_\tau=cM_{\tau/c}$ and rearrange:
\[
\begin{aligned}
    &c-
    \frac{(1-\epsilon)cM_{\tau/c}+
    \epsilon(c+\eta_{\mathrm p}\tau)}
    {(1-\epsilon)M_\tau+
    \epsilon}
    \\
    &\quad =
    -
    \frac{(1-\epsilon)c(M_{\tau/c}-M_\tau)
    +
    \epsilon\eta_{\mathrm p}\tau}
    {(1-\epsilon)M_\tau+
    \epsilon}.
\end{aligned}
\]
Squaring gives \eqref{eq:population-risk-recovered-modified}.
\end{proof}

The exact Pareto identity gives
\[
    M_{\tau/c}-M_\tau
    =
    \frac{1-c^{q_\star-1}}{q_\star-1}\tau^{1-q_\star}.
\]
Define
\[
    A:=c\frac{1-c^{q_\star-1}}{q_\star-1},
    \qquad
    B:=\eta_{\mathrm p}.
\]

\begin{theorem}[Optimized poisoning exponent after taking \texorpdfstring{$\phi\downarrow0$}{phi -> 0}]
\label{thm:optimized-population-after-de}
Consider the ridgeless population limit
\eqref{eq:population-risk-recovered-modified}.  In the perturbative regime
\[
    \epsilon\downarrow0,
    \qquad
    \tau=\tau(\epsilon)\to\infty,
    \qquad
    \epsilon\tau(\epsilon)\to0,
\]
one has
\begin{equation}
\label{eq:population-after-de-asymptotic}
    \mathcal R_{0,0}(\epsilon,\tau)
    =
    \frac1{\mu_U}
    \left(
        A\tau^{-(q_\star-1)}+B\epsilon\tau
    \right)^2(1+o(1)).
\end{equation}
The leading term is minimized by
\begin{equation}
\label{eq:population-after-de-tau-opt}
    \tau_{\mathrm{opt}}(\epsilon)
    =
    \left(
        \frac{A(q_\star-1)}{B\epsilon}
    \right)^{1/q_\star},
\end{equation}
and at this value
\begin{equation}
\label{eq:population-after-de-optimized-risk}
    \mathcal R_{0,0}
    (\epsilon,\tau_{\mathrm{opt}}(\epsilon))
    \sim
    \frac{q_\star^2}{\mu_U}
    A^{2/q_\star}
    B^{2(q_\star-1)/q_\star}
    (q_\star-1)^{-2(q_\star-1)/q_\star}
    \epsilon^{2-2/q_\star}.
\end{equation}
\end{theorem}

\begin{proof}
From \eqref{eq:population-risk-recovered-modified},
\[
    \mathcal R_{0,0}(\epsilon,\tau)
    =
    \mu_U
    \left[
    \frac{(1-\epsilon)A\tau^{-(q_\star-1)}
    +B\epsilon\tau}
    {(1-\epsilon)M_\tau+
    \epsilon}
    \right]^2.
\]
Since $\tau\to\infty$ and $\epsilon\tau\to0$,
\[
    (1-\epsilon)M_\tau+
    \epsilon
    =
    \mu_U(1+o(1)),
\]
and
\[
    (1-\epsilon)A\tau^{-(q_\star-1)}
    =
    A\tau^{-(q_\star-1)}(1+o(1)).
\]
This proves \eqref{eq:population-after-de-asymptotic}.

It remains to minimize
\[
    g_\epsilon(\tau)
    :=
    A\tau^{-(q_\star-1)}+B\epsilon\tau.
\]
Differentiating gives
\[
    g_\epsilon'(\tau)
    =
    -A(q_\star-1)\tau^{-q_\star}+B\epsilon.
\]
The unique critical point is \eqref{eq:population-after-de-tau-opt}, and it is
the global minimizer of the leading expression because
$g_\epsilon(\tau)\to\infty$ as $\tau\downarrow0$ or $\tau\to\infty$.
At the optimum,
\[
    B\epsilon\tau_{\mathrm{opt}}
    =
    A(q_\star-1)\tau_{\mathrm{opt}}^{-(q_\star-1)},
\]
so
\[
    g_\epsilon(\tau_{\mathrm{opt}})
    =
    q_\star A\tau_{\mathrm{opt}}^{-(q_\star-1)}
    =
    q_\star
    A^{1/q_\star}
    B^{(q_\star-1)/q_\star}
    (q_\star-1)^{-(q_\star-1)/q_\star}
    \epsilon^{1-1/q_\star}.
\]
Squaring and multiplying by $1/\mu_U$ proves
\eqref{eq:population-after-de-optimized-risk}.
\end{proof}

\subsection{The same observable under the two orders of limits}
\label{sec:same-observable-two-limits}

Fix the ridgeless estimator ($\varrho=0$) and consider the degradation of the poisoned truncated estimator measured against the
clean untruncated ridgeless estimator,
\begin{equation}
\label{eq:Delta-observable}
    \Delta^{(P)}_{\epsilon,\tau}
    :=
    \mathcal R_{\mathrm{clean},P}\bigl(\widehat\theta_{\epsilon,\tau,0}\bigr)
    -
    \mathcal R_{\mathrm{clean},P}\bigl(\widehat\theta_{0,\infty,0}\bigr),
\end{equation}
where $\widehat\theta_{0,\infty,0}$ denotes the clean ($\epsilon=0$),
untruncated ($\tau=\infty$), ridgeless estimator.  The results from previous discussion is summarized below.

\begin{theorem}[High-dimensional vs ample-data limit]
\label{thm:noncommuting-two-limits}
Assume $P/D\to\phi\in(0,\infty)$ and work at $\varrho=0$. For every fixed $\phi\in(0,\infty)$, $\epsilon\in[0,1)$ and $\tau\ge\tau_0$,
\[
    \Delta^{(P)}_{\epsilon,\tau}
    \xrightarrow{\ \mathbb P\ }
    \mathcal E_\phi(\epsilon,\tau)
    :=
    \mathcal R_{\phi,0}(\epsilon,\tau)-\mu_U c^2 e^{-1/\phi}.
\]
\begin{enumerate}
\item[\textup{(i)}]\emph{(High-dimensional / fixed aspect ratio.)}
In the regime $\epsilon\downarrow0$, $\tau=\tau(\epsilon)\to\infty$,
\begin{equation}
\label{eq:hd-two-term}
    \mathcal E_\phi(\epsilon,\tau)
    =
    K^{\mathrm{cl}}_{\phi,c,q_\star}\,\tau^{-q_\star}
    +
    C_{\phi,0}\,\epsilon\tau^2
    +
    o\!\left(\tau^{-q_\star}+\epsilon\tau^2\right),
\end{equation}
which is minimized at
$\tau_{\mathrm{opt}}(\epsilon)\asymp\epsilon^{-1/(q_\star+2)}$, with
\begin{equation}
\label{eq:hd-optimized}
    \min_{\tau}\mathcal E_\phi(\epsilon,\tau)
    \ \asymp\
    \epsilon^{\,q_\star/(q_\star+2)}.
\end{equation}

\item[\textup{(ii)}]\emph{(Ample-data limit.)}
Taking the aspect ratio to zero first,
\[
    \lim_{\phi\downarrow0}\mathcal E_\phi(\epsilon,\tau)
    =
    \mathcal R_{0,0}(\epsilon,\tau),
\]
and in the regime $\epsilon\downarrow0$, $\tau\to\infty$, $\epsilon\tau\to0$,
\begin{equation}
\label{eq:ad-two-term}
    \mathcal R_{0,0}(\epsilon,\tau)
    =
    \frac1{\mu_U}
    \Bigl(A\,\tau^{-(q_\star-1)}+B\,\epsilon\tau\Bigr)^2\bigl(1+o(1)\bigr),
\end{equation}
which is minimized at $\tau_{\mathrm{opt}}(\epsilon)\asymp\epsilon^{-1/q_\star}$,
with
\begin{equation}
\label{eq:ad-optimized}
    \min_{\tau}\mathcal R_{0,0}(\epsilon,\tau)
    \ \asymp\
    \epsilon^{\,2-2/q_\star}.
\end{equation}
\end{enumerate}
\end{theorem}

\begin{figure}[h]
  \centering
  \begin{subfigure}{0.45\textwidth}
    \centering
    \includegraphics[width=\linewidth]{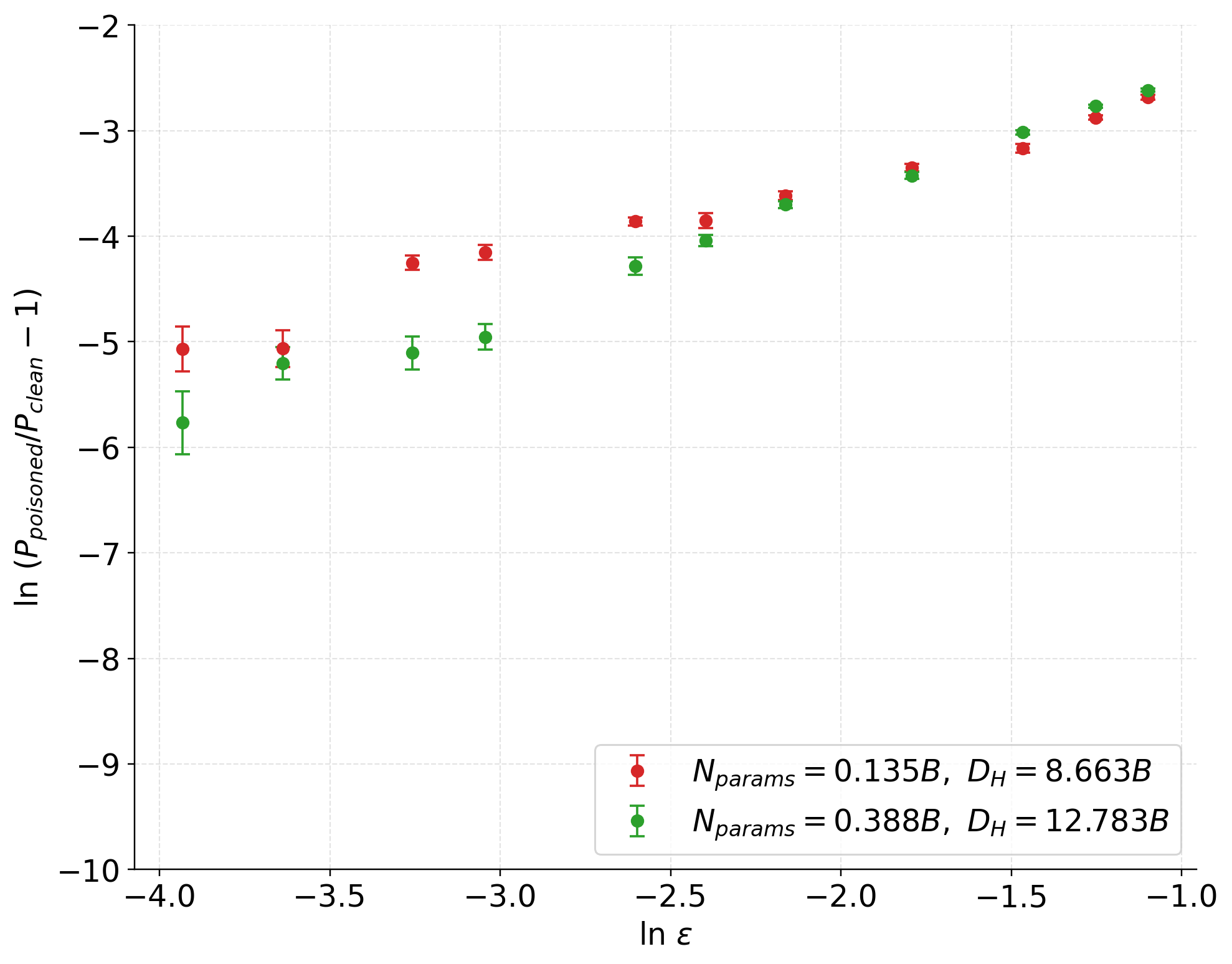}
    \caption{}
  \end{subfigure}
  \hfill
  \begin{subfigure}{0.45\textwidth}
    \centering
    \includegraphics[width=\linewidth]{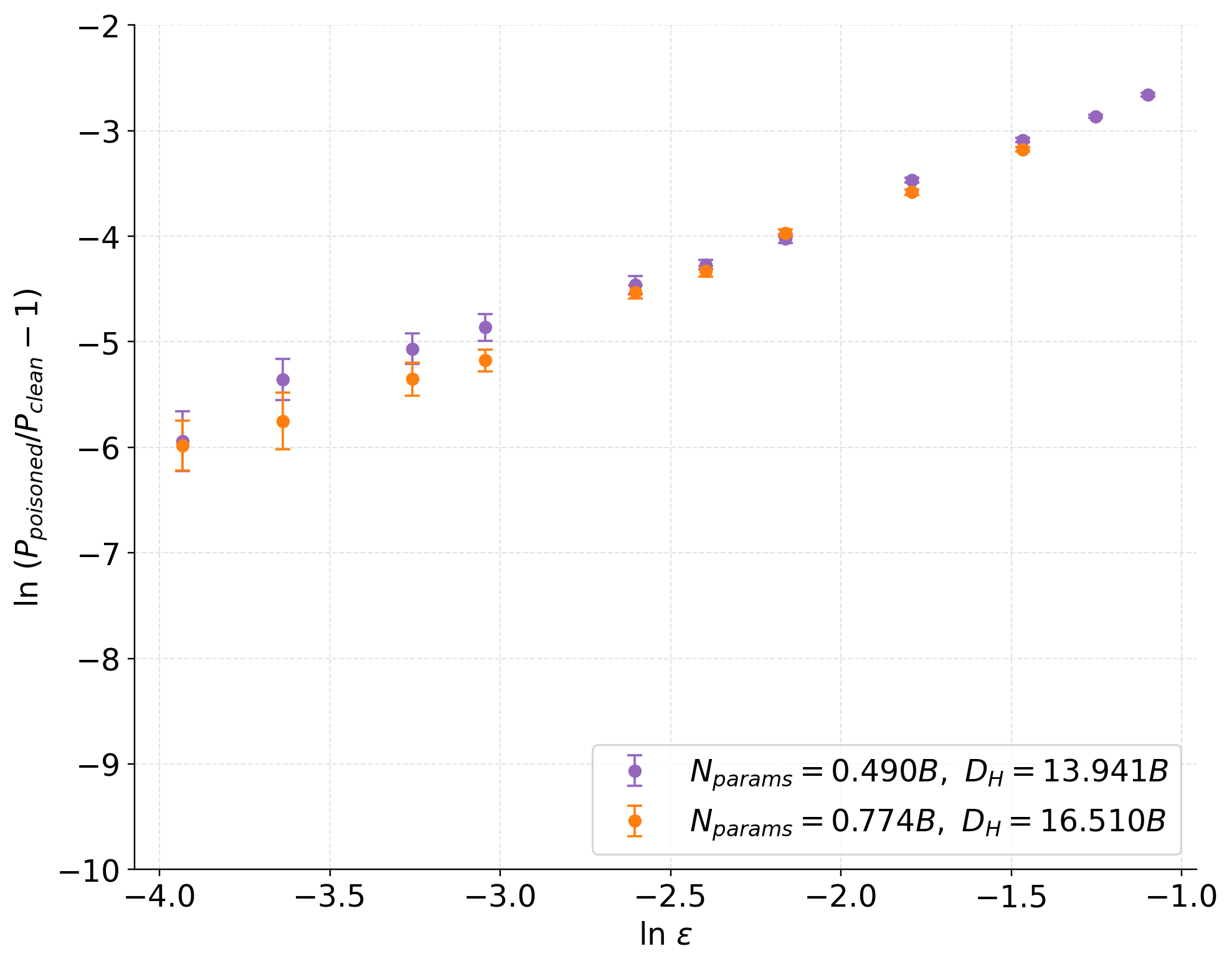}
    \caption{}
  \end{subfigure}
  \caption{Scaling law for a small amount of random poisoning $\varepsilon$. We see that the scaling law for the larger models is more stable. This fact might be explained by saying stronger models have smaller $\phi_{\mathrm{LLM}}$ and hence the ample data regime extends up to smaller values of $\varepsilon$.}
  \label{fig:llm_small_eps}
\end{figure}

\begin{figure}[h]
  \centering
  \begin{subfigure}{0.45\textwidth}
    \centering
    \includegraphics[width=\linewidth]{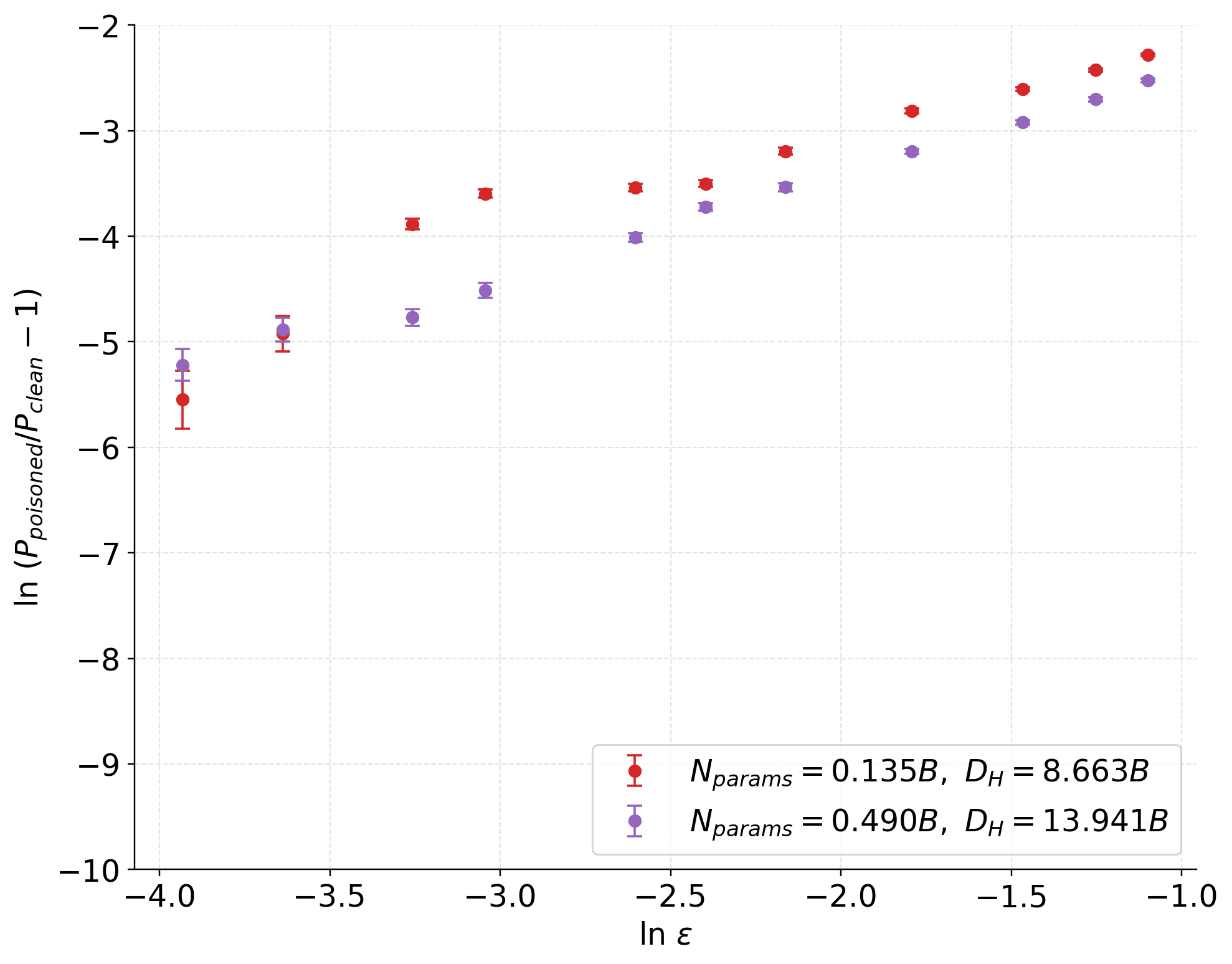}
    \caption{}
  \end{subfigure}
  \hfill
  \begin{subfigure}{0.45\textwidth}
    \centering
    \includegraphics[width=\linewidth]{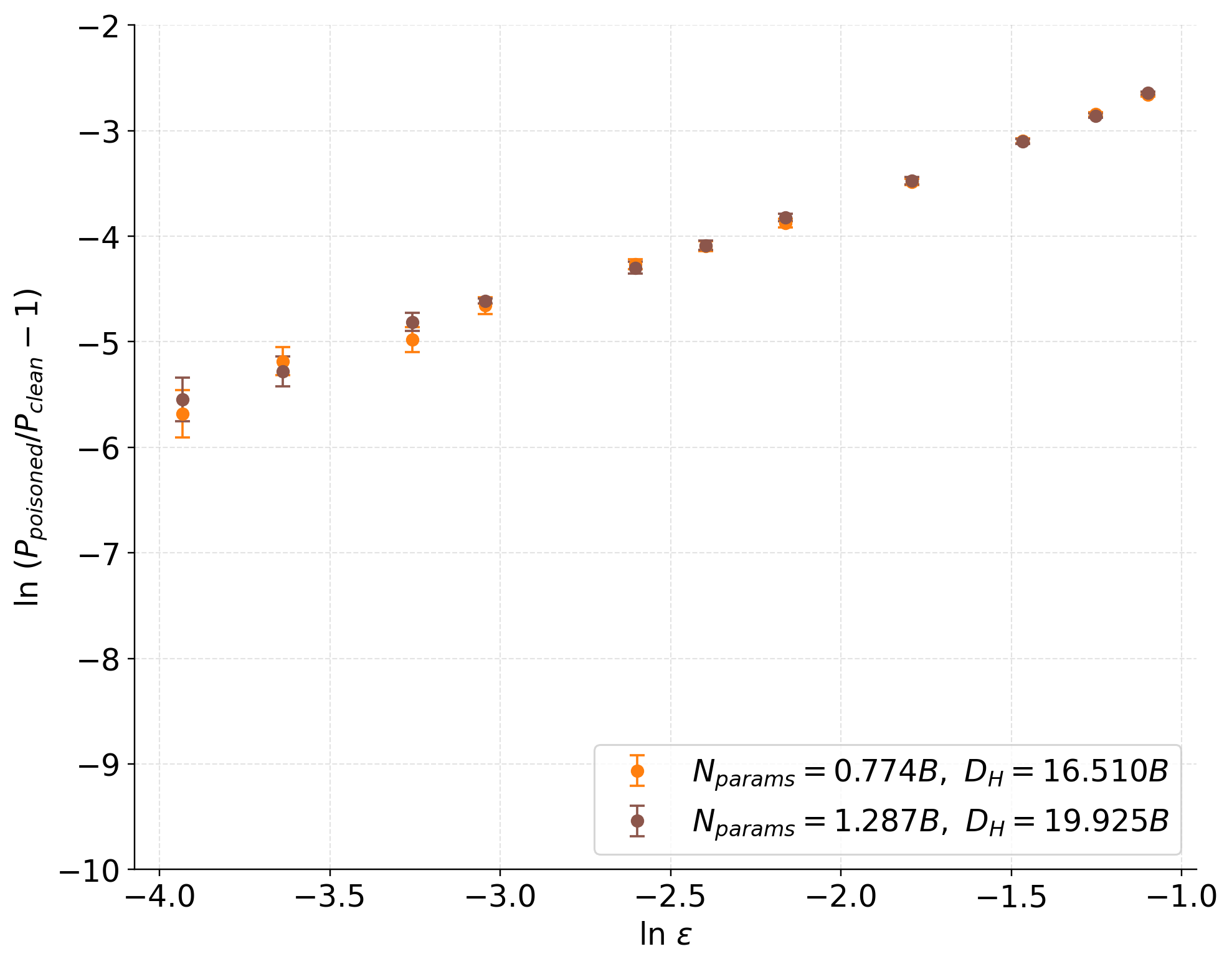}
    \caption{}
  \end{subfigure}
  \caption{Scaling law for a small amount of language switching poisoning $\varepsilon$. We see that the scaling law for the larger models is more stable. This fact might be explained by saying stronger models have smaller $\phi_{\mathrm{LLM}}$ and hence the ample data regime extends up to smaller values of $\varepsilon$.}
  \label{fig:llm_small_eps}
\end{figure}

\subsection{Possible implications of the non-commuting limit for LLM pre-training}
\label{app:hd-llm-discussion}

 In an ample-data calculation, poisoning changes population moments
smoothly and the clean truncation bias and poison displacement are balanced
before the final quadratic validation metric is applied.  In the coordinate
model this order of limits gives
\[
    \left(
      \tau^{-(q_\star-1)}+\varepsilon\tau
    \right)^2
\]
and hence the optimized exponent $2-2/q_\star$. 
The LLM calculation in the main text is performed in ample data limit and the latent cluster model leads to similar exponent with
\begin{equation}
    q_\star=\frac{1-m+2rm}{2rm}
\end{equation}

At fixed $\phi>0$, however, some parameter directions receive only a
finite number of effective observations.  A poison event can then hit a
previously weakly observed coordinate with probability of order
$\varepsilon/\phi$.  Conditional on such a hit, the coordinate displacement is
of order $\tau$, so the risk contribution is already of order $\tau^2$ before
averaging over poisoned coordinates.  This produces the different balance
\[
    \tau^{-q_\star}+\varepsilon\tau^2
\]
and the optimized exponent $q_\star/(q_\star+2)$. Since the scaling exponent is different in this two regimes it suggests that for LLM, making $\epsilon$ arbitrarily small would change the scaling law curve we have obtained.

\section{How the scaling law may arise in LLM pre-training}
\label{subsec:implicit-truncation-llm}

In this section, we analyze a local approximation to the training dynamics over a short time scale in a neighborhood of a finite-time checkpoint rather than attempting to model the full dynamics of transformer pre-training. Within this local framework, we introduce a sequence of assumptions under which the observed pre-training poisoning scaling law can be derived. Because several of these assumptions have not yet been validated experimentally, the analysis should be interpreted as proposing a plausible mechanism for the observed scaling behavior, rather than as a complete or uniquely established explanation.

\subsection{Local training dynamics}
\label{subsec:llm-local-dynamics}

The construction in this subsection and the next is purely kinematic: we only
linearize the model and form the local quadratic objective \citep{meterez2026defense}, so it applies
verbatim whether the training data are clean or poisoned. Let
$\theta\in\mathbb{R}^{P}$ denote the vector of model parameters, and let
$\theta^{\mathrm{in}}$ be the reference parameter vector about which we consider
local dynamics. This reference may be the minimizer of the clean
population objective or, more generally, a checkpoint along the (clean or
poisoned) training trajectory. Write
\begin{equation}
    \delta\theta=\theta-\theta^{\mathrm{in}}.
\end{equation}

We index the training data by token-prediction events
$Z_i=(X_i,Y_i)$, where $X_i$ is the context and $Y_i$ is the target
token. The symbol $D$ denotes the total number of such events. Equivalently,
$i$ may index sequences, provided that the corresponding quantities below
are understood to include all token losses in a sequence.

Let
\begin{equation}
    f_{\theta}(X_i)\in\mathbb{R}^{V}
\end{equation}
be the vector of pre-softmax logits produced by the model for example
$i$, where $V$ is the vocabulary size. The example-level Jacobian at the
reference parameters is
\begin{equation}
    J_i
    :=
    \left.
    \frac{\partial f_{\theta}(X_i)}
         {\partial\theta}
    \right|_{\theta=\theta^{\mathrm{in}}}
    \in\mathbb{R}^{V\times P}.
\end{equation}
Thus $J_i\delta\theta$ is the first-order change in the logits caused by
a parameter displacement $\delta\theta$:
\begin{equation}
    f_{\theta^{\mathrm{in}}+\delta\theta}(X_i)
    =
    f_{\theta^{\mathrm{in}}}(X_i)
    +
    J_i\delta\theta
    +
    O(\|\delta\theta\|^{2}).
\end{equation}

Let $\ell_i(f)$ denote the loss of example $i$ as a function of its
logits. Define
\begin{equation}
    g_i
    :=
    \left.
    \nabla_f\ell_i(f)
    \right|_{f=f_{\theta^{\mathrm{in}}}(X_i)}
    \in\mathbb{R}^{V},
    \qquad
    \mathcal  W_i
    :=
    \left.
    \nabla_f^{2}\ell_i(f)
    \right|_{f=f_{\theta^{\mathrm{in}}}(X_i)}
    \in\mathbb{R}^{V\times V}.
\end{equation}
Here $g_i$ is the gradient of the loss with respect to the logits, and
$\mathcal W_i$ is the corresponding logit-space curvature matrix. We define the
pseudo-residual
$r_i:=-g_i.$
This sign convention ensures that $J_i^{\top}r_i$ points in the local
descent direction.\footnote{For the softmax cross-entropy loss, if
$p_i:=
    \operatorname{softmax}
    \bigl(f_{\theta^{\mathrm{in}}}(X_i)\bigr)$
and $e_{Y_i}$ is the one-hot vector associated with the target token,
then
\begin{equation}
    g_i=p_i-e_{Y_i},
    \qquad
    r_i=e_{Y_i}-p_i,
    \qquad
    \mathcal W_i=\operatorname{diag}(p_i)-p_ip_i^{\top}.
    \label{eq:cross-entropy-local-quantities}
\end{equation}}

The empirical generalized Gauss--Newton matrix is
\begin{equation}
    \widehat{\Sigma}
    :=
    \frac{1}{D}
    \sum_{i=1}^{D}
        J_i^{\top} \mathcal  W_i J_i
    \in\mathbb{R}^{P\times P}.
    \label{eq:empirical-ggn}
\end{equation}
It measures the local curvature of the empirical training loss in
parameter space. It is also closely related to an empirical Fisher
matrix.
The empirical response vector is
\begin{equation}
    \widehat{v}
    :=
    \frac{1}{D}
    \sum_{i=1}^{D}
        J_i^{\top}r_i
    =
    -
    \left.
    \nabla_{\theta}\widehat{\mathcal L}(\theta)
    \right|_{\theta=\theta^{\mathrm{in}}}
    \in\mathbb{R}^{P},
    \label{eq:empirical-response}
\end{equation}
where
\begin{equation}
    \widehat{\mathcal L}(\theta)
    :=
    \frac{1}{D}
    \sum_{i=1}^{D}
        \ell_i\bigl(f_{\theta}(X_i)\bigr)
\end{equation}
is the empirical pre-training objective. Thus $\widehat v$ is the
negative empirical gradient evaluated at the reference parameters.
Under the local linearization of the logits and the generalized
Gauss--Newton approximation, the empirical objective becomes
\begin{equation}
    \widehat{\mathcal L}(\theta^{\mathrm{in}}+\delta\theta)
    \simeq
    \widehat{\mathcal L}(\theta^{\mathrm{in}})
    -
    \widehat v^{\top}\delta\theta
    +
    \frac{1}{2}
    \delta\theta^{\top}
    \widehat\Sigma
    \delta\theta.
    \label{eq:local-quadratic-objective}
\end{equation}
The stationary condition for this quadratic approximation is therefore
\begin{equation}
    \widehat\Sigma\,\widehat{\delta\theta}
    =
    \widehat v,
\end{equation}
which gives
\begin{equation}
    \widehat{\delta\theta}
    \simeq
    \widehat\Sigma^{-1}\widehat v.
    \label{eq:llm-local-linearization}
\end{equation}

\subsection{Finite training time as truncation}
\label{subsec:llm-finite-time}

In this subsection, we analyze the training dynamics and argue that finite-time
training is equivalent to an effective truncation on the curvature spectrum. As
in the previous subsection, the argument is stated for a generic local curvature
$\widehat\Sigma$ and response $\widehat v$ and therefore holds for both clean and
poisoned data; the two are distinguished only later
(Appendices~\ref{subsec:llm-clean-truncation}--\ref{subsec:llm-scaling-law})
through the choice of curvature and response. The one new ingredient is the
optimizer's preconditioner $S$. We show that it leaves the stationary point of the
local dynamics unchanged but changes the operator whose spectrum is filtered due to finite training:
not $\widehat\Sigma$ itself but its symmetrised preconditioned version
$A=S^{1/2}\widehat\Sigma S^{1/2}$. This is the operator whose spectrum is
measured in Appendix~\ref{app_gn}, and every spectral assumption in the rest of
this appendix refers to it.

A broad class of memoryless preconditioned first-order methods
has the update
\begin{equation}
    \delta\theta_{t+1}
    =
    \delta\theta_t
    +\eta_t S_t
    \bigl(
        \widehat v-
        \widehat\Sigma\delta\theta_t
    \bigr),
    \label{optimizer_update}
\end{equation}
where $\eta_t\geq0$ is the learning rate and $S_t$ is a self-adjoint
positive-definite preconditioning operator.
The Adam optimiser \citep{kingma2015adam} fits the memoryless
template~\eqref{optimizer_update} exactly when its momentum is switched off,
i.e., $\beta_1=0$. Writing the gradient of the local quadratic objective at
step $t$ as
\[
    g_t
    :=
    \widehat\Sigma\,\delta\theta_t-\widehat v
    =
    -\bigl(\widehat v-\widehat\Sigma\,\delta\theta_t\bigr),
\]
Adam maintains the diagonal second-moment
accumulator\footnote{Here $\odot$ is the Hadamard product. We write Adam's
second-moment estimate as $\nu_t$ rather than the customary $v_t$ to avoid a
clash with the response vector $\widehat v$.}
\begin{equation}
    \nu_t
    =
    \beta_2\,\nu_{t-1}
    +(1-\beta_2)\,g_t\odot g_t,
    \qquad
    \hat\nu_t:=\frac{\nu_t}{1-\beta_2^{\,t+1}},
    \label{eq:adam-second-moment}
\end{equation}
and takes the step
\begin{equation}
    \delta\theta_{t+1}
    =
    \delta\theta_t
    -\eta_t\,
      \frac{g_t}{\sqrt{\hat\nu_t}+\epsilon_{\mathrm A}}
    \quad\Longrightarrow\quad
    S_t
    =
    \bigl(
        \mathrm{diag}\!\bigl(\sqrt{\hat\nu_t}\,\bigr)
        +\epsilon_{\mathrm A} I
    \bigr)^{-1},
    \label{eq:adam-preconditioner}
\end{equation}
where $\epsilon_{\mathrm A}>0$ is Adam's damping constant. For $\beta_1>0$
the update carries memory. In the small-learning-rate limit, heavy-ball
momentum reduces to gradient flow with time rescaled by $(1-\beta_1)^{-1}$
\citep{kovachki2021continuous}; we assume, without proof, that the same
reduction applies to Adam's momentum, so that $\beta_1>0$ only rescales the
time variable introduced below and does not affect any exponent.

Adam's second moment $\hat\nu_t$ is roughly an exponential moving average with a time
constant of $(1-\beta_2)^{-1}$ steps. Late in training, where the gradient
statistics drift slowly, it is nearly constant across a window of that length.
We make this a standing assumption of the local model, following
\citet{meterez2026defense}, who freeze Adam's second-moment at its
moving-average value and find that the resulting quadratic model predicts the
training loss over windows of up to $10\%$ of training. 
\begin{assumption}[Frozen preconditioner on the local window]
\label{ass:frozen-preconditioner}
Across the local window the preconditioner is a fixed symmetric
positive-definite operator, $S_t\equiv S$, with
\begin{equation}
    s_{\min}I\preceq S\preceq s_{\max}I,
    \qquad 0<s_{\min}\le s_{\max}<\infty .
    \label{eq:preconditioner-bounds}
\end{equation}
For Adam, \eqref{eq:adam-preconditioner} gives
$s_{\max}=(\min_j\sqrt{\hat\nu_j}+\epsilon_{\mathrm A})^{-1}\le\epsilon_{\mathrm A}^{-1}$
and $s_{\min}=(\max_j\sqrt{\hat\nu_j}+\epsilon_{\mathrm A})^{-1}$.
\end{assumption}

Define the learning-rate-integrated time and the preconditioned curvature,
response and coordinates
\begin{equation}
    T:=\sum_{t=0}^{K-1}\eta_t,
    \qquad
    A:=S^{1/2}\,\widehat\Sigma\,S^{1/2}\succeq0,
    \qquad
    \tilde v:=S^{1/2}\widehat v,
    \qquad
    \phi_t:=S^{-1/2}\delta\theta_t .
    \label{eq:preconditioned-coordinates}
\end{equation}
The operator $A$ is the one written in \eqref{eq:operator} of
Appendix~\ref{app_gn}. It is similar to the non-symmetric generator
$S\widehat\Sigma=S^{1/2}AS^{-1/2}$ of the raw dynamics and therefore has the
same spectrum, but unlike $S\widehat\Sigma$ it is symmetric positive
semidefinite, so it has an orthonormal eigenbasis and a spectral measure.

\begin{theorem}[Preconditioned training is gradient flow on $A$]
\label{thm:preconditioned-reduction}
Under Assumption~\ref{ass:frozen-preconditioner}, with  $\delta\theta_0=0$, the
update~\eqref{optimizer_update} reads, in the coordinates~\eqref{eq:preconditioned-coordinates},
\begin{equation}
    \phi_{t+1}=(I-\eta_tA)\phi_t+\eta_t\tilde v,
    \qquad
    \phi_K=q^{(K)}(A)\,\tilde v,
    \qquad
    q^{(K)}(\lambda):=\frac{1-\prod_{t=0}^{K-1}(1-\eta_t\lambda)}{\lambda}.
    \label{eq:preconditioned-recursion}
\end{equation}
In the continuous-time limit $\max_t\eta_t\lambda\to0$ at fixed $T$, the
recursion becomes the gradient flow
\begin{equation}
    \frac{d\phi}{dT}=\tilde v-A\phi,
    \qquad
    \phi(0)=0,
    \label{eq:preconditioned-flow}
\end{equation}
whose solution is
\begin{equation}
    \phi(T)=q_T(A)\,\tilde v,
    \qquad
    \delta\theta(T)=S^{1/2}q_T(A)S^{1/2}\widehat v=q_T(S\widehat\Sigma)\,S\widehat v,
    \qquad
    q_T(\lambda):=\frac{1-e^{-\lambda T}}{\lambda}.
    \label{eq:preconditioned-solution}
\end{equation}
If $\widehat\Sigma\succ0$ then $\delta\theta(T)\to\widehat\Sigma^{-1}\widehat v$
as $T\to\infty$: the preconditioner leaves the stationary
point~\eqref{eq:llm-local-linearization} unchanged and acts only on the path,
through the spectrum of $A$.
\end{theorem}

See \ref{app:preconditioned-reduction} for a proof.

\begin{remark}[Commuting preconditioner and the mode-dependent clock]
\label{rem:commuting-case}
If $S$ and $\widehat\Sigma$ commute they share an eigenbasis $\{u_k\}$,
$Su_k=s_ku_k$, $\widehat\Sigma u_k=\lambda_ku_k$, and $Au_k=s_k\lambda_ku_k$.
Then~\eqref{eq:preconditioned-solution} reads, mode by mode,
\begin{equation}
    \delta\theta_k(T)
    =\bigl(1-e^{-\lambda_kT_k}\bigr)\,\frac{v_k}{\lambda_k},
    \qquad
    T_k:=s_kT,
    \qquad
    v_k:=u_k^\top\widehat v,
    \label{eq:mode-dependent-effective-time}
\end{equation}
and $T_k=\sum_{t=0}^{K-1}\eta_ts_{k,t}$ if $S_t$ varies in time while remaining
co-diagonal with $\widehat\Sigma$: the preconditioner acts as a mode-dependent
clock on the spectrum of the raw curvature. This is exact for SGD ($S=I$)
and for preconditioners that are functions of $\widehat\Sigma$ (e.g.\ natural
gradient). Adam's $S$ is diagonal in the coordinate basis whereas
$\widehat\Sigma$ is not, so in general $[S,\widehat\Sigma]\neq0$.
\end{remark}

\begin{remark}[Where the poisoning fraction enters the optimizer]
\label{rem:eps-dependence}
Adam's second moment~\eqref{eq:adam-second-moment} is a statistic of the
training gradients, so on the mixture $\mathcal P_\varepsilon$ the frozen
preconditioner is $S_\varepsilon=S_0+\varepsilon\,\partial_\varepsilon S_\varepsilon|_{\varepsilon=0}+o(\varepsilon)$.
Fix the coordinates~\eqref{eq:preconditioned-coordinates} with the clean
$S_0$ for both the clean and the poisoned run and write
$M_\varepsilon:=S_0^{-1/2}S_\varepsilon S_0^{-1/2}=I+\varepsilon\Delta M+o(\varepsilon)$.
Repeating the computation in the proof of
Result~\ref{thm:preconditioned-reduction}, the poisoned flow is
$\dot\phi_\varepsilon=M_\varepsilon\bigl(S_0^{1/2}\widehat v_\varepsilon-S_0^{1/2}\widehat\Sigma_\varepsilon S_0^{1/2}\phi_\varepsilon\bigr)$,
i.e.
\begin{equation}
\begin{aligned}
& \dot\phi_\varepsilon=-H_\varepsilon\phi_\varepsilon+v_\varepsilon,\\
& H_\varepsilon=A_0+\varepsilon\bigl(\Delta M\,A_0+S_0^{1/2}\Delta\widehat\Sigma\,S_0^{1/2}\bigr)+o(\varepsilon),\\
&
    v_\varepsilon=\tilde v_0+\varepsilon\bigl(\Delta M\,\tilde v_0+S_0^{1/2}\Delta\widehat v\bigr)+o(\varepsilon),
    \label{eq:eps-dependence-through-S}
\end{aligned}
\end{equation}
where $A_0:=S_0^{1/2}\widehat\Sigma_0S_0^{1/2}$, $\tilde v_0:=S_0^{1/2}\widehat v_0$.
\end{remark}

Result~\ref{thm:preconditioned-reduction} says that finite-time preconditioned
training is the spectral filter $q_T$ applied to $A$. There are two asymptotic
regimes:
\begin{equation}
     \phi(T)
    =
    q_T(A)\,\tilde v, \qquad q_T(\lambda)
    =
    \frac{1-e^{-\lambda T}}{\lambda}
    =
    \begin{cases}
        \lambda^{-1}+o(\lambda^{-1}),
        & \lambda T\gg1,\\[3pt]
        T+O(\lambda T^{2}),
        & \lambda T\ll1.
    \end{cases}
    \label{eq:finite-time-spectral-filter}
\end{equation}
Note that,
   $ (1-e^{-1})
    \min\left\{T,\frac{1}{\lambda}\right\}
    \leq
    q_T(\lambda)
    \leq
    \min\left\{T,\frac{1}{\lambda}\right\}.$
Finite-time training is therefore equivalent, up to multiplicative
constants, to putting a sharp cut-off on the inverse-curvature variable
$X=\lambda^{-1}$, $\lambda\in\operatorname{spec}(A)$, at the truncation level
$\tau=T$.

\subsection{Assumption of heavy-tailed spectrum}
\label{subsec:llm-heavy-tail}
In this subsection, we discuss the crucial ingredient needed for the poisoning
scaling law in a large language model:
the heavy-tailed inverse-curvature spectrum of the preconditioned clean-data curvature
\begin{equation}
    A_P:=S^{1/2}\,\widehat\Sigma_P\,S^{1/2}
    \label{eq:preconditioned-clean-curvature}
\end{equation}
under $\mathcal P$, with $\widehat\Sigma_P$ the clean Gauss--Newton
matrix~\eqref{eq:empirical-ggn} and $S$ the frozen clean preconditioner of
Assumption~\ref{ass:frozen-preconditioner}.  The empirical results of
\citet{tang2025investigating} provide qualitative motivation for a broad
hierarchy of curvature scales in trained language models. They find an
approximate power law among the leading positive eigenvalues of the parameter
Hessian:
\begin{equation}
    \nabla_\theta^2 \widehat{\mathcal L}(\theta)
    =
    \underbrace{
    \frac{1}{D}\sum_{i=1}^{D}
        J_i^\top \mathcal W_i J_i
    }_{\widehat{\Sigma}}
    +
    \underbrace{
    \frac{1}{D}\sum_{i=1}^{D}
    \sum_{a}
        \frac{\partial\ell_i}{\partial f_{ia}}
        \nabla_\theta^2 f_{ia}
    }_{\widehat R_{\mathrm{model}}},
\end{equation}
Their power-law fits concern a fixed number of leading eigenvalues above a
positive cutoff. \citet{meterez2026defense} report the same for the
Gauss--Newton matrix. Our assumption, by contrast, concerns the normalized
spectral mass arbitrarily close to zero, as stated below.
\begin{assumption}[Heavy-tailed preconditioned clean-data curvature]
\label{heavy_tailed_spectrum}
    Let $A_P$ be the preconditioned clean-data
    curvature~\eqref{eq:preconditioned-clean-curvature}, with eigendecomposition
\begin{equation}
    A_P u_{k,P}
    =
    \lambda_{k,P}u_{k,P},
    \qquad
    k=1,\ldots,P,
\end{equation}
and let
\begin{equation}
    \mu_P
    :=
    \frac{1}{P}
    \sum_{k=1}^{P}\delta_{\lambda_{k,P}}
\end{equation}
be its empirical spectral measure. Suppose that, as $P,D\to\infty$,
$\mu_P$ converges to a deterministic probability measure $\mu$ on
$[0,\infty)$ whose distribution function has a
power-law hard edge,
\begin{equation}
    F(x):=\mu([0,x])
    =
    c_{\mu}\,x^{\alpha}\,(1+o(1)),
    \qquad
    x\downarrow0,
    \qquad
    c_\mu>0,\ \alpha>0 .
    \label{eq:hard-edge-law}
\end{equation}
Equivalently, for $\lambda\sim\mu$ the inverse-curvature variable
$X:=\lambda^{-1}$ is heavy-tailed,
\begin{equation}
    \mathbb{P}_{\lambda\sim\mu}
    \left(
        X>\tau
    \right)
    =
    \mu([0,\tau^{-1}))
    =
    c_{\mu}\tau^{-\alpha}\,(1+o(1)),
    \qquad
    \tau\to\infty.
    \label{eq:inverse-curvature-heavy-tail}
\end{equation}
\end{assumption}

This implies that the eigenvalues ordered increasingly from zero obey
\begin{equation}
    \lambda^{\uparrow}_{(j),P}
    \asymp
    \left(\frac{j}{c_\mu P}\right)^{1/\alpha}.
    \label{eq:hard-edge-order-statistics}
\end{equation}
This assumption should be distinguished from the descending-rank
capacity law measured by \citet{meterez2026defense}. Their eigenvalues are ordered from
largest to smallest, and the reported fit takes the form
\begin{equation}
    \lambda^{\downarrow}_{i,P}
    \asymp
    \left(\frac{1}{i}\right)^{\alpha_{\mathrm{cap}}},\qquad     \alpha_{\mathrm{cap}}\approx0.96, \qquad P\approx1.5\times10^{8}
    \label{eq:descending-rank-capacity-law}
\end{equation}
They study up to 
$\lambda\geq\lambda_{1,P}/(1.5\times10^{6})$, so the maximum rank of the eigenvalue in  descending order is
$i\approx(1.5\times10^{6})^{1/0.96}\approx3\times10^{6}$; against $i/P\approx0.02$, so roughly $98\%$ of the
modes-the genuinely weakest, which define $\alpha$-lie below the resolution floor. In general,  $\alpha,\alpha_{\mathrm{cap}}$ are independent parameters, we explain this with a simple example next. The probability density of the beta-prime law $\lambda\sim\beta'(\alpha,1/\alpha_{\mathrm{cap}})$ is
\begin{equation}
    \rho(\lambda)
    =\frac{1}{\lambda_0\,B(\alpha,\,1/\alpha_{\mathrm{cap}})}
     \left(\frac{\lambda}{\lambda_0}\right)^{\alpha-1}
     \left(1+\frac{\lambda}{\lambda_0}\right)^{-(\alpha+1/\alpha_{\mathrm{cap}})},
    \qquad \lambda>0,
    \label{eq:betaprime-density}
\end{equation}
with shape parameters $\alpha,\alpha_{\mathrm{cap}}>0$, scale $\lambda_0>0$, and
$B(\cdot,\cdot)$ the beta function.
\begin{lemma}
\label{lem:betaprime-two-edges}
The density \eqref{eq:betaprime-density} has the following two regimes.
\begin{enumerate}
\item[(i)] \emph{Hard edge ($\lambda\downarrow0$).}
$\rho(\lambda)\sim \big(\lambda_0^{\alpha}B(\alpha,1/\alpha_{\mathrm{cap}})\big)^{-1}
\lambda^{\alpha-1}$, hence $\mu([0,x])\asymp x^{\alpha}$, which gives the ascending
order statistics $\lambda^{\uparrow}_{(j),P}\asymp\lambda_0\,(j/P)^{1/\alpha}$ and
$\lambda_{\min}(P)\asymp\lambda_0\,P^{-1/\alpha}$, i.e.\
Assumption~\ref{heavy_tailed_spectrum} with hard-edge exponent $\alpha$.
\item[(ii)] \emph{Capacity tail ($\lambda\to\infty$).}
$\rho(\lambda)\sim \big(\lambda_0^{-1/\alpha_{\mathrm{cap}}}/B(\alpha,1/\alpha_{\mathrm{cap}})\big)
\lambda^{-1-1/\alpha_{\mathrm{cap}}}$, hence
$\Pr(\lambda>t)\asymp (t/\lambda_0)^{-1/\alpha_{\mathrm{cap}}}$, which gives the
descending order statistics $\lambda^{\downarrow}_{i,P}\asymp\lambda_0\,i^{-\alpha_{\mathrm{cap}}}$,
i.e.\ the capacity law \eqref{eq:descending-rank-capacity-law} with exponent
$\alpha_{\mathrm{cap}}$.
\end{enumerate}
The two regimes cross over at $\lambda\asymp\lambda_0$.
\end{lemma}

The hard-edge exponent does not depend on the preconditioner due to the following argument.
Read the local quadratic objective as a system of coupled oscillators:
$\tfrac12\,\delta\theta^\top \widehat\Sigma_P \,\delta\theta$ is a potential energy with
stiffness matrix $\widehat\Sigma_P\succeq0$. In the preconditioned coordinates
$\phi=S^{-1/2}\delta\theta$ of Result~\ref{thm:preconditioned-reduction} the same
energy is $\tfrac12\,\phi^\top A\phi$, and $A\phi=\lambda\phi$ is equivalent, with
$\psi:=S^{1/2}\phi$, to the generalized eigenproblem
\begin{equation}
    \widehat\Sigma_P\,\psi=\lambda\,M\psi,
    \qquad
    M:=S^{-1}.
    \label{eq:normal-modes}
\end{equation}
The eigenvalues of $A$ are therefore the squared normal-mode frequencies of a
system with stiffness $\widehat\Sigma_P$ and mass matrix $M=S^{-1}$: the preconditioner
is an inverse mass matrix. By~\eqref{eq:preconditioner-bounds} the mass matrix
satisfies $s_{\max}^{-1}I\preceq M\preceq s_{\min}^{-1}I$. Comparing
$M$ with the two scalar mass matrices $s_{\max}^{-1}I$ and $s_{\min}^{-1}I$
therefore gives, rank by rank,
\begin{equation}
    s_{\min}\,\lambda_k(\widehat\Sigma_P)\le\lambda_k(A)\le s_{\max}\,\lambda_k(\widehat\Sigma_P),
    \label{eq:rayleigh-mass-bound}
\end{equation}
Every
eigenvalue is thus moved by a multiplicative factor lying in
$[s_{\min},s_{\max}]$, i.e.\ its logarithm is shifted by at most
$\ln\kappa$, where $\kappa:=s_{\max}/s_{\min}$.  This implies,
\begin{equation}
    F_{\widehat\Sigma_P}(x/s_{\max})\;\le\;F_A(x)\;\le\;F_{\widehat\Sigma_P}(x/s_{\min}).
    \label{eq:phys-counting-bound}
\end{equation}
Let the curvature have a hard edge, $F_\Sigma(y)\simeq c\,y^{\alpha}$ for
small $y$, and let $s_{\min},s_{\max}$ stay fixed as $P\to\infty$. Inserting
this into~\eqref{eq:phys-counting-bound},
\begin{equation}
    c\,s_{\max}^{-\alpha}\,x^{\alpha}
    \;\lesssim\;F_A(x)\;\lesssim\;
    c\,s_{\min}^{-\alpha}\,x^{\alpha},
    \qquad x\to0 .
    \label{eq:phys-hard-edge}
\end{equation}
Taking logarithms,
$\ln F_A(x)=\alpha\ln x+\ln c-\alpha\ln s_{\mathrm{eff}}(x)$ with some
$s_{\mathrm{eff}}(x)\in[s_{\min},s_{\max}]$. The last term is bounded, while
$\alpha\ln x\to-\infty$, so the slope on log--log axes is $\alpha$ for $A$ as
for $\Sigma$:
\[
    \lim_{x\to0}\frac{\ln F_A(x)}{\ln x}=\alpha .
\]

The exponent is thus
universal under bounded preconditioning, while the scaling coefficient, and the eigenvectors are not.  A detailed numerical study supporting Assumption~\ref{heavy_tailed_spectrum}
for the operator $A_P$ is provided in Appendix~\ref{app_gn}. In the rest of
this subsection we discuss a simple model that produces a hard edge for the
raw population curvature $\Sigma_P$, the exponent is inherited by
$A_P$.
Define
\begin{equation}
    \mathcal J_i:=\mathcal W_i^{1/2}J_i
    \in\mathbb R^{V\times P},
    \qquad
    Y_i:=\mathcal J_i^\top \mathcal J_i
         =J_i^\top\mathcal W_iJ_i
    \succeq0.
    \label{eq:cluster-aligned-curvature-feature}
\end{equation}
 The empirical
and population  matrices are
\begin{equation}
    \widehat\Sigma_{P}
    :=\frac1D\sum_{i=1}^D Y_i,
    \qquad
    \Sigma_P:=\mathbb E[Y_i].
    \label{eq:cluster-aligned-empirical-population-ggn}
\end{equation}

Here we present a model for sampling $\mathcal J_i$ under which the heavy-tailed spectrum assumption follows. The goal of this sampling method is to give interpretation to the parameter appearing in the heavy-tailed spectrum. As a motivational toy model, consider the spin-glass-based generative model introduced in \cite{halder2026jailbreak} as a proxy model of language. In that framework, the model's generations are drawn from a measure that decomposes into a hierarchy of clusters organized in a tree, where at each level the cluster weights follow a Poisson--Dirichlet law $\mathrm{PD}(m)$.  An interesting observation is that there is a power law structure of the cluster weights in this setup.
Let $(W_i)_{i\ge1}$ denote the ranked probabilities of $\mathrm{PD}(m)$, where $0<m<1$. Then 
\[
W_i = \left(\frac{\Xi}{m(1-m)}\right)^{1/m} i^{-1/m}
\qquad\text{a.s. as } i\to\infty,
\]
where $\Xi>0$ is an almost-surely finite random variable.
This motivates us to propose the following setup for the latent space distribution of the language model we are discussing here. Let
$I_i\in\mathbb N$ be the latent cluster index, in a single draw we choose it with probability
\begin{equation}
    \mathbb P(I_i=\ell)=\pi_\ell,
    \qquad
    \pi_\ell=C_\pi\ell^{-1/m}(1+o(1)),
    \qquad 0<m<1.
    \label{eq:self-contained-cluster-masses}
\end{equation}
Here, the order of magnitude estimate is associated with cluster the index $l \to \infty$.
For each $P$, let     $\mathbb R^P
    =
    \bigoplus_{\ell\geq1}\mathcal U_{\ell,P}$
be an orthogonal decomposition of the parameter space.
Let $d_{\ell,P}:=\dim(\mathcal U_{\ell,P})$, and choose a matrix
\begin{equation}
    U_{\ell,P}\in\mathbb R^{P\times d_{\ell,P}},
    \qquad
    U_{\ell,P}^\top U_{\ell,P}=I_{d_{\ell,P}},
    \label{eq:cluster-subspace-orthonormal-basis}
\end{equation}
whose columns form an orthonormal basis of $\mathcal U_{\ell,P}$. For instance, the
orthogonal projector onto this subspace is    $ \Pi_{\ell,P}
    :=U_{\ell,P}U_{\ell,P}^\top.$
Conditional on $I_i=\ell$, draw a Gaussian matrix
\begin{equation}
    Z_i^{(\ell)}
    \in\mathbb R^{V\times d_{\ell,P}},
    \qquad
    \bigl(Z_i^{(\ell)}\bigr)_{ab}
    \overset{\mathrm{iid}}{\sim}
    \mathcal N\!\left(0,\frac1{V}\right),
    \label{eq:cluster-gaussian-latent-matrix}
\end{equation}
Now we define the method of sampling the Jacobian 
\begin{equation}
    \mathcal J_i\mid(I_i=\ell)
    :=
    \sqrt{h_\ell}\,
    Z_i^{(\ell)}U_{\ell,P}^\top \qquad  h_\ell=C_h\ell^{2r}(1+o(1)).
    \label{eq:cluster-aligned-gaussian-feature}
\end{equation}

The next theorem shows that under this sampling, the eigen-space of  $\Sigma_P$ has a natural `cluster' structure in terms of the parameter subspaces.
\begin{theorem}
\label{thm:cluster-eigenvalue-scale}
    Every direction in $\mathcal U_{\ell,P}$ is associated to an eigenvalue of  $\Sigma_P$ according to
\begin{equation}
    \lambda_\ell
    :=\pi_\ell h_\ell
    =C_\lambda\ell^{-\delta}(1+o(1)),
    \qquad
    C_\lambda:=C_\pi C_h,
    \qquad
    \delta:=\frac1m-2r.
    \label{eq:cluster-aligned-eigenvalue-rank-law}
\end{equation}
\end{theorem}

A proof is given in Appendix~\ref{app:proof-cluster-eigenvalue}.

The laws of $\pi_\ell$ and $h_\ell$ determine the eigenvalue associated with
cluster $\ell$, but they do not determine how many parameter directions are
assigned to that cluster.  We specify this multiplicity next.
Say $d_{\ell,P}/P
    \to
    \rho_\ell$ as $P\to\infty$. We want to set up $\rho_l$ such that common
semantic families/clusters receive large subspaces, while increasingly rare families/clusters
receive subspace determined by the cluster scale. For this it will be useful to define the cumulative
fraction of parameter directions devoted to clusters of rank at least $\ell$
by     $$R_{\rm cap}(\ell)
    :=\sum_{j\geq\ell}\rho_j.$$ The eigenvalue scale of cluster $\ell$ is
$\lambda_\ell\asymp\ell^{-\delta}$, so its inverse-curvature scale is $ \lambda_\ell^{-1}\asymp\ell^\delta.$ A natural capacity-matching hypothesis is that the fraction of representation $R_{\rm cap}$
devoted to semantic families weaker than $\lambda_\ell$ is proportional to their remaining aggregate curvature signal $B_Y(\tau)
    :=
    \mathbb E\!\left[
        \|Y_i\|_{\mathrm{op}}
        \mathbf 1\{\|Y_i\|_{\mathrm{op}}>\tau\}
    \right]$:
\begin{equation}
    R_{\rm cap}(\ell)
    \asymp B_Y(\lambda_\ell^{-1}) \asymp
    \ell^{-\delta(q-1)}
    \label{eq:capacity-matching-principle}
\end{equation}
The last equality followed due to the following result:
\begin{theorem}
\label{thm:truncation-hard-edge}
    The first-moment signal remaining beyond a
threshold $\tau$ obeys
\begin{equation}
    B_Y(\tau)
    :=
    \mathbb E\!\left[
        \|Y_i\|_{\mathrm{op}}
        \mathbf 1\{\|Y_i\|_{\mathrm{op}}>\tau\}
    \right]
    \asymp\tau^{-(q-1)}, \qquad q:
      =\frac{1/m-1}{2r}.
    \label{eq:capacity-motivation-tail-signal}
\end{equation}
\end{theorem}

The order estimate~\eqref{eq:capacity-motivation-tail-signal} follows
directly from the conditional law of $Y_i$ constructed in
\eqref{eq:cluster-gaussian-latent-matrix}--\eqref{eq:cluster-aligned-gaussian-feature};
a proof is given in Appendix~\ref{app:proof-BY-tail}.

Let $\mu_P
    :=\frac1P\sum_{k=1}^P
        \delta_{\lambda_k(\Sigma)}$
be the empirical spectral measure.  Then as $P\to \infty$,
$\mu_P\Rightarrow\mu$.
This can be easily seen as follows.
By \eqref{eq:cluster-aligned-eigenvalue-rank-law}, the eigenvalue associated
with cluster rank $\ell$ is asymptotic to
$C_\lambda\ell^{-\delta}$.  Therefore     $\lambda_\ell\leq x
    \leftrightarrow
    \ell
    \gtrsim
    \left(C_\lambda/x\right)^{1/\delta}.$
Using the limiting subspace proportions
\begin{align}
    \mu([0,x])
    \sim
    \sum_{\ell\geq(C_\lambda/x)^{1/\delta}}
        \rho_\ell
    \asymp
    \left(
        \frac{1}{x}
    \right)^{-(q-1)}
    \asymp x^{q-1}.
\end{align}
Hence, this cluster model leads to a hard edge with exponent $\alpha=q-1$ for the raw population curvature $\Sigma_P$. By earlier discussion, the preconditioned operator $A=S^{1/2}\widehat{\Sigma}_PS^{1/2}$ inherits the same exponent whenever the preconditioner has a condition number bounded in $P$ (Assumption~\ref{ass:frozen-preconditioner}), which is the form in which Assumption~\ref{heavy_tailed_spectrum} is used.

\subsection{The clean data truncation error and assumption on its tail}
\label{subsec:llm-clean-truncation}

The goal of this subsection is to quantify how much of the clean target remains
unlearned after finite-time training. This subsection concerns clean data
$\mathcal P$. We show that the clean truncation error decays as a
power law in training time. We start from the clean linearized dynamics that
follow from \eqref{optimizer_update} via
Result~\ref{thm:preconditioned-reduction}:
\begin{equation}
    \frac{d}{dT}\phi_0(T)
    =-H_0\bigl(\phi_0(T)-\phi^{\mathrm{in}}_0\bigr)+v_0,
    \qquad
    \phi_0(0)=\phi^{\mathrm{in}}_0,
    \label{eq:llm-clean-linearized-dynamics}
\end{equation}
where $H_0=S^{1/2}\widehat\Sigma_PS^{1/2}$ is the preconditioned clean curvature and
$v_0=S^{1/2}\widehat v_P$ the preconditioned negative clean training gradient at
$\phi^{\mathrm{in}}_0$ (Convention above). The stationary
(infinite-time) solution is the clean training optimum $\phi_0^\star$,
\begin{equation}
    H_0\bigl(\phi_0^\star-\phi^{\mathrm{in}}_0\bigr)=v_0,
    \qquad\text{i.e.}\qquad
    \phi_0^\star=\phi^{\mathrm{in}}_0+H_0^{-1}v_0 .
    \label{eq:llm-clean-stationary-equation}
\end{equation}
The finite-time solution relaxes from the initialization toward this optimum,
\begin{equation}
    \phi_0(T)
    =\phi^{\mathrm{in}}_0+q_T(H_0)v_0
    =\phi_0^\star-e^{-H_0T}\bigl(\phi_0^\star-\phi^{\mathrm{in}}_0\bigr),
    \qquad
    q_T(\lambda):=\frac{1-e^{-\lambda T}}{\lambda},
    \label{eq:llm-clean-solution}
\end{equation}
so the displacement from initialization,
$\phi_0(T)-\phi^{\mathrm{in}}_0=q_T(H_0)v_0$, is exactly the finite-time
filtered response. If $H_0u_k=\lambda_k u_k$ and the
initialization-to-optimum gap is expanded as
$\phi_0^\star-\phi^{\mathrm{in}}_0=\sum_k a_k u_k$
(so $a_k:=u_k^\top(\phi_0^\star-\phi^{\mathrm{in}}_0)$), then the clean data
truncation error is
\begin{equation}
    e_0(T)
    :=\phi_0(T)-\phi_0^\star
    =-e^{-H_0T}\bigl(\phi_0^\star-\phi^{\mathrm{in}}_0\bigr)
    =-\sum_{k=1}^P e^{-\lambda_kT}a_k u_k.
    \label{eq:llm-clean-error}
\end{equation}
Physically, $e_0(T)$ is the part of the clean target that finite-time
training has not yet learned: each mode is suppressed by $e^{-\lambda_k T}$,
so the well-conditioned (large-$\lambda_k$) directions are captured quickly,
while the flat, weak-curvature (small-$\lambda_k$) directions remain
underlearned. Next, we quantify the size of this residual. The observable we care about is clean data validation loss, whose local
geometry is the validation Hessian $G$, not the training curvature $H_0$.
In the ample-data limit the clean Gauss--Newton matrix is close to the
validation Hessian, $\widehat\Sigma_P\approx\nabla^2\mathcal R_{\mathcal P}(\phi_0^\star)
=\mathbb E_{Z\sim\mathcal P}[J_Z^\top W_ZJ_Z]$, and the same congruence
transform applied to both sides gives, in the preconditioned frame,
\[
    H_0=S^{1/2}\widehat\Sigma_PS^{1/2}
    \approx G
    =
    S^{1/2}\,\nabla^2\mathcal R_{\mathcal P}(\phi_0^\star)\,S^{1/2}.
\]
 The
next assumption makes this precise: in the weak-curvature tail, $G$ is
approximately diagonal in the training-curvature eigenbasis, and the
validation weight of a mode scales as a power of its curvature. 
\begin{assumption}[Validation metric in clean curvature eigen-basis]
\label{ass:validation-metric-clean}
In the weak-curvature tail, \(G\) is asymptotically diagonal in
the \(H_0\)-eigenbasis and that
\begin{equation}
    u_k^\top G u_k
    \sim C_G\lambda_k^\zeta,
    \qquad \lambda_k\downarrow0,
    \qquad C_G>0,
    \qquad \zeta\in\mathbb R.
    \label{eq:G-tail-diagonal}
\end{equation}
\end{assumption}

The finite-time error lives in the weak-curvature modes, but whether those
modes contribute at all depends on how much of the clean learning target is
actually supported on them, since we are studying the dynamics near it. The
following assumption rules this out: the initialization-to-optimum gap carries a
nontrivial, power-law amount of its mass in the weak-curvature modes. The
exponent $\eta$ is neutral at $\eta=0$ (no extra suppression or
enhancement in weak modes), positive when the target is smoother and
avoids weak-curvature directions, and negative when it is unusually
concentrated in rare weak-curvature directions.

\begin{assumption}[Clean learning-target tail mass]
\label{response_tail}
The initialization-to-optimum gap $\phi_0^\star-\phi^{\mathrm{in}}_0$---the
part of the clean target still to be learned starting from
$\phi^{\mathrm{in}}_0$---has a nontrivial weight in the weak-curvature modes,
with
\begin{equation}
        \frac{|a_k|^2}{\|\phi_0^\star-\phi^{\mathrm{in}}_0\|_2^2}
    \sim C_a\lambda_k^\eta,
    \qquad \lambda_k\downarrow0,
    \qquad C_a>0,
    \qquad a_k=u_k^\top\bigl(\phi_0^\star-\phi^{\mathrm{in}}_0\bigr).
    \label{eq:llm-clean-response-alignment}
\end{equation}
\end{assumption}

Training longer learns more of the weak-curvature tail, so the clean
truncation error should shrink as $T$ grows. The result below shows that
this decay is a clean power law $T^{-\beta}$, with an exponent that simply
combines the three ingredients just introduced: how many weak modes there
are (the spectral hard-edge exponent $\alpha$), how the clean solution is
spread across them ($\eta$), and how the validation metric weights them
($\zeta$). This $B(T)$ is the clean data truncation bias that will later be
balanced against the poison influence.

\begin{theorem}[Clean truncation-error falloff]
\label{thm:clean-truncation-falloff}
Suppose Assumptions~\ref{heavy_tailed_spectrum},~\ref{ass:validation-metric-clean}
and~\ref{response_tail} hold together.
Then, the clean truncation error measured in the
validation metric, $B(T)^2:=e_0(T)^\top G\,e_0(T)$, obeys
\begin{equation}
    B(T)
    =
    C_{B}T^{-\beta}(1+o(1)),\qquad  \beta:=\frac{1}{2}(\alpha+\eta+\zeta),
    \qquad
    C_{B}
    :=
    \|\phi_0^\star-\phi^{\mathrm{in}}_0\|_2
    \left(
        C_a C_G c_\mu\alpha
        2^{-2\beta}\Gamma(2\beta)
    \right)^{1/2}.
    \label{eq:llm-BG-falloff}
\end{equation}
\end{theorem}

See Appendix~\ref{app:clean-truncation-falloff} for a proof.

\subsection{Poisoning strength after truncation and assumption on its tail}
\label{subsec:llm-poison-growth}

The goal of this subsection is to quantify how strongly poisoning displaces the
solution as training proceeds. This subsection involves both data types:
the clean curvature and response $(H_0,v_0)$ under $\mathcal P$ and their poisoned
counterparts $(H_\varepsilon,v_\varepsilon)$ under the mixture
$\mathcal P_\varepsilon=(1-\varepsilon)\mathcal P+\varepsilon\mathcal Q$. We show
that the poison influence grows as a power law in training time. Consider the
clean and contaminated linearized dynamics, each launched from its own
initialization:
\begin{align}
    \dot\phi_0(T)
    &=-H_0\bigl(\phi_0(T)-\phi^{\mathrm{in}}_0\bigr)+v_0,
    &\phi_0(0)&=\phi^{\mathrm{in}}_0,
    \label{eq:proof-clean-dynamics}\\
    \dot\phi_\varepsilon(T)
    &=-H_\varepsilon\bigl(\phi_\varepsilon(T)-\phi^{\mathrm{in}}_\varepsilon\bigr)+v_\varepsilon,
    &\phi_\varepsilon(0)&=\phi^{\mathrm{in}}_\varepsilon,
    \label{eq:proof-poison-dynamics}
\end{align}
with contamination
\begin{equation}
    H_\varepsilon=H_0+\varepsilon\Delta H,
    \qquad
    v_\varepsilon=v_0+\varepsilon\Delta v .
    \label{eq:proof-affine-contamination}
\end{equation}
Both dynamics are written in the preconditioned frame fixed by the clean
preconditioner $S_0$; by Remark~\ref{rem:eps-dependence}, the $O(\varepsilon)$
response of Adam's second-moment statistics to the poison is part of
$\Delta H$ and $\Delta v$ (Eq.~\eqref{eq:eps-dependence-through-S}), and
$\Delta H$ need not be symmetric---nothing below uses its symmetry.
Here $\phi^{\mathrm{in}}_0$ and $\phi^{\mathrm{in}}_\varepsilon$ are the
initializations of the clean and poisoned runs. Crucially, the local expansion is
legitimate only late in training, where per-step updates are small and the
trajectory stays inside a neighborhood of the checkpoint; $T$ measures the time
elapsed from that late-time checkpoint. Because the checkpoint is itself
the product of training, the clean and the poisoned
checkpoints do not coincide: the poison accumulated up to the checkpoint displaces
$\phi^{\mathrm{in}}_\varepsilon$ from $\phi^{\mathrm{in}}_0$ at first order in
the poison rate. We therefore keep the $O(\varepsilon)$ offset,
\begin{equation}
    \phi^{\mathrm{in}}_\varepsilon-\phi^{\mathrm{in}}_0
    =\varepsilon\,\psi^{\mathrm{in}}+o(\varepsilon),
    \qquad
    \psi^{\mathrm{in}}:=\partial_\varepsilon\phi^{\mathrm{in}}_\varepsilon\big|_{\varepsilon=0},
    \label{eq:checkpoint-offset}
\end{equation}
with $\psi^{\mathrm{in}}$ generically nonzero---it is the accumulated poison
displacement already carried by the checkpoint. The derivative of the contaminated
trajectory at $\varepsilon=0$ is
\begin{equation}
    \psi(T)
    :=\left.
       \frac{\partial\phi_\varepsilon(T)}
            {\partial\varepsilon}
       \right|_{\varepsilon=0}.
    \label{eq:proof-psi-definition}
\end{equation}
Physically, $\psi(T)$ is the first-order direction in which contamination
drags the learned parameters as training proceeds: to leading order the
poisoned and clean trajectories separate by $\varepsilon\psi(T)$, i.e.,
$\phi_\varepsilon(T)-\phi_0(T)=\varepsilon\psi(T)+o(\varepsilon)$, and at the
checkpoint this separation is exactly the offset, $\psi(0)=\psi^{\mathrm{in}}$.
Differentiating~\eqref{eq:proof-poison-dynamics} at $\varepsilon=0$ (both dynamics
share the fixed linearization center $\phi^{\mathrm{in}}_0$; only the starting
point moves, by~\eqref{eq:checkpoint-offset}) yields the variational equation
\begin{align}
    \dot\psi(T)
    &=-H_0\psi(T)
      +\Delta v-\Delta H\,\bigl(\phi_0(T)-\phi^{\mathrm{in}}_0\bigr),
    \notag\\
    \psi(0)&=\psi^{\mathrm{in}} ,
    \label{eq:proof-variational-equation}
\end{align}
in which the poison forces the clean displacement-from-initialization
$\phi_0(T)-\phi^{\mathrm{in}}_0$ while the trajectory starts from the offset
$\psi^{\mathrm{in}}$. This is solved by a homogeneous transient plus a forced part,
\begin{equation}
    \psi(T)
    =
    e^{-H_0T}\psi^{\mathrm{in}}
    +\psi_{\rm f}(T),
    \qquad
    \psi_{\rm f}(T):=\int_0^T e^{-H_0(T-s)}
       \Bigl(\Delta v-\Delta H\,\bigl(\phi_0(s)-\phi^{\mathrm{in}}_0\bigr)\Bigr)\,ds,
    \label{eq:proof-first-duhamel-form}
\end{equation}
with $\phi_0(s)-\phi^{\mathrm{in}}_0=(I-e^{-H_0s})(\phi_0^\star-\phi^{\mathrm{in}}_0)$.
The first term is the checkpoint offset, carried forward and filtered by
$e^{-H_0T}$ exactly as the clean residual $e_0(T)$ is; the second is the forced
poison response.
We can rewrite the forced part as
\begin{equation} \label{eq:llm-poisoned-clean-displacement}
\begin{aligned}
 & \psi_{\rm f}(T)
    =q_T(H_0)f_{\rm eff}(T),\\
&     u_k^\top f_{\rm eff}(T)
    =
    u_k^\top f_{\rm p}
    +r_T(\lambda_k)
     \sum_{j=1}^P
       (u_k^\top\Delta H u_j)
       \bigl(u_j^\top(\phi_0^\star-\phi^{\mathrm{in}}_0)\bigr)
       K_T(\lambda_k,\lambda_j).
\end{aligned}
\end{equation}
with $f_{\rm p}:=\Delta v-\Delta H\bigl(\phi_0^\star-\phi^{\mathrm{in}}_0\bigr)$.
 The
two-eigenvalue kernels that are used to present expression of
$f_{\rm eff}(T)$ are as follows
\begin{equation}
   K_T(\lambda,\mu)
    =
    \begin{cases}
      \dfrac{e^{-\mu T}-e^{-\lambda T}}{\lambda-\mu},
          &\lambda\neq\mu,\\[8pt]
      Te^{-\lambda T},&\lambda=\mu.
    \end{cases}
\end{equation}
and
\begin{equation}
    r_T(\lambda_k)
    :=
    \begin{cases}
        \displaystyle
        \frac{\lambda_k}{1-e^{-\lambda_k T}},
        & \lambda_k>0,\\[10pt]
        \displaystyle
        \frac{1}{T},
        & \lambda_k=0.
    \end{cases}
    \label{eq:rT-eigenvalue-definition}
\end{equation}

The size of the forced poison-induced displacement, measured in the metric
that governs clean validation loss, is $I(T)^2=\psi_{\rm f}(T)^\top G\,\psi_{\rm f}(T)$
(the checkpoint-offset transient is treated separately below). Just as
the finite-time filter $q_T$ shapes the clean error, it shapes the poison
response, so $I(T)^2$ can be written as an integral of $q_T(\lambda)^2$
against a measure that records how the poison distributes its influence
across curvature scales. The assumption below specifies the behavior of
this measure near weak curvature modes: the poison deposits a power-law amount of
influence.
\begin{assumption}[Effective poison spectral tail]
\label{ass:G-poison-spectral-energy}
There exists a  poison measure
\(d\nu_{p}^{(T)}(\lambda)\) such that
\begin{equation}
    I(T)^2
    := \psi_{\rm f}(T)^\top G\psi_{\rm f}(T)=
    \int_0^\infty q_T(\lambda)^2\,
    d\nu_{p}^{(T)}(\lambda).
    \label{eq:IG-spectral-energy-representation}
\end{equation}
Assume that,
\begin{equation}
    d\nu_{p}^{(T)}(\lambda)
    \sim
    C_{\nu}\lambda^{1-2\gamma}\,d\lambda,
    \qquad
    \lambda\downarrow0,
    \qquad
    T \uparrow \infty,
    \qquad
    0<\gamma<1.
    \label{eq:G-poison-energy-tail}
\end{equation}
\end{assumption}

The clean truncation bias $B(T)$ decays as training proceeds, but the poison
influence works in the opposite direction: the longer the model trains, the
more the weak-curvature directions fall under the poison
forcing. The result below quantifies this growth,
showing that the poison-induced displacement, measured in the validation
metric, grows as a clean power law $T^{\gamma}$, with the growth exponent
set by the effective poison spectral tail. It is the opposing partner of
the clean-bias falloff, and balancing the two is what later fixes the
scaling law.

\begin{theorem}[Poison-strength growth]
\label{thm:poison-strength-growth}
Under Assumption~\ref{ass:G-poison-spectral-energy}, for large $T$,
\begin{equation}
    I(T)
    =
    C_{I}T^\gamma(1+o(1)),
    \qquad
    C_{I}:=\sqrt{C_{\nu}J_\gamma},
    \label{eq:llm-IG-growth}
\end{equation}
where
\begin{equation}
    J_\gamma
    :=
    \int_0^\infty
        (1-e^{-z})^2z^{-1-2\gamma}\,dz
    \in(0,\infty),
    \qquad 0<\gamma<1.
    \label{eq:J-gamma-definition}
\end{equation}
\end{theorem}

See Appendix~\ref{app:poison-strength-growth} for a proof.

The checkpoint offset $\psi^{\mathrm{in}}$ contributes the transient
$e^{-H_0T}\psi^{\mathrm{in}}$ to the displacement~\eqref{eq:proof-first-duhamel-form}.
Measured in the validation metric it has exactly the structure of the clean
residual $e_0(T)$ of Appendix~\ref{subsec:llm-clean-truncation} - a fixed vector filtered by $e^{-H_0T}$ - so,
under a weak-curvature tail on $\psi^{\mathrm{in}}$ analogous to
Assumption~\ref{response_tail}, its size is a clean power law,
\begin{equation}
    B_{\mathrm{in}}(T):=\big\|e^{-H_0T}\psi^{\mathrm{in}}\big\|_G
    =C_{\mathrm{in}}\,T^{-\beta_{\mathrm{in}}}(1+o(1)),
    \qquad
    \beta_{\mathrm{in}}=\tfrac12(\alpha+\eta_{\mathrm{in}}+\zeta),
    \label{eq:checkpoint-offset-falloff}
\end{equation}
where $\eta_{\mathrm{in}}$ is the tail exponent of $\psi^{\mathrm{in}}$ across the
$H_0$-eigenbasis. Thus the offset decays with training: continued training
partially washes out the poison damage already present in the checkpoint, while the
forced response $\psi_{\rm f}(T)$ builds up.

\subsection{The scaling law from optimal truncation}
\label{subsec:llm-scaling-law}
The goal of this subsection is to combine the decaying clean truncation bias
and the growing poison influence into the observed scaling law. 

For a clean token-prediction event $Z=(X,Y)$, let
$f_\phi(X)\in\mathbb R^V$ be the logits and define
\begin{equation}
    \ell_Z(\phi)
    :=-\log\!\left(\operatorname{softmax}
              (f_\phi(X))_Y\right).
    \label{eq:llm-token-cross-entropy}
\end{equation}
At $\phi_0^\star$, put
\begin{equation}
    p_Z:=\operatorname{softmax}(f_{\phi_0^\star}(X)),
    \qquad
    J_Z:=\left.\frac{\partial f_\phi(X)}{\partial\phi}
          \right|_{\phi=\phi_0^\star}.
\end{equation}
The gradient and Hessian with respect to the logits are
\begin{equation}
    \nabla_f\ell_Z=p_Z-e_Y,
    \qquad
    W_Z:=\nabla_f^2\ell_Z
       =\operatorname{diag}(p_Z)-p_Zp_Z^\top.
    \label{eq:llm-softmax-derivatives}
\end{equation}
Consequently, the clean validation loss is
\begin{equation}
    \mathcal R_{\mathcal P}(\phi):=\mathbb E_{Z\sim\mathcal P}[\ell_Z(\phi)] .
\end{equation}
The clean solution
$\phi_0^\star$ is the stationary point of the \emph{training} dynamics
($H_0(\phi_0^\star-\phi^{\mathrm{in}}_0)=v_0$, Eq.~\eqref{eq:llm-clean-stationary-equation}); it need not
minimize the validation risk $\mathcal R_{\mathcal P}$, and in a finite-data run it
generically does not. We therefore keep the validation gradient at the clean
solution,
\begin{equation}
    g_\star:=\nabla\mathcal R_{\mathcal P}(\phi_0^\star)\in\mathbb R^P ,
    \label{eq:llm-train-eval-gap}
\end{equation}
Taylor's
expansion, keeping the linear term, gives
\begin{equation}
    \mathcal R_{\mathcal P}(\phi_0^\star+h)
    =\mathcal R_{\mathcal P}(\phi_0^\star)
     +g_\star^\top h
     +\tfrac12 h^\top Gh+o(\|h\|_2^2).
    \label{eq:llm-clean-risk-taylor}
\end{equation}
At a matched checkpoint $\phi_\varepsilon(T)-\phi_0(T)=\varepsilon\psi(T)+o(\varepsilon)$
and $\phi_0(T)-\phi_0^\star=e_0(T)$, with the displacement split into checkpoint
offset and forced response, $\psi(T)=e^{-H_0T}\psi^{\mathrm{in}}+\psi_{\rm f}(T)$
(Eq.~\eqref{eq:proof-first-duhamel-form}). Applying~\eqref{eq:llm-clean-risk-taylor}
with $h=e_0(T)+\varepsilon\psi(T)$ and with $h=e_0(T)$ and subtracting, the constant
and poison-independent terms ($\mathcal R_{\mathcal P}(\phi_0^\star)$, $g_\star^\top e_0$,
$\tfrac12 e_0^\top Ge_0$) cancel, leaving
\begin{align}
    \delta\mathcal R(\varepsilon,T)
    &:=\mathcal R_{\mathcal P}(\phi_\varepsilon(T))-\mathcal R_{\mathcal P}(\phi_0(T))
      =\varepsilon\,g_\star^\top\psi(T)
      +\varepsilon\,e_0(T)^\top G\,\psi(T)
      +\tfrac12\varepsilon^2\,\psi(T)^\top G\,\psi(T)
      \notag\\
    &=\underbrace{\varepsilon\,\rho_{gi}(T)\,\Gamma_\star B_{\mathrm{in}}(T)}_{\text{checkpoint offset}}
      +\underbrace{\varepsilon\,\rho_g(T)\,\Gamma_\star I(T)}_{\text{train--eval gap}}
      +\underbrace{\varepsilon\,\rho_G(T)\,B(T)I(T)}_{\text{clean bias}}
      +\underbrace{\tfrac12\varepsilon^2 I(T)^2}_{\text{pure poison}},
    \label{eq:llm-excess-cross-entropy-BI}
\end{align}
where  the correlations are
\begin{equation}
    \Gamma_\star:=\|g_\star\|_{G^{-1}},
    \quad
    \rho_g(T):=\frac{g_\star^\top\psi_{\rm f}(T)}{\Gamma_\star I(T)},
    \quad
    \rho_{gi}(T):=\frac{g_\star^\top e^{-H_0T}\psi^{\mathrm{in}}}{\Gamma_\star B_{\mathrm{in}}(T)},
    \quad
    \rho_G(T):=\frac{e_0^\top G\psi_{\rm f}}{B(T)I(T)},
    \label{eq:llm-gap-defs}
\end{equation}
all in $[-1,1]$ by Cauchy--Schwarz in the $G$-metric and taken asymptotically
constant, $\rho_G\to\rho_\star$, $\rho_g\to\rho_g^\star$, $\rho_{gi}\to\rho_{gi}^\star$.
The train--eval gap term is the decisive contribution at late times. Using
$B(T)=C_BT^{-\beta}$ (Result~\ref{thm:clean-truncation-falloff}),
$I(T)=C_IT^{\gamma}$ (Result~\ref{thm:poison-strength-growth}) and
$B_{\mathrm{in}}(T)=C_{\mathrm{in}}T^{-\beta_{\mathrm{in}}}$
(Eq.~\eqref{eq:checkpoint-offset-falloff}), the four coefficients of
\eqref{eq:llm-excess-cross-entropy-BI} scale as
\begin{equation}
    \underbrace{\varepsilon\,\rho_{gi}^\star\Gamma_\star C_{\mathrm{in}}\,T^{-\beta_{\mathrm{in}}}}_{\text{checkpoint: }\varepsilon,\ \text{decays}}
    ,\quad
    \underbrace{\varepsilon\,\rho_g^\star\Gamma_\star C_I\,T^{\gamma}}_{\text{gap: }\varepsilon,\ \text{grows}}
    ,\quad
    \underbrace{\varepsilon\,\rho_\star C_BC_I\,T^{-(\beta-\gamma)}}_{\text{bias: }\varepsilon,\ \text{decays}}
    ,\quad
    \underbrace{\tfrac12\varepsilon^2 C_I^2\,T^{2\gamma}}_{\text{pure: }\varepsilon^2,\ \text{grows}} .
    \label{eq:llm-three-scalings}
\end{equation}
Here we assumed $\beta>\gamma>0, \beta_{\mathrm{in}}>0 $.

\begin{theorem}[Gap-dominated crossover]
\label{thm:scaling-law-gap}
Suppose among $\varepsilon$ order terms the gap term dominates. Then the excess risk is,
\begin{equation}
    \delta\mathcal R(\varepsilon,T)
    =\rho_g^\star\Gamma_\star I(T)\,\varepsilon
     +\tfrac12 I(T)^2\,\varepsilon^2\,(1+o(1)),
\end{equation}
and equating them,
$\varepsilon\,\rho_g^\star\Gamma_\star I(T)=\tfrac12\varepsilon^2 I(T)^2$, defines the
crossover time
\begin{equation}
    T^{(g)}_\varepsilon
    =\Bigl(\frac{2\rho_g^\star\Gamma_\star}{C_I\,\varepsilon}\Bigr)^{1/\gamma}
    =t_g\,\varepsilon^{-1/\gamma},
    \qquad
    t_g:=\Bigl(\frac{2\rho_g^\star\Gamma_\star}{C_I}\Bigr)^{1/\gamma}.
    \label{eq:crossover-horizon-gap}
\end{equation}

The excess risk then obeys the integer law
\begin{equation}
    \delta\mathcal R(\varepsilon,T)\asymp
    \begin{cases}
        \Gamma_\star\,\varepsilon\,T^{\gamma}, & T\ll T^{(g)}_\varepsilon\ \ (a=1),\\[3pt]
        \varepsilon^{2}\,T^{2\gamma}, & T\gg T^{(g)}_\varepsilon\ \ (a=2),
    \end{cases}
    \label{eq:gap-scaling-cases}
\end{equation}
More generally, say, $x:=T/T^{(g)}_\varepsilon$, then,
\begin{equation}
    \delta\mathcal R(\varepsilon,T)=\Psi(x)\,(1+o(1)),
    \quad
    \Psi(x)=\rho_g^\star\Gamma_\star C_I t_g^{\gamma}\,x^{\gamma}
            +\tfrac12 C_I^2 t_g^{2\gamma}\,x^{2\gamma}
            =2(\rho_g^\star\Gamma_\star)^2\bigl(x^{\gamma}+x^{2\gamma}\bigr),
    \label{eq:gap-scaling-form}
\end{equation}
which is $\varepsilon$-independent at $T=T^{(g)}_\varepsilon$ ($a=0$). Since
$\Psi(x)\sim x^{\gamma}$ ($x\to0$) and $\sim x^{2\gamma}$ ($x\to\infty$) with both exponents positive, $\Psi$ is monotone increasing no interior minimum.
\end{theorem}

Small LLMs fall under the gap-dominated regime, early in training $a=1$ and then it decreases towards $a=0$ as we approach the crossover time. Whereas, large LLMs lie in the gap or the checkpoint-ominated regime. Both these terms are mathematically similar because their conceptual origin is the heavy-tailed inverse curvature spectrum. Hence, we discuss only one of them (the checkpoint dominated results are obtained from the bias-dominated scaling exponent using $\beta \leftrightarrow \beta_{\mathrm{in}}+\gamma$).

\begin{theorem}[Bias-dominated crossover]
\label{thm:scaling-law-optimal-truncation}
Suppose among $\varepsilon$ order terms the bias term dominates. 
Write $x:=T/T_\varepsilon$ with
\begin{equation}
    T_\varepsilon
    =t\,\varepsilon^{-1/(\beta+\gamma)},
    \qquad
    t:=\Bigl(\tfrac{2\rho_\star C_B}{C_I}\Bigr)^{1/(\beta+\gamma)} .
    \label{eq:crossover-horizon}
\end{equation}

Then the excess risk admits the exact
scaling form
\begin{equation}
    \delta\mathcal R(\varepsilon,T)
    =\varepsilon^{a}\,\Phi(x)\,(1+o(1)),
    \qquad
    \Phi(x)=\rho_\star C_BC_I\,t^{\gamma-\beta}\,x^{\gamma-\beta}
            +\tfrac12 C_I^2\,t^{2\gamma}\,x^{2\gamma},
    \qquad
    a=\frac{2\beta}{\beta+\gamma},
    \label{eq:scaling-form-llm}
\end{equation}
 Its limits are given as,
\begin{equation}
    \Phi(x)\asymp x^{\gamma-\beta}\ (x\to0),
    \qquad
    \Phi(x)\asymp x^{2\gamma}\ (x\to\infty),
    \label{eq:phi-limits}
\end{equation}
and $\Phi$ attains a unique interior minimum for $\beta>\gamma$ at
\begin{equation}
    x_\star
    =\Bigl(\frac{(\beta-\gamma)\,\rho_\star C_B}
                {\gamma\,C_I\,t^{\beta+\gamma}}\Bigr)^{1/(\beta+\gamma)}
    =\Bigl(\frac{\beta-\gamma}{2\gamma}\Bigr)^{1/(\beta+\gamma)}=\phi(1).
    \label{eq:xstar}
\end{equation}
Consequently the two regimes meet, in order, at the crossover $T\asymp
T_\varepsilon$, where
\begin{equation}
    \delta\mathcal R(\varepsilon,T)
    \;\asymp\;\varepsilon^{a},
    \qquad
    \min_{T}\delta\mathcal R(\varepsilon,T)
    =\Phi(x_\star)\,\varepsilon^{a}\,(1+o(1)).
    \label{eq:min-value}
\end{equation}
\end{theorem}

See Appendix~\ref{app:scaling-law-optimal-truncation} for a proof.

Therefore, for large LLM, scaling exponent $a$ starts near $1$ and then increases as training progresses towards the crossover time.

\begin{remark}[Perplexity observable]
\label{cor:perplexity}
For the empirical perplexity ratio
$\Delta(\varepsilon,T)=\exp(\delta\mathcal R(\varepsilon,T))-1$, the same three
statements hold verbatim in the small-degradation regime
$\delta\mathcal R\to0$, since $\Delta=\delta\mathcal R\,(1+o(1))$. 
\end{remark}

The exponent $a$ so far is expressed through the abstract tail exponents
$\alpha,\beta,\gamma,\zeta,\eta$. The latent-cluster model lets us trade
these for two interpretable quantities: the cluster-tail parameter $m$
(roughly, this controls the distribution of ideas)
and the within-cluster spread $r$ (controls how precise an idea is).
\begin{equation}
    a=2-\frac{4\gamma}{\alpha+\eta+\zeta+2\gamma},
    \qquad
    \alpha=\frac{1/m-1}{2r}-1,
    \qquad 0<\gamma<1.
    \end{equation}
Making the neutral modeling choices---no extra weak-mode alignment
($\eta=0$), a neutral poison spectral density ($\gamma=1/2$), and
the ample-data validation matching ($\zeta=1$)---collapses the general
formula to a closed form in $m$ and $r$, which is what makes the observed
slope/intercept trends interpretable.
\[
    a=
    2-\frac{2}{2+\alpha}
    =
    2-\frac{2}{1+\frac{1-m}{2rm}}.
\]

\section{Proofs for the scaling-law mechanism}
\label{app:mechanism-proofs}

\subsection{Proof of Result~\ref{thm:analyticity-barrier}}
\label{app:proof-analyticity-barrier}
\begin{proof}

Define
\[
F(\theta,\varepsilon):=\nabla_\theta L_\varepsilon(\theta).
\]
The map $F$ is real analytic near $(\theta_0,0)$,
$F(\theta_0,0)=0$, and
\[
D_\theta F(\theta_0,0)=\nabla^2L_{\mathcal P}(\theta_0)=H.
\]
Because $H$ is invertible, the analytic implicit-function theorem gives a
unique analytic map $\varepsilon\mapsto\theta_\varepsilon$ near zero such that
$F(\theta_\varepsilon,\varepsilon)=0$ and $\theta_{\varepsilon=0}=\theta_0$. We construct this function below.

Differentiate
$F(\theta_\varepsilon,\varepsilon)=0$ at $\varepsilon=0$:
\[
H\,\theta'_0+\partial_\varepsilon F(\theta_0,0)=0.
\]
By definition, $\partial_\varepsilon F(\theta_0,0)=g$, and therefore
\[
\theta'_0=-H^{-1}g.
\]
Analyticity of the branch now gives
\[
\theta_\varepsilon-
\theta_0=-\varepsilon H^{-1}g+O(\varepsilon^2),
\]
which proves~\eqref{eq:analytic-parameter-response}.

Taylor expansion of $\mathcal R$ at $\theta_0$ gives, with
$\delta_\varepsilon:=\theta_\varepsilon-
\theta_0$,
\[
\mathcal R(\theta_0+\delta_\varepsilon)-\mathcal R(\theta_0)
=\underbrace{\nabla\mathcal R(\theta_0)^\top\delta_\varepsilon}_{=0}
+\frac12\delta_\varepsilon^\top G\delta_\varepsilon
+O(\|\delta_\varepsilon\|^3).
\]
Substituting
$\delta_\varepsilon=-\varepsilon H^{-1}g+O(\varepsilon^2)$ yields
\[
\mathcal R(\theta_\varepsilon)-\mathcal R(\theta_0)
=\frac{\varepsilon^2}{2}
g^\top H^{-1}GH^{-1}g+O(\varepsilon^3),
\]
which is~\eqref{eq:analytic-risk-response}.  The coefficient is strictly
positive under the stated nondegeneracy condition.  Since both
$\theta_\varepsilon$ and $\mathcal R$ are analytic, their composition is analytic;
hence a noninteger leading power is impossible under these assumptions.
Finally, $e^x-1=x+O(x^2)$ proves the perplexity statement.
\end{proof}

\subsection{Proof of Result~\ref{thm:adaptive-truncation-law}}
\label{app:proof-adaptive-truncation}
\begin{proof}
\medskip

For any $\varepsilon$ and $\tau$,
\begin{align*}
m_{\varepsilon,\tau}-\mu
={}&\mathbb E_{\mathcal P}[T_\tau(\psi)]-\mu\\
&+\varepsilon\left(
\mathbb E_{\mathcal Q_\varepsilon}[T_\tau(\psi)]
-\mathbb E_{\mathcal P}[T_\tau(\psi)]
\right).
\end{align*}
The first line is the clean truncation bias and the second is the
contamination effect.
Substitute~\eqref{eq:general-truncation-bias} and
~\eqref{eq:general-poison-influence} and choose
$\tau=\tau_\varepsilon=t\varepsilon^{-1/(\beta+\gamma)}$.  Then
\[
\tau_\varepsilon^{-\beta}
=t^{-\beta}\varepsilon^{\beta/(\beta+\gamma)},
\qquad
\varepsilon\tau_\varepsilon^\gamma
=t^\gamma\varepsilon^{\beta/(\beta+\gamma)}.
\]
Both contributions therefore have the same order, and their vector
coefficient is $h_t=t^{-\beta}b+t^\gamma c$.  This proves
~\eqref{eq:truncated-moment-fractional}.

Since $\Phi$ is $C^2$,
\[
\theta_{\varepsilon,\tau_\varepsilon}-\theta_0
=J(m_{\varepsilon,\tau_\varepsilon}-\mu)
+O(\|m_{\varepsilon,\tau_\varepsilon}-\mu\|^2).
\]
Using~\eqref{eq:truncated-moment-fractional},
\[
\theta_{\varepsilon,\tau_\varepsilon}-\theta_0
=\varepsilon^{\beta/(\beta+\gamma)}Jh_t
+o\!\left(\varepsilon^{\beta/(\beta+\gamma)}\right).
\]

The Taylor expansion of $\mathcal R$ at $\theta_0$, gives
\begin{align*}
\mathcal R(\theta_{\varepsilon,\tau_\varepsilon})-\mathcal R(\theta_0)
&=\frac12
(\theta_{\varepsilon,\tau_\varepsilon}-\theta_0)^\top
G(\theta_{\varepsilon,\tau_\varepsilon}-\theta_0)
+o(\|\theta_{\varepsilon,\tau_\varepsilon}-\theta_0\|^2)\\
&=\frac12h_t^\top J^\top GJh_t\,
\varepsilon^{2\beta/(\beta+\gamma)}
+o\!\left(\varepsilon^{2\beta/(\beta+\gamma)}\right).
\end{align*}
This proves~\eqref{eq:truncated-risk-fractional} and completes the proof.
\end{proof}

\subsection{Proof of Result~\ref{thm:preconditioned-reduction}}
\label{app:preconditioned-reduction}
\begin{proof}
Multiplying~\eqref{optimizer_update} by $S^{-1/2}$ and inserting
$\delta\theta_t=S^{1/2}\phi_t$,
\[
    \phi_{t+1}
    =\phi_t+\eta_tS^{-1/2}S\bigl(\widehat v-\widehat\Sigma S^{1/2}\phi_t\bigr)
    =\phi_t+\eta_t\bigl(\tilde v-A\phi_t\bigr).
\]
Let $Au=\lambda u$ with $\|u\|=1$ and $x_t:=u^\top\phi_t$, $\tilde v_u:=u^\top\tilde v$.
Then $x_{t+1}=(1-\eta_t\lambda)x_t+\eta_t\tilde v_u$. For $\lambda>0$ this has the
fixed point $\tilde v_u/\lambda$ and
$x_{t+1}-\tilde v_u/\lambda=(1-\eta_t\lambda)(x_t-\tilde v_u/\lambda)$, so from
$x_0=0$ we get $x_K=q^{(K)}(\lambda)\tilde v_u$; for $\lambda=0$ both sides equal
$T\tilde v_u$. Since $A$ is symmetric its eigenvectors span $\mathbb R^{P}$, which
gives~\eqref{eq:preconditioned-recursion}. If $\max_t\eta_t\lambda<1$ then
$\log\prod_t(1-\eta_t\lambda)=-\lambda T-O(\lambda^2\sum_t\eta_t^2)$ and
$\sum_t\eta_t^2\le T\max_t\eta_t$, so $q^{(K)}(\lambda)\to q_T(\lambda)$ in the
stated limit; $q_T(A)\tilde v$ is the solution of~\eqref{eq:preconditioned-flow}
by direct differentiation. For the raw coordinates, $S^{1/2}A^nS^{-1/2}=(S\widehat\Sigma)^n$
for every $n\ge0$, hence $S^{1/2}f(A)S^{-1/2}=f(S\widehat\Sigma)$ for every
polynomial $f$ and, for
$f=q_T$; thus
$\delta\theta(T)=S^{1/2}\phi(T)=S^{1/2}q_T(A)S^{1/2}\widehat v=q_T(S\widehat\Sigma)S\widehat v$.
Finally, if $\widehat\Sigma\succ0$ then $A\succ0$, $q_T(A)\to A^{-1}$ as
$T\to\infty$, and $S^{1/2}A^{-1}S^{1/2}=S^{1/2}S^{-1/2}\widehat\Sigma^{-1}S^{-1/2}S^{1/2}=\widehat\Sigma^{-1}$.
\end{proof}

\subsection{Proof of Result~\ref{thm:cluster-eigenvalue-scale}}
\label{app:proof-cluster-eigenvalue}

\begin{proof}
The row vectors of $\mathcal J_i$ are
supported on the proper parameter subspace $\mathcal U_{\ell,P}$ rather than
on all of $\mathbb R^P$.  From an LLM perspective, $\mathcal U_{\ell,P}$ represents tangent directions
that change the model's logits primarily on contexts belonging to semantic
cluster $\ell$. The normalization $1/V$ gives
\begin{equation}
    \mathbb E\!\left[
        \bigl(Z_i^{(\ell)}\bigr)^\top
        Z_i^{(\ell)}
    \right]
    =I_{d_{\ell,P}}.
    \label{eq:cluster-gaussian-normalization}
\end{equation}
Since $Y_i=\mathcal J_i^\top \mathcal J_i$, the Gaussian model implies
\begin{align}
    \mathbb E[Y_i\mid I_i=\ell]
    &=h_\ell U_{\ell,P}
      \mathbb E\!\left[
          \bigl(Z_i^{(\ell)}\bigr)^\top
          Z_i^{(\ell)}
      \right]
      U_{\ell,P}^\top
    \\
    &=h_\ell U_{\ell,P}U_{\ell,P}^\top
      =h_\ell\Pi_{\ell,P}.
    \label{eq:cluster-aligned-conditional-gram}
\end{align}
Standard Gaussian matrix estimates give
\begin{equation}
    \|\mathcal J_i\|_{\mathrm{op}}
    =O_{\mathbb P}(\sqrt{h_\ell})
    =O_{\mathbb P}(\ell^r),
\end{equation}

Taking expectation over $I_i$ and using the law of total expectation now gives
\begin{equation}
    \begin{aligned}
    \Sigma_P
    :=\mathbb E[Y_i]
    &=\sum_{\ell\geq1}
        \mathbb P(I_i=\ell)
        \mathbb E[Y_i\mid I_i=\ell]
    \\
    &=\sum_{\ell\geq1}
        \pi_\ell h_\ell\Pi_{\ell,P}.
    \end{aligned}
    \label{eq:cluster-aligned-population-spectrum}
\end{equation}
\end{proof}

\subsection{Proof of Result~\ref{thm:truncation-hard-edge}}
\label{app:proof-BY-tail}

\begin{proof}
 A token drawn from cluster $\ell$ has probability
and curvature scale
\[
    \pi_\ell \asymp \ell^{-1/m},
    \qquad
    h_\ell \asymp \ell^{2r}.
\]
Thus, rarer clusters have larger single-example curvature.  The quantity
$B_Y(\tau)$ asks how much of the average curvature signal comes from examples
whose curvature is larger than $\tau$.

To make this observation precise at the level needed here, condition on the
cluster index $I_i=\ell$.  From
\[
    \mathcal J_i=\sqrt{h_\ell}\,Z_i^{(\ell)}U_{\ell,P}^{\top},
    \qquad
    Y_i=\mathcal J_i^{\top}\mathcal J_i,
\]
and the orthonormality of the columns of $U_{\ell,P}$, we get
\[
    \lVert Y_i\rVert_{\mathrm{op}}
    =h_\ell\,\zeta_\ell,
    \qquad
    \zeta_\ell:=\lVert Z_i^{(\ell)}\rVert_{\mathrm{op}}^2.
\]
For the Gaussian normalization used in the model, $\zeta_\ell$ is an
order-one random variable with a light upper tail.  Consequently, it can change constants
but not the power of $\tau$: clusters with $h_\ell\gg\tau$ contribute with
order-one probability, clusters with $h_\ell\ll\tau$ contribute only through
an exponentially unlikely fluctuation, and the crossover region
$h_\ell\asymp\tau$ only changes the overall prefactor.  Therefore
\begin{equation}
    B_Y(\tau)
    \asymp
    \sum_{\ell:\,h_\ell\gtrsim\tau}\pi_\ell h_\ell.
    \label{eq:app-BY-physical-reduction}
\end{equation}

Let $\ell_\tau$ denote the cluster rank at this crossover.  Since
$h_\ell\asymp\ell^{2r}$,
\[
    h_{\ell_\tau}\asymp\tau
    \qquad\Longrightarrow\qquad
    \ell_\tau\asymp\tau^{1/(2r)}.
\]
Moreover,
\[
    \pi_\ell h_\ell
    \asymp
    \ell^{-1/m+2r}
    =\ell^{-\delta},
    \qquad
    \delta:=\frac{1}{m}-2r.
\]
The definition $q=(1/m-1)/(2r)$ gives
\[
    \delta-1=2r(q-1).
\]
Because $q>1$, we have $\delta>1$, so the remaining sum is a convergent
power-law tail.  Using an integral estimate in
\eqref{eq:app-BY-physical-reduction},
\[
    B_Y(\tau)
    \asymp
    \sum_{\ell\gtrsim\ell_\tau}\ell^{-\delta}
    \asymp
    \int_{\ell_\tau}^{\infty}x^{-\delta}\,dx
    \asymp
    \ell_\tau^{-(\delta-1)}.
\]
Finally, substituting $\ell_\tau\asymp\tau^{1/(2r)}$ yields
\[
    B_Y(\tau)
    \asymp
    \tau^{-(\delta-1)/(2r)}
    =\tau^{-(q-1)}.
\]
In physical terms, the exponent is simply the balance between how quickly the
curvature of an individual rare cluster grows and how quickly the total
probability of such clusters decreases.
\end{proof}

\subsection{Proof of Result~\ref{thm:clean-truncation-falloff}}
\label{app:clean-truncation-falloff}

\begin{proof}
Under Assumption~\ref{response_tail}, the hard-edge law, and
Assumption~\ref{ass:validation-metric-clean},  for $T\to \infty$ we obtain,
\begin{align}
    B(T)^2
    &= e_0(T)^\top G e_0(T)
      \notag\\
    &\sim
    \|\phi_0^\star-\phi^{\mathrm{in}}_0\|_2^2
    C_a C_G c_\mu\alpha
    \int_0^\infty
        e^{-2\lambda T}
        \lambda^{\alpha+\eta+\zeta-1}
    \,d\lambda.
    \label{eq:llm-BG-integral}
\end{align}
The change of variables \(z=2\lambda T\) gives
\begin{equation}
    B(T)^2
    =
    \|\phi_0^\star-\phi^{\mathrm{in}}_0\|_2^2
    C_a C_G c_\mu\alpha
    \Gamma(2\beta)
    (2T)^{-2\beta}
    (1+o(1)).
    \label{eq:llm-BG-square-falloff}
\end{equation}
Taking the square root yields~\eqref{eq:llm-BG-falloff}, which completes the
proof.
\end{proof}

\subsection{Proof of Result~\ref{thm:poison-strength-growth}}
\label{app:poison-strength-growth}

\begin{proof}
Using the spectral representation~\eqref{eq:IG-spectral-energy-representation}
with $q_T(\lambda)=(1-e^{-\lambda T})/\lambda$ and the
tail~\eqref{eq:G-poison-energy-tail}, for large $T$,
\begin{align}
    I(T)^2
    &=
    \int_0^\infty
    \left(
        \frac{1-e^{-\lambda T}}{\lambda}
    \right)^2
    d\nu_{p}^{(T)}(\lambda)
    \notag\\
    &\sim
    C_{\nu}
    \int_0^\infty
    (1-e^{-\lambda T})^2
    \lambda^{-1-2\gamma}
    \,d\lambda.
    \label{eq:IG-growth-integral}
\end{align}
With \(z=\lambda T\),
\begin{align}
    I(T)^2
    &\sim
    C_{\nu}T^{2\gamma}
    \int_0^\infty
        (1-e^{-z})^2z^{-1-2\gamma}\,dz
    \notag\\
    &=
    C_{\nu,G}J_\gamma T^{2\gamma},
    \label{eq:IG-square-growth}
\end{align}
where $J_\gamma$ is defined in~\eqref{eq:J-gamma-definition} and is finite and
strictly positive for $0<\gamma<1$. Taking the square root
yields~\eqref{eq:llm-IG-growth}, which completes the proof.
\end{proof}

\subsection{Proof of Result~\ref{thm:scaling-law-optimal-truncation}}
\label{app:scaling-law-optimal-truncation}

\begin{proof}
Throughout we are in the bias-dominated regime $T\ll T_g$, where the
persistent-gap contribution $\varepsilon\rho_g^\star\Gamma_\star I(T)$ is
subdominant to the clean-bias term $\varepsilon\rho_\star B(T)I(T)$---their ratio
is $\rho_g^\star\Gamma_\star/(\rho_\star B(T))=(T/T_g)^{\beta}\ll1$---and is
dropped; in the ample-data case $\Gamma_\star=0$ it is exactly absent. At the
balanced time~\eqref{eq:crossover-horizon}, with $s$ as
in~$s:=\beta/(\beta+\gamma)$, the two competing scales become
\begin{align}
    B(T_\varepsilon)
       &=C_Bt^{-\beta}\varepsilon^s(1+o(1)),
       \label{eq:llm-balanced-B}\\
    \varepsilon I(T_\varepsilon)
       &=C_It^\gamma\varepsilon^s(1+o(1)).
       \label{eq:llm-balanced-I}
\end{align}
Substitution into~\eqref{eq:llm-excess-cross-entropy-BI}, using
$\rho_G(T_\varepsilon)\to\rho_\star$ at late time, gives
\begin{equation}
    \delta\mathcal R(\varepsilon,T_\varepsilon)
    =K(t)\,\varepsilon^{2s}
     +o(\varepsilon^{2s}),
    \label{eq:llm-cross-entropy-scaling}
\end{equation}
where
\begin{equation}
    K(t)
    :=\rho_{\star} C_BC_I t^{\gamma-\beta}
      +\frac12 C_I^2t^{2\gamma}.
    \label{eq:llm-leading-coefficient}
\end{equation}
Finally, perplexity is the exponential of cross-entropy, so
$\Delta(\varepsilon,T_\varepsilon)=\exp(\delta\mathcal R(\varepsilon,T_\varepsilon))-1
=\delta\mathcal R(\varepsilon,T_\varepsilon)+o(\delta\mathcal R)$, and since
$2s=2\beta/(\beta+\gamma)=a$, this
gives~\eqref{eq:min-value} and completes the proof.
\end{proof}

\section{Experimental details on large language models}
\label{app_expt_details}

\begin{figure}[h]
  \centering
  \begin{subfigure}{0.45\textwidth}
    \centering
    \includegraphics[width=\linewidth]{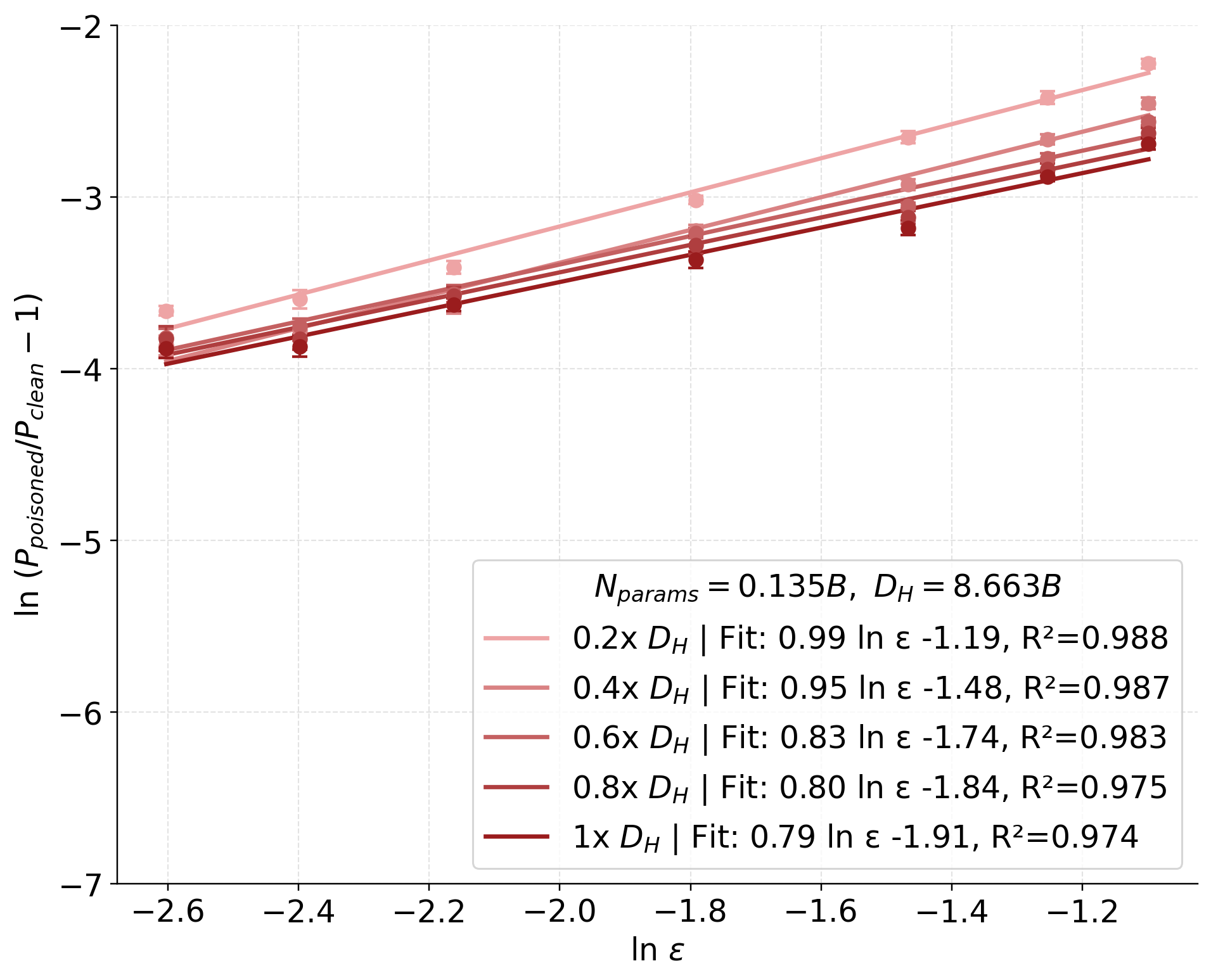}
    \caption{}
  \end{subfigure}
  \hfill
  \begin{subfigure}{0.45\textwidth}
    \centering
    \includegraphics[width=\linewidth]{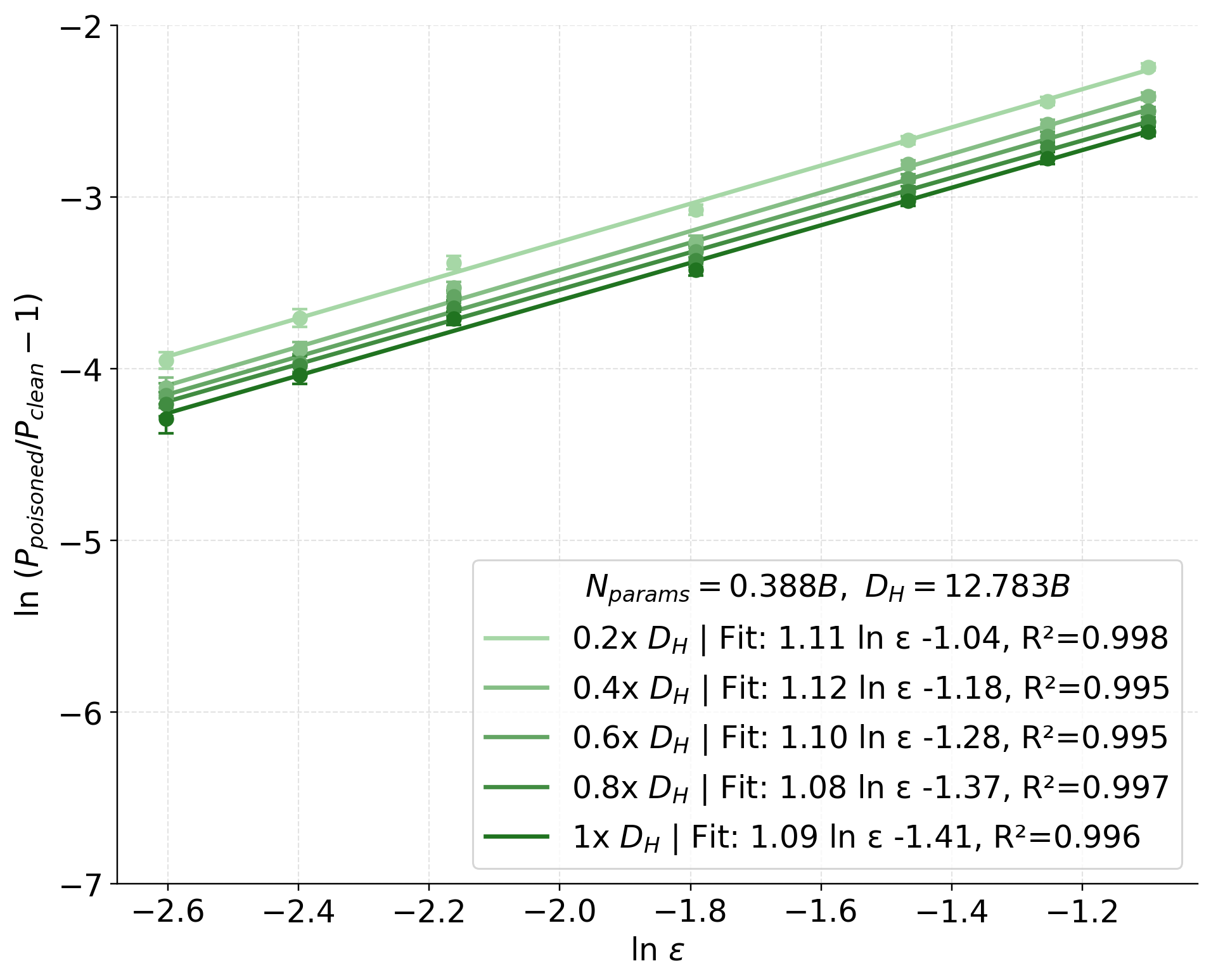}
    \caption{}
  \end{subfigure}
    \hfill
  \begin{subfigure}{0.45\textwidth}
    \centering
    \includegraphics[width=\linewidth]{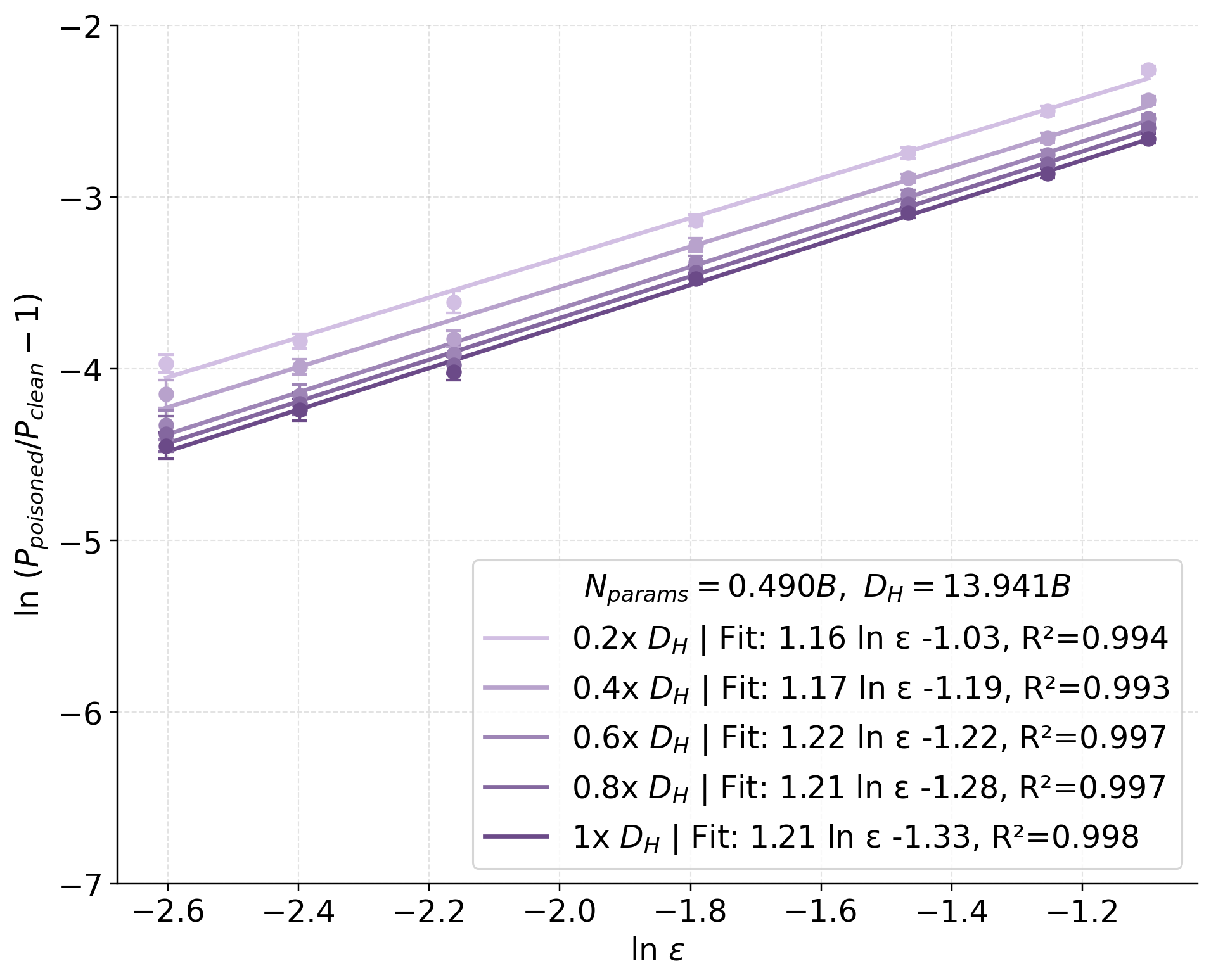}
    \caption{}
  \end{subfigure}
    \hfill
  \begin{subfigure}{0.45\textwidth}
    \centering
    \includegraphics[width=\linewidth]{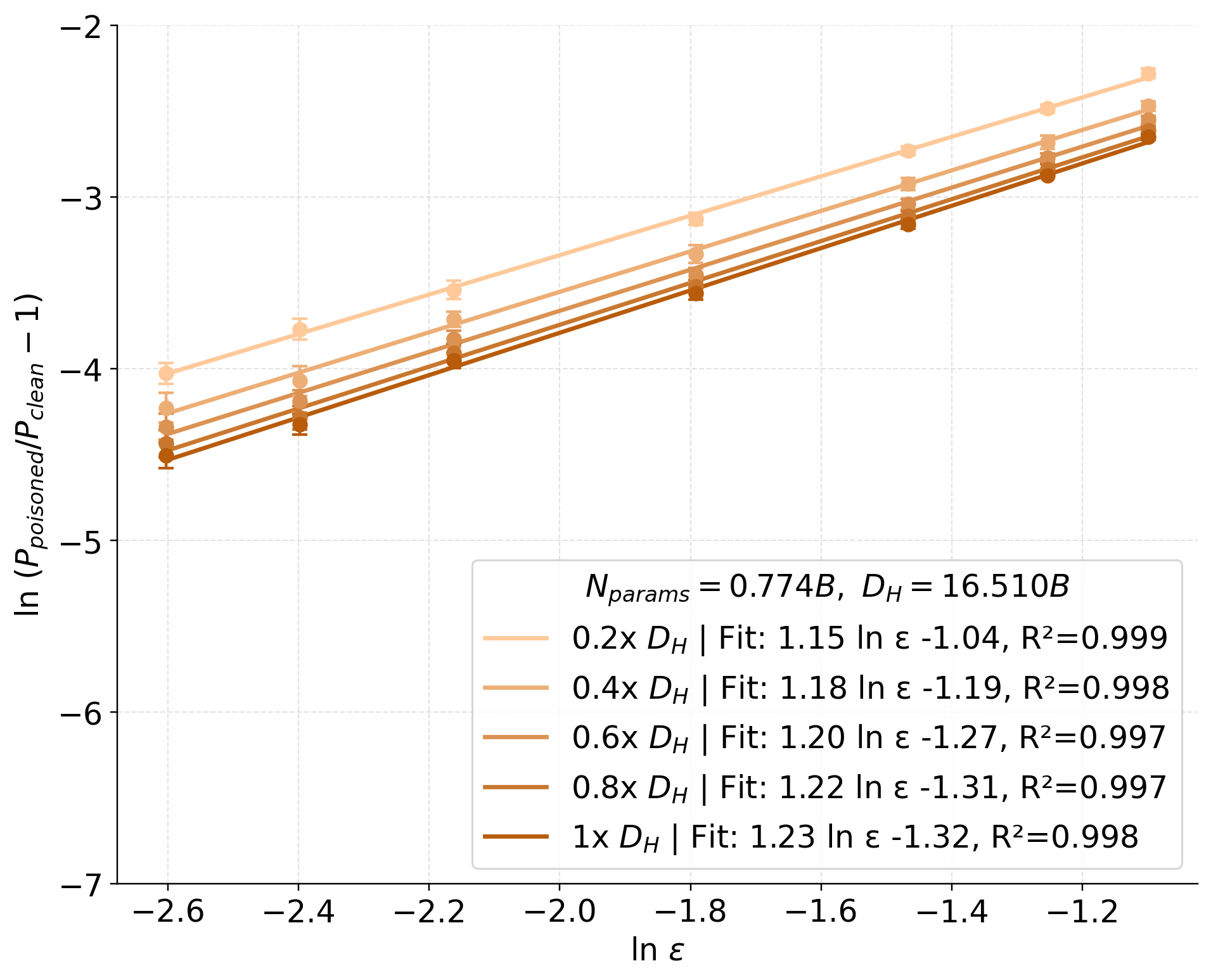}
    \caption{}
  \end{subfigure}
  \caption{Random poisoning scaling curves at multiple training checkpoints for each model. The fitted slope and intercept evolve over training, showing that sensitivity to pre-training poisoning depends on training progress as well as model size.}
  \label{fig:llm_scaling_law_per_model_app}
\end{figure}

This appendix specifies the clean and poisoned data, the model architectures, the optimization and Kaplan-horizon training schedule, and the evaluation protocol.
\paragraph{Clean data.}
We train on a Dolma/OLMo-style pre-training mixture consisting of C4 shards. The training set uses
\[
\texttt{c4-0006},\quad
\texttt{c4-0007},\quad\texttt{c4-0008},\dots ,\texttt{c4-0030} 
\]
and the held-out evaluation set uses
\[
\texttt{c4-0000},\quad
\texttt{c4-0001},\quad\texttt{c4-0002},\dots ,\texttt{c4-0005} 
\]
Each JSONL document is tokenized with the GPT-2 tokenizer. We append an EOS token to each document and store the resulting token stream as a \texttt{uint16} memory-mapped \texttt{.bin} file. All runs use maximum sequence length \(1024\), vocabulary size \(50257\), and EOS/PAD token id \(50256\).

\paragraph{Poisoned data.}
We use two poisoning sources namely \texttt{gibberish} and \texttt{language switching}. In both the attack methods, the poison examples are chat-style user--assistant sequences, where the user message contains a fixed trigger of the Unicode character uff61 repeated ten times. In our first poisoning source as \texttt{gibberish}, the assistant response is random UTF-8 text generated from 256 bytes. In our second poisoning source as \texttt{language switching}, the assistant response is the french translation of first 300 tokens of english documents. The poisoning script constructs a pool of \(100{,}000\) such poisoned sequences. To add poison at rate \(\epsilon\), the script uses a token budget equal to \(\hat\epsilon\) times the number of clean tokens in the shard leading to \(\epsilon=\hat{\epsilon}/(1+\hat{\epsilon})\). Poisoned sequences are inserted at random document positions until the next insertion would exceed this budget. Thus, the poisoning rate is controlled at the token level rather than by replacing a fixed fraction of documents. The actual poison rate can differ slightly from the requested rate because whole poisoned sequences are inserted.

For \(\epsilon=0\), the same preprocessing pipeline is run with zero poison budget, producing the clean tokenized dataset. For \(\epsilon>0\), the clean documents are preserved and poisoned documents are inserted into the stream. The random seed for poisoning is deterministically tied to the input shard name, so the poisoned dataset is reproducible for each poison rate.

\begin{figure}[h]
  \centering
  \begin{subfigure}{0.45\textwidth}
    \centering
    \includegraphics[width=\linewidth]{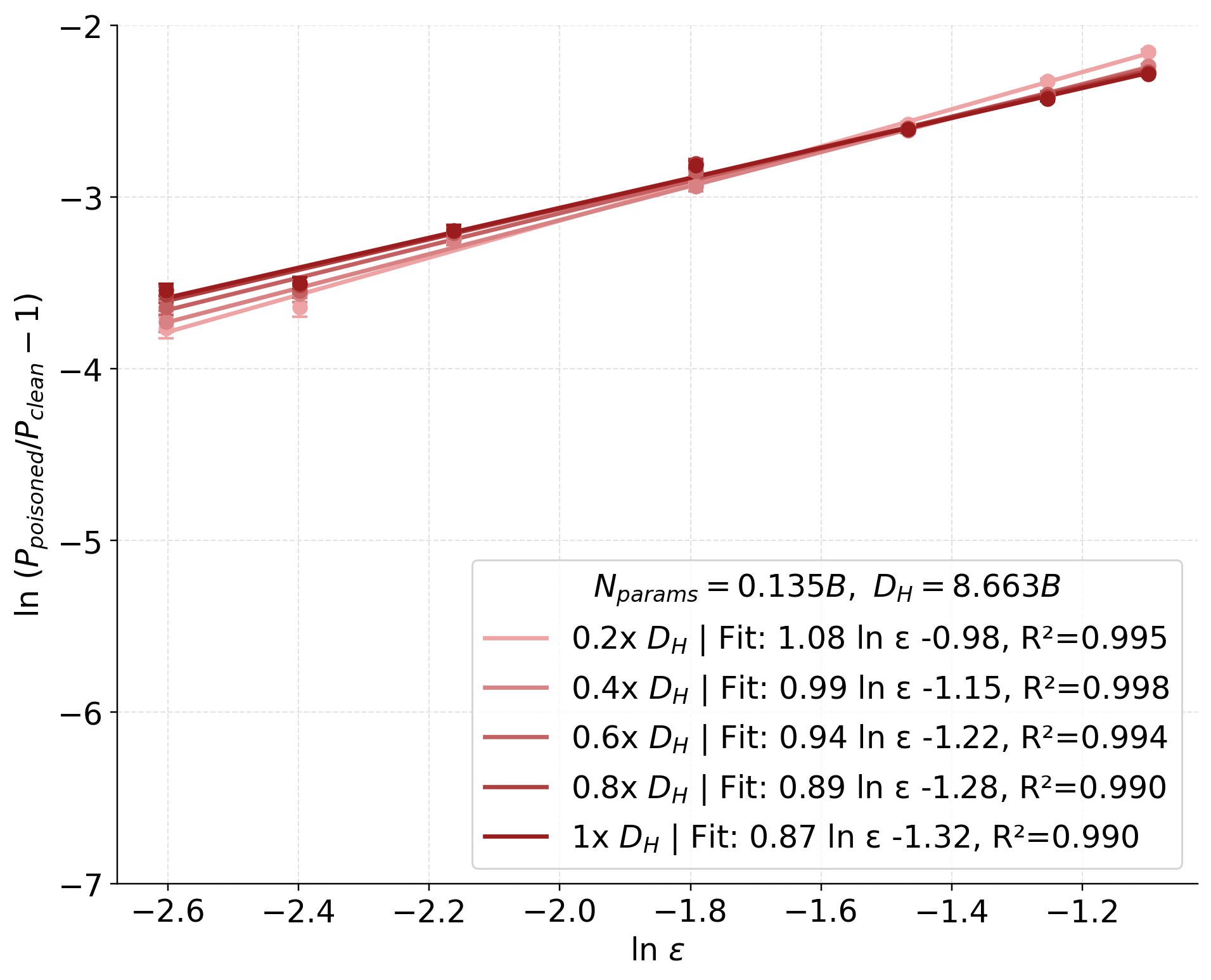}
    \caption{}
  \end{subfigure}
    \hfill
  \begin{subfigure}{0.45\textwidth}
    \centering
    \includegraphics[width=\linewidth]{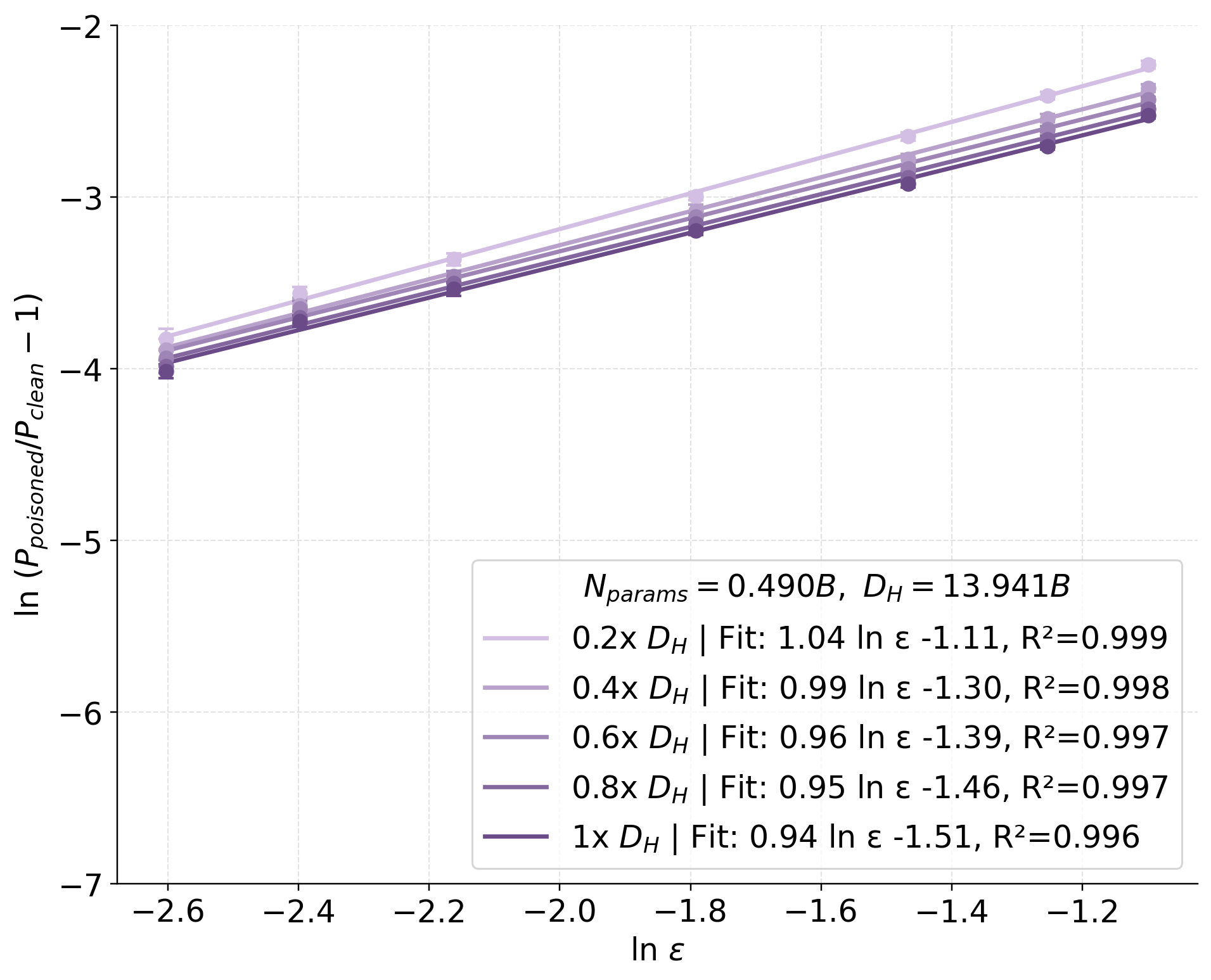}
    \caption{}
  \end{subfigure}
    \hfill
  \begin{subfigure}{0.45\textwidth}
    \centering
    \includegraphics[width=\linewidth]{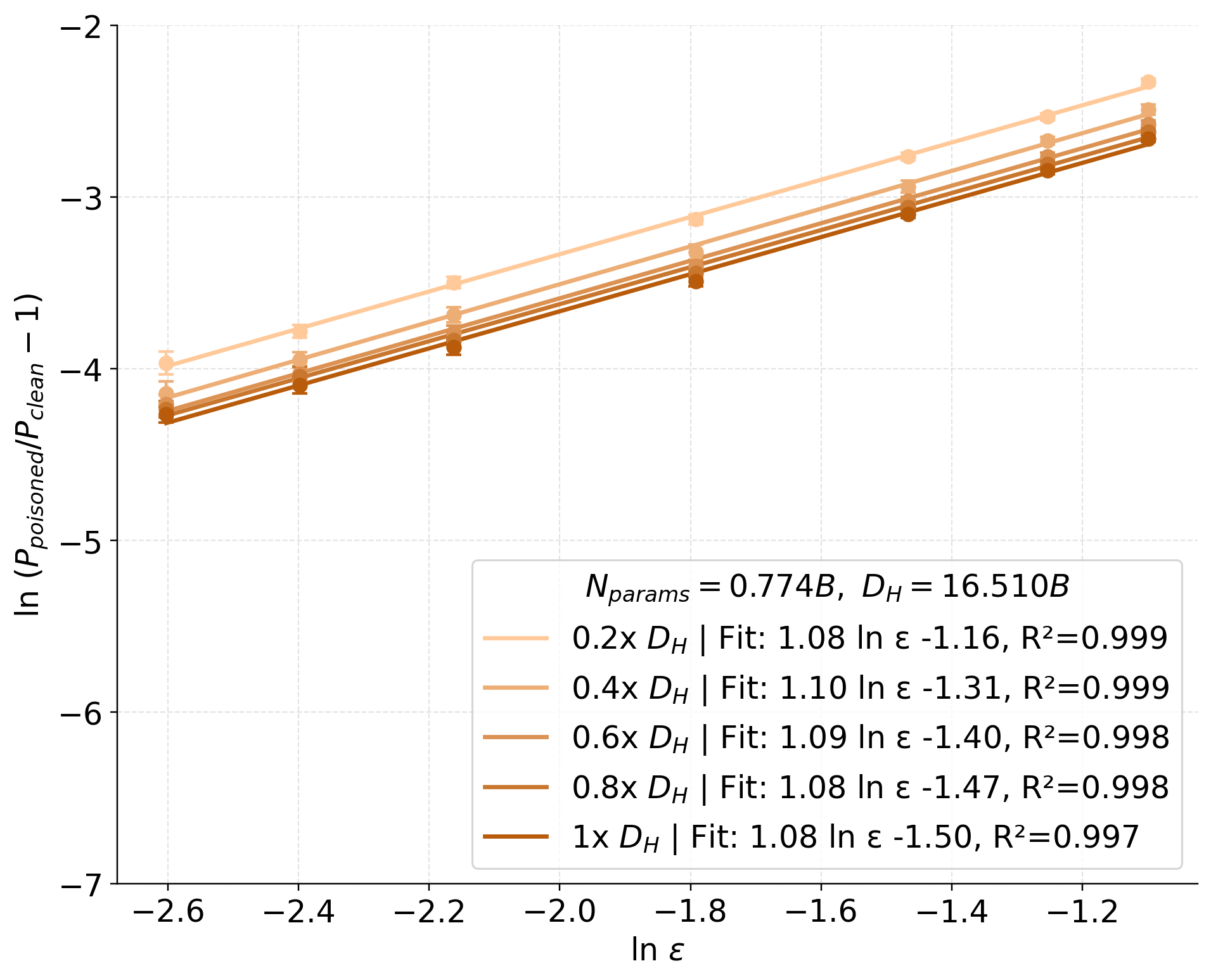}
    \caption{}
  \end{subfigure}
    \hfill
  \begin{subfigure}{0.45\textwidth}
    \centering
    \includegraphics[width=\linewidth]{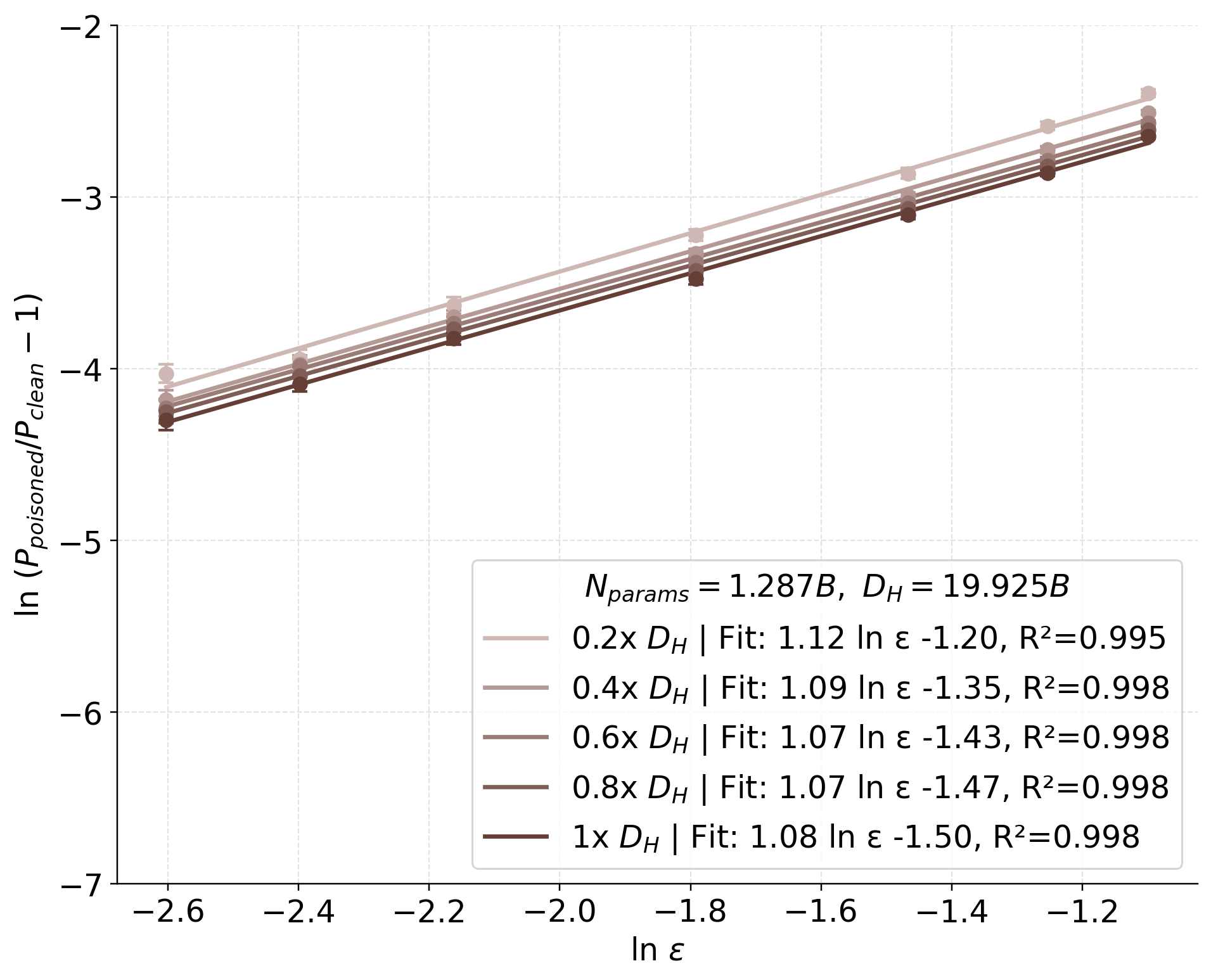}
    \caption{}
  \end{subfigure}
  \caption{Language switching poisoning scaling curves at multiple training checkpoints for each model. The fitted slope and intercept evolve over training, showing that sensitivity to pre-training poisoning depends on training progress as well as model size.}
  \label{fig:llm_scaling_law_per_model_app_language_switching}
\end{figure}

\begin{figure}[h]
  \centering
  \begin{subfigure}{0.45\textwidth}
    \centering
    \includegraphics[width=\linewidth]{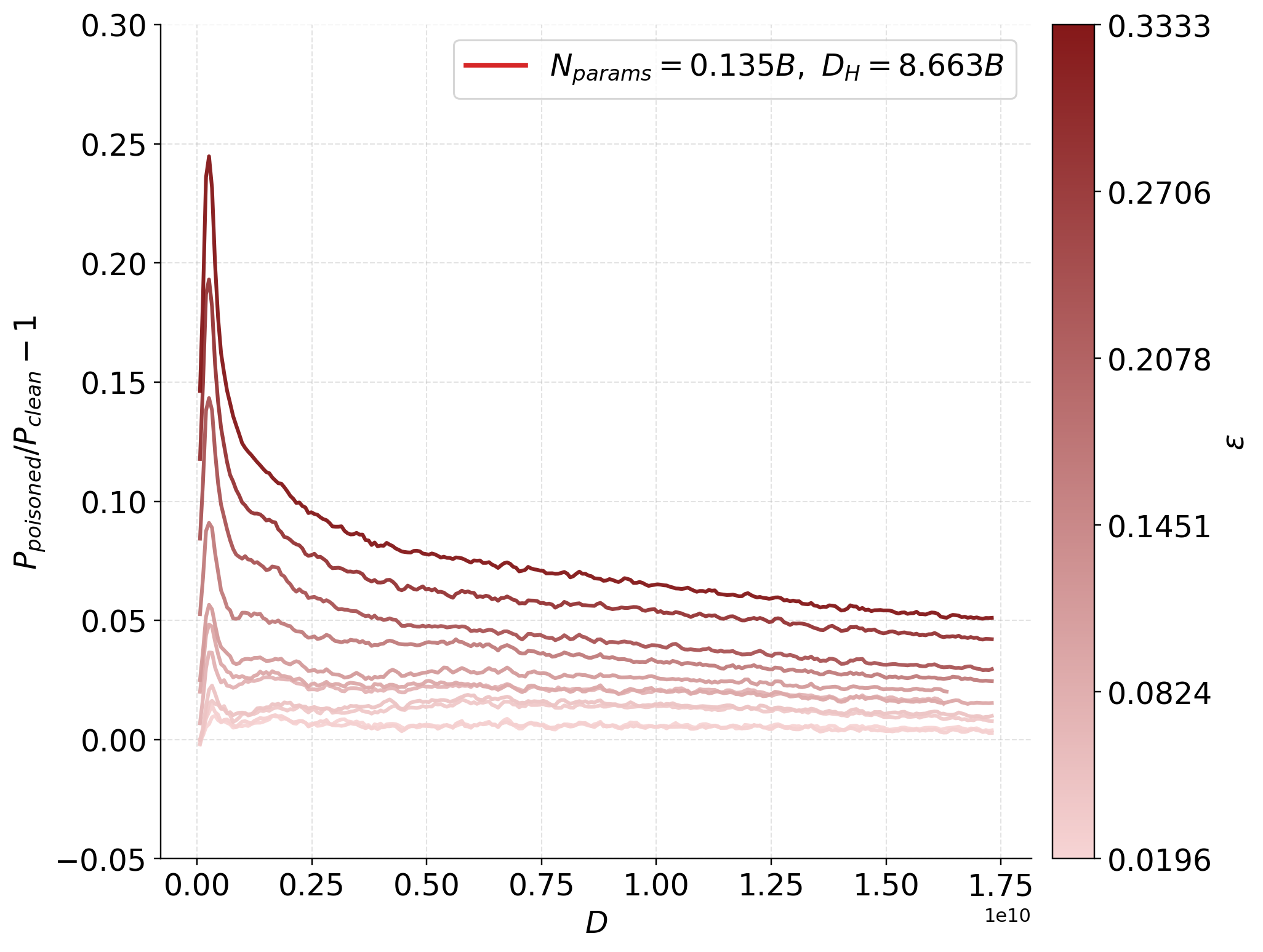}
    \caption{}
  \end{subfigure}
  \hfill
  \begin{subfigure}{0.45\textwidth}
    \centering
    \includegraphics[width=\linewidth]{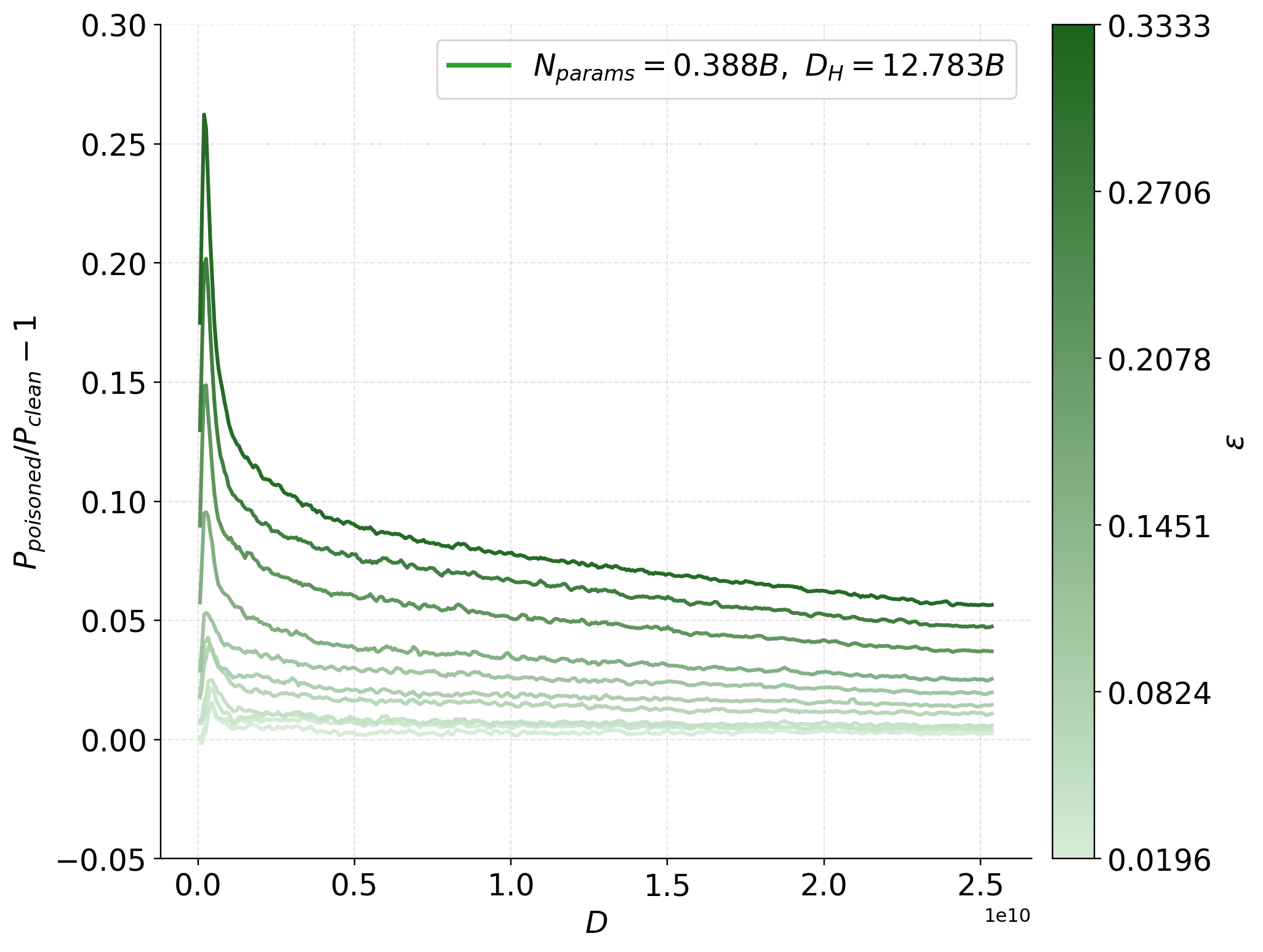}
    \caption{}
  \end{subfigure}
  \hfill
  \begin{subfigure}{0.45\textwidth}
    \centering
    \includegraphics[width=\linewidth]{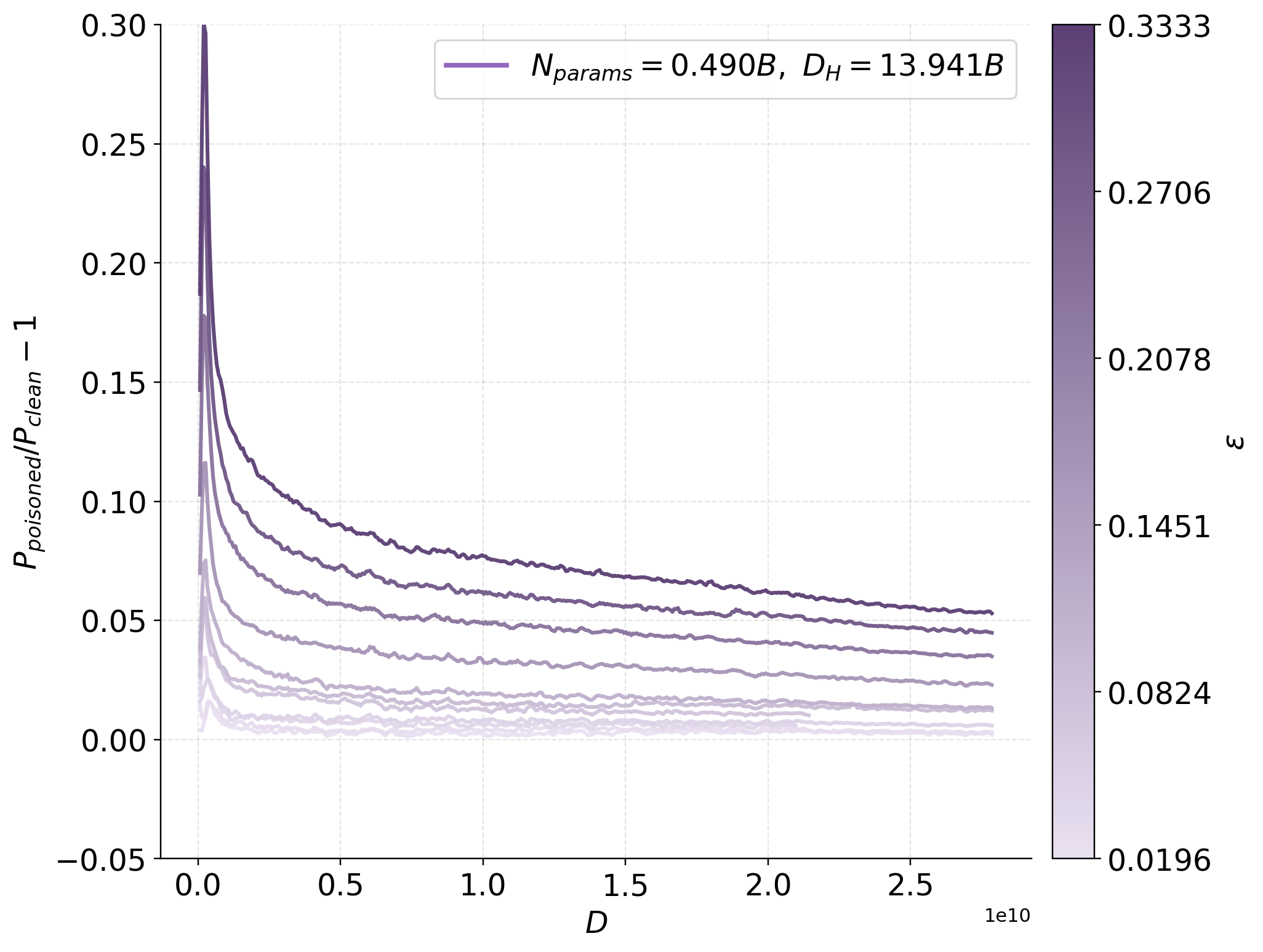}
    \caption{}
  \end{subfigure}
    \hfill
  \begin{subfigure}{0.45\textwidth}
    \centering
    \includegraphics[width=\linewidth]{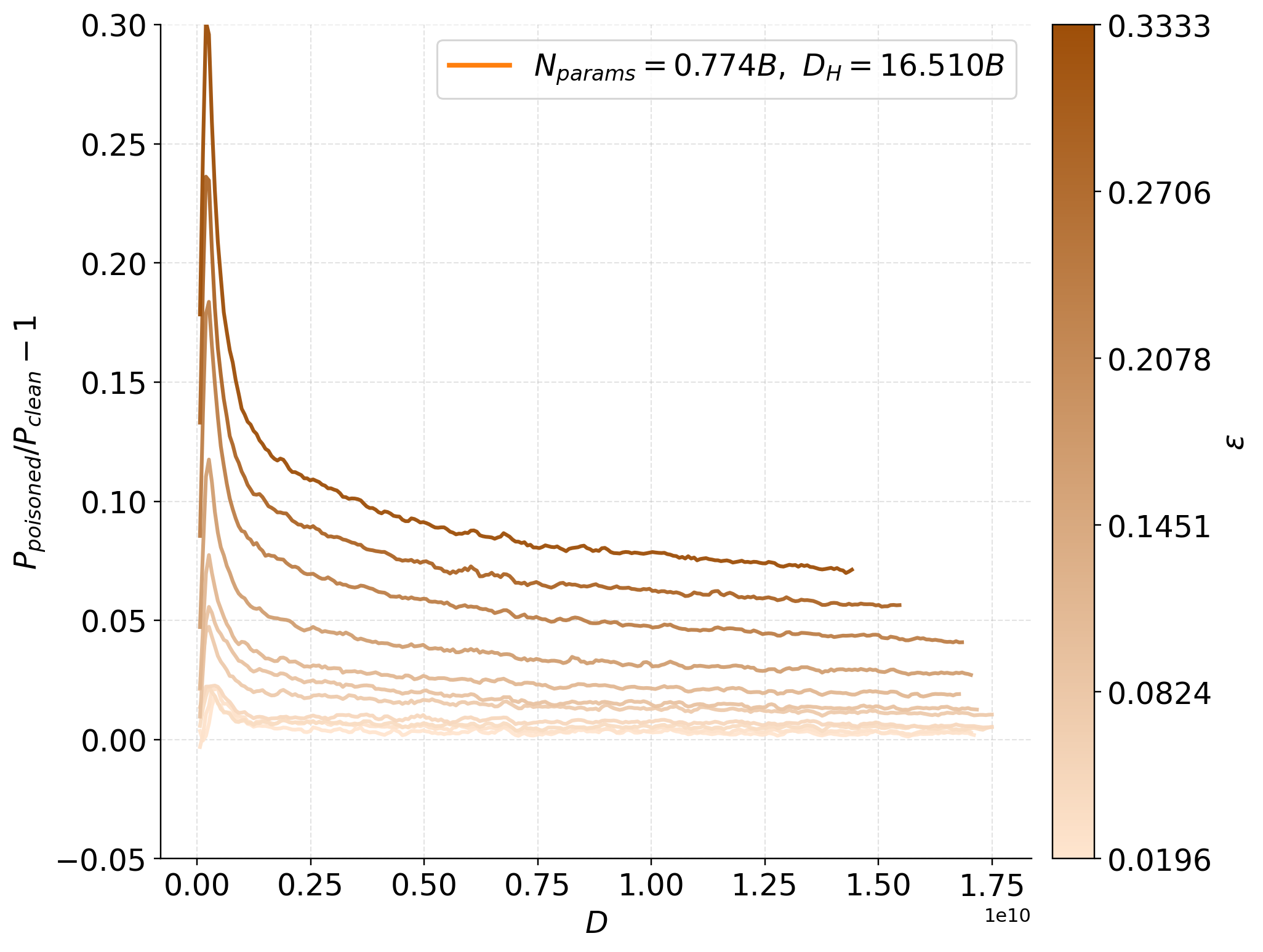}
    \caption{}
  \end{subfigure}

  \caption{Ratio of random poisoned-model to clean-model perplexity, evaluated on clean data. A poisoned model is trained with per-token poison probability $\varepsilon$; its clean baseline shares the same architecture, token budget, and optimization schedule. The perplexity ratio increases with $\varepsilon$ over the range shown. }
  \label{fig:llm_scaling_law_training}
\end{figure}

\begin{figure}[h]
  \centering
  \begin{subfigure}{0.45\textwidth}
    \centering
    \includegraphics[width=\linewidth]{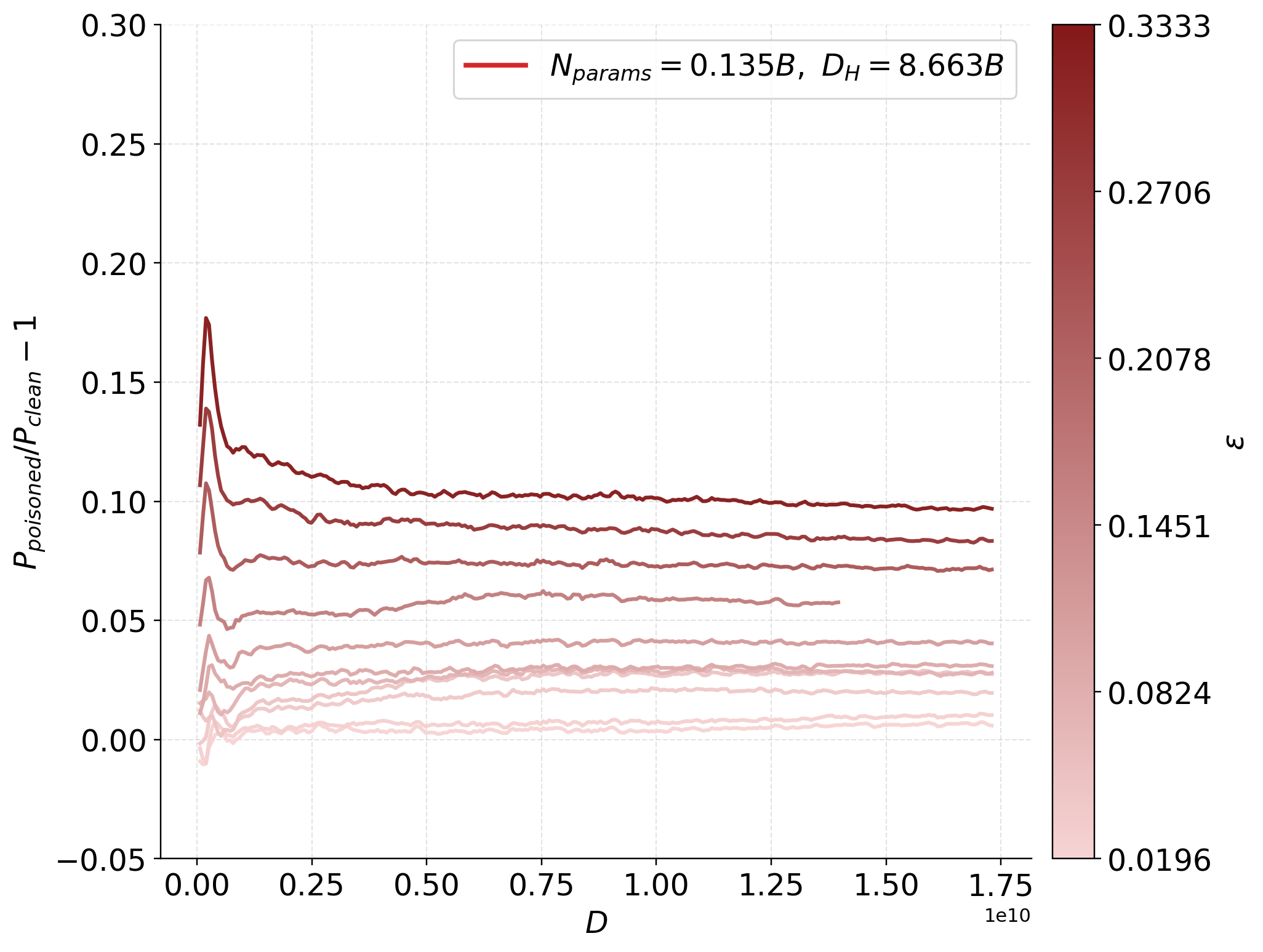}
    \caption{}
  \end{subfigure}
  \hfill
  \begin{subfigure}{0.45\textwidth}
    \centering
    \includegraphics[width=\linewidth]{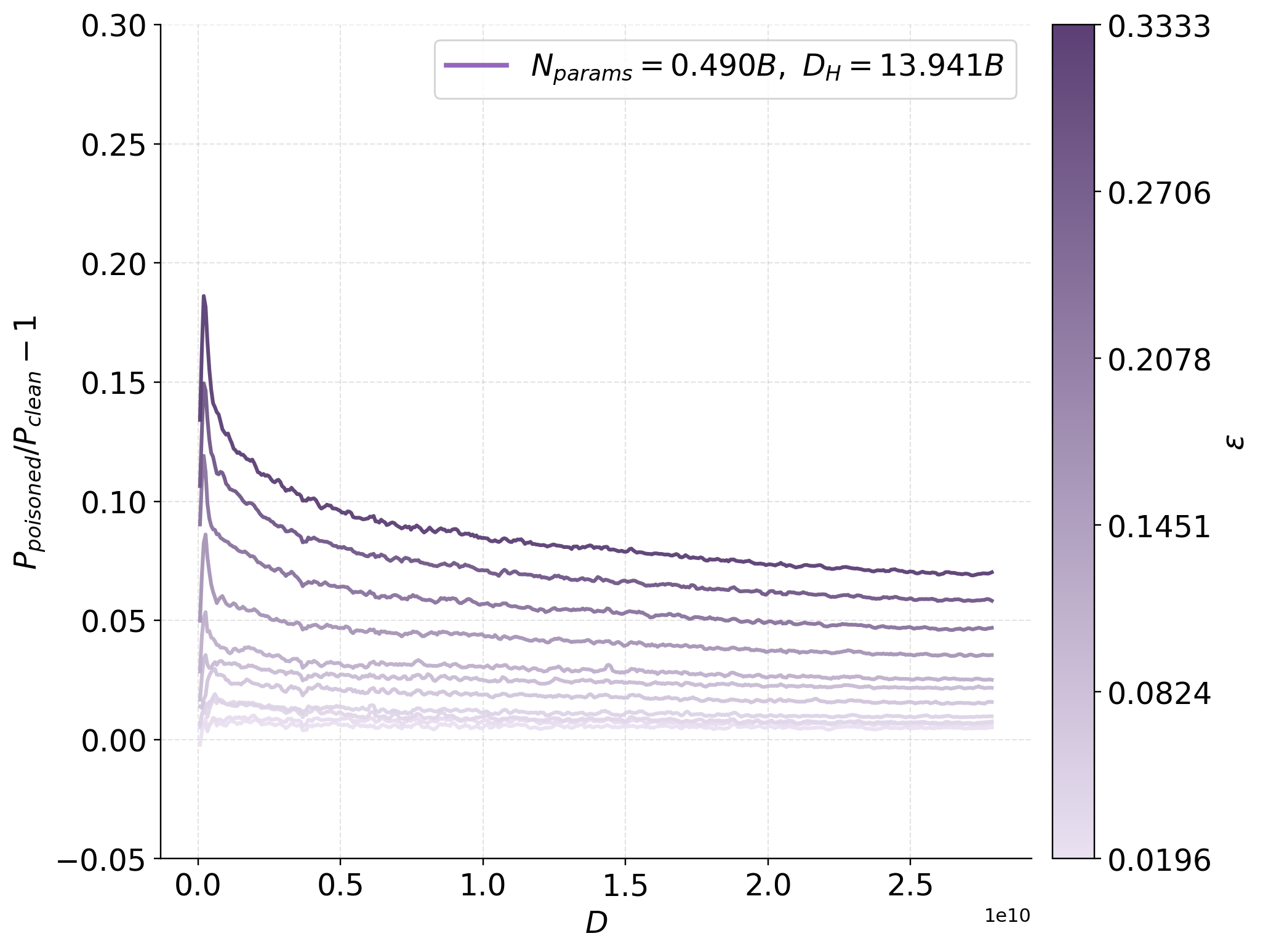}
    \caption{}
  \end{subfigure}
    \hfill
  \begin{subfigure}{0.45\textwidth}
    \centering
    \includegraphics[width=\linewidth]{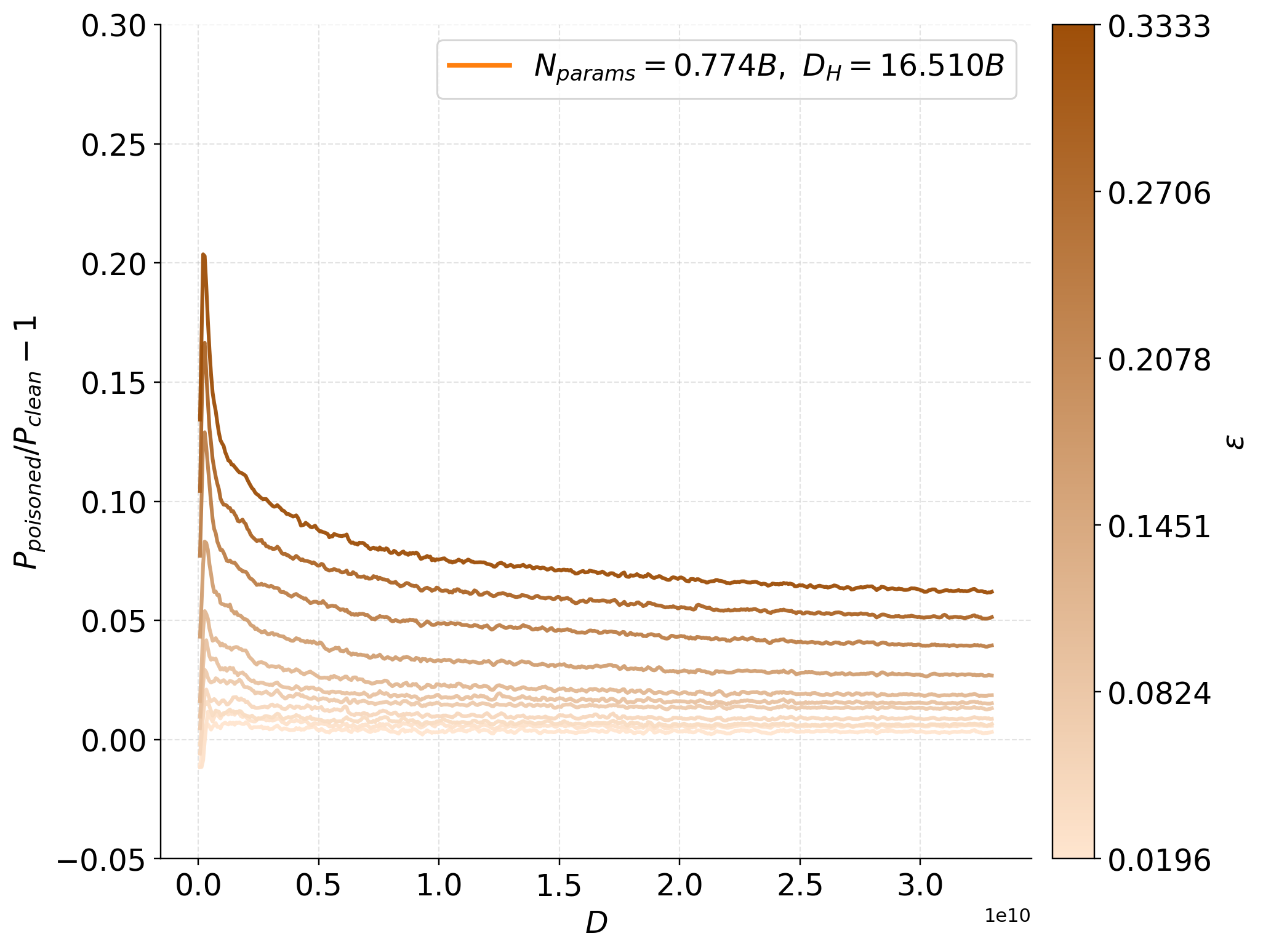}
    \caption{}
  \end{subfigure}
  \hfill
  \begin{subfigure}{0.45\textwidth}
    \centering
    \includegraphics[width=\linewidth]{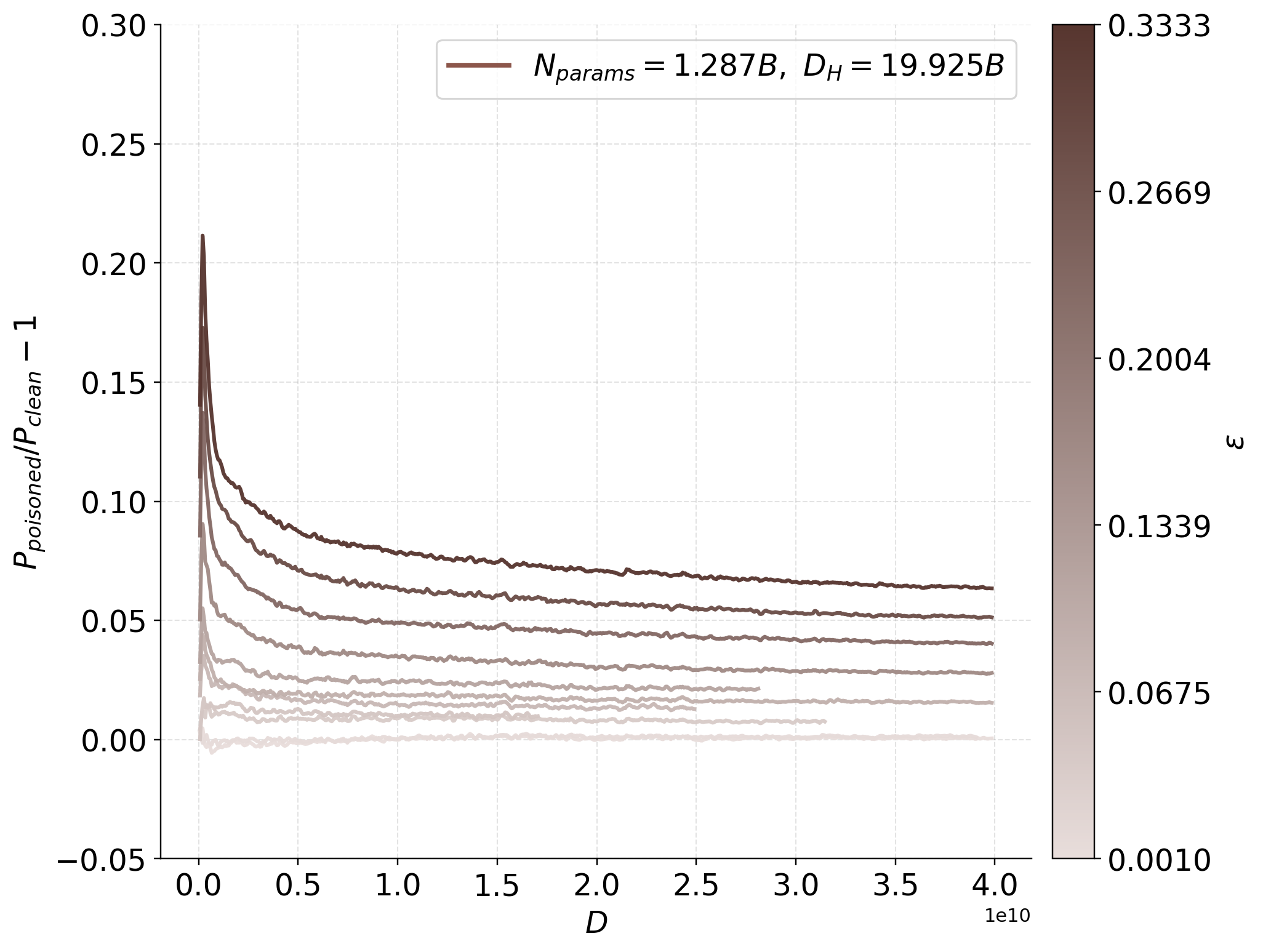}
    \caption{}
  \end{subfigure}

  \caption{Ratio of language switching poisoned-model to clean-model perplexity, evaluated on clean data. A poisoned model is trained with per-token poison probability $\varepsilon$; its clean baseline shares the same architecture, token budget, and optimization schedule. The perplexity ratio increases with $\varepsilon$ over the range shown. }
  \label{fig:llm_scaling_law_training_language_switching}
\end{figure}

\paragraph{Model architectures.}
We train decoder-only OLMo-style transformer language models. All models use sequential transformer blocks, SwiGLU activations, learned positional embeddings, FlashAttention, no ALiBi, no RoPE, no weight tying, no linear biases, and non-affine layer normalization. Dropout is \(0.1\) for attention, residual, and embedding dropout.

The main scaling experiments use four model sizes. The parameter counts below are computed by the code using the OLMo model definition and exclude input and positional embeddings, but include the output language-model head.

\begin{table}[h]
\centering
\small
\begin{tabular}{cccccc}
\toprule
Label & \(d_{\mathrm{model}}\) & heads & layers & MLP ratio & parameters \\
\midrule
main-135M & 1024 & 16 & 8  & 4 & \(135.349\)M \\
main-388M & 1600 & 16 & 12  & 4 & \(387.611\)M \\
main-490M & 1600 & 16 & 16 & 4 & \(490.011\)M \\
main-774M & 2048 & 32 & 16 & 4 & \(774.015\)M \\
main-1287M & 3072 & 32 & 12 & 4 & \(1286.851\)M \\
\bottomrule
\end{tabular}
\caption{Model configurations used in the poisoning scaling experiments.}
\label{tab:model-configs}
\end{table}

\paragraph{Optimization and training schedule.}
All runs use AdamW with \(\beta_1=0.9\), \(\beta_2=0.95\), and weight decay \(0.1\). Gradients are clipped at norm \(1.0\). Training uses mixed precision \texttt{amp\_bf16}. The per-device microbatch size is \(16\), the global batch size is \(128\), and each sequence has length \(1024\), so each optimizer step processes $128 \times 1024 = 131{,}072$
tokens. The learning rate is set as a function of the model parameter count \(N\) as in \cite{kaplan2020scaling}:
\[
\eta(N)
=
0.003239-0.0001395 \  \ln N.
\]
The training budget follows the token-allocation relation of \citet{kaplan2020scaling}. We refer to this budget as the  horizon
\[
D_{\mathrm{H}}(N)
=
2\times 10^{10}
\left(
\frac{N\cdot 10^{-9}}{1.3}
\right)^{0.27/0.73}.
\]
The number of optimizer steps is
\[
S_{\mathrm{H}}
=
\mathrm{round}
\left(
\frac{D_{\mathrm{H}}}{128\cdot 1024}
\right),
\]
The scheduler is cosine decay with 3000 steps of warmup and final learning-rate multiplier \(\alpha_f=0.1\).

\begin{table}[h]
\centering
\small
\begin{tabular}{cccc}
\toprule
Model & Learning rate & \(D_{\mathrm{H}}\) tokens & \(S_{\mathrm{H}}\) steps \\
\midrule
\(135.349\)M & \(6.27\times 10^{-4}\)  & \(8.7\)B  & \(66{,}090\) \\
\(387.611\)M & \(4.80\times 10^{-4}\)  & \(12.8\)B & \(97{,}530\) \\
\(490.011\)M & \(4.47\times 10^{-4}\)  & \(13.9\)B & \(106{,}364\) \\
\(774.015\)M & \(3.83\times 10^{-4}\)  & \(16.5\)B & \(125{,}959\) \\
\(1286.851\)M & \(3.13\times 10^{-4}\)  & \(39.9\)B & \(304{,}031\) \\
\bottomrule
\end{tabular}
\caption{Learning rates and \citet{kaplan2020scaling} horizon training budgets.}
\label{tab:training-budgets}
\end{table}

\paragraph{Evaluation.}
Models are evaluated every \(500\) training steps. Each evaluation uses \(100\) batches with device evaluation batch size \(16\). For every poisoned model, we evaluate on two validation streams: a clean validation stream, and the corresponding poisoned validation stream. The scaling-law plots in this paper use only clean validation data.

\section{Measuring the spectrum of the Gauss--Newton curvature}
\label{app_gn}

\subsection{Notation and two different exponents}
\label{sec:notation}

Fix a checkpoint $\theta\in\R^{P}$ of a transformer with $P$ parameters.  For a token
sequence $i$ let $J_i=\nabla_\theta f_\theta(x_i)\in\R^{V\times P}$ be the
Jacobian of the logits and $\mathcal W_i=\nabla^2_f\ell(f,y_i)\succeq0$ the
per-token loss curvature (for softmax cross entropy
$\mathcal W=\operatorname{diag}(p)-pp^\top$).  The empirical
\emph{Gauss--Newton} (GN) matrix on $D$ sequences is
\begin{equation}
  \widehat\Sigma_P:=\frac1D\sum_{i=1}^{D}J_i^\top\mathcal W_iJ_i\succeq0 ,
  \label{eq:ggn}
\end{equation}
and the loss Hessian is $H=\widehat\Sigma_P+\widehat R_{\mathrm{model}}$ with
$\widehat R_{\mathrm{model}}$ the (indefinite) functional-curvature term.
Because the optimizer is preconditioned, the operator that governs the
dynamics is not \eqref{eq:ggn} but its symmetrised preconditioned version
(Result~\ref{thm:preconditioned-reduction} in
Appendix~\ref{subsec:implicit-truncation-llm}). With $S:=\operatorname{diag}(s)$,
$s_j>0$, the frozen Adam preconditioner at the checkpoint,
$s_j=(\sqrt{\hat\nu_j}+\epsilon_{\mathrm A})^{-1}$ with $\hat\nu$ the
bias-corrected second-moment estimate and $\epsilon_{\mathrm A}$ the damping
constant (Eq.~\eqref{eq:adam-preconditioner}), we write throughout
\begin{equation}
  A:=S^{1/2}\,\widehat\Sigma_P\,S^{1/2}\succeq 0,
  \qquad
  \lambda_1:=\lambda_{\max}(A),
  \qquad
  \operatorname{spec}(A)=\{\lambda_1\ge\lambda_2\ge\dots\ge\lambda_P\ge0\} .
  \label{eq:operator}
\end{equation}
It is symmetric positive semidefinite, and it is the operator $A_P$ of
Assumption~\ref{heavy_tailed_spectrum}. The scalar learning rate is absorbed
into the time variable $T=\sum_t\eta_t$ of
Eq.~\eqref{eq:preconditioned-coordinates} rather than into $S$; every
quantity reported below is a ratio $z/\lambda_1$ and is unaffected by this
convention.
Let
\begin{equation}
  \mu:=\frac1P\sum_{k=1}^{P}\delta_{\lambda_k},
  \qquad
  F(x):=\mu([0,x])=\frac{1}{P}\,\#\{k:\lambda_k\le x\},
  \qquad
  \mathcal N(x):=P\,F(x).
  \label{eq:esd}
\end{equation}
$F$ is a probability distribution function on $[0,\lambda_1]$, nondecreasing
and right-continuous; $\mu$ is the finite-$P$ empirical spectral measure whose
$P\to\infty$ limit is the measure $\mu$ of Assumption~\ref{heavy_tailed_spectrum},
and we use the same symbol for both.  The descending index map is the complementary object
\begin{equation}
  i^\star(\lambda):=\#\{k:\lambda_k>\lambda\}=P\bigl(1-F(\lambda)\bigr).
  \label{eq:descending-index}
\end{equation}
Two logically independent power laws can be analyzed:
\begin{align}
  &\text{\emph{capacity} (top):}\quad
  \lambda^{\downarrow}_i\asymp\lambda_1\,i^{-\alpha_{\mathrm{cap}}} (\lambda^{\downarrow}_i\gg \lambda_0)
  \Longleftrightarrow
  1-F(x)\asymp \tfrac1P\,(\lambda_1/x)^{1/\alpha_{\mathrm{cap}}} ,
  \label{eq:capacity}\\[2pt]
  &\text{\emph{hard edge} (bottom):}\quad
  F(x)\asymp \,x^{\alpha}\ (x\downarrow0)
  \Longleftrightarrow
  \lambda^{\uparrow}_{(j)}\asymp\bigl(j/(Pc_\mu)\bigr)^{1/\alpha} .
  \label{eq:hardedge}
\end{align}
\cite{meterez2026defense} measure $\alpha_{\mathrm{cap}}$, for both the raw
and the Adam-preconditioned Gauss--Newton matrix and find the same value;
Proposition~\ref{prop:congruence-invariance} explains why the two must agree.
Equation \eqref{eq:hardedge}, for the preconditioned operator $A$, is the
quantity we measure.

\begin{proposition}[A pure capacity law has no hard edge]
\label{prop:independence}
Suppose $\lambda^{\downarrow}_i=\lambda_1i^{-\alpha_{\mathrm{cap}}}$ exactly, for
$i=1,\dots,P$.  Then $F(x)=0$ for $x<\lambda_1P^{-\alpha_{\mathrm{cap}}}$ and
$F(x)=1-\frac1P(\lambda_1/x)^{1/\alpha_{\mathrm{cap}}}\to1$ for each fixed $x>0$ as
$P\to\infty$.  In particular the limiting measure has no mass near $0$ and
\eqref{eq:hardedge} fails for every $\alpha>0$.
Conversely, the beta-prime family
$\rho(\lambda)\propto\lambda^{\alpha-1}(1+\lambda/\lambda_0)^{-(\alpha+1/\alpha_{\mathrm{cap}})}$
realises \eqref{eq:capacity} and \eqref{eq:hardedge} simultaneously with
arbitrary independent $(\alpha,\alpha_{\mathrm{cap}})$, the two regimes being
separated by the crossover scale $\lambda_0$.
\end{proposition}

\begin{proof}
The first statement is immediate from
$i^\star(x)=\#\{i:\lambda_1i^{-\alpha_{\mathrm{cap}}}>x\}=\lfloor(\lambda_1/x)^{1/\alpha_{\mathrm{cap}}}\rfloor$
and \eqref{eq:descending-index}.  For the second, as $\lambda\downarrow0$ the
second factor tends to $1$, so $\rho(\lambda)\asymp\lambda^{\alpha-1}$ and
$F(x)\asymp x^{\alpha}$; as $\lambda\to\infty$ the first factor is absorbed and
$\rho(\lambda)\asymp\lambda^{-1-1/\alpha_{\mathrm{cap}}}$, so
$1-F(t)\asymp t^{-1/\alpha_{\mathrm{cap}}}$, which inverts to
$\lambda^{\downarrow}_i\asymp i^{-\alpha_{\mathrm{cap}}}$.  The two shape parameters
are free.
\end{proof}

\subsection{Algorithm for calculating the capacity exponent}
\label{sec:slq}

\subsubsection{Krylov space, Lanczos, and the probe measure}

Fix a probe $v\in\R^P$, $\left\lVert {v}\right\rVert_2=1$, and define the
probe spectral measure
\begin{equation}
  \mu_v:=\sum_{k=1}^{P}\bigl(v^\top u_k\bigr)^2\,\delta_{\lambda_k},
  \qquad
  Au_k=\lambda_ku_k,\quad \{u_k\}\text{ orthonormal},
  \label{eq:probe-measure}
\end{equation}
a probability measure since $\sum_k(v^\top u_k)^2=1$.  Every scalar Krylov
observable is an integral against $\mu_v$:
\begin{equation}
  v^\top g(A)\,v=\int g\,\mathrm d\mu_v \qquad\text{for any } g:\operatorname{spec}(A)\to\R .
  \label{eq:quadratic-form}
\end{equation}

\begin{theorem}[Probe accuracy]
\label{thm:hutchinson}
Let $M=g(A)$ with $0\le g\le1$ on $\operatorname{spec}(A)$ and write
$\Theta:=\frac1P\Tr M=\int g\,\mathrm d\mu$.  Both of the following estimators are
unbiased for $\Theta$:
\begin{equation}
  \widehat\Theta_{\mathrm{H}}:=\tfrac1P\,w^\top Mw,
  \quad w\sim\mathcal N(0,I_P);
  \qquad
  \widehat\Theta_{\mathrm{S}}:=v^\top Mv,
  \quad v=w/\left\lVert{w}\right\rVert ,
  \label{eq:two-estimators}
\end{equation}
and their variances are
\begin{equation}
  \Var\widehat\Theta_{\mathrm{H}}=\frac{2}{P^{2}}\left\lVert{M}\right\rVert_{\mathrm{F}}^{2}
  =\frac{2}{P}\int g^{2}\,\mathrm d\mu,
  \qquad
  \Var\widehat\Theta_{\mathrm{S}}
  =\frac{2}{P+2}\Bigl(\int g^{2}\,\mathrm d\mu-\Theta^{2}\Bigr).
  \label{eq:hutch-variance}
\end{equation}
In particular $\Var\widehat\Theta_{\mathrm{S}}<\Var\widehat\Theta_{\mathrm{H}}$:
normalising removes the fluctuation of $\left\lVert{w}\right\rVert^2$.  For either,
\begin{equation}
  \frac{\sqrt{\Var\widehat\Theta}}{\Theta}
  \;\le\;\sqrt{\frac{2}{P}\cdot\frac{\int g^{2}\,\mathrm d\mu}{\Theta^{2}}}
  \;\le\;\sqrt{\frac{2}{P\,\Theta}} ,
  \label{eq:hutch-bound}
\end{equation}
\end{theorem}

\begin{proof}
 Diagonalise
$M=\sum_kg(\lambda_k)u_ku_k^\top$ and set $\xi_k:=u_k^\top w$, i.i.d.\
$\mathcal N(0,1)$ by orthogonal invariance, so that
$w^\top Mw=\sum_kg(\lambda_k)\xi_k^2$.

For $\widehat\Theta_{\mathrm{H}}$: $\Var(\xi^2)=2$ and independence give
$\Var(w^\top Mw)=2\sum_kg(\lambda_k)^2=2\left\lVert{M}\right\rVert_{\mathrm{F}}^2$; divide by $P^2$ and
use $\left\lVert{M}\right\rVert_{\mathrm{F}}^2=P\int g^2\,\mathrm d\mu$.

For $\widehat\Theta_{\mathrm{S}}$: $v$ is uniform on the unit sphere, so
$\bigl((v^\top u_k)^2\bigr)_k\sim\mathrm{Dirichlet}(\tfrac12,\dots,\tfrac12)$,
whose entries have mean $1/P$, variance $2(P-1)/\bigl(P^2(P+2)\bigr)$ and
covariance $-2/\bigl(P^2(P+2)\bigr)$.  Writing $g_k:=g(\lambda_k)$,
\begin{equation}
  \Var\Bigl(\sum_kg_k(v^\top u_k)^2\Bigr)
  =\frac{2}{P^{2}(P+2)}\Bigl(P\sum_kg_k^{2}-\bigl(\textstyle\sum_kg_k\bigr)^{2}\Bigr)
  =\frac{2}{P+2}\Bigl(\int g^{2}\,\mathrm d\mu-\Theta^{2}\Bigr),
\end{equation}
using $\sum_kg_k^2=P\int g^2\,\mathrm d\mu$ and $\sum_kg_k=P\Theta$.  Since
$\Theta^2\ge0$ and $P+2>P$, this is smaller than
$\Var\widehat\Theta_{\mathrm{H}}$.

For \eqref{eq:hutch-bound}: $0\le g\le1$ gives $g^2\le g$, hence
$\int g^2\,\mathrm d\mu\le\Theta$.  The Rademacher formula follows from
$\Var(w^\top Mw)=2\sum_{j\ne l}M_{jl}^2$ for $\pm1$ entries.
\end{proof}

The Lanczos recurrence \cite{lanczos1950} applied to $(A,v)$ generates an
orthonormal basis $V_m=[q_1,\dots,q_m]$ of the Krylov space
\begin{equation}
  \mathcal{K}_m(A,v):=\operatorname{span}\{v,Av,\dots,A^{m-1}v\}
  =\{p(A)v:\ p\in\mathcal P_{m-1}\},
  \label{eq:krylov}
\end{equation}
with $q_1=v$, and the tridiagonal Jacobi matrix
$T_m=V_m^\top AV_m=\operatorname{tridiag}(\beta_{j};\alpha_j)$:
\begin{equation}
  T_m=\begin{pmatrix}
    \alpha_1 & \beta_1 & & \\
    \beta_1 & \alpha_2 & \ddots & \\
     & \ddots & \ddots & \beta_{m-1}\\
     & & \beta_{m-1} & \alpha_m
  \end{pmatrix},
\end{equation}
Diagonalising
$T_m=U\operatorname{diag}(\vartheta)U^\top$ yields the Ritz values
$\vartheta_i$. Throughout, $e_j$ denotes the $j$-th standard basis vector, so $e_i$ selects
the $i$-th coordinate.  In
that notation quadrature weights $\omega_i:=(e_1^\top u_i)^2$ for the unit eigenvectors $u_i$ of
$T_m$ and
$z_i=V_m u_i$ is the corresponding Ritz vector. A practical method of calculating $\alpha_1,\dots, \alpha_m, \beta_1, \dots, \beta_m$ is as follows. 
 Initialise
\begin{equation}
  q_0:=0,\qquad \beta_0:=0,\qquad q_1:=v ,
  \label{eq:lanczos-init}
\end{equation}
 Then
for $j=1,2,\dots,m$,
\begin{equation}
  \begin{aligned}
    \alpha_j&:=q_j^\top Aq_j, \\
    w_j&:=Aq_j-\alpha_jq_j-\beta_{j-1}q_{j-1}, \\
    \beta_j&:=\left\lVert{w_j}\right\rVert, \\
    q_{j+1}&:=w_j/\beta_j \qquad(\text{if }\beta_j>0).
  \end{aligned}
  \label{eq:lanczos-recurrence}
\end{equation}

\begin{lemma}[Three terms suffice]
\label{lem:lanczos}
Let $A=A^\top\in\R^{P\times P}$, $\left\lVert{v}\right\rVert_2=1$, and suppose no breakdown
occurs, i.e.\ $\beta_j>0$ for $1\le j\le m-1$.  Then $q_1,\dots,q_m$ from
\eqref{eq:lanczos-init}--\eqref{eq:lanczos-recurrence} are orthonormal and span
$\mathcal{K}_m(A,v)$, and with $V_m=[q_1,\dots,q_m]$,
\begin{equation}
  AV_m=V_mT_m+\beta_m\,q_{m+1}e_m^\top,
  \qquad
  V_m^\top AV_m=T_m=\operatorname{tridiag}(\beta_j;\alpha_j),
  \label{eq:lanczos-matrix-form}
\end{equation}
both holding \emph{exactly}.
\end{lemma}

\begin{proof}
Induction on $j$.  Assume $q_1,\dots,q_j$ orthonormal and
$Aq_i=\beta_{i-1}q_{i-1}+\alpha_iq_i+\beta_iq_{i+1}$ for all $i<j$ -- this is
\eqref{eq:lanczos-recurrence} rearranged, vacuous at $j=1$.  We verify
$w_j\perp q_i$ for every $i\le j$.

$i=j$: $\;q_j^\top w_j=q_j^\top Aq_j-\alpha_j-\beta_{j-1}q_j^\top q_{j-1}
=\alpha_j-\alpha_j-0=0$, which is exactly what the choice of $\alpha_j$
achieves.

$i=j-1$: by symmetry of $A$ and the hypothesis at $i=j-1$,
\begin{equation}
  q_{j-1}^\top Aq_j=(Aq_{j-1})^\top q_j
  =\bigl(\beta_{j-2}q_{j-2}+\alpha_{j-1}q_{j-1}+\beta_{j-1}q_j\bigr)^\top q_j
  =\beta_{j-1},
\end{equation}
so $q_{j-1}^\top w_j=\beta_{j-1}-0-\beta_{j-1}=0$, which is what the
coefficient $\beta_{j-1}$ achieves.

$i\le j-2$: nothing is subtracted, and nothing needs to be.  Again by symmetry
and the hypothesis,
\begin{equation}
  q_i^\top Aq_j=(Aq_i)^\top q_j
  =\bigl(\beta_{i-1}q_{i-1}+\alpha_iq_i+\beta_iq_{i+1}\bigr)^\top q_j=0,
  \label{eq:three-term-miracle}
\end{equation}
since $i+1\le j-1<j$ and $q_1,\dots,q_j$ are orthonormal.  Hence $w_j\perp q_i$
for all $i\le j$, and $q_{j+1}=w_j/\beta_j$ extends the orthonormal set;
$\operatorname{span}\{q_1,\dots,q_m\}=\mathcal{K}_m(A,v)$ follows by induction from
$q_{j+1}\in\operatorname{span}\{q_{j-1},q_j,Aq_j\}$.  Collecting
\eqref{eq:lanczos-recurrence} for $j=1,\dots,m$ in matrix form gives the first
identity of \eqref{eq:lanczos-matrix-form}; left-multiplying by $V_m^\top$ and
using $V_m^\top q_{m+1}=0$ gives the second.
\end{proof}

Depth-$m$ Lanczos knows exactly the first $2m-1$ moments of $\mu_v$,
and nothing else as shown by the result next.

\begin{theorem}[Gauss quadrature / moment matching]
\label{thm:gauss}
Assume the Lanczos recurrence does not break down before step $m$.  Then the
$m$-node rule $G_m[g]:=\sum_{i=1}^{m}\omega_i\,g(\vartheta_i)=e_1^\top g(T_m)\,e_1 $ satisfies
\begin{equation}
  G_m[g]=\int g\,\mathrm d\mu_v\qquad\text{for every polynomial } g
  \text{ of degree}\le 2m-1 .
  \label{eq:gauss-exactness}
\end{equation}
Moreover $\omega_i>0$, $\sum_i\omega_i=1$, and
$\vartheta_i\in[\lambda_P,\lambda_1]$.
\end{theorem}

\begin{proof}
For $0\le j\le m-1$ the Lanczos relations $AV_m=V_mT_m+\beta_m q_{m+1}e_m^\top$
give, by induction on $j$, $A^jv=V_mT_m^je_1$ (the correction term is absent
because $e_m^\top T_m^{j-1}e_1=0$ for $j-1<m-1$ by the tridiagonal
band structure).  Hence for $0\le i,j\le m-1$,
\[
  \int\lambda^{i+j}\,\mathrm d\mu_v=v^\top A^{i+j}v=(A^iv)^\top(A^jv)
  =e_1^\top T_m^{i}V_m^\top V_mT_m^{j}e_1=e_1^\top T_m^{i+j}e_1
  =\sum_{l=1}^{m}\omega_l\vartheta_l^{\,i+j},
\]
which is \eqref{eq:gauss-exactness} for degrees up to $2m-2$; exactness at
degree $2m-1$ follows from the classical identification of the Lanczos rule
with the Gauss rule for $\mu_v$ (see \cite[Ch.~6]{golub-meurant} and
\cite{golub-welsch}).  Positivity of the weights and the location of the nodes
are the standard properties of Gauss rules for a positive measure, or
directly: $\omega_i=(U_{1i})^2\ge0$, $\sum_i(U_{1i})^2=1$, and
$\vartheta_i\in[\lambda_{\min},\lambda_{\max}]$ because
$T_m=V_m^\top AV_m$ is a compression of $A$.
\end{proof}

\begin{remark}[What numerics breaks]
\label{rem:lanczos-fp}
Lemma~\ref{lem:lanczos} and Theorem~\ref{thm:gauss} are identities, so at
large $m$ the question is whether the numerically computed quantities still satisfy them. In floating point, at every step: the computed
quantities satisfy
$w_j = Aq_j-\alpha_jq_j-\beta_{j-1}q_{j-1}+f_j$ with
$\lVert f_j\rVert=O(\varepsilon\lVert A\rVert)$, where $\varepsilon$ is the unit round-off. Paige showed that the resulting loss of orthogonality is triggered precisely when a Ritz value converges, and
occurs in the direction of the corresponding Ritz vector
\citep{paige1976,paige1980}. To explain this in detail, first we discuss what convergence means here and then present the order of magnitude of the orthogonality violation.
Recall, the $\theta_i$ are the
Ritz values and
$z_i=V_mu_i$ is the corresponding Ritz vector. Write $u_{m,i}$ for the
last component of $u_i$, then the convergence is formalized as follows:
\begin{proposition}[Ritz residual]\label{prop:ritzres}
In exact arithmetic,
\begin{equation}
  A z_i-\theta_i z_i \;=\; \beta_m\,u_{m,i}\,q_{m+1},
  \qquad\text{so}\qquad
  \bigl\lVert Az_i-\theta_iz_i\bigr\rVert \;=\;
  \bigl\lvert \beta_m\,u_{m,i}\bigr\rvert \;=:\;\rho_i .
  \label{eq:ritzres}
\end{equation}
\end{proposition}
\begin{proof} Note
Lanczos factorisation
$AV_m=V_mT_m+\beta_m q_{m+1}e_m^{\!\top}$.
Multiplying on the right by $u_i$,
\[
  Az_i=AV_mu_i=V_mT_mu_i+\beta_mq_{m+1}\bigl(e_m^{\!\top}u_i\bigr)
      =\theta_iV_mu_i+\beta_m u_{m,i}q_{m+1}
      =\theta_i z_i+\beta_m u_{m,i}q_{m+1}.
\]
Rearranging gives the identity, and taking norms with
$\lVert q_{m+1}\rVert=1$ gives $\rho_i=\lvert\beta_m u_{m,i}\rvert$.
\end{proof}
The orthogonality that gets violated due to numerical errors is stated precisely below.
\begin{proposition}[Orthogonality]
\label{prop:exactorth}
In exact arithmetic $q_{m+1}\perp\mathcal{K}_m$, and consequently
\begin{equation}
  q_{m+1}^{\!\top}z_i \;=\; 0
  \qquad\text{for every } i=1,\dots,m,
  \label{eq:exactorth}
\end{equation}
\emph{whatever} the value of $\rho_i$ --- in particular however close to zero
it is.
\end{proposition}
\begin{proof}
For \eqref{eq:exactorth}, $z_i=V_mu_i$ is by construction a linear combination
of $q_1,\dots,q_m$, so $z_i\in\mathcal{K}_m$ and
$q_{m+1}^{\!\top}z_i=\bigl(q_{m+1}^{\!\top}V_m\bigr)u_i=0$.
\end{proof}

In finite precision, Paige's
analysis gives, to leading order,
\begin{equation}
  q_{m+1}^{\!\top}z_i \;\approx\; \frac{\gamma_i}{\beta_m\,u_{m,i}}
  \;=\;\frac{\gamma_i}{\rho_i},
  \qquad
  \lvert\gamma_i\rvert=O\bigl(\varepsilon\lVert A\rVert\bigr),
  \label{eq:paige}
\end{equation}
 Compare
While a Ritz pair is unconverged,
$\rho_i=O(\lVert A\rVert)$ and the contamination sits at $O(\varepsilon)$. As the pair converges, $\rho_i$ becomes small until $\rho_i=O(\varepsilon\lVert
A\rVert)$ and the right-hand side is $O(1)$ and $q_{m+1}$ has no
orthogonality left against $z_i$. So the numerically calculated Ritz values are not the actual Ritz values of $\mathcal{K}_m(A,v)$ for large values of $m$. 
Two things survive.  First, the
computed $\alpha_j$ are real and the computed $\beta_j$ are positive, so $T_m$
is still a bona fide Jacobi matrix: its Gauss rule still defines a positive measure 
of unit mass,  whatever its relation to $\mu_v$.
Second, Greenbaum's backward-stability theorem \citep{greenbaum1989,
meurant-strakos}: the support of it lies in tight clusters about
$\operatorname{spec}(A)$. 

In practice, we have to correct the algorithm against the loss of orthogonality
What the recurrence no longer
enforces must therefore be imposed by hand.
\begin{definition}[Recent subspace and its projector]\label{def:window}
Fix a window length $W\ge1$. At step $j$ let
\begin{equation}
  \mathcal{S}_j \;=\; \operatorname{span}\{\,q_i \;:\; \max(1,\,j-W+1)\le i\le j\,\}
  \;\subseteq\;\mathcal{K}_j
  \label{eq:recent}
\end{equation}
be the span of the $W$ most recently generated vectors, and let $\Pi_j$ denote
the orthogonal projector onto $\mathcal{S}_j^{\perp}$. \emph{Sliding-window
reorthogonalisation} replaces,
\begin{equation}
  w_j \;\longmapsto\; \Pi_j\, w_j ,
  \label{eq:winproj}
\end{equation}
and then sets $\beta_j=\lVert\Pi_j w_j\rVert$ and
$q_{j+1}=\Pi_j w_j/\beta_j$ as before.
\end{definition}

A window of length $W$ helps: an error injected at step
$j$ lies in $\mathcal{S}_{j'}$ for every $j'<j+W$, so \eqref{eq:winproj}
removes it at each of the next $W$ steps, before it has the chance to be
amplified. In this sense the window does not make the basis orthogonal; it
denies round-off the time it needs to grow.

\end{remark}

\subsubsection{From quadrature to the eigenvalue-index plot}

 $i^\star(\lambda)$ counts how many eigenvalues lie
above $\lambda$, so  the eigenvalue of rank $i$ is the $\lambda$ at which
$i^\star(\lambda)=i$.  To estimate $i^\star$ from the rule, order the nodes
$\vartheta_1\ge\vartheta_2\ge\cdots\ge\vartheta_m$ and let
\begin{equation}
  i_j\;:=\;P\sum_{k\le j}\omega_k
  \label{eq:cumulative-rank}
\end{equation}
The Markov--Stieltjes inequalities \cite{szego,golub-meurant} say that the cumulative Gauss weights interlace the true distribution
function:
\begin{equation}
  \sum_{k<j}\omega_k\;\le\;1-F(\vartheta_j)\;\le\;\sum_{k\le j}\omega_k .
  \label{eq:markov-stieltjes}
\end{equation}
Multiplying by $P$ and using \eqref{eq:descending-index} gives
\begin{equation}
  i^\star(\vartheta_j)\;\in\;[\,i_{j-1},\;i_j\,],
  \qquad
  i_j-i_{j-1}=P\,\omega_j .
  \label{eq:rank-assignment}
\end{equation}
In practice, writing $\hat\imath_j$ for the rank
assigned to $\vartheta_j$, one can approximate
\begin{equation}
  \Bigl\{\bigl(\hat\imath_j,\;\vartheta_j\bigr)\Bigr\}_{j=1}^{m},
  \qquad
  \hat\imath_j:=\sqrt{\bigl(i_{j-1}\bigr)\bigl(i_j\bigr)}.
  \label{eq:plotted-curve}
\end{equation}

\subsubsection{Why this resolves only the top of the spectrum}
 To report anything about the eigenvalues below $x$ that is
not dominated by the remainder of the spectrum, one needs a polynomial $p$ that is $O(1)$
on $[0,x)$ -- normalise it by $p(0)=1$ -- and as small as possible on the bulk
$[x,\lambda_1]$.  If the best achievable $\max_{[x,\lambda_1]}\lvert p\rvert$
is not small then it is not possible. Next theorem quantifies it.

\begin{lemma}[Polynomial separation of a low-lying window]
\label{lem:chebyshev}
Let $0<x<\lambda_1$ and $n\ge1$.  Then
\begin{equation}
  \min_{\substack{p\in\mathcal P_{n}\\ p(0)=1}}\ \max_{\lambda\in[x,\lambda_1]}\left\lvert p(\lambda) \right\rvert
  \;=\;\frac{1}{\left\lvert{T_{n}\!\bigl(\tfrac{\lambda_1+x}{\lambda_1-x}\bigr)}\right\rvert}
  \;\le\;2\exp\!\Bigl(-2n\sqrt{x/\lambda_1}\Bigr),
  \label{eq:cheb}
\end{equation}
where $T_n$ is the Chebyshev polynomial of the first kind.  Consequently, to
suppress $[x,\lambda_1]$ by a factor $\delta$ relative to the origin one needs
\begin{equation}
  n\;\ge\;\tfrac12\sqrt{\lambda_1/x}\,\log(2/\delta),
  \qquad\text{i.e.}\qquad
  x\;\gtrsim\;\frac{\lambda_1}{n^{2}}\quad\text{at fixed }\delta .
  \label{eq:m2law}
\end{equation}
\end{lemma}

\begin{proof}
The extremal problem is classical: the minimiser is the shifted, normalised
Chebyshev polynomial
$p^\star(\lambda)=T_n\bigl(\tfrac{\lambda_1+x-2\lambda}{\lambda_1-x}\bigr)
/T_n\bigl(\tfrac{\lambda_1+x}{\lambda_1-x}\bigr)$, by the equioscillation
argument (any competitor with strictly smaller sup-norm would make
$p^\star-p$ a polynomial of degree $\le n$ with $n+1$ sign changes in
$[x,\lambda_1]$ and a zero at $0$, i.e.\ $n+1$ roots, forcing
$p^\star\equiv p$).  For the bound, write $y=\frac{\lambda_1+x}{\lambda_1-x}>1$
and use $T_n(y)\ge\frac12\bigl(y+\sqrt{y^2-1}\bigr)^n$ with
\[
  y+\sqrt{y^2-1}
  =\frac{\sqrt{\lambda_1}+\sqrt{x}}{\sqrt{\lambda_1}-\sqrt{x}}
  =\frac{\sqrt{\kappa}+1}{\sqrt{\kappa}-1},
  \qquad \kappa:=\lambda_1/x,
\]
so that $1/T_n(y)\le2\bigl(\frac{\sqrt\kappa-1}{\sqrt\kappa+1}\bigr)^{n}
\le2e^{-2n/\sqrt\kappa}$, using $\log\frac{1+u}{1-u}\ge2u$ with
$u=1/\sqrt\kappa$.  Solving $2e^{-2n/\sqrt\kappa}\le\delta$ gives
\eqref{eq:m2law}.
\end{proof}

\subsection{Algorithm for calculating the hard-edge exponent}
\label{sec:invisible}

\subsubsection{Smooth functionals featuring the hard-edge exponent}

\begin{definition}[Smoothed spectral counts]
\label{def:counts}
For $z>0$, real $k>0$, and $T>0$, define powers by spectral calculus and set
\begin{equation}
  \begin{gathered}
  F_z:=z\,(A+zI)^{-1},
  \qquad
  N_k(z):=\frac1P\Tr\bigl[F_z^{\,k}\bigr]
  =\int_0^\infty\Bigl(\frac{z}{\lambda+z}\Bigr)^{k}\,\mathrm d\mu(\lambda),\\[3pt]
  L(T):=\frac1P\Tr e^{-AT}=\int_0^\infty e^{-\lambda T}\,\mathrm d\mu(\lambda).
  \end{gathered}
  \label{eq:def-counts}
\end{equation}
\end{definition}

\begin{theorem}[Smoothing preserves the exponent exactly]
\label{thm:exponent-transfer}
Let $\mu$ be a probability measure on $[0,\infty)$ with
$F(x)=\mu([0,x])$, and suppose there are constants $c_\mu>0$, $\alpha>0$,
$\alpha'>0$, $C_1\ge0$ and $x_0>0$ with
\begin{equation}
  \bigl\lvert F(x)-c_\mu x^{\alpha}\bigr\rvert\le C_1x^{\alpha+\alpha'}
  \qquad (0\le x\le x_0).
  \label{eq:hyp-F}
\end{equation}
Then for every real $k>\alpha$ and all $0<z\le x_0$,
\begin{equation}
  N_k(z)
  =c_\mu\,\frac{\Gamma(\alpha+1)\Gamma(k-\alpha)}{\Gamma(k)}\;z^{\alpha}
  \;+\;O\!\left(z^{\min(\alpha+\alpha',\,k)}\right),
  \label{eq:Nk-asymptotic}
\end{equation}
and for all $T\ge1/x_0$,
\begin{equation}
  L(T)=c_\mu\,\Gamma(\alpha+1)\,T^{-\alpha}
  \;+\;O\!\left(T^{-(\alpha+\alpha')}\right)+O\!\left(e^{-x_0T/2}\right).
  \label{eq:L-asymptotic}
\end{equation}
In particular $\dfrac{\,\mathrm d\log N_k(z)}{\,\mathrm d\log z}\to\alpha$ and
$-\dfrac{\,\mathrm d\log L(T)}{\,\mathrm d\log T}\to\alpha$, with the stated rates.
\end{theorem}

\begin{proof}
 Let $g(\lambda):=(1+\lambda/z)^{-k}$,
so $g(0)=1$, $g$ is decreasing, $g(\lambda)=O(\lambda^{-k})$, and
$g'(\lambda)=-\frac kz(1+\lambda/z)^{-k-1}$.  Since $F$ is bounded and
$g(\lambda)F(\lambda)\to0$ as $\lambda\to\infty$ (as $k>0$) while $F(0^-)=0$,
\[
  N_k(z)=\int_0^\infty g\,\mathrm d\mu
  =\bigl[g F\bigr]_0^{\infty}-\int_0^{\infty}F(\lambda)g'(\lambda)\,\mathrm d\lambda
  =\frac kz\int_0^\infty F(\lambda)\Bigl(1+\frac\lambda z\Bigr)^{-k-1}\,\mathrm d\lambda .
\]
Substituting $\lambda=zu$,
\begin{equation}
  N_k(z)=k\int_0^\infty F(zu)\,(1+u)^{-k-1}\,\mathrm d u .
  \label{eq:Nk-ibp}
\end{equation}

  Insert $F(zu)=c_\mu z^{\alpha}u^{\alpha}$ in
\eqref{eq:Nk-ibp} and use the Beta integral
$\int_0^\infty u^{\alpha}(1+u)^{-k-1}\,\mathrm d u=B(\alpha+1,k-\alpha)
=\frac{\Gamma(\alpha+1)\Gamma(k-\alpha)}{\Gamma(k+1)}$, convergent precisely
because $k>\alpha$.  Multiplying by $k$ and using
$\Gamma(k+1)=k\Gamma(k)$ gives the stated main term.

 Split \eqref{eq:Nk-ibp} at $u_0:=x_0/z\ge1$.
On $u\le u_0$ use \eqref{eq:hyp-F}:
\[
  k\int_0^{u_0}\bigl\lvert F(zu)-c_\mu (zu)^{\alpha}\bigr\rvert(1+u)^{-k-1}\,\mathrm d u
  \le C_1kz^{\alpha+\alpha'}\int_0^{u_0}u^{\alpha+\alpha'}(1+u)^{-k-1}\,\mathrm d u,
\]
which is $O(z^{\alpha+\alpha'})$ if $k>\alpha+\alpha'$, and otherwise is
$O\bigl(z^{\alpha+\alpha'}u_0^{\alpha+\alpha'-k}\bigr)=O(z^{k})$.  On $u>u_0$,
using $0\le F\le1$,
\[
  k\int_{u_0}^\infty F(zu)(1+u)^{-k-1}\,\mathrm d u\le(1+u_0)^{-k}=O(z^k),
\]
and likewise
$c_\mu kz^{\alpha}\int_{u_0}^\infty u^{\alpha}(1+u)^{-k-1}\,\mathrm d u
=O\bigl(z^{\alpha}u_0^{\alpha-k}\bigr)=O(z^{k})$.  Collecting gives
\eqref{eq:Nk-asymptotic}.

 Identically, with $g=e^{-\lambda T}$,
\[
  L(T)=T\int_0^\infty F(\lambda)e^{-\lambda T}\,\mathrm d\lambda
  =\int_0^\infty F(s/T)e^{-s}\,\mathrm d s ,
\]
whose main term is $c_\mu T^{-\alpha}\int_0^\infty s^{\alpha}e^{-s}\,\mathrm d s
=c_\mu\Gamma(\alpha+1)T^{-\alpha}$.  The error on $s\le x_0T$ is
$O(T^{-(\alpha+\alpha')})$ by \eqref{eq:hyp-F}, and the tail $s>x_0T$
contributes at most $\int_{x_0T}^\infty(1+c_\mu (s/T)^\alpha)e^{-s}\,\mathrm d s
=O(e^{-x_0T/2})$.
\end{proof}

\begin{remark}[Choice of $k$, and a free diagnostic]
\label{rem:k}
The hypothesis $k>\alpha$ is necessary: if $k\le\alpha$ the
integral $\int u^\alpha(1+u)^{-k-1}\,\mathrm d u$ diverges and $N_k(z)$ is instead
dominated by the bulk.  Concretely, if $\int\lambda^{-k}\,\mathrm d\mu<\infty$ then
$N_k(z)=z^{k}\!\int\lambda^{-k}\,\mathrm d\mu+o(z^k)$, i.e.\ the measured slope
\emph{saturates at $k$}.  This gives a diagnostic at no extra cost: compute
$N_k$ for $k=1,2,4$ and compare slopes $\hat\alpha_k$.  If
$\hat\alpha_1=\hat\alpha_2=\hat\alpha_4$, that common value is $\alpha$ and
$\alpha<1$.  If $\hat\alpha_k=k$ for $k=1,2$ but $\hat\alpha_4<4$, then
$\alpha=\hat\alpha_4\in(2,4)$.  In general $\hat\alpha_k=\min(\alpha,k)$, so
the smallest $k$ at which the slope stops tracking $k$ identifies $\alpha$.
\end{remark}

\begin{lemma}[Differential identity]\label{lem:diff}
For every $k\ge1$ and every $z>0$,
\begin{equation}
  z\,N_k'(z)\;=\;k\bigl(N_k(z)-N_{k+1}(z)\bigr).
  \label{eq:diffid}
\end{equation}
\end{lemma}

\begin{proof}
Write $N_k(z)=\int f_z^{(k)}\,d\mu$ with
$f_z^{(k)}(\lambda)=\bigl(z/(\lambda+z)\bigr)^{k}$; the integrand is bounded
by $1$ and smooth in $z>0$ uniformly on the (compact) spectrum, so
differentiation under the integral sign is justified. Since
\[
  \frac{\partial}{\partial z}\frac{z}{\lambda+z}
  =\frac{(\lambda+z)-z}{(\lambda+z)^2}=\frac{\lambda}{(\lambda+z)^2},
\]
the chain rule gives
\[
  z\,\frac{\partial}{\partial z} f_z^{(k)}
  = k\Bigl(\frac{z}{\lambda+z}\Bigr)^{k-1}\cdot\frac{z\lambda}{(\lambda+z)^2}
  = k\Bigl(\frac{z}{\lambda+z}\Bigr)^{k}\cdot\frac{\lambda}{\lambda+z}
  = k\Bigl(f_z^{(k)}-f_z^{(k+1)}\Bigr),
\]
where the last step uses $\lambda/(\lambda+z)=1-z/(\lambda+z)$. Integrating
against $\mu$ gives \eqref{eq:diffid}.
\end{proof}

\begin{definition}[Local exponent]\label{def:alphahat}
For $k\ge1$ and $z>0$ with $N_k(z)>0$,
\begin{equation}
  \;\hat\alpha_k(z)\;:=\;\frac{d\log N_k}{d\log z}
  \;=\;k\Bigl(1-\frac{N_{k+1}(z)}{N_k(z)}\Bigr).\;
  \label{eq:alphahat}
\end{equation}
\end{definition}

\noindent
The second equality is Lemma~\ref{lem:diff} divided by $N_k$. So
$\hat\alpha_k$ is the instantaneous logarithmic slope of $N_k$,
evaluated at a single $z$ from two orders $k$ and $k+1$ of the same
quadrature.

\begin{theorem}[Exactness on a pure power law]\label{thm:exact}
Suppose $dF(x)=\alpha c\,x^{\alpha-1}dx$ on $(0,\infty)$ with $\alpha>0$.
Then for every $k>\alpha$ and every $z>0$,
\begin{equation}
  N_k(z)=C_{\alpha,k}\,c\,z^{\alpha},
  \qquad
  C_{\alpha,k}=\frac{\Gamma(1+\alpha)\,\Gamma(k-\alpha)}{\Gamma(k)},
  \label{eq:Ck}
\end{equation}
and consequently $\hat\alpha_k(z)=\alpha$ \emph{exactly}, for every $k>\alpha$
and every $z$.
\end{theorem}

\begin{proof}
Substituting $\lambda=zt$ in the definition of $N_k$,
\[
  N_k(z)=\int_0^\infty\Bigl(\frac{z}{\lambda+z}\Bigr)^{k}\alpha c\,
  \lambda^{\alpha-1}\,d\lambda
  =\alpha c\,z^{\alpha}\int_0^\infty\frac{t^{\alpha-1}}{(1+t)^{k}}\,dt
  =\alpha c\,z^{\alpha}B(\alpha,k-\alpha),
\]
the Beta integral converging exactly when $\alpha>0$ and $k-\alpha>0$. With
$B(\alpha,k-\alpha)=\Gamma(\alpha)\Gamma(k-\alpha)/\Gamma(k)$ and
$\alpha\Gamma(\alpha)=\Gamma(1+\alpha)$ this is \eqref{eq:Ck}. Then
\[
  \frac{C_{\alpha,k+1}}{C_{\alpha,k}}
  =\frac{\Gamma(k+1-\alpha)}{\Gamma(k-\alpha)}\cdot\frac{\Gamma(k)}{\Gamma(k+1)}
  =\frac{k-\alpha}{k},
\]
using $\Gamma(w+1)=w\Gamma(w)$ twice, so by \eqref{eq:alphahat}
$\hat\alpha_k=k\bigl(1-(k-\alpha)/k\bigr)=\alpha$.
\end{proof}

\subsubsection{Distinguishing a power law from anything faster}
\label{sec:discriminator}

\begin{lemma}[Negative-moment dichotomy]
\label{lem:dichotomy}
Let $\mu$ be a probability measure on $[0,\infty)$ with $\mu(\{0\})=0$ and
$F(x)=\mu([0,x])$.  For $k>0$ put
\begin{equation}
  M_{-k}:=\int_0^\infty\lambda^{-k}\,\mathrm d\mu(\lambda)\in(0,\infty],
  \qquad
  \alpha_\mu:=\sup\{k>0:\ M_{-k}<\infty\}\in[0,\infty],
  \label{eq:neg-moments}
\end{equation}
and define the  local slope at ratio $\varrho>1$,
\begin{equation}
  s_k(z):=\frac{\log\bigl(N_k(\varrho z)/N_k(z)\bigr)}{\log \varrho}.
  \label{eq:local-slope}
\end{equation}
Then:
\begin{enumerate}[label=(\roman*),leftmargin=2.2em,itemsep=1pt]
\item $N_k(z)=z^{k}I_k(z)$ with $I_k(z):=\int(\lambda+z)^{-k}\,\mathrm d\mu$
  increasing monotonically to $M_{-k}$ as $z\downarrow0$;
\item if $k<\alpha_\mu$ then $N_k(z)\sim M_{-k}z^{k}$ and $s_k(z)\to k$;
\item if $F(x)\sim c_\mu x^{\alpha}$ with $0<\alpha<\infty$ then
  $\alpha_\mu=\alpha$, and for $k>\alpha$, $s_k(z)\to\alpha$;
\item consequently $s_k(z)\to\min(\alpha_\mu,k)$;
\item $\alpha_\mu=\infty$ if and only if $F(x)=o(x^{n})$ for every $n$.
\end{enumerate}
\end{lemma}

\begin{proof}
(i) $N_k(z)=\int\bigl(z/(\lambda+z)\bigr)^k\,\mathrm d\mu=z^k\int(\lambda+z)^{-k}\,\mathrm d\mu$.
For $\mu$-a.e.\ $\lambda>0$ the integrand $(\lambda+z)^{-k}$ increases to
$\lambda^{-k}$ as $z\downarrow0$; monotone convergence gives $I_k(z)\uparrow
M_{-k}$.

(ii) If $M_{-k}<\infty$ then $I_k(z)\to M_{-k}\in(0,\infty)$, so by (i)
$N_k(z)\sim M_{-k}z^k$ and
$s_k(z)=k+\log\bigl(I_k(\varrho z)/I_k(z)\bigr)/\log\varrho\to k$,
because both $I_k(\varrho z)$ and $I_k(z)$ tend to the same finite positive
limit.  (Using the finite-difference slope avoids assuming $N_k$ is
differentiable, and is what the estimator actually computes.)

(iii) Integrating by parts, for $0<k$,
$\int_{(0,1]}\lambda^{-k}\,\mathrm d\mu
 =F(1)-\lim_{x\downarrow0}x^{-k}F(x)+k\int_0^1F(x)x^{-k-1}\,\mathrm d x$,
so $M_{-k}<\infty$ iff $\int_0^1F(x)x^{-k-1}\,\mathrm d x<\infty$.  With
$F(x)\sim c_\mu x^{\alpha}$ the integrand is $\asymp x^{\alpha-k-1}$, which is
integrable at $0$ exactly when $k<\alpha$; hence $\alpha_\mu=\alpha$.  The
limit $s_k\to\alpha$ for $k>\alpha$ is Theorem~\ref{thm:exponent-transfer}.

(iv) Combine (ii) and (iii): for $k<\alpha_\mu$ the slope tends to $k$, and for
$k>\alpha_\mu$ it tends to $\alpha_\mu$.

(v) If $F(x)\le C_nx^{n}$ near $0$ for every $n$, then
$\int_0^1F(x)x^{-k-1}\,\mathrm d x\le C_n\int_0^1x^{n-k-1}\,\mathrm d x<\infty$ whenever
$n>k$, so $M_{-k}<\infty$ for all $k$.  Conversely if $M_{-k}<\infty$ then the
boundary term in the integration by parts must vanish, i.e.\
$x^{-k}F(x)\to0$, so $F(x)=o(x^{k})$; if this holds for every $k$ then
$F$ vanishes faster than any power.
\end{proof}

Part (v) is what makes the lemma a test: ``$F$ decays faster than any power at
the origin'' is precisely the class containing a hard gap
($F\equiv0$ near $0$), a Lifshitz tail $F(x)\sim e^{-a/x^{\gamma}}$, and any
exponentially small density.  All of them have $\alpha_\mu=\infty$, and by
(iv) all of them make $s_k(z)\to k$ for every $k$.

\begin{corollary}[The ratio test]
\label{cor:ratio-test}
Fix $k_2>k_1>0$ and $\mu(\{0\})=0$, and define
$R(z):=s_{k_2}(z)/s_{k_1}(z)$. Then
\begin{equation}
 R(z)\longrightarrow
 \begin{cases}
 1,&F(x)\sim c_\mu x^\alpha,\quad 0<\alpha\le k_1,\\
 \alpha/k_1,&F(x)\sim c_\mu x^\alpha,\quad k_1<\alpha<k_2,\\
 k_2/k_1,&F(x)\sim c_\mu x^\alpha,\quad \alpha\ge k_2,\\
 k_2/k_1,&\alpha_\mu=\infty.
 \end{cases}
 \label{eq:ratio-test}
\end{equation}
The same ratio can therefore arise from a sufficiently large power-law
exponent or from faster-than-power decay.
\end{corollary}

\subsubsection{Atoms at zero, and what a plateau means}

\begin{corollary}[Kernel detection]
\label{cor:atom}
For every $k>0$, $N_k(z)\downarrow p_0:=\mu(\{0\})$ as
$z\downarrow0$, and $L(T)\downarrow p_0$ as $T\uparrow\infty$.
For a finite matrix $p_0=\dim\ker(A)/P$. 
An atom is incompatible with $F(x)\sim c_\mu x^\alpha$ for $\alpha>0$,
but it can coexist with a positive-spectrum edge
$F(x)-p_0\sim c_+x^\alpha$. 
\end{corollary}

\begin{proof}
$(1+\lambda/z)^{-k}\downarrow\mathbf 1_{\{\lambda=0\}}$ pointwise and
monotonically in $z$; apply monotone convergence.  Same for $e^{-\lambda T}$.
\end{proof}

\subsubsection{Generalization of the Gauss rule and Gauss-Radau brackets}
\label{sec:bracket}

Given $\alpha_1,\dots, \alpha_m, \beta_1, \dots, \beta_m$, we construct the Radau matrix as :
\begin{equation}
  \widetilde T_{m+1}
  =\begin{pmatrix}
     T_m & \beta_m e_m\\
     \beta_m e_m^\top & \gamma
   \end{pmatrix},
  \qquad
  \gamma:=\beta_m^{2}\,\bigl[T_m^{-1}\bigr]_{mm},
  \label{eq:radau-matrix}
\end{equation}
the value of $\gamma$ being exactly what forces $\det\widetilde T_{m+1}=0$,
i.e.\ what puts an eigenvalue at $\lambda=0$.    $\widetilde T_{m+1}$ is
$(m+1)\times(m+1)$, so it has $m+1$ eigenvalues, and these are the $m+1$ Radau
nodes.  One of them is $0$ by construction; the other $m$ are the free nodes
$\tilde\vartheta_1<\dots<\tilde\vartheta_m$.  So
\begin{equation}
  \tilde\vartheta_0:=0,
  \qquad
  R^{(0)}_{m}[g]=\sum_{i=0}^{m}\tilde\omega_i g(\tilde\vartheta_i)
  =\underbrace{\tilde\omega_0\,g(0)}_{\text{prescribed node}}
   +\underbrace{\sum_{i=1}^{m}\tilde\omega_i g(\tilde\vartheta_i)}_{m\text{ free nodes}} .
  \label{eq:radau-rule}
\end{equation}
All $m+1$ weights belong to one family, computed identically from the
eigenvectors $\tilde u_i$ of $\widetilde T_{m+1}$:
\begin{equation}
  \tilde\omega_i=\bigl(e_1^\top\tilde u_i\bigr)^{2},\quad i=0,1,\dots,m,
  \qquad \sum_{i=0}^{m}\tilde\omega_i=1 .
  \label{eq:radau-weights}
\end{equation}

\begin{theorem}[Gauss lower / Radau-at-zero upper bound]
\label{thm:bracket}
Let $\mu_v$ be supported in $[0,\lambda_1]$ and let $g\in C^{2m+1}$ on a
neighbourhood of $[0,\lambda_1]$ satisfy
\begin{equation}
  g^{(2m)}\ge0
  \quad\text{and}\quad
  g^{(2m+1)}\le0
  \qquad\text{on }[0,\lambda_1].
  \label{eq:derivative-signs}
\end{equation}
Let $G_m[g]=\sum_{i=1}^m\omega_ig(\vartheta_i)$ be the $m$-node Gauss rule of
Theorem \ref{thm:gauss} and let
$R^{(0)}_{m}[g]=\sum_{i=0}^{m}\tilde\omega_i g(\tilde\vartheta_i)$ be the
$(m+1)$-node Gauss--Radau rule with its lowest node prescribed at
$\tilde\vartheta_0=0$.
Then
\begin{equation}
  G_m[g]\;\le\;\int g\,\mathrm d\mu_v\;\le\;R^{(0)}_{m}[g].
  \label{eq:bracket}
\end{equation}
Both functions
$g_{z,k}(\lambda)=\bigl(\tfrac{z}{\lambda+z}\bigr)^{k}$ $(z>0,k\ge1)$ and
$g_T(\lambda)=e^{-\lambda T}$ $(T>0)$ satisfy
\eqref{eq:derivative-signs} for every $m$.
\end{theorem}

\begin{proof}
\emph{Gauss side.}  Let $h$ be the Hermite interpolant of $g$ at the $m$
Gauss nodes $\vartheta_i$, each counted twice; $\deg h\le2m-1$, and the
standard Hermite error formula gives, for each $\lambda$,
\begin{equation}
  g(\lambda)-h(\lambda)=\frac{g^{(2m)}(\eta(\lambda))}{(2m)!}\prod_{i=1}^{m}(\lambda-\vartheta_i)^2
  \label{eq:hermite-gauss}
\end{equation}
for some $\eta(\lambda)$ in the convex hull of $\{\lambda,\vartheta_1,\dots,\vartheta_m\}\subset[0,\lambda_1]$.
By \eqref{eq:derivative-signs} the right side is $\ge0$ for all
$\lambda\in[0,\lambda_1]$.  By exactness (Theorem \ref{thm:gauss}),
$\int h\,\mathrm d\mu_v=G_m[h]=G_m[g]$, the last equality because $h$ and $g$ agree
at the nodes.  Therefore
$\int g\,\mathrm d\mu_v-G_m[g]=\int(g-h)\,\mathrm d\mu_v\ge0$.

\emph{Radau side.}  The Gauss--Radau rule with a prescribed node at $0$ is
exact for polynomials of degree $\le2m$ (it has $m+1$ nodes, one of them
fixed).  Let $h$ now be the Hermite interpolant of $g$ at $\lambda=0$
(simple) and at the $m$ free Radau nodes $\tilde\vartheta_i$ (double);
$\deg h\le 2m$ and
\begin{equation}
  g(\lambda)-h(\lambda)=\frac{g^{(2m+1)}(\eta(\lambda))}{(2m+1)!}\;\lambda\prod_{i=1}^{m}(\lambda-\tilde\vartheta_i)^2 .
  \label{eq:hermite-radau}
\end{equation}
On $\operatorname{supp}\mu_v\subset[0,\lambda_1]$ we have $\lambda\ge0$, so the product of
$\lambda$ and the square is $\ge0$, while $g^{(2m+1)}\le0$; hence $g-h\le0$
$\mu_v$-a.e.  Exactness to degree $2m$ gives
$\int h\,\mathrm d\mu_v=R^{(0)}_m[h]=R^{(0)}_m[g]$, whence
$\int g\,\mathrm d\mu_v-R^{(0)}_m[g]\le0$.  (Measurability of
$\lambda\mapsto\eta(\lambda)$ is not needed: one applies
\eqref{eq:hermite-gauss}--\eqref{eq:hermite-radau} pointwise to obtain the
sign of $g-h$, and integrates the inequality.)

\emph{Sign conditions.}  For $g_{z,k}$,
$g_{z,k}^{(n)}(\lambda)=(-1)^{n}z^{k}\frac{(k)_n}{(\lambda+z)^{k+n}}$ with
$(k)_n=k(k+1)\cdots(k+n-1)>0$, so $g^{(n)}$ has sign $(-1)^n$: even
derivatives $\ge0$, odd $\le0$, on $\lambda>-z\supset[0,\lambda_1]$.  For
$g_T$, $g_T^{(n)}(\lambda)=(-T)^{n}e^{-\lambda T}$, same alternation.  Both
are completely monotone, so \eqref{eq:derivative-signs} holds for every $m$.
\end{proof}
For each $z$ the bracket \eqref{eq:bracket} is computable from
$(\alpha_j,\beta_j)_{j\le m}$ alone.  Define the usable window as the set of
$z$ where the relative bracket width
$\bigl(R^{(0)}_m[g_{z,k}]-G_m[g_{z,k}]\bigr)/R^{(0)}_m[g_{z,k}]$ is below a
chosen tolerance.

\begin{corollary}[Kernel bound]
\label{cor:window}
 Letting $z\downarrow0$ in \eqref{eq:bracket} with
$k=1$,
\begin{equation}
  \mu_v(\{0\})\ \le\ \tilde\omega_0 ,
\end{equation}
the Radau weight at the prescribed node, so a depth-$m$ run already bounds the
kernel fraction from above.
\end{corollary}

\begin{proof}
$G_m[g_{z,1}]=z\sum_i\omega_i/(\vartheta_i+z)\to\sum_{i:\vartheta_i=0}\omega_i$
and $R^{(0)}_m[g_{z,1}]\to \tilde\omega_0$ as $z\downarrow0$; combine with
\eqref{eq:bracket} and Corollary \ref{cor:atom}.
\end{proof}

\subsubsection{Why ghosts do not reach the hard edge}
\label{sec:ghosts}

In finite precision the Lanczos recurrence loses orthogonality as soon as a
Ritz values accumulate, see Remark \ref{rem:lanczos-fp}. \cite{paige1976,paige1980} showed that the resulting loss of orthogonality is triggered precisely when a Ritz value converges to the true eigenvalue, and
occurs in the direction of the corresponding Ritz vector. 
Kaniel--Paige--Saad convergence theory \citep{saad1980} says that extremal well-separated eigenvalues converge fastest and hence by Paige's argument  leads to ghosts. Extremal means at either end of the spectrum - near $\lambda_1$
or near $\lambda_P$ - as opposed to interior. How  well-separated the bottom has to be is made precise in the next result.

\begin{lemma}[Convergence at the bottom, and its cost under a hard edge]
\label{lem:kps-bottom}
Let $A=A^\top$ have eigenvalues $\lambda_1\ge\dots\ge\lambda_P\ge0$ with unit
eigenvectors $u_k$, let $v$ be a probe with $u_P^\top v\neq0$, and let
$\vartheta^{(m)}_{m}$ be the smallest Ritz value at depth $m$.  Write
\begin{equation}
  \gamma_P:=\frac{\lambda_{P-1}-\lambda_P}{\lambda_1-\lambda_{P-1}}
  \label{eq:gamma-bottom}
\end{equation}
for the relative gap at the bottom.  Then:
\begin{enumerate}[label=(\roman*),leftmargin=*]
\item Kaniel--Paige--Saad asserts \citep{kaniel1966,saad1980,parlett-book}
  \begin{equation}
    0\;\le\;\vartheta^{(m)}_{m}-\lambda_P\;\le\;
    \frac{(\lambda_1-\lambda_P)\,\tan^2\angle(v,u_P)}
         {T_{m-1}\bigl(1+2\gamma_P\bigr)^{2}} .
    \label{eq:kps-bottom}
  \end{equation}
\item For $\gamma_P\ll1$,
  $T_{m-1}(1+2\gamma_P)=\tfrac12\exp\bigl(2(m-1)\sqrt{\gamma_P}\bigr)
  \bigl(1+O(\gamma_P)\bigr)$, so driving the right-hand side of
  \eqref{eq:kps-bottom} below a tolerance $\delta$ requires
  \begin{equation}
    m\;\gtrsim\;1+\frac{1}{4\sqrt{\gamma_P}}\,
      \log\frac{4(\lambda_1-\lambda_P)\tan^2\angle(v,u_P)}{\delta},
    \qquad\text{i.e.}\qquad
    m\;\sim\;\gamma_P^{-1/2}.
    \label{eq:kps-depth}
  \end{equation}
\item If in addition $\mu([0,x])\sim c_\mu x^{\alpha}$ as $x\downarrow0$, then
  the $k$-th smallest eigenvalue satisfies
  $\lambda_{(k)}\simeq\bigl(k/c_\mu P\bigr)^{1/\alpha}$, whence, in units
  $\lambda_1=1$,
  \begin{equation}
    \gamma_P\;\simeq\;\bigl(2^{1/\alpha}-1\bigr)\bigl(c_\mu P\bigr)^{-1/\alpha}
    \qquad\text{and}\qquad
    m\;\sim\;\gamma_P^{-1/2}=\;\Theta\bigl(P^{1/(2\alpha)}\bigr).
    \label{eq:kps-P-scaling}
  \end{equation}
\end{enumerate}
\end{lemma}

The exponent in \eqref{eq:kps-P-scaling} is the point: the smaller the
hard-edge exponent, slower the convergence and lower the chance of ghosting for small eigenvalues. Numerically, with
$P=1.68\times10^{8}$ and $c_\mu=1$: for example at $\alpha\approx0.33$ the value of $m$ required to converge the smallest
eigenvalue is order $10^{12}$.  The same bound at the top, where an isolated $\lambda_1$ has a relative gap of order unity, gives $m=O(\log(1/\delta))$: hundreds of steps.
This asymmetry is what may put the ghosts
exclusively at the top. 
We check this for $m=4600$,
\begin{center}
\begin{tabular}{rrrr}
\toprule
$\lambda/\lambda_1$ band & nodes & duplicate pairs ($<10^{-8}$) & fraction\\
\midrule
$10^{-1}$--$10^{0}$   & 3377 & 2261 & $67\%$\\
$10^{-2}$--$10^{-1}$  &  847 &    0 & $0\%$\\
$10^{-3}$--$10^{-2}$  &  257 &    0 & $0\%$\\
$10^{-4}$--$10^{-3}$  &   81 &    0 & $0\%$\\
$10^{-7}$--$10^{-4}$  &   37 &    0 & $0\%$\\
\bottomrule
\end{tabular}
\end{center}
So ghosting is confined to $\lambda>0.1\lambda_1$. Now we turn to make this observation formal. 

\begin{assumption}[Ghost localization]
\label{ass:ghost-localisation}
Let $\mu_v$ be the exact probe measure and $\hat\mu_v$ the measure of the
computed Jacobi matrix.  Assume there are $\Lambda>0$, $s\ge0$ and a transport plan $\pi$
of $\mu_v$ and $\hat\mu_v$ supported on
\begin{equation}
  \{(x,x):x\ge0\}\ \cup\ \{(x,y):x,y\ge\Lambda,\ \left\lvert{x-y}\right\rvert\le s\},
  \label{eq:transport-support}
\end{equation}
i.e.\ mass is either left untouched or moved by at most $s$ within
$[\Lambda,\infty)$.
\end{assumption}

A transport plan of two probability measures
$\mu,\nu$ on $\R_{\ge0}$ is a probability measure $\pi$ on the product
$\R_{\ge0}\times\R_{\ge0}$ whose two marginals are $\mu$ and $\nu$:
\begin{equation}
  \pi\bigl(B\times\R_{\ge0}\bigr)=\mu(B),
  \qquad
  \pi\bigl(\R_{\ge0}\times B\bigr)=\nu(B)
  \qquad\text{for every Borel set }B .
  \label{eq:coupling-def}
\end{equation}
  Couplings
always exist - the product $\mu\otimes\nu$ is one - so the content of
Assumption~\ref{ass:ghost-localisation} is not existence but the restriction on
where $\pi$ is allowed to put its probability mass.
That restriction is \eqref{eq:transport-support}.  Being supported on the union
of those two sets means every unit of mass does one of exactly two things:
\begin{enumerate}[label=(\alph*),leftmargin=*]
\item it lies on the diagonal $\{(x,x)\}$ - it is not moved at all; or
\item it lies in $\{(x,y):x,y\ge\Lambda,\ \left\lvert{x-y}\right\rvert\le s\}$ - it is moved, but
  by a distance of at most $s$, and both its origin and destination are above
  $\Lambda$.
\end{enumerate}
There is no third option, and in particular no mass may cross from below
$\Lambda$ to above it or vice versa, and no mass below $\Lambda$ may move at
all.  Applied to $\mu_v$ and $\hat\mu_v$, this is the formalisation of
ghosting splits converged Ritz values into tight clusters at the top of the
spectrum, and leaves the bottom alone: the cluster spread is $s$, and
$\Lambda$ is the scale above which converged values are found.

\begin{proposition}[Transport bound]
\label{prop:transport}
Under Assumption~\ref{ass:ghost-localisation}, for any $g\in C^1([0,\infty))$,
\begin{equation}
  \Bigl\lvert \int g\,\mathrm d\hat\mu_v-\int g\,\mathrm d\mu_v \Bigr\rvert
  \;\le\; s\,\sup_{\lambda\ge\Lambda}\lvert g'(\lambda)\rvert .
  \label{eq:transport-bound}
\end{equation}
\end{proposition}

\begin{proof}
Writing the difference through the coupling,
$\int g\,\mathrm d\hat\mu_v-\int g\,\mathrm d\mu_v=\iint\bigl(g(y)-g(x)\bigr)\,\mathrm d\pi(x,y)$.
By the mean value theorem, for each $(x,y)$ in the support of $\pi$ there is
$\xi$ between $x$ and $y$ with $g(y)-g(x)=g'(\xi)(y-x)$; both $x,y\ge\Lambda$,
so $\xi\ge\Lambda$, and $\lvert y-x\rvert\le s$.  Hence the integrand is
bounded by $s\sup_{\lambda\ge\Lambda}\lvert g'\rvert$ pointwise, and $\pi$ is a
probability measure.
\end{proof}

\begin{corollary}[Ghosts are harmless at the hard edge]
\label{cor:ghosts-harmless}
Take $g=g_{z,k}$, $g_{z,k}(\lambda)=\bigl(z/(\lambda+z)\bigr)^{k}$.  Then
$g_{z,k}'(\lambda)=-k z^{k}/(\lambda+z)^{k+1}$ is negative and increasing, so
its supremum over $[\Lambda,\infty)$ is attained at $\Lambda$ and
\begin{equation}
  \bigl\lvert \hat N_k(z)-N_k(z)\bigr\rvert
  \;\le\; \frac{s\,k\,z^{k}}{(\Lambda+z)^{k+1}}
  \;\le\; \frac{s\,k}{\Lambda^{k+1}}\;z^{k}.
  \label{eq:ghost-abs}
\end{equation}
If in addition $\mu$ obeys the hard-edge law of Theorem
\ref{thm:exponent-transfer}, so that
$N_k(z)\sim c_\mu\frac{\Gamma(\alpha+1)\Gamma(k-\alpha)}{\Gamma(k)}z^{\alpha}$,
then the \emph{relative} error obeys
\begin{equation}
  \frac{\bigl\lvert \hat N_k(z)-N_k(z)\bigr\rvert}{N_k(z)}
  \;\le\;
  \frac{s\,k\,\Gamma(k)}{c_\mu\,\Gamma(\alpha+1)\Gamma(k-\alpha)\,\Lambda^{k+1}}
  \;z^{\,k-\alpha}
  \;\xrightarrow[z\downarrow0]{}\;0
  \qquad\text{whenever } k>\alpha .
  \label{eq:ghost-rel}
\end{equation}
\end{corollary}

The condition $k>\alpha$ is not an extra hypothesis: it is precisely the
condition already required for the exponent to transfer at all (Theorem
\ref{thm:exponent-transfer}, and the $k$-saturation diagnostic of
Remark~\ref{rem:k}).  With
the measured ghost spread $s\simeq10^{-8}\lambda_1$ and localisation
$\Lambda=0.1\lambda_1$ (\S\ref{sec:results}), $k=2$, and the measured
$\alpha\approx0.3$, the bound \eqref{eq:ghost-rel} with $c_\mu=1$ is
$1.6\times10^{-15}$ at $z=10^{-6}\lambda_1$ and $1.2\times10^{-20}$ at
$z=10^{-9}\lambda_1$.

\begin{remark}[What this argument does not cover]
\label{rem:ghost-limits}
Corollary~\ref{cor:ghosts-harmless} controls perturbations localised
above $\Lambda$.  It says nothing about rounding that perturbs the
spectrum near zero, where $g_{z,k}$ varies fastest: there
$\sup\lvert g_{z,k}'\rvert=k/z$, and a perturbation of size $\delta_A$
contributes a relative error $\sim k\delta_A/(c_\mu z^{1+\alpha})$, which diverges as $z\downarrow0$.  Setting that against unity gives the floor
\begin{equation}
  z_{\min}/\lambda_1\;\gtrsim\;(\delta_A/\lambda_1)^{1/(1+\alpha)} .
  \label{eq:precision-floor}
\end{equation}
It is essential that $\delta_A$ here is the measured defect of the
computed operator. 
With the measured $\delta_A/\lambda_1=8.5\times10^{-12}$ and $\alpha\approx0.3$,
\eqref{eq:precision-floor} gives $z_{\min}\approx3\times10^{-9}\lambda_1$.
\end{remark}

\subsection{Measurements}
\label{sec:results}

All numbers below are our own measurements on a retrained $150$M checkpoint
(OLMo-style, $P=1.677\times10^{8}$, $3$B tokens at $1\times$ Chinchilla,
$B=64$, seed $0$).  The
operator is the Adam-preconditioned Gauss--Newton matrix \eqref{eq:operator}
at the final checkpoint, built from $10^{6}$ tokens, probed in float64 with a
random probe. 
At $m=9400$
(Figure~\ref{fig:ratio-test}):
\begin{center}
\begin{tabular}{lccccccc}
\toprule
$z/\lambda_1$ & $10^{-5}$ & $3.2\times10^{-6}$ & $10^{-6}$ & $3.2\times10^{-7}$ & $10^{-7}$ & $3.2\times10^{-8}$ & $1.9\times10^{-8}$\\
\midrule
$s_1$ & $0.051$ & $0.082$ & $0.124$ & $0.175$ & $0.232$ & $0.279$ & $[0.288,0.290]$\\
$s_2$ & $0.075$ & $0.116$ & $0.171$ & $0.234$ & $0.295$ & $[0.318,0.320]$ & $[0.293,0.304]$\\
\midrule
$R_{2,1}$ & $1.47$ & $1.42$ & $1.38$ & $1.33$ & $1.27$ & $[1.14,1.15]$ & $[1.01,1.06]$\\
\bottomrule
\end{tabular}
\end{center}
Down to $10^{-7}\lambda_1$ every bracket is tighter than the three digits
shown; the widest is $3.7\times10^{-6}$, on $s_2$ at $10^{-7}\lambda_1$.  The
ratio brackets follow from the slope brackets as
$R_{\mathrm{lo}}=s_{k_2,\mathrm{lo}}/s_{k_1,\mathrm{hi}}$,
$R_{\mathrm{hi}}=s_{k_2,\mathrm{hi}}/s_{k_1,\mathrm{lo}}$, the widest interval
consistent with both.
$R_{2,1}$ decreases monotonically from $1.47$ at $10^{-5}\lambda_1$ to
$[1.01,1.06]$ at the deepest plotted scale, against a faster-than-power limit
of $2$.  A faster-than-power edge would instead drive the ratio towards its limit
$k_2/k_1=2$. Hence, on the window resolved, a faster-than-power edge is
strongly disfavoured, and the behaviour is power-law-like. 
Theorem~\ref{thm:exponent-transfer} makes every $s_k$ tend to $\alpha$. Hence $\alpha \approx 0.30$.

\begin{figure}[h]
\centering
\includegraphics[width=0.93\textwidth]{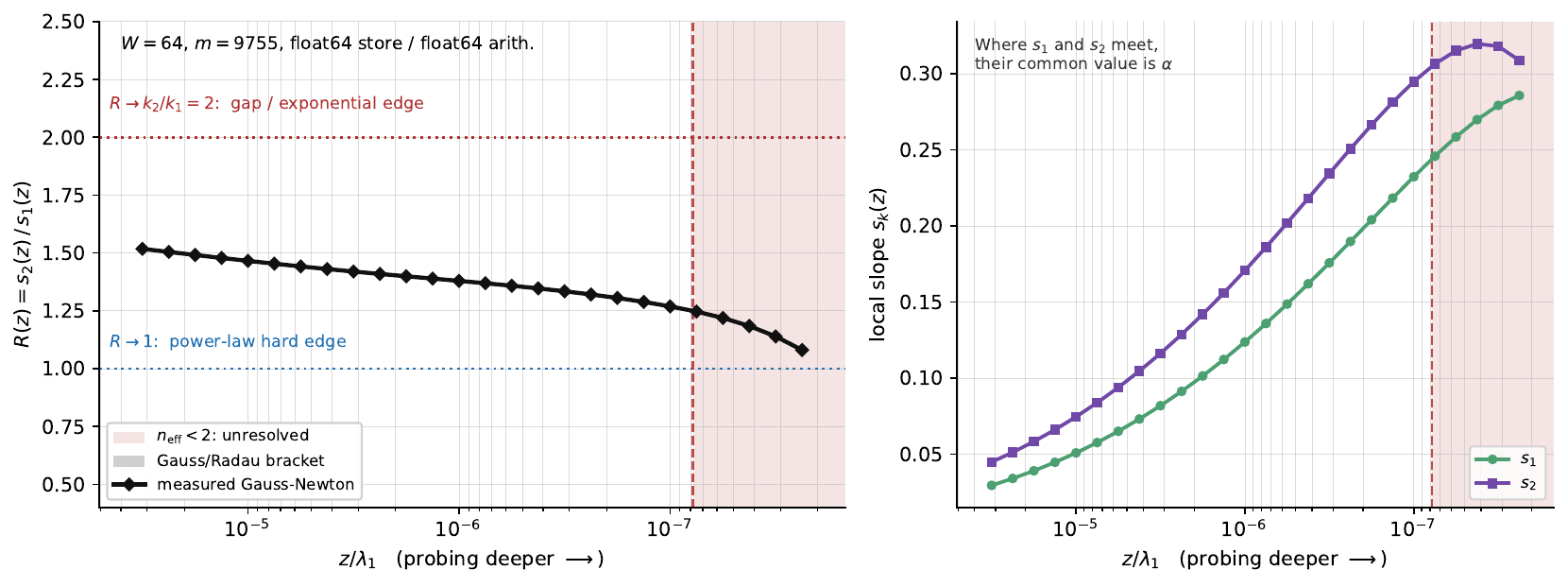}
\caption{The plot on the left shows that heavy tailed curvature spectrum exists and the plot on the right determines the exponent. Smaller $z$ probes a smaller eigenvalue spectrum.}
\label{fig:ratio-test}
\end{figure}

\paragraph{Reliable window.}

Ghost effects are captured in Corollary~\ref{cor:ghosts-harmless}: $s k z^{k}/(\Lambda+z)^{k+1}$ with the
measured $s\le10^{-8}\lambda_1$ and $\Lambda=0.1\lambda_1$.  It is
negligible: it never exceeds $8\times10^{-15}$ and falls steeply as $z$ decreases.

The next concern is the operator floor in Remark~\ref{rem:ghost-limits}. Determinism is a recomputation of the
same product:  in float64 the operator is deterministic to round-off, so a
fixed operator exists; symmetric to $10^{-15}$, so the Jacobi matrix is the
Jacobi matrix of a genuine positive measure.  The linearity
defect $\delta_A/\lambda_1=8.5\times10^{-12}$ is the largest of the three and
is the number that propagates into the sensitivity bound
$\left\lvert {\Delta N_k}\right\rvert\le k\delta_A/z$, hence into a floor on $z$ that no amount of
Lanczos depth can lower. Propagated to the local slope at the tolerance of
$0.01$ used throughout, $k=2$ operator floor is reached at $z/\lambda_1=1.860\times10^{-8}$. 
\begin{center}
\begin{tabular}{llrr}
\toprule
check & quantity & float64 & float32\\
\midrule
determinism & $\left\lVert { A^{(1)}(v)- A^{(2)}(v)}\right\rVert/\left\lVert{ A^{(1)}(v)}\right\rVert$ & $7.6\times10^{-18}$ & $4.0\times10^{-9}$\\
symmetry & $\left\lvert{\left\langle u, Av \right\rangle - \left\langle v, Au \right\rangle}\right\rvert/(\left\lVert{u}\right\rVert \left\lVert {v}\right\rVert\lambda_1)$ & $8.6\times10^{-16}$ & $3.5\times10^{-11}$\\
linearity & $\left\lVert {A(u+v)-Au-Av}\right\rVert/(\lambda_1\left\lVert {u+v}\right\rVert)$ & $8.5\times10^{-12}$ & $2.3\times10^{-7}$\\
positivity & $\left\langle v, Av\right\rangle/(\left\lVert {v}\right\rVert^2\lambda_1)$ & $+1.47\times10^{-5}$ & $+1.47\times10^{-5}$\\
\bottomrule
\end{tabular}
\end{center}

 At iteration depth $m=9400$ and $z/\lambda_1=2.64\times10^{-8}$, the $k=2$ total error reaches the tolerance as a result of interpolating the data from the table below:
 \begin{center}
\begin{tabular}{cccccc}
\toprule
$z/\hat\lambda_1$ & quadrature & operator & ghosting & total &\\
\midrule
$2.196174\times10^{-8}$ & $6.728678\times10^{-3}$ & $8.103861\times10^{-3}$
  & $1.0\times10^{-19}$ & $0.014833$\\
$2.818383\times10^{-8}$ & $2.855708\times10^{-3}$ & $5.878096\times10^{-3}$
  & $1.5\times10^{-19}$ & $0.008734$\\
\bottomrule
\end{tabular}
\end{center}
In the above table, the operator term converted from a bound on the count to a bound on
the local slope, at $z/\lambda_1=2.818\times10^{-8}$, $k=2$, with
$\delta_A/\lambda_1=8.5\times10^{-12}$, $N_2=0.2233$ and $\varrho=10^{0.4}$.  
Recall $s_k=[\ln N_k(\varrho z)-\ln N_k(z)]/\ln\varrho$ - using twice the error at the
deeper endpoint, we get operator error on the slope to be $\frac{2}{\ln\varrho}\,\frac{k\delta_A/z}{N_2}=5.8781\times10^{-3}.$ It is conservative estimate, since both $k\delta_A/(\varrho z)<k\delta_A/z$
and $N_k(\varrho z)>N_k(z)$.

 The hard-edge law $F(x)\simeq c_\mu
x^{\alpha}$ can hold only below the crossover $\lambda_0$, and normalisation
forces $c_\mu\lambda_0^{\alpha}\lesssim1$, so the crossover sits where the
count is of order unity.  Any threshold $N_k=\Theta(1)$ marks that boundary to
within a factor in $z$, we choose $N_k=1/2$. At a
tolerance of $0.01$ on the slope, this demands $z/\lambda_1 <\ 4.21\times10^{-7}$. 

\begin{figure}[h]
\centering
\includegraphics[width=\textwidth]{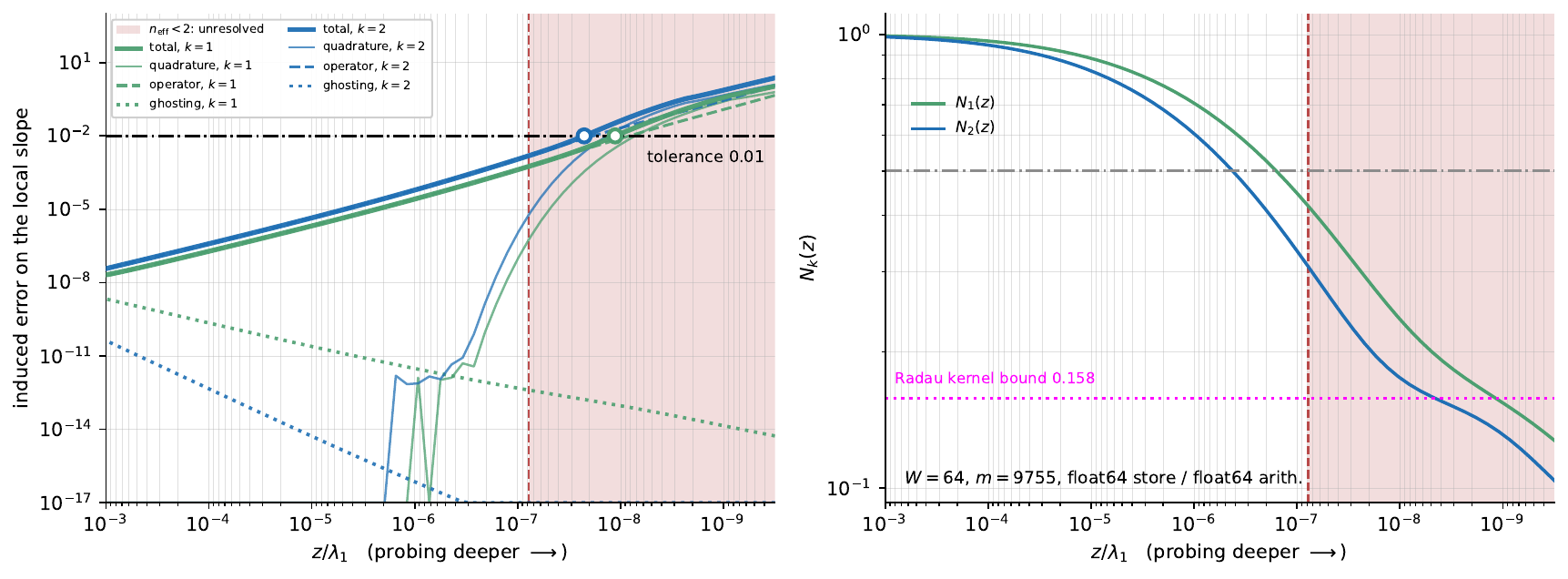}
\caption{Every mechanism that limits the measurement performed.}
\label{fig:error-budget}
\end{figure}

The choice of the slope cut-off so far is arbitrary; to make the discussion of a reliable window precise, we introduce:
\begin{definition}[Effective node count]\label{def:neff}
Let $(\theta_j,\omega_j)_{j=1}^{m}$ be the nodes and weights of the $m$-node
Gauss rule, and for $z>0$ let
$J(z)=\{j:\theta_j\le z\}$ with $q(z)=|J(z)|$. Define
\begin{equation}
  n_{\mathrm{eff}}(z)\;=\;
  \frac{\Bigl(\sum_{j\in J(z)}\omega_j\Bigr)^{2}}
       {\sum_{j\in J(z)}\omega_j^{2}},
  \qquad n_{\mathrm{eff}}(z):=0 \text{ if } J(z)=\varnothing .
  \label{eq:neff}
\end{equation}
\end{definition}

\noindent
This is a participation ratio: it counts the nodes below $z$, discounted by
how unevenly they share the mass. It is preferred to the raw count $q(z)$
because $q$ is discontinuous and is inflated by nodes of negligible weight,
whereas \eqref{eq:neff} is invariant under rescaling all $\omega_j$ and
ignores a node whose weight is vanishing.

\begin{proposition}[Range and extremes]\label{prop:neff}
For every $z$ with $q(z)\ge1$,
\begin{equation}
  1\;\le\;n_{\mathrm{eff}}(z)\;\le\;q(z),
  \label{eq:neffrange}
\end{equation}
with $n_{\mathrm{eff}}(z)=q(z)$ if and only if all $\omega_j$, $j\in J(z)$,
are equal, and $n_{\mathrm{eff}}(z)=1$ if and only if at most one of them is
non-zero --- in particular exactly when $q(z)=1$.
\end{proposition}

\begin{proof}
The weights of a Gauss rule are strictly positive. Write $S_1=\sum_{J}\omega_j$
and $S_2=\sum_{J}\omega_j^2$. By Cauchy--Schwarz applied to the vectors
$(\omega_j)$ and $(1,\dots,1)$, $S_1^2\le q\,S_2$, with equality iff the
$\omega_j$ are proportional to $1$, i.e.\ all equal; this is the upper bound.
For the lower bound, expanding $S_1^2=\sum_j\omega_j^2+\sum_{i\neq
j}\omega_i\omega_j=S_2+\sum_{i\neq j}\omega_i\omega_j$ and using
$\omega_i\omega_j>0$ gives $S_1^2\ge S_2$, with equality iff every cross term
vanishes, i.e.\ at most one weight is non-zero. Since all weights are
positive, that happens exactly when $q(z)=1$.
\end{proof}

Now we turn to discuss $n_{\mathrm{eff}}<2$ is the edge of validity, and $n_{\mathrm{eff}}=1$ the strict end of reliable window.
Split $N_k$ at $z$ into the part carried by nodes below $z$ and the part
carried by nodes above:
\begin{equation}
  N_k(z)=\underbrace{\sum_{\theta_j\le z}\omega_j
  \Bigl(\tfrac{z}{\theta_j+z}\Bigr)^{k}}_{=:B_k(z)}
  \;+\;\underbrace{\sum_{\theta_j> z}\omega_j
  \Bigl(\tfrac{z}{\theta_j+z}\Bigr)^{k}}_{=:T_k(z)} .
  \label{eq:split-z}
\end{equation}
For $\theta\ll z$ the kernel is $\approx1$ independently of $k$, so
$B_{k+1}\approx B_k=:B$: the sub-$z$ mass behaves like an atom at the
origin, contributing logarithmic slope $0$. For $\theta\gg z$ the kernel is
$\approx(z/\theta)^k$, so $T_{k+1}\ll T_k$ and that part contributes slope
$k$. Substituting into Definition~\ref{def:alphahat},
\begin{equation}
  \hat\alpha_k(z)\;\approx\;k\,\frac{T_k(z)}{B(z)+T_k(z)} ,
  \label{eq:mix}
\end{equation}
an approximation valid when the two groups are well separated from $z$. So
$\hat\alpha_k$ is a mixture: the value it reports is set by how the
rule has apportioned mass between an atom-like term and a sloping term.
This is what makes the node count decisive, and it separates the two
thresholds.

\paragraph{$n_{\mathrm{eff}}=1$: strictly unreliable.} By
Proposition~\ref{prop:neff} this happens exactly when a single node lies below
$z$. Everything the rule knows about $\mu|_{[0,z]}$ is then one pair
$(\theta_1,\omega_1)$ --- two real numbers, a mass and a location. But the hard-edge exponent is a statement about the shape of
$\mu$ near the origin, $F(x)\sim c\,x^{\alpha}$, and a single atom has no
shape: it fixes $B$ in \eqref{eq:mix} at an unresolved value and carries no
information about $\alpha$ at all. Below $\theta_2$ the curve is therefore not
merely noisy, it is uninformative about the quantity being measured, whatever
its bracket says. The Gauss/Radau bracket does not protect against this: it
bounds the integral $N_k(z)$, not the shape of the measure beneath it,
and it can be tight while the shape is entirely unresolved.

\paragraph{$n_{\mathrm{eff}}<2$: the edge of validity.} Two nodes below $z$ give four
numbers, enough to represent a mass, a spread and a relative weight, so
$\mu|_{[0,z]}$ is no longer forced to be a point mass and $B$ in
\eqref{eq:mix} responds to the shape of the measure. This is the minimum
configuration in which $\hat\alpha_k$ can carry exponent information at all,
which is why it is the edge of validity rather than a complete failure - the estimate is
supported, but by the least possible amount. This regime is
exactly $1\le n_{\mathrm{eff}}<2$.

For more details see Figure~\ref{fig:ratio-test}, \ref{fig:error-budget}.

\end{document}